\documentclass[twoside]{article}

\usepackage[accepted]{aistats2023}

\usepackage[round]{natbib}

\usepackage[T1]{fontenc}

\usepackage{bbm}

\usepackage{microtype}
\usepackage{graphicx}
\usepackage{caption}
\usepackage{subcaption}
\usepackage{booktabs} % for professional tables
\usepackage[algo2e,linesnumbered]{algorithm2e}

\usepackage{dsfont}
\usepackage{soul}

\usepackage{csquotes}
\usepackage{amsmath,amssymb,amsthm}
\usepackage{mathtools,enumitem,systeme,bm,bbm,physics}
\usepackage{thmtools}
\usepackage{thm-restate}
\usepackage{xcolor}
\usepackage{systeme}

\newlist{enumthm}{enumerate}{1}
\setlist[enumthm]{label=(\Alph*)}

\newtheorem{definition}{Definition}
\makeatletter
\newcommand{\qualification}[1]{\ifthmt@thisistheone#1\fi}
\makeatother

\newtheorem{theorem}{Theorem}
\newtheorem{lemma}{Lemma}
\newtheorem{proposition}{Proposition}
\newtheorem{corollary}{Corollary}
\newtheorem{claim}{Claim}

\newtheorem{remark}{Remark}
\newtheorem{assumption}{Assumption}

\newcommand{\ep}{\hfill $\Box$}

\newcommand{\supp}{\mathrm{supp}}
\DeclareMathOperator*{\argmax}{\arg\!\max}
\DeclareMathOperator*{\argmin}{\arg\!\min}

\DeclarePairedDelimiter\floor{\lfloor}{\rfloor}

\newcommand{\cA}{\mathcal{A}}

\newcommand{\cC}{\mathcal{C}}
\newcommand{\cD}{\mathcal{D}}
\newcommand{\cE}{\mathcal{E}}
\newcommand{\cF}{\mathcal{F}}

\newcommand{\cH}{\mathcal{H}}
\newcommand{\cI}{\mathcal{I}}
\newcommand{\cJ}{\mathcal{J}}
\newcommand{\cK}{\mathcal{K}}
\newcommand{\cL}{\mathcal{L}}
\newcommand{\cM}{\mathcal{M}}
\newcommand{\cN}{\mathcal{N}}

\newcommand{\cP}{\mathcal{P}}

\newcommand{\cT}{\mathcal{T}}
\newcommand{\cR}{\mathcal{R}}
\newcommand{\cX}{\mathcal{X}}
\newcommand{\cY}{\mathcal{Y}}
\newcommand{\cZ}{\mathcal{Z}}
\newcommand{\cS}{\mathcal{S}}
\newcommand{\cU}{\mathcal{U}}
\newcommand{\cV}{\mathcal{V}}
\newcommand{\cO}{\mathcal{O}}

\newcommand{\NN}{\mathbb{N}}
\newcommand{\EE}{\mathbb{E}}
\newcommand{\PP}{\mathbb{P}}
\newcommand{\RR}{\mathbb{R}}
\newcommand{\Var}{\mathrm{Var}}
\newcommand{\Cov}{\mathrm{Cov}}
\newcommand{\diag}{\mathrm{diag}}
\newcommand{\KL}{\mathrm{KL}}

\newcommand{\indicator}{\mathds{1}}

\newcommand{\poly}{\mathrm{poly}}

\newcommand{\trim}{\mathrm{Trim}}
\newcommand{\mix}{\mathrm{mix}}

\allowdisplaybreaks

\usepackage[colorlinks=true,linkcolor=red,citecolor=purple,pagebackref=true]{hyperref}

\makeatletter
\newcommand{\printfnsymbol}[1]{%
	\textsuperscript{\@fnsymbol{#1}}%
}
\makeatother

\SetKwInput{KwInput}{Input}
\SetKwInput{KwOutput}{Output}

\newtheorem{manualtheoreminner}{Theorem}
\newenvironment{manualtheorem}[1]{%
	\renewcommand\themanualtheoreminner{#1}%
	\manualtheoreminner
}{\endmanualtheoreminner}

\begin{document}



\twocolumn[

\aistatstitle{Nearly Optimal Latent State Decoding in Block MDPs}

\aistatsauthor{ Yassir Jedra$^\star$\footnotemark[1] \And Junghyun Lee$^\star$\footnotemark[2] \And Alexandre Prouti\`{e}re\footnotemark[1] \And Se-Young Yun\footnotemark[2] }

\aistatsaddress{ \footnotemark[1]School of Electrical Engineering and Computer Science, KTH Royal Institute of Technology, Stockholm, Sweden \\ \footnotemark[2]Kim Jaechul Graduate School of AI, KAIST, Seoul, Republic of Korea \\ \texttt{\{jedra, alepro\}@kth.se} \quad \texttt{\{jh\_lee00, yunseyoung\}@kaist.ac.kr } } ]
	
\begin{abstract}
	We consider the problem of model estimation in episodic Block MDPs. In these MDPs, the decision maker has access to rich observations or contexts generated from a small number of latent states. We are interested in estimating the latent state decoding function (the mapping from the observations to latent states) based on data generated under a fixed behavior policy. We derive an information-theoretical lower bound on the error rate for estimating this function and present an algorithm approaching this fundamental limit. In turn, our algorithm also provides estimates of all the components of the MDP.
	We apply our results to the problem of learning near-optimal policies in the reward-free setting. Based on our efficient model estimation algorithm, we show that we can infer a policy converging (as the number of collected samples grows large) to the optimal policy at the best possible asymptotic rate. Our analysis provides necessary and sufficient conditions under which exploiting the block structure yields improvements in the sample complexity for identifying near-optimal policies. When these conditions are met, the sample complexity in the minimax reward-free setting is improved by a multiplicative factor $n$, where $n$ is the number of contexts.
\end{abstract}

% !TEX root = ./main.tex
\section{Introduction}\label{sec:intro}

In Reinforcement Learning, leveraging succinct representations of the system state is empirically known to considerably accelerate the search for near-optimal policies, see, e.g., \cite{laskin20a,guo2020,stooke21a} and references therein.
The design of RL algorithms with provable performance guarantees that learn and exploit such representations remains largely open.

In this paper, we address this challenge for a specific class of models, namely episodic Block MDPs (BMDPs). In these MDPs, introduced in \cite{krishnamurthyAL16} and since then widely studied and motivated (see \cite{DuKJAD19a,zhang2022BMDP} and references therein), the decision maker has, in each round, access to rich observations, referred to as {\it contexts}, generated from a small number of {\it latent} states. More precisely, to each context $x$ corresponds a unique latent state $s=f(x)$ where $f$ is referred to as the latent state decoding function. If the decision maker selects control action $a$, the system moves from context $x$ to context $y$ with probability $q(y | s')p(s'|s,a)$ where $s=f(x)$ and $s'=f(y)$. The emission distributions $q$, the latent state transition rates $p$ and the decoding function $f$ are initially unknown. Intuitively, if we could learn the latent state decoding function, we would be able to efficiently summarize the environment with a low-dimensional state space, and hence devise RL algorithms that can quickly learn near-optimal policies. In this work, our objective is to substantiate this intuition. Specifically, we aim at answering the following questions: (i) How fast can we learn the latent state decoding function, as well as $p$ and $q$? (ii) What optimal gains in terms of sample complexity can we expect when exploiting the existing yet initially unknown structure?

Most existing studies on representation learning in BMDPs~\citep{jiang2017,dann2018,DuKJAD19a,misra2020kinematic,foster2021,zhang2022BMDP} address these questions using the function approximation framework. Specifically, the latent state decoding function or some of its functionals is assumed to belong to a parametrized class of functions. This assumption somehow introduces an additional structure in the BMDP. Now, the performance guarantees obtained for algorithms learning within this class of functions depend on the cardinality or complexity of this class. For classes with moderate complexity, these algorithms are able to learn near-optimal policies quickly. However, as it turns out, in absence of any prior knowledge of the class of functions that could contain the true latent state decoding function, these algorithms cannot exploit the block structure (the class becomes too complex and the guarantees are not better than those achieved in plain tabular MDPs -- refer to \textsection\ref{sec:related-works} for details). Moreover, to exploit the function approximation framework, the algorithms proposed in the aforementioned work rely on strong computational oracles (e.g., ERM, MLE).
In this paper, we depart from the function approximation framework and analyze scenarios where no additional structure is imposed on the BMDP. For these scenarios, we wish to understand how to optimally exploit the block structure alone to speed up the convergence of learning algorithms.

{\bf Contributions.} Our main results concern the estimation of the latent state decoding function in BMDPs using data generated from a fixed behavior policy. We apply our results to the problem of learning near-optimal policies in the reward-free learning setting, where the reward function is revealed after the data has been collected. More precisely, our contributions are as follows.

 {\it 1. Learning the latent state decoding function}

(a) We first derive information-theoretical lower bounds on the latent state decoding error rates. When the BMDP $\Phi=(p,q,f)$ satisfies some regularity assumptions, we establish that the expected number $\mathbb{E}_\Phi[|{\cal E}|]$ of contexts for which we do not infer the corresponding latent state correctly must be (roughly) greater than $n\exp(-{TH\over n}I(\Phi))$ where $T$ is the number of episodes in the collected data, $H$ the duration of an episode, $n$ the size of the context space and $I(\Phi)$ a well-defined and non-negative rate function. This lower bound provides conditions on $T$, $H$, $n$, and $\Phi$ under which one can hope to estimate the block structure accurately.

(b) We present a structure estimation algorithm whose performance approaches our lower bound. Its design is inspired by spectral methods typically used for inference tasks in the Stochastic Block Model~\citep{Abbe18}, Degree-Corrected Block Model~\citep{Gao18} and Block Markov Chain~\citep{SandersPY20}. To analyze its performance, we develop new concentration inequalities for functionals of Markov chains with {\it restarts} (restarts are needed to model the episodic nature of the MDP).

{\it 2. Learning near-optimal policies in the reward-free setting} 

In the reward-free setting, the performance of an RL algorithm is quantified through $\Delta({\cal R})=\sup_{r\in {\cal R}} {1\over H}(V^\star(r) - V^{\hat{\pi}_r}(r))$ where $V^\star(r)$ and $V^{\hat{\pi}_r}(r)$ are the values of the optimal policy and of that of the policy $\hat{\pi}_r$ returned by the algorithm for the (possibly context-dependent) rewards $r$, and where ${\cal R}$ is a given class of reward functions. 

(a) We present lower bounds on the number of observed episodes required to infer $\epsilon$-optimal policies or more precisely to ensure that $\Delta({\cal R})\le \epsilon$. These lower bounds hold even if the data can be collected in an adaptive manner. In the {\it minimax} setting where ${\cal R}$ contains all possible bounded reward functions, this sample complexity lower bound is $TH=\Omega({n\over \epsilon^2})$. In the {\it reward-specific} setting where ${\cal R}$ reduces to a single function $r$, the lower bound becomes $TH=\Omega(n\log ({1\over \epsilon})+{1\over \epsilon^2})$ (the first term corresponds to the data required to learn the block structure accurately, and the second term to the data required to learn an $\epsilon$-optimal policy given the block structure). These lower bounds quantify the gains achieved when exploiting the block structure: in the minimax and reward-specific settings, without any structure, these bounds would be $\Omega({n^2\over \epsilon^2})$ and $\Omega({n\over \epsilon^2})$, respectively.

(b) We study the performance of RL algorithms built upon our structure estimation algorithm using data generated with a fixed behavior policy. We show that these algorithms learn $\epsilon$-optimal policies at the fastest rate possible. More precisely, when $\epsilon=o(1)$ (as $n$ grows large), their sample complexities match (up to logarithmic factors) our lower bounds in the minimax and reward-specific settings.
% !TEX root = ./main.tex
\section{Models and Objectives}\label{sec:prelim}

\subsection{Episodic block MDPs}

A Block MDP (BMDP) $\Phi$ is defined by $(\mathcal{S}, \mathcal{X}, \mathcal{A}, \mu, p, q, f, H)$ where $\mathcal{S}$, $\mathcal{X}$, $\mathcal{A}$ denote the sets of hidden latent states, observed contexts, and actions, respectively.
Let $S, n, A$ be their cardinalities.
$H \ge 2$ is the length of each episode and $\mu$ denotes the distribution of the initial context $x_1 \in \cX$.
$p$ represents the latent state time-homogenous dynamics: $p(s'|s,a)$ is the probability of moving from latent state $s$ to $s'$ given that the learner selects action $a$. The {\it decoding function} $f$ maps each context $x$ to a unique latent state $f(x) \in \cS$.
When the system is in the latent state $s$, the learner observes a context $x$ drawn according to the {\it emission distribution} $q(\cdot |s)$ with support $f^{-1}(s)\subset {\cal X}$.
Naturally, the decoding function $f$ induces a partition of $\cX$ into {\it clusters} indexed by $s$, and is initially unknown to the learner.
To simplify the notation, we write $\Phi=(p,q,f)$ to specify the dynamics and the structure of the BMDP.
We further denote $\PP_\Phi$ and $\EE_\Phi$ as the probability measure and the expectation for observations generated under the model $\Phi$, respectively.

\subsection{Latent state decoding and model estimation}
In this paper, our main objective is an offline model estimation: given some dataset, the learner wishes to estimate the latent state decoding function $f$ and the transition dynamics $p$ and $q$ of the BMDP.

The learner is given a dataset of transitions: $\cD = \left\{ (x_1^{(t)},a_2^{(t)}, \dots, x_{H-1}^{(t)},a_{H-1}^{(t)},x_{H}^{(t)})_{t \in [T]} \right\}$.
The dataset has been generated using the uniform behavior policy $\rho \sim \cU(\cA)$\footnote{This assumption can be relaxed. For further details,  we refer the reader to Appendix \ref{app:adaptive}.}.% \footnote{Our results can be easily extended to fixed behavior policies exploring all actions in all states with some strictly positive probability. Appendix \ref{app:adaptive} also provides some guidelines for extending our results to active behavior policies.} 
 The use of passive (also referred to as memoryless) behavior policy is common to derive theoretical guarantees~\citep{Azizzadenesheli16a,Azizzadenesheli16b}, and more importantly, to accommodate practical offline RL applications (such as healthcare) where active data collection strategies are impractical or dangerous~\citep{levine2020offline}.
Also note that recently, \cite{xiao2022curse} showed for passive data collection in offline RL, the uniform policy is the best behavior policy. An additional important remark is that we consider policy-induced data collection~\citep{xiao2022curse}, i.e., we do not assume access to any generative model.

The objective of the learner is to build $\hat{f}$, an estimate of the latent decoding function.
Since the latent states are not observed, $\hat{f}$ is only defined up to a permutation of these states, as in any clustering task~\citep{Abbe18,YunP19,SandersPY20}.
The accuracy of $\hat{f}$ is measured through the cardinality of the set ${\cal E}$ of misclassified contexts defined as $\mathcal{E} =  \cE_\nu$, where
\begin{equation}
	\cE_\sigma = \bigcup_{s \in \cS} \hat{f}^{-1}(\sigma(s)) \setminus f^{-1}(s), \quad \sigma \in \Upsilon(\cS),
\end{equation}
$\nu\in \argmin_{\sigma\in \Upsilon(\cS)} \left| \cE_\sigma \right|$, and $\Upsilon(S)$ denotes the set of permutations over ${\cal S}$.
We say that the clusters are recovered {\it asymptotically accurately} if $\EE[|\cE|] = o(n)$ and {\it asymptotically exactly} if $\EE[|\cE|] = o(1)$.
The lower and upper bounds of the error rates have been derived for various block models. Such clustering guarantees for BMDPs have {\it not} been studied, and they are highly nontrivial for the following reasons. First, because we deal with Markovian data, standard concentration techniques used in clustering Stochastic Block Models (SBMs) and their extensions~\citep{Abbe18,YunP19,Gao18} do not apply. Then, we have controllable actions that can change the transition probability, even between the same pair of states, which prohibits us from using the clustering algorithms of Block Markov Chains (BMCs; \cite{SandersPY20}) directly.

Using $\hat{f}$, the latent state transition probabilities and emission distribution are estimated using the {\it plug-in estimators}, which are very effective despite their simplicity~\citep{duan2020minimax,ren2021nearly,xiao2021optimality}. Their precise forms are provided in Section \ref{sec:performance}. The accuracies of these estimates, $\hat{p}$ and $\hat{q}$, are assessed through the $\ell_1$-distances: for all $(s, s', a) \in \cS^2 \times \cA$, $\Vert \hat{p}(\nu(\cdot) |\nu(s),a)- p(\cdot|s,a) \rVert_1$, and  $\lVert \hat{q}(\cdot |\nu(s))- q(\cdot|s) \rVert_1$. Without loss of generality, for simplicity, we assume that $\nu=\mathrm{Id}$, i.e., $\nu(s)=s$ for all $s$.

\subsection{Assumptions}

We impose the following assumptions, which allow us to invoke and translate theoretical tools and ideas from various literature on clustering in block models \citep{Abbe18,Gao18,YunP19,SandersPY20}:
\begin{assumption}[{Asymptotics}]
	\label{assumption:SA}
	$S, A, p$ are independent of $n$
\end{assumption}
\begin{assumption}[{Linear context cluster sizes}]
\label{assumption:linear}
The sizes of the clusters grow linearly with $n$. Specifically, there exists $\alpha=(\alpha_s)_{s\in \cS}$ independent of $n$ such that (i) $\alpha_s > 0$ for all $s \in \cS$, and (ii) $\sum_s \alpha_s = 1$ and $| f^{-1}(s)| = \alpha_s n$.
\end{assumption}
\begin{restatable}[$\eta$-regularity]{assumption}{separability}
	\label{assumption:regularity} There exists $\eta > 1$, independent of $n$, such that:
	\begin{align*}
	(i)& \ \ \ \max_{s_1, s_2 \in \cS} \frac{\alpha_{s_1}}{\alpha_{s_2}}\le \eta,\\
	(ii)& \ \ \ \max_{a \in \mathcal{A}} \max_{s_1, s_2, s_3 \in \mathcal{S}} \left\{ \frac{p(s_2 | s_1, a)}{p(s_3 | s_1, a)}, \frac{p(s_1 | s_2, a)}{p(s_1 | s_3, a)} \right\} \le \eta,\\
	(iii)&\ \ \ \max_{s \in \cS} \max_{x, y \in f^{-1}(s)} \frac{q(x | s)}{q(y | s)} \leq \eta.
	\end{align*}
\end{restatable}
 \begin{assumption}[Uniform initial context]
	\label{assumption:uniform}
	$\mu \sim \cU(\cX)$.
\end{assumption}

Assumptions \ref{assumption:SA}, \ref{assumption:linear}, and \ref{assumption:regularity}(i), which express some kind of homogeneity in the number and sizes of the clusters, are imposed in the majority of the clustering literature, see \cite{Abbe18,Gao18,YunP19,SandersPY20}. Such homogeneity is required to establish theoretical guarantees for clustering tasks.

Assumption \ref{assumption:regularity}(ii) has been used for the clustering task in BMCs~\citep{SandersPY20} to ensure a minimum level of separability between clusters. A similar assumption is made in SBMs~\citep{YunP16,YunP19,Gao18}. In Appendix \ref{app:beyond-eta}, we discuss how to relax Assumption \ref{assumption:regularity}(ii) and only assume that the transition kernel $p$ is aperiodic and communicating.

Assumption \ref{assumption:regularity} (iii) has been considered in the Degree Corrected Block Model (DCBM; \cite{Gao18}), where the heterogeneity via degree-correction of each cluster was allowed, but limited.

Assumption \ref{assumption:uniform} is not crucial but simplifies many statements in the proofs. It ensures that all contexts can be reached with positive probability.  We can replace it with $\mu(x) = \Theta\left( \frac{1}{n} \right)$, and given Assumption \ref{assumption:regularity}, even remove it, as long as $H \geq 3$.

% !TEX root = ./main.tex
\section{Fundamental Lower Bound of Latent State Decoding}\label{sec:lower}
To derive instance-specific lower bounds on the clustering error rates, we leverage change-of-measure arguments~\citep{Lai85} where we pretend that the observations are generated by a BMDP model $\Psi$ that is slightly different than the true model $\Phi$. To obtain instance-specific lower bounds, we always need to restrict the attention to a class of {\it adaptive} algorithms. Indeed, an algorithm always returning $\hat{f}=f$ would have no error for the BMDP $\Phi = (p, q, f)$ but would fail for any other BMDP with $\tilde{\Phi} = (p,q,\tilde{f})$ with $\tilde{f}$ different than $f$. Here, we restrict our attention to the wide class of so-called {\it $\beta$-locally better-than-random} clustering algorithms. 
\begin{definition}[$\beta$-locally better-than-random clustering algorithms]\label{def:locally-better-than-random}
	A clustering algorithm is {\textit{$\beta$-locally better-than-random}} for the BMDP $\Phi=(p,q,f)$ if for any BMDP $\tilde{\Phi}=(p,\tilde{q},\tilde{f})$ with $\max_{y : f(y) = \tilde{f}(y)} \max_{s} |q(y|s)-\tilde{q}(y|s)| \le \beta$ and $| \{y\in {\cal X}: f(y)\neq \tilde{f}(y)\}| \}\le 1$, for all $x \in \cX$, $x$ is misclassified with probability at most $1 - \frac{1}{S}$.
\end{definition}

Considering locally better-than-random algorithms is not restrictive, as algorithms with reasonable performance should indeed adapt to the BMDP they are facing, or in other words, such algorithms behave similarly when the considered BMDP model changes very slightly.

We derive a lower bound for the clustering error rate of {\it each context} $x \in \cX$ characterized by a rate function $I(x;\Phi)$ defined as follows.
Consider an alternate BMDP $\Psi^{(x,j)}=(p,\tilde{q},\tilde{f})$ obtained from the true model $\Phi$ by changing the latent state of context $x$ from $f(x)$ to $j$, for $j \neq f(x)$.
The alternate latent decoding function $\tilde{f}$ is defined as $\tilde{f}(y)=f(y)\indicator_{y\neq x}+j\indicator_{y = x}$ for all $y \in \cX$.
The alternate emission function $\tilde{q}$ remains unchanged for latent states different than $f(x)$ and $j$: for any $s \neq f(x),j$, $\tilde{q}(\cdot |s)=q(\cdot |s)$.
Under $\Psi^{(x,j)}$, $x$ has an emission probability $\tilde{q}(x|j) = cq(x|i)$ for $c>0$ to be chosen later, and we define for all $y\in f^{-1}(i)$, $\tilde{q}(y|i)=q(y|i)/(1-q(x|i))$ and for all $y\in f^{-1}(j)$, $\tilde{q}(y|j)=q(y|j)(1 - c q(x|i))$.
Note that with this definition, $\tilde{q}(\cdot | s)$ is a well-defined probability distribution over $\tilde{f}^{-1}(s)$, for all $s \in \cS$.
%\junghyun{The variable $c$ actually takes into account the inhomogeneous clusters, as we recall from Assumption \ref{assumption:regularity} (iii) that even in the same cluster, contexts may have (slightly) different emission probabilities and thus different transition dynamics.}

We now define the rate function $I(x;\Phi)$:
\begin{equation}
I(x; \Phi) = \min_{j \not= f(x)} \inf_{c > 0} I_j(x; c, \Phi),
\end{equation}
where $I_j(x;c,\Phi)$ is (almost)\footnote{Refer to Appendix \ref{app:lower-bound} for the exact expression of $I_j(x;c,\Phi)$.} equal to ${H\over n}$ times the expected log-likelihood ratio of the observations over one episode made under $\Psi^{(x,j)}$ and $\Phi$, and is $\infty$ when $c$ is such that $\Psi_j$ is not well-defined.
Importantly, the rate function $I(x;\Phi)$ does not scale with $n$, and we can identify necessary and sufficient conditions for the {\it clusterability} of $x$, i.e., for $I(x;\Phi)>0$. In Appendix \ref{sec:I-examples}, we provide illustrative examples of $I(x;\Phi)$ for clusterable and non-clusterable BMDPs.

In the next theorem, whose proof is presented in Appendix \ref{app:lower-bound}, we present our lower bounds for the clustering error rates.
\begin{theorem}
\label{thm:lower-bound}
Let $\Phi$ be a BMDP satisfying Assumptions \ref{assumption:SA},\ref{assumption:linear},\ref{assumption:regularity},\ref{assumption:uniform}. Consider a clustering algorithm that is $\beta$-locally better-than-random for $\Phi$ with $\beta \ge \frac{2S\eta^2}{n}$,
when applied to the data gathered over $T$ episodes, each of length $H$, using $\rho$.
Then, for all $x\in {\cal X}$,
\begin{equation*}
	\mathbb{P}_\Phi[x \in {\cal E}] \ge \frac{1}{2\eta S} \exp\left( - \frac{TH}{n} I(x; \Phi)(1+o(1)) \right).
\end{equation*}
As a consequence, we have that
\begin{equation*}
	\mathbb{E}_\Phi[|{\cal E}|]\ge n\exp\left(-\frac{TH}{n} I(\Phi)(1+o(1))\right),
\end{equation*}
with $I(\Phi) := -\frac{n}{TH} \log\left(\frac{1}{2 \eta S n} \sum_{x \in \cX} \exp\left( - \frac{TH}{n} I(x; \Phi) \right) \right)$.
%\end{equation}
\end{theorem}
%
%We now present a fundamental lower bound for the clustering error rate in expectation, whose proof is presented in Appendix \ref{app:lower-bound}:
%\begin{restatable}[Lower Bound on the Clustering Error]{theorem}{lowerbound}
%\label{thm:lower-bound}
%Assume that $T$ trajectories of our BMDP, each of length $H \geq 2$, are observed beforehand using an arbitrarily fixed memoryless policy $\pi = \{\pi^{(t)}\}_{t=1}^T$.
%Then the clustering error under any algorithm must satisfy the following:
%\begin{equation}
%\varepsilon_x^\pi \geq \frac{1}{4} \exp\left( - \frac{TH}{n} \left( I(x; \pi, \alpha, p) + \mathcal{O}\left(\frac{1}{n}\right) \right) \right).
%\end{equation}
%\end{restatable}

The above theorem provides necessary conditions for the existence of algorithms classifying a given context $x$ asymptotically accurately, i.e., $\mathbb{P}_\Phi[x\in {\cal E}]\to 0$ as $n\to\infty$. These conditions are $I(x;\Phi)>0$ and $TH=\omega(n)$.
In Appendix \ref{subsec:ratefunction}, we show that $I(x;\Phi)=0$ {\it if and only if} there exists $j\neq f(x)$ and $c > 0$ such that for all $(s,a)$,  $p(s | f(x), a) = p(s | j, a)$ and $p(f(x) | s, a) = c p(j | s, a)$.
Equivalently, these conditions state that the transition rates to and out of the latent states $f(x)$ and $j$ are identical, in which case, of course, we cannot determine whether $x$ belongs to $f(x)$ or $j$. We note that, thanks to Assumption \ref{assumption:regularity}, if $I(x;\Phi)>0$, then $I(y;\Phi)>0$ for all $y$ satisfying $f(y) = f(x)$. In words, the condition for accurate classification of a context only depends on its corresponding latent state.

We further deduce, from Theorem \ref{thm:lower-bound}, necessary conditions for the existence of algorithms recovering the clusters asymptotically accurately, i.e., $\mathbb{E}_\Phi[|\mathcal{E}|]=o(n)$. These conditions are $I(\Phi)>0$ and $TH=\omega(n)$. Note that $I(\Phi)> 0$ if and only if for all $x$ such that $I(x;\Phi)>0$ (we need to classify each context asymptotically accurately if we wish to do the same over all the latent states). Similarly, necessary conditions for the existence of an asymptotically exact clustering algorithm, i.e., $\mathbb{E}_\Phi[\mathcal{E}|]=o(1)$, are $I(\Phi) > 0$ and
$TH-{n\log(n)\over I(\Phi)} = \omega(1)$. In particular, $TH$ must be larger than $n\log(n)$. In the {\it critical} regime where $TH=n\log(n)$, the necessary condition for exact recovery is $I(\Phi)>1$.
These regimes are reminiscences of the clustering guarantees for SBMs~\citep{Abbe18,YunP19}, DCBMs~\citep{Gao18}, and BMCs~\citep{SandersPY20}.

\begin{remark}
We can extend Theorem \ref{thm:lower-bound} to scenarios where the behavior policy $\rho$ is adaptive, in which we just obtain a different rate function $I'(\Phi) > I(\Phi)$; see Appendix \ref{app:adaptive}.
\end{remark}

%\paragraph{Illustrative Example. }
%
%To illustrate the divergence $I$, we provide a simple example.
%Consider a BMDP with two clusters and three actions, each of which its latent transition transition is given as
%\begin{equation*}
%	P_1 = \begin{bmatrix}
%		\frac{1}{2} - \epsilon & \frac{1}{2} + \epsilon \\ \frac{1}{2} + \epsilon & \frac{1}{2} - \epsilon
%	\end{bmatrix},
%	\
%	P_2 = \begin{bmatrix}
%		\frac{1}{2} + \epsilon & \frac{1}{2} - \epsilon \\ \frac{1}{2} - \epsilon & \frac{1}{2} + \epsilon
%	\end{bmatrix},
%	\
%	P_3 = \begin{bmatrix}
%		\frac{1}{2} & \frac{1}{2} \\ \frac{1}{2} & \frac{1}{2}
%	\end{bmatrix},
%\end{equation*}
%where $\epsilon \in [0, 1/2)$ some user-defined parameter that determines the hardness of clustering.
%We assume that the emission transitions and the (exploration) policy are uniformly random distribution over their respective domains.
%
%\junghyun{Compute $I$ in function of $\epsilon$!}

% !TEX root = ./main.tex

\section{Latent State Decoding and Model Estimation Algorithms}\label{sec:algo}
In this section, we present our algorithms to recover the latent state decoding function $f$ as well as the parameters $(p, q)$ defining the BMDP dynamics. We also analyze their performance and sketch the elements of the analyses.

\subsection{Algorithms}

\RestyleAlgo{ruled}
\begin{algorithm2e}[t]
	\SetAlgoLined
	%	\tcc{Simulate the MDPs}
	\KwInput{$\cD = \lbrace (x_1^{(t)},a_2^{(t)}, \dots, x_{H-1}^{(t)},a_{H-1}^{(t)},x_{H}^{(t)})_{t \in [T]} \rbrace$}
	%	\tcc{Observation matrices}
	\For{$a \in \cA$}{
		for all $(x, y) \in \cX^2$,
		\begin{equation*}
			\hat{N}_a(x, y) \gets \sum_{t,h}  \indicator[(x_h^{(t)}, a_h^{(t)}, x_{h+1}^{(t)}) = (x,a,y)]\;
		\end{equation*}
  
		$\Gamma_a \gets $ $\mathcal{X}$ after removing $\floor*{n \exp\left( - (TH /nA) \log(TH /nA) \right)}$ contexts with the highest $\hat{N}_a(x)=\sum_y\hat{N}_a(x,y)$\;

		$\hat{N}_{a, \Gamma_a} \gets (\hat{N}_a(x, y) \indicator_{\{ (x,y) \in \Gamma_a\}})_{x, y \in \cX}$\;

		$\hat{M}_a \gets $ rank-$S$ approximation of $\hat{N}_{a,\Gamma_a}$\;
	}
	%		\tcc{Low-rank approximation, where the singular values are $\sigma_{a, 1} \geq \cdots \geq \sigma_{a, n} \geq 0$}
	$\hat{M} \gets \begin{bmatrix}
		(\hat{M}_1)^\top & \cdots & (\hat{M}_A)^\top & \hat{M}_1 & \cdots & \hat{M}_A
	\end{bmatrix}$\;

	Normalize the rows of $\hat{M}$ by the $\ell_1$-norm\;
	%	\tcc{$|\mathcal{S}|$-means clustering}

	Obtain $\hat{f}_1$ by applying the K-medians algorithm to the rows of $\hat{M}$\;

	%	$\hat{M} := [\hat{M}_{in}; \hat{M}_{out}; \hat{R}] \in \mathbb{R}^{(2n + 1)|\mathcal{A}| \times n}$\;
	\KwOutput{$\hat{f}_{1}$ (initial estimate of the decoding function)}
	\caption{Initial Spectral Clustering}
	\label{alg:spectral-clustering}
\end{algorithm2e}

The latent decoding function $f$ of the given BMDP is estimated using a two-step procedure.
The first step consists in leveraging spectral methods to obtain rough estimates of the latent state decoding function or clusters and is described in Algorithm \ref{alg:spectral-clustering}. The second step iteratively improves the clusters using log-likelihoods of each cluster and is presented in Algorithm \ref{alg:likelihood-improvement}.
At the end of the second step, we get accurate estimates of the latent state decoding function $f$, from which we use the plug-in estimator to estimate the transition dynamics $p$ and $q$.

Our two-step procedure is inspired by the clustering algorithm designed for DCBMs~\citep{Gao18} and BMCs~\citep{SandersPY20}, but there are significant differences mainly due to the fact that the data consists of episodes of a controlled Markov chain (the selected actions do matter) and to the presence of non-uniform emission distributions.

{\bf Initial spectral clustering step.}
As most clustering algorithms with optimality guarantees~\citep{Luxburg07}, our procedure starts with a spectral decomposition of matrices built from the data $\cD$, observed from the underlying BMDP model.
We wish to distinguish whether a context belongs to one cluster or another whenever this is statistically possible, from the observations corresponding to at least one of the actions.
Hence, for each action $a \in \cA$, we build from the data an observation matrix $\hat{N}_a$ and inspect its spectral properties.
%{\color{red}
%In Appendix \ref{app:negative}, we show that algorithms based on the analysis of a single matrix combining the observations for all actions, e.g., $\hat{N}(x,y)=\sum_a\pi(a|x)\hat{N}_a(x,y)$, would be suboptimal.
%}
We compute the rank-$S$ approximation $\hat{M}_a$ of each $\hat{N}_a$ after {\it trimming}\footnote{In the sparse regime ($TH = o(n \log n)$), trimming is necessary to remove contexts with too many observations that would perturb the empirical spectral distribution~\citep{Feige05, Keshavan10a,YunP16,SandersPY20}. In the dense regime ($TH = \Omega(n \log n)$), such a trimming procedure is not necessary~\citep{Sanders21}.}.
We then {\it concatenate}\footnote{For large BMDPs, concatenation may not be feasible due to memory or computational limitations. One possible workaround is to utilize random linear combination, as done in \cite{YunP16} for clustering labeled SBMs; this is left as future work.} these matrices in
$\hat{M}$, a $(n \times 2nA)$ matrix, whose $x$-th row contains all the information we have for context $x$ (transitions starting from $x$ {\it and} ending in $x$ when selecting action $a$ for any possible $a$). Applying the $S$-median algorithm to the rows of $\hat{M}$ yields our initial cluster estimate encoded in $\hat{f}_1$.

%\begin{remark}
%	Our previous Example 1 illustrates why concatenating is considered.
%	If the transition matrices for each $a$ are averaged, then the resulting matrix is $\begin{bmatrix}
%		1/2 & 1/2 \\ 1/2 & 1/2
%	\end{bmatrix}$, from which the clusters can never be recovered.
%	Concatenation allows us to leverage each action's information separately, which makes clustering possible in such cases, even when only a single action is informative.
%	\junghyun{Indeed, there always exists a BMDP with a similar issue for any deterministic linear combination of the transition matrices.
%	One alternative approach is to thus consider a random linear combination of the $\hat{N}_a$'s, as done in \cite{YunP16}; we leave this for future work.}
%\end{remark}

{\bf Iterative likelihood improvement step.} This step takes as input our initial cluster estimates $\hat{f}_{1}$ as well as the counters of transitions $\hat{N}=(\hat{N}_a(x,y))_{a,x,y}$ (see Algorithm \ref{alg:spectral-clustering}). For any $X,Y \subseteq \cX$, we denote by $\hat{N}_a(X,Y)$ the number of observed transitions, when action $a$ is chosen, from any context $x \in X$ to any context $y \in Y$.
In the $\ell$-th iteration, we first use our current cluster estimates $\hat{f}_\ell$ and $\hat{N}$ to estimate the empirical latent transition $\hat{p}_\ell(s | j, a)$ and the empirical backward latent transition $\hat{p}^{bwd}_\ell(s, a | j)$ ($\hat{p}^{bwd}$ is in a sense similar to the backward probability vector used in \cite{DuKJAD19a}).
Using $\hat{p}_\ell, \hat{p}^{bwd}_\ell$, $\hat{f}_\ell$, and $\hat{N}_a$'s, we compute for each context $x$ the log-likelihood ${\cal L}^{(\ell)}(x, j)$ of the event that $x$ is assigned to the latent state $j$.
$x$ is then re-assigned to the cluster with the highest ${\cal L}^{(\ell)}(x, j)$, and we update $\hat{f}_{\ell+1}$ accordingly.
We perform $L=\lfloor \log(nA)\rfloor$ iterations and then output the final estimated decoding function $\hat{f} = \hat{f}_{L+1}$.
%In each iteration, the previous cluster estimates are exploited to estimate the transition probabilities $p$ of the BMDP; from these estimates, we can obtain the likelihood of a context to belong to various clusters; and contexts are then re-assigned to clusters by maximizing this likelihood.

%{\color{red} Lastly, to ensure the quality of estimation, we independently observe $\ceiling{T/2}$ episodes and {\it redefine} the $\hat{N}_a$'s from the newly obtained trajectories.
%Then we obtain the final estimates $(\hat{p},\hat{q})$ computed based on $\hat{f}$ and the new $\hat{N}_a$'s.}

\begin{algorithm2e}[!t]
	\SetAlgoLined
	\KwInput{Initial cluster estimates $\hat{f}_1$ and $\{ \hat{N}_a \}_{a \in \cA}$}
	\For{$\ell=1$ to $L = \floor{\log (nA)}$}{
		for all $(s,j,a) \in \cS^2 \times \cA$, 
		\begin{equation}
			\hat{p}_\ell(s|j,a) \gets {\hat{N}_a(\hat{f}_\ell^{-1}(j), \hat{f}_\ell^{-1}(s)) \over \hat{N}_a(\hat{f}_\ell^{-1}(j), \cX)},
		\end{equation}
		\begin{equation}
			\hat{p}^{bwd}_\ell(s, a | j) \gets {\hat{N}_a(\hat{f}_\ell^{-1}(s), \hat{f}_\ell^{-1}(j)) \over \sum_{\tilde{a} \in \cA} \hat{N}_{\tilde{a}}(\cX, \hat{f}_\ell^{-1}(j))}, \;
		\end{equation}
		
		for all $x \in \cX$, $\hat{f}_{\ell+1}(x) \gets \argmax_{j \in \mathcal{S}} \mathcal{L}^{(\ell)}(x, j)$,\\
		where
		\begin{align*}
			\label{eq:likelihood}
			\mathcal{L}^{(\ell)}(x, j)  &= \sum_{a \in \cA} \sum_{s \in \cS} \hat{N}_a(x, \hat{f}_\ell^{-1}(s)) \log \hat{p}_\ell(s | j, a) \\
			&+ \hat{N}_a(\hat{f}_\ell^{-1}(s), x) \log \hat{p}^{bwd}_\ell(s, a | j) \;
		\end{align*}
	
	}

	\KwOutput{$\hat{f} = \hat{f}_{L+1}$ (final estimate of the decoding function)}
	\caption{Iterative Likelihood Improvement}
	\label{alg:likelihood-improvement}
\end{algorithm2e}

{\bf Estimating $p$ and $q$.}
For simplicity of the analysis, we decouple the estimation of $f$ and that of $p$ and $q$.
Precisely, we first estimate $\hat{f}$ via Algorithms \ref{alg:spectral-clustering} and \ref{alg:likelihood-improvement} using the first $\lfloor T/2\rfloor$ episodes.
We then use the remaining $\lfloor T/2\rfloor$ episodes to get the estimates $\hat{p}$ and $\hat{q}$ using the plug-in estimators: for each $(s, s', a) \in \cS^2 \times \cA$ and $x \in \hat{f}^{-1}(s)$,
\begin{equation}\label{eq:est2}
	\left\{
	\begin{array}{ll}
		\hat{p}(s'|s,a) = {\hat{N}_a(\hat{f}^{-1}(s), \hat{f}^{-1}(s'))\over \hat{N}_a(\hat{f}^{-1}(s),\cX)},\\
		\hat{q}(x|s) =  \frac{\sum_{t=1}^{\floor{T/2}}\sum_{h=1}^H \mathds{1}[(x_h^{(t)}, \hat{f}(x_h^{(t)})) = (x,s)]}{\sum_{t=1}^{\floor{T/2}}\sum_{h=1}^H \mathds{1}[\hat{f}(x_h^{(t)}) = s]}
	\end{array}
	\right.
\end{equation}
If the denominator in some of the above fractions is $0$, then the corresponding estimate is set to be $0$. Here, we leverage $\hat{f}$ (from the first half of the data) but use the second half of the data to {\it redefine} the number of transitions $\hat{N}_a(x,y)$ for all $(x,y,a)$.

\subsection{Performance guarantees}
\label{sec:performance}

The following theorem, whose proof is deferred to Appendix \ref{app:init1}, provides a performance guarantee for Algorithm \ref{alg:spectral-clustering}, the initial spectral clustering algorithm. It states that under the necessary conditions for asymptotically accurate cluster recovery, the algorithm indeed misclassifies only a vanishing proportion of contexts.
\begin{theorem}
\label{thm:initial-spectral}
Assume that $TH = \omega(n)$ and $I(\Phi)>0$. Then after Algorithm \ref{alg:spectral-clustering}, the following holds w.h.p.\footnote{w.h.p. means that with probability tending to 1 as $n \rightarrow \infty$.}:
\begin{equation*}
	\frac{|\mathcal{E}|}{n} = \mathcal{O}\left(\sqrt{\frac{nSA}{TH}} \right).
\end{equation*}
\end{theorem}
%Theorem \ref{thm:initial-spectral} states that we achieve an asymptotically accurate detection whenever this is at all possible (when $TH = \omega(n)$ and $I(\Phi) > 0$), by applying the initial spectral clustering step alone.
%\begin{remark}
%As in \cite{SandersPY20}, we include in our spectral clustering algorithm the trimming step, removing context states visited too often, which regularizes the spectrum.
%As pointed out in \cite{Sanders21}, in the ``dense'' regime (i.e. when $TH = \Omega(n \log n)$), Theorem \ref{thm:initial-spectral} holds even without the trimming step.
%\end{remark}

The next theorem, whose proof is presented in Appendix \ref{app:likelihood-improvement}, provides the clustering error rate guarantee after the iterative likelihood improvement steps (Algorithm \ref{alg:likelihood-improvement}), as well as the estimation error upper bounds for $\hat{p}$ and $\hat{q}$:
% \begin{theorem}
% \label{thm:likelihood-improvement}
% Assume that $TH = \omega(n)$ and $I(\Phi)>0$. Then, after Algorithm \ref{alg:likelihood-improvement},\\
% (i) there exists a universal constant $C = \poly(\eta) > 0$ such that the following holds w.h.p.:
% \begin{equation*}
% { |{\cal E}| } = \cO\left( \sum_{x \in \cX} \exp\left( - C \frac{TH}{n} I(x; \Phi)\right) \right).
% \end{equation*}
% (ii) For all $(s,s',a) \in \mathcal{S}^2 \times \mathcal{A}$, we have w.h.p.
% 	\begin{align*}
% 	\Vert  {\hat{p}(\cdot | s, a) - p(\cdot | s, a)} \Vert_1 & = \mathcal{O}\left( \sqrt{\frac{S^3A^2\log\left(n S A\right) }{TH} } \right.\\
% 	& \ \ \ \ \ \ \ \ \ \ \ \ \ \ \ \ \ \ \ \ \ \ \left. + \frac{SA|\cE|}{n} \right),\\
% 	\Vert \hat{q}(\cdot \vert s) - q(\cdot \vert s)\Vert_1 & = \mathcal{O} \left( \sqrt{ \frac{Sn}{TH}} + \frac{S \vert \cE\vert}{n}  \right).
% 	\end{align*}
% \end{theorem}
\begin{theorem}
\label{thm:likelihood-improvement}
Assume that $TH = \omega(n)$ and $I(\Phi)>0$. Then, after Algorithm \ref{alg:likelihood-improvement},\\
(i) there exists a universal constant $C = \poly(\eta) > 0$ such that the following holds w.h.p.:
\begin{equation*}
{ |{\cal E}| } = \cO\left( \sum_{x \in \cX} \exp\left( - C \frac{TH}{n} I(x; \Phi)\right) \right).
\end{equation*}
(ii) Let us denote $E_p = \Vert  {\hat{p}(\cdot | s, a) - p(\cdot | s, a)} \Vert_1$ and $E_q = \Vert \hat{q}(\cdot \vert s) - q(\cdot \vert s)\Vert_1$.
For all $(s,s',a) \in \mathcal{S}^2 \times \mathcal{A}$, we have w.h.p.,
\begin{align*}
	 E_p &= \mathcal{O}\left( \sqrt{\frac{S^3A^2\log\left(n S A\right) }{TH} } + \frac{SA|\cE|}{n} \right),\\
	 E_q & = \mathcal{O} \left( \sqrt{ \frac{Sn}{TH}} + \frac{S \vert \cE\vert}{n}  \right).
\end{align*}
\end{theorem}

Theorem \ref{thm:likelihood-improvement} {\it(i)} refines the result of Theorem \ref{thm:initial-spectral} and shows that our two-step procedure approaches the optimal recovery rate identified in Theorem \ref{thm:lower-bound} up to the constant $C$ in the exponential. Moreover, it indicates that our algorithm exactly recovers the clusters if $TH - \frac{n\log(n)}{C I(x;\Phi)} =\omega(1)$ for all $x$ and $I(\Phi) > 0$.

\subsection{Sketches of the proofs}
\label{sec:proofs}

The proofs of Theorems \ref{thm:initial-spectral} and \ref{thm:likelihood-improvement} rely on two concentration inequalities presented in Appendix \ref{app:concentration} and \ref{app:init1}.
The first inequality is a novel Bernstein-type inequality for functionals of Markov chains with restarts (to account for the episodic nature of the MDP). This contrasts with existing concentration results for Markov chains that concern a single trajectory of the chain starting in a stationary regime~\citep{Avrachenkov13, Avrachenkov18,Paulin15}, which we do not assume here.
The second inequality characterizes how the observation matrix $\hat{N}_a$ for a given action $a$ concentrates around its mean. The main challenge for the analysis of $\hat{N}_a$ stems from the Markovian dependence in its entries, and is addressed using similar techniques as in \cite{SandersPY20,Sanders21}.

% !TEX root = ./main.tex
\section{Reward-Free RL in BLock MDPs}\label{sec:reward-free}

\subsection{Setting and objectives}
A policy $\pi$ for the BMDP $\Phi$ is a collection of $H$ mappings $\pi_h: {\cal X}\to {\cal P}({\cal A})$ ($\cP(\cA)$ is the set of probability distributions on $\cA$). We denote by $\pi_h(a|x)$ the probability of selecting $a$ when the context $x$ is observed in stage $h$. The value of a policy $\pi$ is defined as the expected reward accumulated over an episode $V^\pi(r)= \mathbb{E}_{\pi,\Phi}[ \sum_{h=1}^H r_h(x_h^\pi,a_h^\pi)]$, where $x_h^\pi$ and $a_h^\pi$ are the context and the selected action under $\pi$ in stage $h$. The reward function $r=(r_h)_{h\in [H]}$ is non-stationary (it depends on the stage $h$), deterministic, and  depends on contexts, not only on the corresponding latent states. We assume that the reward $r_h(x,a)$ gathered in (context, action) pair $(x,a)$ in stage $h$ lies in $[0,1]$. We denote by $\pi^\star_r$ an optimal policy for the reward $r$, and by $V^\star(r)$ its value.

In this section, we study the offline reward-free RL task~\citep{yin2021optimal}. Specifically, the learner does not know nor utilize the reward function in the BMDP until the model estimation is completed~\citep{jin2020reward}. After that, the reward function $r$ is revealed. The learner then computes $\hat{\pi}_r$, the optimal policy for the MDP $\hat{\Phi}$ with reward function $r$. This procedure is sometimes referred to as the {\it plug-in algorithm}. The performance of this model-based RL algorithm is characterized by $\Delta(\cR) = \sup_{r\in {\cal R}} (V^\star(r)-V^{\hat{\pi}_r}(r))$ for a given class ${\cal R}$ of reward functions. The case where ${\cal R}$ includes all possible reward functions is referred to as the {\it minimax reward} setting, whereas the case where ${\cal R}$ reduces to a single reward function $r$ to as the {\it reward-specific} setting.

%, since rewards play little role in our reward-free analysis as in \cite{DuKJAD19a}.
%We consider general non-stationary reward functions, i.e., $r_h(x,a)\in [0,1]$ is the reward gathered when visiting the (context, action) pair $(x,a)$\footnote{Note that rewards may depend on contexts, not only on the corresponding latent state; see Appendix B of \cite{DuKJAD19a} for more detailed discussions.} at stage $h$.

\subsection{Lower bounds}

We derive lower bounds on the sample complexity for identifying $\epsilon$-optimal policies in both the minimax reward and reward-specific settings. These lower bounds hold even when {\it adaptive} behavior policies are used to gather the data. As shown later in this section, the performance of the plug-in algorithm based on the model estimation procedure presented in previous sections almost matches these lower bounds, even if the behavior policy is restricted to being passive.

For the minimax setting, we first define 
%$\Lambda(\Phi) = \max_{\bm{v} \in [-1,1]^S}\frac{1}{S}\sum_{s=1}^S\max_{a_{1}, a_{2}} \langle p(\cdot | s, a_1) - p(\cdot | s, a_2), \bm{v} \rangle$.
%\junghyun{
\begin{equation*}
	\Lambda(\Phi) = \max_{\bm{v} \in [-1,1]^S}\frac{1}{S}\sum_{s=1}^S\max_{a_{1}, a_{2}} \left\langle p(\cdot | s, a_1) - p(\cdot | s, a_2), \bm{v} \right\rangle,
\end{equation*}
where $p(\cdot | s, a) = \left( p(s' | s, a) \right)_{s' \in \cS}$ is a vector of length $S$. Note that as $p(s'|s,a)$ does not depend on $n$ for all $(s,s')$, $\Lambda(\Phi)$ is strictly positive and does not depend on $n$. 

\begin{theorem}\label{thm:lower-boundSC1}{[Minimax reward setting]}
	Consider a BMDP $\Phi$ such that $\Lambda(\Phi)>0$. Any algorithm that guarantees $\sup_r {1\over H}(V^\star(r) - V^{\hat{\pi}(r)}(r))\le\epsilon$ with probability at least $1/2$ requires $TH=\Omega({n\Lambda(\Phi) \over \epsilon^2})$ samples.
\end{theorem}
%\junghyun{
%	write it like this?
%\begin{theorem}\label{thm:lower-boundSC1}{[Minimax setting]}
%	Consider a BMDP $\Phi$ such that $\Lambda(\Phi)>0$. Any algorithm that guarantees $\Delta(\cR) \le \epsilon$ with probability at least $1/2$ requires $TH=\Omega({n\Lambda(\Phi) \over \epsilon^2})$.
%\end{theorem}
%}
The proof of the above theorem, presented in Appendix \ref{app:reward-free-lower}, reveals that to get a minimax reward-free guarantee, the minimal sample complexity is mainly dictated by the estimation of the emission distributions $q$. This contrasts with the case where we target reward-specific guarantees. In this case, the minimal sample complexity is limited by the block structure estimation as shown in the following theorem, whose proof is deferred to Appendix \ref{app:reward-free-lower}.
\begin{theorem}\label{thm:lower-boundSC2}{[Reward-specific setting]}
	Let $\epsilon=o(1)$. Consider an algorithm with the following guarantees: for any BMDP $\Phi$ satisfying Assumptions \ref{assumption:SA}-\ref{assumption:uniform} and $I(\Phi)>0$, and for any reward function $r$ initially revealed to the algorithm, $\mathbb{E}_\Phi[{1\over H}(V^\star(r) - V^{\hat{\pi}(r)}(r))]\le\epsilon$.
	Then the algorithm requires $TH=\Omega(n\log(\frac{1}{\epsilon})+{SA\over \epsilon^2})$ samples.
\end{theorem}
%\begin{theorem}\label{thm:lower-boundSC2}{[Reward-specific setting]}
%	Let $\epsilon=o(1)$. Consider an algorithm with the following guarantees: for any BMDP $\Phi$ satisfying Assumptions \ref{assumption:SA}, \ref{assumption:linear}, \ref{assumption:regularity}, and such that $I(\Phi)>0$, for any reward function $r$ initially revealed to the algorithm, $\mathbb{E}_\Phi[{1\over H}(V^\star(r) - V^{\hat{\pi}(r)}(r))]\le\epsilon$. Then the algorithm requires $TH=\Omega(n\log(\frac{1}{\epsilon})+{SA\over \epsilon^2})$ samples.
%\end{theorem}
%\junghyun{
%	write it like this?
%	\begin{theorem}\label{thm:lower-boundSC2}{[Reward-specific setting]}
%		Let $\epsilon=o(1)$.
%		Consider an algorithm such that for any BMDP $\Phi$ satisfying Assumptions \ref{assumption:SA}-\ref{assumption:uniform}, and such that $I(\Phi)>0$, and for any reward function $r$ revealed to the algorithm $\mathbb{E}_\Phi[\Delta(\{r\})]\le\epsilon$.
%		Then the algorithm requires $TH=\Omega(n\log(\frac{1}{\epsilon})+{SA\over \epsilon^2})$.
%	\end{theorem}
%}

The first term $n\log(1/\epsilon)$ in the sample complexity lower bound may be interpreted as the number of samples required to learn the block structure accurately, and the second term to the data required to learn an $\epsilon$-optimal policy given the block structure. We establish in the next subsection that our model-based method achieves the limits predicted in the above theorems. In addition, note that our method does not require any adaptive exploration procedure. And for the reward-specific guarantees, it does not even require any prior knowledge of the reward function (this knowledge is assumed for the lower bound in Theorem \ref{thm:lower-boundSC2}).

%It's worth mentioning that for general tabular MDPs and non-reward-free setting (i.e. a single reward is given at the beginning), \cite{xiao2022curse} 

\subsection{Performance of the plug-in algorithm}
%\junghyun{We now provide sample complexity upper bounds for both settings.}

Consider the plug-in algorithm computing the optimal policy $\hat{\pi}_r$ for the estimated model $\hat{\Phi}=(\hat{p},\hat{q},\hat{f})$ and the reward function $r$. In the minimax reward setting, we have the following performance guarantees (their proof is presented in Appendix \ref{app:reward-free-upper}).

\begin{theorem}[Minimax reward setting]
\label{thm:rf1}
	Consider a BMDP $\Phi$ satisfying Assumptions \ref{assumption:SA}-\ref{assumption:uniform}. Further assume that $TH=\omega(n)$ and $I(\Phi)>0$. Then we have w.h.p.,
	$$
	\sup_r {1\over H} (V^\star(r) - V^{\hat{\pi}_r}(r))= {\cal O}\left(\sqrt{\frac{nS^2A^2\log(SAH)}{TH}}\right).
	$$
\end{theorem}
%\begin{restatable*}[Minimax reward setting]{theorem}{minimax}
%\label{thm:rf1}
%	Consider a BMDP $\Phi$ satisfying Assumptions \ref{assumption:SA}-\ref{assumption:uniform}. Further assume that $TH=\omega(n)$ and $I(\Phi)>0$. Then we have: w.h.p.
%$$
%\sup_r {1\over H} (V^\star(r) - V^{\hat{\pi}_r}(r))= {\cal O}\left(\sqrt{\frac{nS^2A^2\log(SAH)}{TH}}\right).
%$$
%\end{restatable*}

The above minimax guarantees match those predicted by the lower bounds presented in Theorem \ref{thm:lower-boundSC1}: $\sup_r {1\over H} (V^\star(r) - V^{\hat{\pi}_r}(r))$ scales as $\sqrt{n\over TH}$ (up to a multiplicative constant $S\sqrt{A}$ and logarithmic factors). Further note that the gap between the value of the optimal policy and that of $\hat{\pi}_r$ decreases as $H$ increases. This is the first time such decay is proved (for general MDPs, the gap should scale at least as $\frac{1}{\sqrt{T}}$~\citep{jin2020reward}; we get a gap scaling as $\frac{1}{\sqrt{TH}}$ due the specific nature of our BMDPs and more specifically, the fact that they enjoy a bounded mixing time).

The next theorem (see Appendix \ref{app:reward-free-upper} for the proof) provides performance guarantees in the reward-specific setting.
\begin{theorem}\label{thm:rf2}[Reward-specific setting]
	Let $C$ be the constant introduced in Theorem \ref{thm:likelihood-improvement}(i). Under the assumptions of Theorem \ref{thm:rf1}, we have for any reward function $r$, w.h.p.
	\begin{align*}
		\frac{1}{H}(V^\star(r) - &  V^{\hat{\pi}_r}(r)) = {\cal O}\left( \sqrt{\frac{S^3A^2H\log(SAHn)}{T}} \right. \\
		&\ \ \ + \left. \frac{SH^2}{n} \sum_{x \in \cX} \exp\left( - C \frac{TH}{n} I(x; \Phi) \right) \right).
	\end{align*}
\end{theorem}
%\begin{theorem}\label{thm:rf2}[Reward-specific setting]
%Let $C$ be the universal constant introduced in Theorem \ref{thm:likelihood-improvement}(i). Under the assumptions of Theorem \ref{thm:rf1}, we have for any reward function $r$, w.h.p.
%\begin{align*}
% \frac{1}{H}(V^\star(r) - V^{\hat{\pi}_r}(r)) &= {\cal O}\left( \sqrt{\frac{S^3A^2H\log(SAHn)}{T}} \right. \\
% &+ \left. \frac{SH^2}{n} \sum_{x \in \cX} \exp\left( - C \frac{TH}{n} I(x; \Phi) \right) \right).
%\end{align*}
%\end{theorem}

The upper bound shown in Theorem \ref{thm:rf2} consists of two terms. The first scaling as $\sqrt{1/T}$ corresponds to the error made when learning the optimal policy assuming the block structure has been accurately inferred. The second term, scaling as $e^{ - D \frac{TH}{n}}$ for some $D > 0$, corresponds to the error made due to mistakes in the block structure estimation. These two terms also match the lower bound derived in Theorem \ref{thm:lower-boundSC1}: there, we have shown that $\frac{1}{H}(V^\star(r) - V^{\hat{\pi}_r}(r))$ should scale at least as $\sqrt{1/T}$ and $e^{-D{TH\over n}}$. Finally note that in the case our algorithm recovers the clusters asymptotically exactly, i.e., when $TH-{n\log(n)\over C' I(x;\Phi)} = \omega(1)$ for all $x$, then we can remove the second term in our upper bound, and hence prove that $\frac{1}{H}(V^\star(r) - V^{\hat{\pi}_r}(r))$ scales as $\sqrt{1/T}$.
% (see Appendix \ref{app:reward-free-upper}).

\begin{remark}
    The condition that $I(\Phi) > 0$ is closely related to the previous separability notions considered in BMDP literature~\citep{DuKJAD19a,misra2020kinematic}.
    In Appendix \ref{app:separability-others}, we show that our separability condition is strictly \textit{{weaker}} than the $\gamma$-separability~\citep{DuKJAD19a} and kinematic separability~\citep{misra2020kinematic} in that ours encompass a larger set of ``separable'' BMDPs.
    Additionally, unlike the previous works where the separability condition was imposed as assumptions, we have naturally derived our separability condition from an information-theoretic argument (e.g., from the proof of Theorem \ref{thm:lower-bound}).
\end{remark}

% !TEX root = ./main.tex
\section{Related Work}\label{sec:related-works} 
%\textbf{\color{red}see Appendix~\ref{app:related-works} for further related works.}

\paragraph{Structure recovery in block models.}
The problem of learning the latent state decoding function is in a sense similar to the cluster recovery problem in SBM~\citep{Holland83}; see \cite{Abbe18} for a brief survey and references therein. In the SBM, the learner observes a random graph whose edges are drawn independently of each other and with probabilities that depend only on the cluster ids of the two corresponding vertices. From this observation, the objective is to infer the initial clusters. Information-theoretical lower bounds and optimal algorithms have been proposed for this simple block model~\citep{YunP16,Abbe18}, and extensions, e.g., DCBM~\citep{Gao18, Dasgupta04,Karrer11}.
For BMDPs, the data consists of trajectories of a controlled Markov chain. This implies that we deal with correlated samples, which significantly complicates the cluster recovery task. Dealing with Markovian data has been investigated in the case of {\it uncontrolled} Markov chains: in \cite{zhu2019learning,zhang2019spectral,duan2018state}, the authors analyze scenarios under which the transition kernel of the Markov chain exhibits a low-rank structure. Closer to our problem, \cite{SandersPY20} studies cluster recovery in BMC. The main differences with our model are that in \cite{SandersPY20}, (i) the Markov chain is uncontrolled (there are no control actions); (ii) more importantly, the emission distributions are known and uniform for each latent state (i.e., within each context cluster, the contexts are indistinguishable); (iii) finally, the observations come from a single long trajectory of the Markov chain, which simplifies the analysis (since it can leverage existing concentration results for Markov chains in the stationary regime~\citep{Paulin15}).

\paragraph{Reinforcement learning in BMDPs.}
There have been considerable research efforts recently towards the design of efficient RL algorithms for BMDPs~\citep{jiang2017,dann2018,DuKJAD19a,misra2020kinematic,foster2021,zhang2022BMDP}, as well as low-rank MDPs~\citep{sun2019,agarwal2020flambe,Modi21,uehara2022}.
All the aforementioned studies (both for block and low-rank MDPs) rely on rich function approximators. When applied to our BMDP problem formulation, this means that the block structure can be accurately represented through a function belonging to some function class $\cF$. When the reward function is fixed and given (similar to our reward-specific setting), the sample complexity for identifying an $\epsilon$-optimal policy is proved to scale at most as $\textrm{poly}(S,A,H){\log |\cF|}/\epsilon^2$ where $S$, $A$ and $H$ denotes the number of latent states, of actions and of rounds per episode, respectively.
However, without any prior or any additional structural assumption imposed in BMDPs, $\cF$ should correspond to the set of all possible assignments of contexts to latent states, so that $\log |\cF|=n\log(S)$. As a consequence, the algorithms presented in the aforementioned papers do not provably exploit the block structure. Indeed, for tabular MDPs (without any block structure), there are algorithms with sample complexity scaling as $\textrm{poly}(H)An/\epsilon^2$, see \cite{Menard21}. We should also mention that the use of function approximation in Block MDPs and low-rank MDPs requires strong computational oracle assumptions (ERM or MLE estimators on the function class $\Upsilon$). The computational issues for RL with function approximation have been recently discussed in many papers including \cite{kane2022computational, golowich2022learning, zhang2022making, sam2022lowrank, shah2020lowrank}. 
%In \cite{kane2022computational}, the authors establish a reduction between solving linear MDPs and 3-SAT problem. Although there the features vectors are assumed to be known, there is still a computational statistical gap to be adressed. In \cite{zhang2022making}, the authors point out the intractability of oracle based estimators. In \cite{golowich2022learning}, the authors also address the statistical-computational complexity of solving POMDPs.

\cite{Feng20} follow a different path towards efficient algorithms compared to the aforementioned studies, but assume that the learner has access to an unsupervised learning oracle that outputs the correct decoding function with high probability after $\textrm{poly}(S)$ observed samples. Finally, it is worth mentioning \cite{Azizzadenesheli16b}, an early work where the authors analyze the regret of online algorithms using spectral methods to infer the structure. However, their regret guarantees scale as (and is valid only after) $n^2$ rounds.

To the best of our knowledge, this paper is the first to analyze when and how a block structure in MDPs can be learned and exploited without the assumption that this structure can be represented using function approximation and without resorting to any kind of oracles.
%\cite{Azizzadenesheli16b} considered tensor spectral methods for inference and learning of POMDPs

\paragraph{Reward-free RL}
The last part of this paper deals with reward-free RL. Such RL task has been extensively studied recently for tabular MDPs, see e.g., \cite{jin2020reward,kaufmann2021adaptive,Menard21}. In tabular episodic MDP with $n$ states, the sample complexity of algorithms leading to $\epsilon$-optimal policies for any reward function (in the minimax setting) is $\textrm{poly}(H)n^2A/\epsilon^2$. If the algorithm needs to return an $\epsilon$-optimal policy for a given reward function, the sample complexity becomes $\textrm{poly}(H)nA/\epsilon^2$~\citep{Menard21}. By exploiting the block structure and under appropriate conditions, we show in this paper that we can significantly reduce the sample complexity in both scenarios (e.g. by a factor $n$ in the minimax setting).

Our work is also related to the so-called offline RL problems. For example, \cite{xiao2022curse} study a learning task where the reward is specified at the very beginning and is given as part of the static data $\cD$ (the authors referred to this setting as {\it batch policy optimization (BPO)}). %To the best of our knowledge, we are the first to consider offline reward-free RL on BMDPs.

% !TEX root = ./main.tex

\section{Numerical Experiments}\label{sec:experiments}

\begin{figure*}[!th]
\centering
\includegraphics[width=\textwidth]{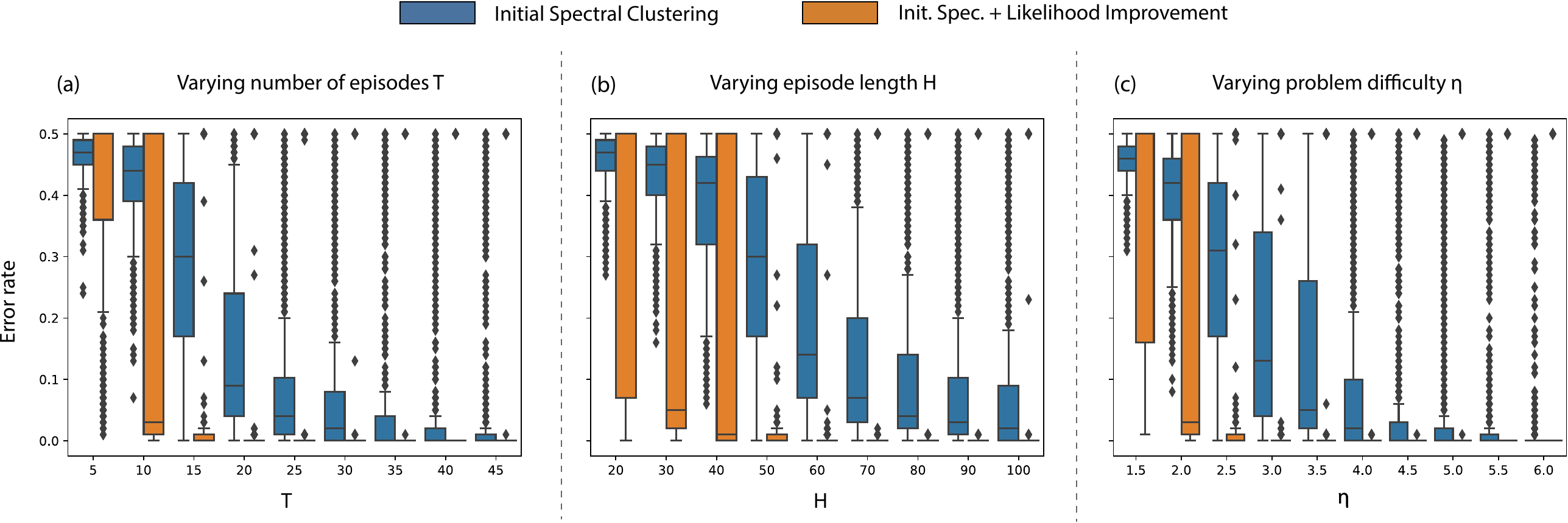}
\caption{The clustering error rates for various choices of (a) $T$'s, (b) $H$'s, and (c) $\eta$'s.}
\label{fig:exp1}
\end{figure*}

\paragraph{Setting.}
To show the efficacy of our clustering algorithm, we consider a simple synthetic $\eta$-regular block MDP with $n = 100$, $S = 2$, and $A = 3$.
The latent transition matrix of each action is given as
\begin{equation*}
P_1 = \begin{bmatrix}
	\frac{\eta}{1 + \eta} & \frac{1}{1 + \eta} \\ \frac{1}{1 + \eta} & \frac{\eta}{1 + \eta}
\end{bmatrix},
\ \
P_2 = \begin{bmatrix}
	\frac{1}{1 + \eta} & \frac{\eta}{1 + \eta} \\ \frac{\eta}{1 + \eta} & \frac{1}{1 + \eta}
\end{bmatrix},
\ \
P_3 = \begin{bmatrix}
	\frac{1}{2} & \frac{1}{2} \\ \frac{1}{2} & \frac{1}{2}
\end{bmatrix},
\end{equation*}
where $\eta > 1$ is the parameter determining the hardness of our BMDP instance and is pre-determined.
Lastly, we use the uniform behavior policy to generate the trajectories.

This particular example shows that it is {\it necessary} to consider all actions via concatenation in the initial spectral clustering, as 1. playing the third action does not provide any useful information for clustering, and 2. considering the ``marginalized'' Markov chain, i.e., a single Markov chain with the transition matrix $\frac{1}{3}(P_1 + P_2 + P_3)$ also renders clustering impossible.

\paragraph{Results.}
For $T, H$, and $\eta$, we considered three settings where for each, we vary one parameter while fixing the other two.
All experiments are repeated 1000 times, and the codes are available in our GitHub repository\footnote{\url{https://github.com/nick-jhlee/optimal-block-mdp}}.

The results are shown in Figure~\ref{fig:exp1}, where we show all the outliers for the sake of transparency.
First note that with sufficiently many observations, we obtain exact clustering using initial clustering and likelihood improvement.
In (a) and (b), the BMDP instance is fixed, and we observe a decay in the clustering error rate as the number of observations increases.
In (c), the BMDP instance changes with the number of observations fixed, and as expected, we also observe a decay in the clustering error rate as $\eta$ increases, i.e., as the instance gets easier.

We believe that the outliers are there because the initial spectral clustering sometimes results in poor initialization for the likelihood improvement step. This is not contradictory to our guarantees because they hold in probability in the limit of $TH = \omega(n \log n)$ as $n \rightarrow \infty$. 
Also, one can note that increasing $T$ is more effective than increasing $H$, even though theoretically, the guarantees depend only on $TH$.
We believe that this is due to a finite-size effect.
Lastly, observe that in all experiments, there seems to be a sharp phase transition after which the algorithm consistently results in exact clustering.
This is consistent with our lower bound (Theorem \ref{thm:lower-bound}), which suggests an asymptotic phase transition at $TH = \omega(n)$ and $TH = \omega(n \log n)$.

In Appendix \ref{app:experiments}, we provide the precise experimental details as well as additional results suggesting that our algorithm is somewhat robust to random corruption.

% In Appendix \ref{app:experiments}, we provide additional results showing that rather surprisingly, our algorithm is somewhat robust to random corruption.

% !TEX root = ./main.tex
\section{Conclusion and Future Work}\label{sec:conclusion}

In this paper, we address learning problems in episodic BMDPs. We provide, for the first time, information-theoretical lower bounds on the  latent state decoding error rate, as well as on the sample complexity for near-optimal policy identification in the reward-free setting (valid even for algorithms with adaptive exploration). We also devise simple algorithms that approach these fundamental limits. Importantly, by exploiting the block structure, we demonstrate that we can significantly accelerate the search for near-optimal policies (in most cases, by a factor $n$, the size of the context space). This is also verified empirically in a synthetic BMDP environment.

In this paper, we mostly restricted the analysis to the case where the data is generated using a fixed behavior policy. There are cases where active exploration is more desirable~\citep{jin2020reward,kaufmann2021adaptive,zhang2020nearly,Tarbouriech19,Tarbouriech20a}.
In Appendix \ref{app:adaptive}, we already provide some discussions on how to extend our theoretical results to such active behavior policies. We plan to make these results and their proofs precise as well as conduct numerical experiments to assess the practical relevance of active exploration.

Another interesting research direction would be to apply to BMDPs the concept of $\delta$-significant latent state introduced in \cite{jin2020reward} (later extended in \cite{zhang2020nearly}) for tabular MDPs, which allows the analysis of RL tasks when some states are hard to reach.
Using this concept, we might be able to further relax the $\eta$-regularity assumption (we have already proposed an interesting relaxation in Appendix \ref{app:beyond-eta}).
Then we would have to recover the context to latent state mapping for significant latent states only.

\subsubsection*{Acknowledgments}
We would like to thank Laura Schmid for helping with polishing the figures, and also the anonymous reviewers for their helpful comments and suggestions.
Junghyun Lee and Se-Young Yun are supported by the Institute of Information \& communications Technology Planning \& Evaluation (IITP) grants funded by the Korean government(MSIT) (No.2022-0-00311, Development of Goal-Oriented Reinforcement Learning Techniques for Contact-Rich Robotic Manipulation of Everyday Objects; No.2019-0-00075, Artificial Intelligence Graduate School Program(KAIST)).
Yassir Jedra and Alexandre Prouti\`{e}re are supported by the Wallenberg AI, Autonomous Systems and Software Program (WASP) funded by the Knut and Alice Wallenberg Foundation.

\bibliographystyle{plainnat}
\bibliography{references}

\newpage

\appendix
\onecolumn

\newpage
\tableofcontents
\newpage
\section{Further Related Work}
\label{app:related-works}

\paragraph{Estimation of (controlled) Markov chains}
The plug-in estimator used for the offline reward-free RL has been studied extensively in the context of estimating the transition matrix of some underlying finite, ergodic Markov chain, which is of both theoretical and practical importance.
For uncontrolled Markov chains, \cite{billingsley61} was the first to consider this as a non-parametric estimation problem.
Recently, \cite{wolfer2019markov,wolfer2021markov} established the PAC-type minimax sample complexity bound w.r.t. any metric.
There has also been some recent interest in obtaining improved guarantees or improving the estimator itself when a certain structural assumption is imposed on the Markov chain.
As mentioned previously, \cite{SandersPY20} introduced and studied the problem of clustering and estimation of the transition kernel (and the emission probabilities) of BMCs; further spectral analyses have been done in \cite{Sanders21,Sanders22}.
It is also worthwhile to mention \cite{VanWerder22}, in which a thorough evaluation of the BMC model in a realistic setting has been done, showing that indeed the BMC model as well as the proposed clustering algorithm can provide meaningful insights into exploratory data analyses and more.

However, all the aforementioned works are for uncontrolled Markov chains, whereas most applications in RL and control theory deal with {\it controlled} Markov chains (e.g., action in RL literature).
Although the efficacy of the plug-in estimators in offline RL is known~\citep{duan2020minimax,yin2020off,ren2021nearly}, the precise characterization of the non-parametric estimation of {\it controlled} Markov chain in an offline setting has been an open question.
\cite{banerjee2022markov} first tackled this by showing that the usual plug-in estimator (per action) is minimax optimal in the PAC sense and provides results for various situations such as Markov chain with restarts and offline policy evaluation~\citep{sutton2018rl}.

\paragraph{Offline, reward-free RL}
The main focus of the first half is on the statistically tight characterization of clustering, as well as guarantees on how well the transition dynamics can be estimated.
In RL, this is often referred to as {\it model estimation}.
As the learner cannot interact further with the environment other than the given trajectories that have no reward information and are generated via some fixed, passive behavior policy $\rho$, our setting naturally extends to the intersection of (model-based) offline RL and reward-free RL.
To the best of our knowledge, we are the first to consider offline reward-free RL on BMDPs.

{\it Offline RL} is a framework in which the learner employs a static dataset $\cD$ collected by some behavior policy $\rho$ and performs training {\it without} any further interaction with the environment~\citep{levine2020offline}. Such a framework is especially useful when the data collection is expensive or dangerous in cases such as healthcare.
Several recent works~\citep{yin2021optimal, ren2021nearly, zhang2021difficult, uehara2022, xiao2022curse} studied the statistical efficiency of offline RL.

{\it Reward-free RL} (or reward-free exploration), pioneered by \cite{jin2020reward}, is the framework in which the goal is to estimate the optimal policy under {\it any} reward function that is revealed to the learner {\it after} a single exploration phase. Thus the main focus is to explore well such that the collected data has sufficient coverage.
Recently, this has been extensively studied theoretically for tabular MDPs, see e.g., \cite{jin2020reward,kaufmann2021adaptive,Menard21,zhang2020nearly,yin2021optimal}.

For a fair comparison, we consider a tabular episodic MDP with $n$ states such that every context-action pair can be reached with probability $\Theta(1 /nA)$.
we now recall the state-of-the-art sample complexity results for {\it offline reward-free} settings, in which the (later revealed) rewards can be {\it non-stationary}.
If the algorithm needs to return an $\epsilon$-optimal policy for any reward functions (reward-free setting), \cite{yin2021optimal} showed the bound $\widetilde{\cO}\left( \frac{H^2 n^2 A}{\epsilon^2} \right)$, which matches the known lower bound up to logarithmic factors~\citep{jin2020reward}.
If the algorithm needs to return an $\epsilon$-optimal policy for a single reward function (reward-specific setting), \cite{ren2021nearly} showed\footnote{Technically, they considered the so-called {\it offline policy optimization} scenario in which instead of the reward function being revealed to the learner, the trajectories outputted by some (potentially unknown) behavior policy contain the rewards. Thus, the learner has to also estimate the underlying reward function, while in our setting, the true reward function is revealed, i.e., estimation of the reward function is not necessary. Despite this difference, one can immediately conclude from a quick inspection of their proofs that their upper and lower bounds still hold for our reward-specific setting.}
the bound $\widetilde{\cO}\left( \frac{nA}{\epsilon^2} + \frac{H^2 n^2 A}{\epsilon} \right)$, as well as nearly matching minimax lower bound $\Omega\left( \frac{n A}{\epsilon^2} \right)$. The reason for the extra $H$ terms compared to \cite{ren2021nearly} is because they consider the total bounded reward assumption ($\sum_h r_h \in [0, 1]$), and thus to compare to our setting in which we assume bounded reward ($r_h \in [0, 1]$), we need to replace $n$ with $nH$, $\epsilon$ with $\epsilon H$, and multiply by $H$ in their guarantees.
By exploiting the block structure and under appropriate conditions (e.g., when $\epsilon = o(1)$), we significantly improve the sample complexities in all cases by a factor of $n$, and especially in reward-specific setting, exponentially in $\epsilon$ for a certain regime of $\epsilon$.
\newpage
% !TEX root = ./main.tex

\section{Notations}\label{app:notation}

\begin{table}[htbp]
\begin{center}
\small
\begin{tabular}{c c p{10cm} }
\toprule
\multicolumn{3}{l}{\bf Generic notations}\\
\hline
$\bm{1}$ & & Column vector with all entries equal to 1 \\
$\Vert \cdot \Vert_{1\; \,}$ & & The $\ell_1$ norm on vectors \\
$\Vert \cdot \Vert_{2\,}$  & & The $\ell_2$ norm on vectors \\
$\Vert \cdot \Vert_\infty$ & & The infinity norm on vectors \\
$\Vert \cdot \Vert_{\phantom{\infty}} $ & & Operator norm on matrices \\
$\Vert \cdot \Vert_{F\,}$ & & Frobeinus norm on matrices \\
$ x \vee  y $ & & To mean $ \max(x, y)$ \\
$ x \wedge y$ & & To mean $ \min(x, y)$ \\
$\mathbb{S}^{d-1}$ & & Unit sphere in $\mathbb{R}^d$ \\
$[K]$ & & For a given integer$K$, denotes the set $\lbrace 1, \dots, K \rbrace$ \\
$\cP(\cZ)$ & & The set of probability distributions over $\cZ$ \\
$\eta$-regular & & A discrete probability $\nu$  in $\cP(\cZ$) is said $\eta$-regular if $\max_{z_1, z_2 \in \cZ} \frac{\nu(z_1)}{\nu(z_2)} \le \eta$ \\
$\cP(\cZ, \eta)$ & & The set of probability distributions over $\cZ$ that are $\eta$-regular \\
$d_{TV}(\cdot, \cdot)$ & & The total variation distance between probability measures\\
$\KL(\cdot || \cdot)$ & & The Kullback-Liebler divergence between probability measures \\
$f(n) \sim g(n)$  &  & To mean $\lim_{n \rightarrow \infty} \frac{f(n)}{g(n)} = 1$ \\
$f(n) \asymp g(n)$ & & To mean there exists $c, C  > 0$  such that for all $n \geq 1$, $c g(n) \leq f(n) \leq Cg(n)$. \\
$f(n) \gtrsim g(n)$ & & To mean there exists $c  > 0$  such that for all $n \geq 1$, $ f(n) \ge cg(n)$. \\
$f(n) \lesssim g(n)$ & & To mean there exists $C  > 0$  such that for all $n \geq 1$, $ f(n) \le Cg(n)$. \\
% $f(n) = \Omega(g(n))$ & & To mean that \\
% $f(n) = \omega(g(n))$ & & To mean that $\liminf_{n \rightarrow \infty} \frac{g(n)}{f(n)} = 0$ \\
\bottomrule \\

\multicolumn{3}{c}{}\\
\hline
\multicolumn{3}{l}{\bf Block MDPs}\\
\hline
$n$ &   & Number of contexts\\
$S$ & & Number of latent states\\
$A$ & & Number of actions\\
$f(x)$ & & Latent state of context $x$\\
$p(s'|s,a)$ & & Transition probability from latent state $s$ to $s'$ with action $a$\\
$q(x|s)$ & & Emission probability for context $x$ given the latent state $s$ \\
$P(y|x,a)$ & & Transition probability from context $x$ to $y$ with action $a$ \\
$\rho(a | x)$ & & Behavior policy (probability of choosing action $a$ at context $x$)\\
\bottomrule

\end{tabular}
\end{center}
\end{table}

\newpage
% !TEX root = ./main.tex

\section{Markov Chains induced in BMDPs and their Mixing Times}\label{app:equilibrium}

In this appendix, we analyze some of the Markov chains induced by the dynamics in a BMDP.
Specifically, we study the Markov chain capturing the dynamics of the context, the Markov chain representing the evolution of the (action, next context) pair, and finally the Markov chain whose state is the triple (context, action, next context).
To do that, we first define the {\it contextual transition kernel} $P(y | x, a) = p(f(y) | f(x), a) q(y | f(y))$.
We denote by $MC_0
$, $MC_1$, $MC_2$ these chains, respectively. Under the (memoryless) behavior policy $\rho \sim \cU(\cA)$, the transition kernels and the initial state distributions of these chains are: for all $(x,x',y,y')\in {\cal X}^4$, and all $(a,b)\in {\cal A}^2$,
$$
\begin{cases}
P_0(y|x) = \sum_a \rho(a|x)P(y|x,a), & \mu_0(x)=\mu(x),\\
P_1((b,y)|(a,x)) = \rho(b|x)P(y|x,b), & \mu_1(a,x)=\sum_{y \in \cX} \mu(y) \rho(a | y) P(x | y, a),\\
P_2((y,b,y') |(x,a,x')) = \indicator_{y=x'} \rho(b|y)P(y'|y,b), & \mu_2(x,a,x')=\mu(x) \rho(a|x)P(x'|x,a).
\end{cases}
$$

Inspired by our Assumption \ref{assumption:regularity} (ii), we introduce the notion of $\eta$-regular Markov chains:
\begin{definition}[$\eta$-regular Markov chain] A time-homogenous Markov chain $(Z_t)_{t\ge 1}$ with finite state space ${\cal Z}$ and transition kernel $P(z'|z)=\mathbb{P}[Z_{h+1}=z' |Z_h=z]$ is $\eta$-regular for some $\eta\ge 1$ if and only if
$$
\max_{(x,y,z)\in {\cal Z}^3} \max\left\{ {P(y|x)\over P(z|x)}, {P(x|y)\over P(x|z)} \right\}\le \eta.
$$
\end{definition}

We remark that Assumption \ref{assumption:regularity} implies that: for all $(s,s',a)$ and all $x\in f^{-1}(s)$,
\begin{equation*}
\frac{1}{\eta S}\le \alpha_s \le \frac{\eta}{S},\qquad \frac{1}{\eta S} \le p(s' \vert s,a) \le \frac{\eta}{S}, \qquad \frac{1}{\eta \alpha_s n} \leq q(x | s) \leq \frac{\eta}{\alpha_s n}.
\end{equation*}
This observation will be instrumental in all the proofs presented in the remainder of the Appendix.

\subsection{Regularity and stationary distributions}
We start with the following:
\begin{proposition}
	For all $(x, y, a) \in \cX^2 \times \cA$,
	\begin{equation}
		\frac{1}{\eta^3 n} \leq P(y | x, a) \leq \frac{\eta^3}{n}
	\end{equation}
\end{proposition}

The following two propositions provide basic properties of the Markov chains $MC_0$ and $MC_1$.

\begin{proposition}\label{prop:mc0}
Under Assumption \ref{assumption:regularity}, the Markov chain $MC_0$ is $\eta^3$-regular, and irreducible aperiodic. Let $\Pi_0$ denote its stationary distribution. We have: for all $(x, y) \in \cX^2$,
\begin{equation}
\frac{1}{\eta^3 n} \leq P_0(y | x) \leq \frac{\eta^3}{n}, \quad\frac{1}{\eta^3 n} \leq \Pi_0(x) \leq \frac{\eta^3}{n}.
\end{equation}
\end{proposition}
%\begin{proof}[(can be deleted if there's no error)]
%	\begin{equation*}
%		{P_0(x | y) \over P_0(x | z)} = {\sum_a \rho(a | y) P(x | y, a) \over \sum_a \rho(a | z) P(x | z, a)}
%		\leq {\eta^3 \sum_a \rho(a | z) P(x | z, a) \over \sum_a \rho(a | z) P(x | z, a)} = \eta^3
%	\end{equation*}
%\end{proof}

\begin{proposition}\label{prop:mc1}
Under Assumptions \ref{assumption:regularity} and \ref{assumption:uniform}, the Markov chain $MC_1$ is $\eta^3$-regular, and irreducible aperiodic. Let $\Pi_1$ denote its stationary distribution. We have: for all $(x, y) \in \cX^2$, and all $(a,b)\in {\cal A}^2$,
\begin{equation}
\frac{1}{\eta^3 nA} \leq \mu_1( a,x) \leq \frac{\eta^3}{nA}, \quad\frac{1}{\eta^3 nA} \leq P_1((b,y) | (a,x)) \leq \frac{\eta^3}{nA}, \quad\frac{1}{\eta^3 nA} \leq \Pi_1(a,x) \leq \frac{\eta^3}{nA}.
\end{equation}
In addition, for all $(a,x)\in {\cal A}\times {\cal X}$, $\Pi_1(a,x) =\sum_{y \in \cX} \Pi_0(y) \rho(a | y) P(x | y, a)$.
\end{proposition}

One may easily check that the Markov chain $MC_2$ is not $\eta$-regular for any $\eta$, but our analyses will actually leverage the properties of Markov chains with kernel $P_2^2$. More precisely, we introduce the following two Markov chains: $MC_{2, odd}$ and $MC_{2, even}$. As the name suggests, they correspond to $Z_h=(X_{2h-1}, A_{2h-1}, X_{2h})$ and $Z_h=(X_{2h}, A_{2h}, X_{2h+1})$ (for $h\ge 1$), respectively. These chains share the same transition kernel, defined as follows:
\begin{align}
	P_{2}^2\left( (y, b, y') | (x, a, x') \right) &:= \left( {P}_2 \right)^2((y, b, y') | (x, a, x')) \nonumber \\
	&= \sum_{\tilde{x}, \tilde{a}, \tilde{y}} {P}_{2}\left( (y, b, y') | (\tilde{x}, \tilde{a}, \tilde{y} )\right) {P}_{2}\left( ( \tilde{x}, \tilde{a}, \tilde{y}) | (x, a, x') \right) \nonumber \\
	&= \sum_{\tilde{a} \in \cA} \rho(b | y) P(y' | y, b) \rho(\tilde{a} | x') P(y | x', \tilde{a}).
%	&= \sum_{\tilde{a} \in \cA} P_{1}((y', b) | (y, \tilde{a})) P_{1}((y, \tilde{a}) | (x', a)).
\end{align}
The two chains have different initial distributions:
\begin{align}
	\mu_{2,odd}(x, a, x') & = \mu(x)\rho(a|x)P(x' | x, a),\\
	\mu_{2,even}(x, a, x') &= \sum_{y,b} \mu(y)\rho(b|y)P(x|y,b)\rho(a|x)P(x'|x,a).
\end{align}

\begin{proposition}\label{prop:mc2}
Under Assumptions \ref{assumption:regularity} and \ref{assumption:uniform}, the Markov chains $MC_{2,odd}$ and $MC_{2,even}$ are $\eta^3$-regular, and irreducible aperiodic. The two Markov chains share the same stationary distribution, denoted by $\Pi_2$. We have: for all $(x,x',y,y') \in \cX^2$, and all $(a,b)\in {\cal A}^2$,
\begin{equation}
{1\over \eta^3 n^2A}\le \mu_{2,odd}(x, a, x') \le {\eta^3 \over n^2 A}, \quad {1\over \eta^6 n^2 A}\le \mu_{2,even}(x, a, x') \le {\eta^6 \over n^2 A},
\end{equation}
\begin{equation}
{1\over \eta^6 n^2 A}\le P_{2}^2\left( (y, b, y') | (x, a, x') \right) \le {\eta^6 \over n^2 A}, \quad {1 \over \eta^6 n^2 A }\le \Pi_2(x, a, x') \le {\eta^6 \over n^2 A}.
\end{equation}
In addition, for all $(x,x',a)\in {\cal X}^2\times {\cal A}$, $\Pi_2(x,a,x') =\Pi_0(x) \rho(a | x) P(x' | x, a)$.
\end{proposition}
%\begin{proof}[(can be deleted if there's no error)]
%	\begin{equation*}
%		{P_2^2(y, b, y'  | x, a, x') \over P_2^2(z, c, z' | x, a, x')} = {\sum_{\tilde{a}} P_1(y', b | y, \tilde{a}) P_1(y, \tilde{a} | x', a) \over \sum_{\tilde{a}} P_1(z', c | y, \tilde{a}) P_1(y, \tilde{a} | x', a)}
%		\leq \eta^3 {\sum_{\tilde{a}} P_1(z', c | y, \tilde{a}) P_1(y, \tilde{a} | x', a) \over \sum_{\tilde{a}} P_1(z', c | y, \tilde{a}) P_1(y, \tilde{a} | x', a)} = \eta^3
%	\end{equation*}
%\end{proof}

The proofs of the above propositions are straightforward. We just justify the expressions of the stationary distributions $\Pi_1$ and $\Pi_2$.

\begin{proof}({\it Stationary distribution of $\Pi_1$}) It suffices to show that $\sum_{y \in \cX} \Pi_0(y) \rho(b | y) P(z | y, b)$ satisfies the balance equation for the $MC_1$:
	\begin{align*}
		\sum_{(z, b) \in \cX \times \cA} \left[ \sum_{y \in \cX} \Pi_0(y) \rho(b | y) P(z | y, b) \right] P_1(x, a | z, b) &= \sum_{(z, b) \in \cX \times \cA} \left[ \sum_{y \in \cX} \Pi_0(y) \rho(b | y) P(z | y, b) \right] \rho(a | z) P(x | z, a) \\
		&= \sum_{z \in \cX} \rho(a | z) P(x | z, a) \sum_{y \in \cX} \Pi_0(y) \left[ \sum_{b \in \cA} \rho(b | y) P(z | y, b) \right] \\
		&= \sum_{z \in \cX} \rho(a | z) P(x | z, a) \sum_{y \in \cX} \Pi_0(y) P_0(z | y) \\
		&= \sum_{z \in \cX} \rho(a | z) P(x | z, a) \Pi_0(z).
	\end{align*}
\end{proof}

\begin{proof}({\it Stationary distribution of $\Pi_2$})
Again we show that $\Pi_0(x) \rho(a | x) P(x' | x, a)$ satisfies the balance equations of the  Markov chain $MC_2$:
\begin{align*}
	\sum_{(y,y',b) \in \cX^2 \times \cA} &\left[ \Pi_0(y)\rho(b|y) P(y' | y, b) \right] P_2^2((x, a, x') | (y,b, y')) \\
&= \sum_{(y,y',b)} \left[ \Pi_0(y)\rho(b|y) P(y' | y, b) \right] \sum_{\tilde{a} \in \cA} \rho(a | x) P(x' | x, a) \rho(\tilde{a} | y') P(x | y', \tilde{a})\\
& = \rho(a | x) P(x' | x, a) \sum_{(y,y',b)} \Pi_0(y)\rho(b|y) P(y' | y, b) \underbrace{\sum_{\tilde{a} \in \cA}\rho(\tilde{a} | y') P(x | y', \tilde{a})}_{ = P_0(x|y')}\\
& = \rho(a | x) P(x' | x, a) \sum_{(y,y')} \Pi_0(y)P_0(y'|y)P_0(x|y')\\
& = \rho(a | x) P(x' | x, a)\Pi_0(x).
\end{align*}
\end{proof}

\subsection{Bounds on multi-hop transition probabilities}

We can establish that the bounds for $P$ presented in the above three propositions hold for $P^h$ for any $h \geq 1$. To this aim, we use the following generic lemma:

\begin{lemma}
\label{lem:power}
Let $\cZ$ be a finite state space. For a row-stochastic matrix $P$, the following holds for all $x, y \in \cZ$ and all $h \geq 1$:
	\begin{equation}
		\min_{x, y \in \cZ} P(x, y) \leq P^h(x, y) \leq \max_{x, y \in \cZ} P(x, y).
	\end{equation}
\end{lemma}
\begin{proof}
	One important observation is that $P^h$ is also row-stochastic, for any $h \geq 1$.
	We only prove the upper bound, as the lower bound follows in the exact same manner.

	$h = 1$ is trivial, and thus let $h \geq 2$.
	Then, for any $(x, y) \in \cZ \times \cZ$
	\begin{align*}
		P^h(x, y) = \sum_{z \in \cZ} P(x, z) P^{h-1}(z, y)
		\leq \left( \max_{x, z \in \cZ} P(x, z) \right) \sum_{z \in \cZ} P^{h-1}(z, y)
		=\max_{x, z \in \cZ} P(x, z).
	\end{align*}
\end{proof}

Combining the results of the above lemma and those of Propositions \ref{prop:mc1} and \ref{prop:mc2}, we simply deduce:

\begin{corollary}
\label{cor:P-MC}
	For all $(x, a, x'), (y, b, y') \in \cX \times \cA \times \cX$ and $h \geq 1$,
	\begin{equation}
	\frac{1}{\eta^3 n} \leq \left( P_0 \right)^h(y | x) \leq \frac{\eta^3}{n},
\end{equation}
	\begin{equation}
		\frac{1}{\eta^3 nA} \leq \left( P_1 \right)^h((b, y) | (a, x)) \leq \frac{\eta^3}{nA},
	\end{equation}
	\begin{equation}
		\frac{1}{\eta^6 n^2A} \leq \left( P_2^2 \right)^h((y, b, y') | (x, a, x')) \leq \frac{\eta^6}{n^2A}.
	\end{equation}
\end{corollary}

%\begin{proof}
%	Again, we prove each one by a line of (direct) computations.
%
%	\underline{\textit{Proof of (i)}}\\
%	We have that by definition and by separability assumption,
%	\begin{align*}
%		\Pi^\rho_{MC}(x) &= \sum_{y \in \cX} \Pi^\rho_{MC}(y) P^\rho_{MC}(x | y) \\
%		&= \sum_{y \in \cX} \Pi^\rho_{MC}(y) \sum_{b \in \cA} \rho(b | y) P(x | y, b) \\
%		&\leq \sum_{y \in \cX} \Pi^\rho_{MC}(y) \sum_{b \in \cA} \rho(b | y) \frac{\eta}{n \alpha(f(x)) |\cS|} \\
%		&= \frac{\eta}{n \alpha(f(x)) |\cS|} \\
%		&\leq \frac{\eta}{n a_{\min}}
%	\end{align*}
%
%	as well as
%	\begin{align*}
%		\Pi^\rho_{MC}(x) &= \sum_{y \in \cX} \Pi^\rho_{MC}(y) P^\rho_{MC}(x | y) \\
%		&= \sum_{y \in \cX} \Pi^\rho_{MC}(y) \sum_{b \in \cA} \rho(b | y) P(x | y, b) \\
%		&\geq \sum_{y \in \cX} \Pi^\rho_{MC}(y) \sum_{b \in \cA} \rho(b | y) \frac{1}{\eta n \alpha(f(x)) |\cS|} \\
%		&= \frac{1}{\eta n \alpha(f(x)) |\cS|} \\
%		&\geq \frac{1}{\eta  n a_{\max}}
%	\end{align*}
%\end{proof}

\subsection{Mixing times}

Consider an irreducible aperiodic Markov chain with initial distribution $\mu$, transition kernel $P$, and stationary distribution $\Pi$. Its mixing time is defined as $\inf\{ h\ge 1: d_{TV}(\mu P^h,\Pi) \le 1/4\}$. The following proposition provides upper bounds of the Markov chains $MC_0$, $MC_1$, $MC_{2,odd}$ and $MC_{2,even}$.

\begin{proposition}\label{prop:mixing}
Assume that $\rho(\cdot | x) \sim \cU(\cA)$.
Under Assumptions \ref{assumption:regularity} and \ref{assumption:uniform}, we have:
	\begin{itemize}
		\item[\textit{(i)}] The mixing times of $MC_0$ is upper bounded by {$2\eta^2$}.

		\item[\textit{(i)}] The mixing times of $MC_1$ is upper bounded by {$2\eta^2$}.

		\item[\textit{(ii)}] The mixing times of $MC_{2, odd}$ and $MC_{2, even}$ are both upper bounded by {$\eta^2 + 1$}.
	\end{itemize}
\end{proposition}

%Note that despite the various degree of regularity of $MC_0, MC_1, MC_{2,odd}, MC_{2,odd}$ (e.g. $MC_0$ is $\eta^3$-regular and $MC_{2,odd}$ is $\eta^5$-regular), their mixing times are almost identical.
As described below, to obtain our tight upper bounds for the mixing times, instead of simply using the loose (regularity) bounds for the transition kernels, we use more sophisticated arguments.

The proof of the above results relies on Dobrushin's ergodic coefficient:
\begin{definition}[\cite{Dobrushin56a, Dobrushin56b}] \label{def:dob-coef}
	For any row-stochastic matrix $P$, define the {Dobrushin's ergodic coefficient} $\delta(P)$ as follows:
	\begin{equation}
		\delta(P) := \frac{1}{2} \max_{x, y \in \mathcal{X}} \sum_{z \in \mathcal{X}} \left| P(z | x) - P(z | y) \right|.
	\end{equation}
\end{definition}
The Dobrushin's ergodic coefficient can be equivalently rewritten (Exercise 4.4.12 of \cite{Bremaud20}) as follows:
\begin{equation}
	\delta(P) = 1 - \min_{x, y \in \mathcal{X}} \sum_{z \in \mathcal{X}} \left( P(z | x) \wedge P(z | y) \right).
\end{equation}

Now, $\delta(P)$ is directly related to the convergence rate of the stationary distribution:
\begin{theorem}[Theorem 4.3.15 of \cite{Bremaud20}]
	For any $h$,
	\begin{equation}
		d_{TV}\left( \mu P^h, \Pi \right) \leq \left( \delta(P) \right)^h d_{TV}\left( \mu, \Pi \right)
	\end{equation}
	and
	\begin{equation}
		d_{TV}\left( P^{h+1}(z, \cdot), \Pi \right) \leq \left( \delta(P) \right)^h d_{TV}\left( P(z, \cdot), \Pi \right)
	\end{equation}
\end{theorem}
%
%\begin{remark}
%	Above implies that $\delta(P)$ is an upper bound of the second-largest eigenvalue modulus (SLEM).
%\end{remark}

\paragraph{Proof of Proposition \ref{prop:mixing} (i)} We analyze the mixing time $t_{\mix, MC_0}(\varepsilon)=\inf\{ h\ge 1: d_{TV}(\mu P_0^h,\Pi_0) \le \varepsilon\}$, and apply the results to $\varepsilon = 1/4$. We prove that:
\begin{equation}
		t_{\mix, MC_0}(\varepsilon) \leq \eta^2 \log \frac{1}{\varepsilon}.
\end{equation}
\begin{proof}
For any $x, y, z \in \mathcal{X}$,
\begin{align*}
		P_0(z | x) \wedge P_0(z | y) &= \left(\sum_{a \in \mathcal{A}} \rho(a|x)P(z | x, a) \right) \wedge \left( \sum_{a \in \mathcal{A}} \rho(a|y) P(z | y, a) \right) \\
		&= \frac{1}{A} \left( \sum_{a \in \mathcal{A}} P(z | x, a) \right) \wedge \left( \sum_{a \in \mathcal{A}} P(z | y, a) \right) \\
		&\geq \frac{1}{A} \left( \sum_{a \in \mathcal{A}} \frac{1}{\eta \alpha_{f(z)} n} p(f(z) | f(x), a) \right) \wedge \left( \sum_{a \in \mathcal{A}} \frac{1}{\eta \alpha_{f(z)} n} p(f(z) | f(y), a) \right) \\
		&\geq \frac{1}{\eta \alpha_{f(z)} n A} \left( \sum_{a \in \mathcal{A}} \frac{1}{\eta} p(f(z) | f(y), a) \right) \wedge \left( \sum_{a \in \mathcal{A}} p(f(z) | f(y), a) \right) \\
		&\geq \frac{1}{\eta^2 \alpha_{f(z)} n A} \sum_{a \in \mathcal{A}} p(f(z) | f(y), a)
	\end{align*}
Then the Dobrushin's coefficient can be bounded as follows:
	\begin{align*}
		\delta(P_0) &= 1 - \min_{x, y \in \mathcal{X}} \sum_{z \in \mathcal{X}} \left( P_0(z | x) \wedge P_0(z | y) \right) \\
		&\leq 1 - \min_{x, y \in \mathcal{X}} \sum_{z \in \mathcal{X}} \frac{1}{\eta^2 \alpha_{f(z)} n A} \sum_{a \in \mathcal{A}} p(f(z) | f(y), a) \\
		&= 1 - \frac{1}{\eta^2 A} \min_{y \in \mathcal{X}} \sum_{a \in \mathcal{A}} \sum_{s \in \mathcal{S}} \sum_{z \in f^{-1}(s)} \frac{1}{n \alpha_s} p(s | f(y), a) \\
		&= 1 - \frac{1}{\eta^2}.
	\end{align*}
%	Observe that $0 \leq \zeta^\rho \leq 1$, always.

	Thus, we have:
	\begin{align*}
		d_{TV}\left( P_0^h(z, \cdot), \Pi \right) \leq \left( \delta(P_0) \right)^h d_{TV}\left( \mu_0, \Pi_0 \right)
		\leq \left( 1 - \frac{1}{\eta^2} \right)^h d_{TV}\left( \mu_0, \Pi_0 \right) \leq e^{-\frac{h}{\eta^2}} d_{TV}\left( \mu_0, \Pi_0 \right).
	\end{align*}

Note that $d_{TV}\left( \mu_0, \Pi_0 \right) = \frac{1}{2} \sum_{z \in \cX} \left| \mu(z) - \Pi(z) \right| \leq \frac{1}{2} \sum_{z \in \cX} \left( \mu_0(z) + \Pi_0(z) \right)= 1$. In summary, we have $d_{TV}\left( P_0^h(z, \cdot), \Pi_0 \right) \leq \varepsilon$ whenever $h \geq \eta^2 \log \frac{1}{\varepsilon}$. This completes the proof.
\end{proof}

%At the end, letting $\varepsilon = \frac{1}{4}$ results in $t_{mix}^U \leq \eta \log 4 \leq 2\eta$.
%
%\begin{remark}
%\label{rmk:eta-mixing}
%	One can easily see that in general, any transition matrix that is $\eta$-regular has its Dobrushin's coefficient bounded by $1 - \frac{1}{\eta}$.
%\end{remark}

%\begin{remark}
%	Following the same line of reasoning, Proposition \ref{prop:mixing-MC} can be generalized to arbitrary policy $\rho$ as follows:
%	\begin{equation}
%		t_{mix}^\rho(\varepsilon) \leq \frac{\eta}{\zeta^\rho} \log \frac{1}{\varepsilon},
%	\end{equation}
%	where $\zeta^\rho \triangleq \min_{x, y \in \mathcal{X}} \sum_{a \in \mathcal{A}} \left( \rho(a | x) \wedge \rho(a | y) \right)$.
%
%	If $\zeta^\rho = 0$, then above result is void in the sense that the inequality becomes $t_{mix}(\varepsilon) \leq \infty$.
%	For $\zeta^\rho = 0$, there must exist some $x, y \in \cX$ such that $\sum_{a \in \cA} \left( \rho(a | x) \wedge \rho(a | y) \right) = 0$ i.e. $\forall a \in \cA$, $\rho(a | x) = 0$ or $\rho(a | y) = 0$.
%	In other words, for the mixing time to be bounded (i.e. be constant $O(1)$), it is sufficient to ensure that $\rho(a | x) > 0$ for all $(x, a) \in \cX \times \cA$.
%\end{remark}

\paragraph{Proof of Proposition \ref{prop:mixing} (ii)} The proof follows from exactly the same arguments as those used in the proof of (i).

\paragraph{Proof of Proposition \ref{prop:mixing} (iii)}

We start with the following lemma, which relates the power of ${P}_2$ to the power of $P_1$::
\begin{lemma}
\label{lem:P-power}
	${P}_2^{h+1}(y', a', z | x, a, y) = \sum_{\tilde{a} \in \cA} P_1((a',z) | (\tilde{a},y')) P_1^h((\tilde{a},y') | (a,y))$.
\end{lemma}
\begin{proof}
	We proceed by induction.
	\begin{itemize}
		\item For $h=1$:
		\begin{align*}
			{P}_2^2((y', a', z) | (x, a, y)) &= \sum_{(\tilde{x}, \tilde{a}, \tilde{y}) \in \cX \times \cA \times \cX} {P}_2((y', a', z) | (\tilde{x}, \tilde{a}, \tilde{y})) {P}_2((\tilde{x}, \tilde{a}, \tilde{y}) | (x, a, y)) \\
			&= \sum_{\tilde{a} \in \cA} P_1((a',z) | (\tilde{a},y')) P_1((\tilde{a},y') | (a,y)).
		\end{align*}
		
		\item For $h \geq 2$:
		\begin{align*}
			& {P}_2^{h+1}((y', a', z) | (x, a, y)) = \sum_{(\tilde{x}, \tilde{a}, \tilde{y}) \in \cX \times \cA \times \cX} {P}_2((y', a', z) | (\tilde{x}, \tilde{a}, \tilde{y})) {P}_2^h((\tilde{x}, \tilde{a}, \tilde{y}) | (x, a, y)) \\
			&\quad = \sum_{(\tilde{x}, \tilde{a}) \in \cX \times \cA} P_1((a',z) | (\tilde{a},y')) \sum_{\tilde{a}' \in \cA} P_1((\tilde{a},y') | (\tilde{a}',\tilde{x})) P_1^{h-1}((\tilde{a}',\tilde{x}) | (a,y)) \\
			&\quad = \sum_{\tilde{a} \in \cA} P_1((a',z) | (\tilde{a},y')) \sum_{(\tilde{x}, \tilde{a}') \in \cX \times \cA} P_1((\tilde{a},y') | (\tilde{a}',\tilde{x})) P_1^{h-1}((\tilde{a}'\tilde{x}) | (a,y)) \\
			&\quad= \sum_{\tilde{a} \in \cA} P_1((a',z) | (\tilde{a},y')) P_1^h((\tilde{a},y') | (a,y)).
		\end{align*}
	\end{itemize}
\end{proof}
\begin{lemma}\label{lem:dtv}
	For all $h \geq 1$ and all $(x, a, y) \in \cX \times \cA \times \cX$,
	\begin{equation}
		d_{TV}\left( \left({P}_2\right)^{h+1}(\cdot | (x, a, y)), \Pi_2 \right) \leq d_{TV}\left( \left( P_1\right)^h(\cdot | (a,y)), \Pi_1 \right).
	\end{equation}
\end{lemma}
\begin{proof} 
Noting that $\Pi_2(y', a', z) = \sum_{\tilde{a} \in \cA} P_1((a',z) | (\tilde{a},y')) \Pi_1(\tilde{a},y')$, we have:
	\begin{align*}
		& 2 d_{TV}\left( {P}_2^{h+1} (\cdot | x, a, y), \Pi_2 \right) \\
		&= \sum_{(y', a', z) \in \cX \times \cA \times \cX} \left| {P}_2^{h+1}(y', a', z | x, a, y) - \Pi_2(y', a', z) \right| \\
		&\overset{(*)}{=} \sum_{(y', a', z) \in \cX \times \cA \times \cX} \left| \sum_{\tilde{a} \in \cA} P_1((a',z) | (\tilde{a},y')) P_1^h((\tilde{a},y') | (a,y)) - \sum_{\tilde{a} \in \cA} P_1((a',z) | (\tilde{a},y')) \Pi_1((]\tilde{a},y') \right| \\
		&\leq \sum_{(y', \tilde{a}) \in \cX \times \cA} \left| P_1^h((\tilde{a},y') | (a,y)) - \Pi_1(\tilde{a},y') \right| \sum_{(z, a') \in \cX \times \cA} P_1((a',z) | (\tilde{a},y')) \\
		&= 2 d_{TV} \left( P_1^h(\cdot |(a,y)), \Pi_1 \right),
	\end{align*}
	where $(*)$ follows from Lemma \ref{lem:P-power}.
\end{proof}

We now complete the proof of Proposition \ref{prop:mixing} (iii): from Lemma \ref{lem:dtv}, we have that for all $h \geq 1$ and all $(x, a, y) \in \cX \times \cA \times \cX$,
	\begin{equation*}
		d_{TV}\left( \left(P_2^2\right)^h(\cdot | (x, a, y)), \Pi_2 \right) \leq d_{TV}\left( \left( P_1\right)^{2h - 1}(\cdot | (a, y)), \Pi_1 \right).
	\end{equation*}
In particular, this implies that $t_{\mix, MC_2} \leq \frac{t_{\mix, MC_1} + 1}{2} \leq \frac{t_{\mix, MC_1}}{2} + 1 \leq \eta^2 + 1$.

\newpage
% !TEX root = ./main.tex

\section{Proof of Theorem \ref{thm:lower-bound} -- Fundamental Lower Bound of Latent State Decoding}\label{app:lower-bound}
The proof is based upon an appropriate {\it change-of-measure} argument~\citep{Lai85}. The clustering error rate lower bound in SBMs~\citep{YunP14, YunP16} and Block Markov Chains~\citep{SandersPY20} also leveraged a change-of-measure argument, but different than ours. More precisely, our confusing model is constructed by first fixing a specific context and moving it to some other cluster. Of course, the $q$'s (emission probabilities) are changed appropriately. We note that since each cluster is inhomogeneous (the emission distributions are not uniform), we derive a clustering error rate lower bound for each context $x \in \cX$.

{\bf The confusing model.} Denote the $T$ observed trajectories as $\mathcal{T} = \{ \mathcal{T}^{(t)} \}_{t=1}^{T}$. Fix a context $x$ and denote by $i= f(x)$. Note that $\{ x\in {\cal E} \}= \{ \hat{f}(x) \not= i \}$ is the event that $x$ is mis-classified, where we recall that $\cE$ is the set of mis-classified contexts under $\rho$ and our chosen algorithm $\cA$. Let $\Phi$ be the true BMDP model, induced by $(p, q, f)$, from which $\mathcal{T}$ is ``actually" generated. We define the confusing BMDP model by moving $x$ from its original cluster $i$ to some other cluster $j \not= i$, which will be determined later on.
More precisely, let $\Psi^{(x,j)}$ be the confusing model, induced by $(p, \tilde{q}, \tilde{f})$, where $\tilde{f}(x) = j$ and $\tilde{f} \equiv f$ on all other contexts and $\tilde{q}$, the context emission distribution of $\Psi^{(x,j)}$, is defined as follows:
\begin{equation}
	\tilde{q}(x | j) = c q(x | i), \quad  \tilde{q}(y | j) = (1 - c q(x | i)) q(y | j), \quad y \in f^{-1}(j) \setminus \{x\},
\end{equation}
\begin{equation}
	\tilde{q}(z | i) = \frac{q(z | i)}{1 - q(x | i)}, \quad z \in f^{-1}(i) \setminus \{x\},
\end{equation}
\begin{equation}
	\tilde{q}(y | s) = q(y | s), \quad y \in f^{-1}(s), s \in \cS \setminus \{i, j\}.
\end{equation}

Here, $c \geq 0$ is to be chosen later. We now provide the possible values taken by $c$, so that $\Psi^{(x,j)}$ is a possible BMDP compatible with the $\beta$-locality assumption (for $\beta \gtrsim \frac{1}{n}$) at $\Phi$. First of all, for $\tilde{q}$ to be a well-defined probability distribution, we must have that $\tilde{q} \in [0, 1]$. From the regularity condition on $q$, we have that $0 \leq c \leq \frac{1}{\eta^2} \frac{n}{S}$. Now, from the $\beta$-locality, we must have that
\begin{equation*}
	\left| \frac{q(y | i)}{1 - q(x | i)} - q(y | i) \right| \leq \beta, \quad |(1 - c q(x | i)) q(z | j) - q(z | j)| \leq \beta.
\end{equation*}
The first inequality is trivially true for $\beta \geq 2\eta^2 \frac{S}{n}$, with sufficiently large $n$, or precisely speaking, with $n \geq 2\eta^2 S$.
The second inequality is true for $0 \leq c \leq \frac{2}{\eta^2} \frac{n}{S}$. Overall, the domain of $c$ is given as $0 \leq c \leq \frac{1}{\eta^2} \frac{n}{S}$. This completes the description of $\Psi^{(x,j)}$. To simplify the notation, we use $\Psi$ to represent $\Psi^{(x,j)}$ in the remaining of the proof.

{\bf Log-likelihood ratio and its connection to the error rate.}  The log-likelihood ratio of the observed trajectories under $\Phi$ and $\Psi$ is
\begin{equation}
	\mathcal{L} = \log \frac{\mathbb{P}_\Psi[\left\{ \mathcal{T}^{(t)} \right\}_{t=1}^T]}{\mathbb{P}_\Phi[\left\{ \mathcal{T}^{(t)} \right\}_{t=1}^T]} \\
	= \sum_{t=1}^{T} \underbrace{\log\frac{\mathbb{P}_\Psi[\mathcal{T}^{(t)} | \mathcal{T}^{(1)}, \cdots, \mathcal{T}^{(t-1)} ]}{\mathbb{P}_\Phi[\mathcal{T}^{(t)} | \mathcal{T}^{(1)}, \cdots, \mathcal{T}^{(t-1)}]}}_{\triangleq \mathcal{L}^{(t)}},
\end{equation} 
with by convention $\cT^{(-1)} = \emptyset$ and $\mathcal{T}^{(t)} = (x_1^{(t)}, a_1^{(t)}, \cdots, x_{H-1}^{(t)}, a_{H-1}^{(t)}, x_H^{(t)}, a_H^{(t)})$. In our case in which the policies are non-adaptive, the conditioning is meaningless i.e. $\PP[\cT^{(t)} | \cT^{(1)}, \cdots, \cT^{(t-1)}] = \PP[\cT^{(t)}]$. The conditioning will be important when considering adaptive behavior policies (refer to Appendix \ref{app:adaptive}). The following proposition relates $\cL$ to the classification error rate of $x$.

\begin{proposition}\label{prop:LL}
For $\beta$-locally better-than-random clustering algorithm that outputs $\hat{f}$ given some dataset (of trajectories), we have:\\
\begin{align*}
(i) &\quad \PP_\Psi[\hat{f}(x) \not= f(x)] \geq \frac{1}{S},\\
(ii) &\quad \varepsilon_x = \PP_\Phi[x\in {\cal E}]\geq \frac{1}{2S} \exp\left( -\EE_\Psi[\cL] - \sqrt{2S \Var_\Psi[\cL]} \right).
 \end{align*}
\end{proposition}

\begin{proof} The proof is analogous to that of Proposition 4 of \cite{SandersPY20}.

\textbf{\underline{\textit{Proof of (i)}}}\\
From the definition of $\beta$-locally better-than-random algorithm (Definition \ref{def:locally-better-than-random}),
\begin{equation*}
	\PP_\Psi[\hat{f}(x) \not= f(x)] \geq \PP_\Psi[\hat{f}(x) = \tilde{f}(x)] = 1 - \PP_\Psi[\hat{f}(x) \not= \tilde{f}(x)] \geq 1 - \left(1 - \frac{1}{S} \right) = \frac{1}{S},
\end{equation*}
where the first inequality follows from the observation that $\{ \hat{f}(x) = \tilde{f}(x) \} \subset \{ \hat{f}(x) \not= f(x) \}$.\\

\textbf{\underline{\textit{Proof of (ii)}}}\\
For simplicity, we denote the event $\xi_x = \{ \hat{f}(x) \not= f(x) \}$.
First, consider any function $R(n, T) : \NN_+^2 \rightarrow \RR$. We have:
	\begin{equation*}
		\PP_\Psi[\cL \leq R(n, T)] = \PP_\Psi[\cL \leq R(n, T), \xi_x^\complement] + \PP_\Psi[\cL \leq R(n, T), \xi_x].
	\end{equation*}
From (i), we deduce:	
\begin{equation*}
		\PP_\Psi[\cL \leq R(n, T), \xi_x^\complement] \leq \PP_\Psi[\xi_x^\complement] = 1 - \PP_\Psi[\xi_x] \leq 1 - \frac{1}{S}.
	\end{equation*}
From our change-of-measure,
	\begin{equation*}
		\PP_\Psi[\cL \leq R(n, T), \xi_x] \leq e^{R(n, t)} \PP_\Phi[\cL \leq R(n, T), \xi_x]
		\leq e^{R(n, T)} \PP_\Phi[\xi_x].
	\end{equation*}
Combining the above results, we obtain:
	\begin{equation*}
		\PP_\Psi[\cL \geq R(n, T)] \geq 1- \left(1 - \frac{1}{S} \right) - e^{R(n, T)} \PP_\Phi[\xi_x].
	\end{equation*}
	Specify $R(n, T) = \log\frac{1}{2S} + \log\frac{1}{\PP_\Phi[\xi_x]}$ and apply Chebyshev's inequality to get
	\begin{equation*}
		\PP_\Psi\left[ \cL \geq \EE_\Psi[\cL] + \sqrt{2S \Var_\Psi[\cL]} \right] \leq \frac{1}{2S} \leq \PP_\Psi\left[ \cL \geq \log\frac{1}{2S} - \log\PP_\Phi[\xi_x] \right],
	\end{equation*}
from which the result follows.
\end{proof}

{\bf The rate (or divergence) function and its connection to the log-likelihood ratio.}
We first introduce the following divergence or rate function:
\begin{align}
\label{eqn:I}
	I^{(t)}_j(x; c, \Phi) &:= n\sum_{a \in \cA} \sum_{s \in \cS} \left\{ c  q(x | f(x)) p(j | s, a) m_\rho^{\Psi,(t)}(s, a) \log\frac{c p(j | s, a)}{p(f(x) | s, a)}\right. \nonumber \\
	&\quad + c  q(x | f(x)) m_\rho^{\Psi,(t)}(j, a) p(s | j, a) \log\frac{p(s|j,a)}{p(s|f(x),a)} \nonumber \\
	&\quad + \left.  (1 - c q(x | f(x)) p(j | s, a)) m_\rho^{\Psi,(t)}(s, a) \log\frac{1 - c q(x | f(x)) p(j | s, a)}{1 - q(x | f(x)) p(f(x) | s, a)} \right\},
\end{align}
where $m_\rho^{\Psi,(t)}(s, a)$ denotes the expected proportion of rounds spent in (latent state, action) pair $(s,a)$ under policy $\rho$ and model $\Psi$, in the $t$-th episode:
\begin{equation}
	m_\rho^{\Psi,(t)}(s, a) := \frac{1}{H - 1} \sum_{h=1}^{H-1} \PP_\Psi[\tilde{f}(x_h^{(t)} )= s, a_h^{(t)} = a].
\end{equation}
Note that since the behavior policy is not changing from one episode to the other, $m_\rho^{\Psi,(t)}(s, a)$ and hence $I^{(t)}_j(x; c, \Phi)$ do not depend on $t$.
However, we keep separating the different episodes, so that the analysis will remain valid under adaptive behavior policies; see Appendix \ref{app:adaptive} for more detailed discussions.

Next, we define $I_j(x; c, \Phi):= \frac{1}{T} \sum_{t=1}^T I_j^{(t)}(x; c, \Phi)$, and finally, the rate function:
\begin{equation}
	I(x; \Phi) := \min_{j: j \not= f(x)} \inf_{c > 0} I_j(x; c, \Phi).
\end{equation}

Note that the choice of the "most" confusing cluster $j^\star$ as well as the "optimal'' $c^\star$ depends on $x$, i.e., such choices can (and will differ, generally) for each $x \in \cX$.

%\begin{remark}
%	More generally, let $\delta > 1$ be arbitrary.
%	Then any $c \in [0, \delta]$ and $\beta \geq \delta \eta^4 \frac{S^2}{n^2}$ satisfies above inequalities, but we set $\delta = 2$ for simplicity.
%\end{remark}
We now consider two cases; when $I(x; \Phi) = 0$ and when $I(x; \Phi) > 0$.
\newpage
\begin{proposition}
\label{prop:I=0-error}
	If $I(x; \Phi) = 0$, then $\PP_\Phi[\hat{f}(x) \not= f(x)] \geq \alpha_{\min}$.
\end{proposition}
\begin{proof}
	Proposition \ref{prop:I=0} (proved in Section \ref{subsec:ratefunction}) asserts that $I(x; \Phi)=0$ iff there exists some $j \not= f(x)$ and $c > 0$ such that
	\begin{enumerate}
		\item $p(f(x) | s, a) = c p(j | s, a), \quad \forall (s, a) \in \cS \times \cA$,
		\item $p(s | f(x), a) = p(s | j, a), \quad \forall (s, a) \in \cS \times \cA$.
	\end{enumerate}

	We now show that when $I(x; \Phi)=0$, we can construct a different BMDP model in which clusters $i = f(x)$ and $j$ are merged to a single cluster, yet the likelihoods for any given observation(trajectory) are the same for this alternate model or the true model.
	This then implies that the clustering error rate of the algorithm is at least $\alpha_{\min}$.
	
	Our original model is driven by the contextual transition kernel $P(y | x, a) = p(f(y) | f(x), a) q(y | f(y))$ for any $(x, y, a) \in \cX^2 \times \cA$.
	For our alternate model, let $k$ be the merged cluster index of $f(x)$ and $j$, and first denote $\tilde{\cS} = \cS \setminus \{f(x), j\}$.
	Then the alternate set of latent states is $\cS = \cS' \cup \{k\}$.
	We then define $f'$, $p'$, and $q'$ as follows:
	\begin{enumerate}
		\item $p'(s | k, a) = p(s | j, a)$ and $q'(\cdot | s) = q(\cdot | s)$ for any $(s, a) \in \tilde{\cS} \times \cA$,
		
		\item $p'(k | s, a) = (1 + c) p(j | s, a)$ for any $(s, a) \in  \times \cA$,
		
		\item $q'(y | k) = \frac{c}{1 + c} q(y | i)$ and $q'(z | k) = \frac{1}{1 + c} q(z | j)$ for any $y \in f^{-1}(i)$ and any $z \in f^{-1}(j)$,
		
		\item $p'(k | k, a) = (1 + c) p(j | i, a)$.
	\end{enumerate}
	Then it is straightforward to check that 1. $p'$ and $q'$'s are all well-defined probability distributions over their respective supports and more importantly, 2. the alternate contextual transition kernel $P'(y | x, a) = p'(f'(y) | f'(x), a) q'(y | f'(y))$ is {\it identical} to that of $P$.
	As $P' = P$, for any given observation(trajectory), the likelihoods under true and alternate models are exactly the same, and we are done.
\end{proof}

Now, assuming that $I(x; \Phi) > 0$ i.e. $I_j(x; c, \Phi) > 0$ for any $j \not= f(x)$ and $c > 0$, the next propositions assert that $\EE_\Psi[\cL]$ is precisely the leading term with $I_j(x; c, \Phi)$ to be defined later, and $\Var_\Psi[\cL]$ is negligible:
\begin{proposition}
\label{prop:E-likelihood}
$\EE_\Psi[\cL] \leq \frac{TH}{n} (I_j(x;c, \Phi)+{\cal O}({1\over n}))$.
\end{proposition}

\begin{proposition}
\label{prop:Var-likelihood}
$\Var_\Psi[\cL] \leq \cO\left( \frac{TH}{n} \right)$.
\end{proposition}

Combining all the above results will complete the proof of Theorem \ref{thm:lower-bound}.
%When $I(x; \Phi) = 0$, we have that $\PP_\Phi[\hat{f}(x) \not= f(x)] \geq \frac{1}{2}$. When $I(x; \Phi) > 0$, 
Denote by $(j^\star,c^\star)$ such that $I(x;\Phi) = I_{j^\star}(x;c^\star, \Phi)$.
Now Propositions \ref{prop:LL}, \ref{prop:I=0-error}, \ref{prop:E-likelihood} and \ref{prop:Var-likelihood} applied to $(j^\star,c^\star)$ imply:
\begin{equation*}
	\PP_\Phi[\hat{f}(x) \not= f(x)] \geq \frac{1}{2\eta S} \exp\left( - I(x; \Phi) \frac{TH}{n} (1 + o(1)) \right).
\end{equation*}

\subsection{Proof of Proposition \ref{prop:E-likelihood}}\label{sec:KL-I}

Here we compute $\EE_\Psi[\cL]$. We fix our attention to episode $t$ and focus on $\cL^{(t)}$.
For full generality, we consider a more general form of behavior policy that can change per episode i.e. $\rho = (\rho^{(t)})_{t \in [T]}$.
Recall that $x_h^{(t)}$ is the context observed in step $h$ of this episode. Note that the corresponding latent state $s_h^{(t)}$ may depend on the model $\Phi$ or $\Psi$ considered. Further recall that the transition kernels under the two models are:
\begin{align*}
P^\Phi(y | x, a) &= q(y | f(y)) p(f(y) | f(x), a),\\
P^\Psi(y | x, a) &= \tilde{q}(y | \tilde{f}(y)) p(\tilde{f}(y) | \tilde{f}(x), a).
\end{align*}
Then the likelihoods of the observed trajectory on the $t$-episode under both models are:
\begin{align*}
\mathbb{P}_\Phi[\mathcal{T}^{(t)}] &= \left[ \mu(x_1^{(t)}) \prod_{h = 1}^H \rho^{(t)}(a_h^{(t)} | x_h^{(t)}) \right] \left[ \prod_{h = 1}^{H-1} P^\Phi(x_{h+1}^{(t)} | x_h^{(t)}, a_h^{(t)}) \right],\\
\mathbb{P}_\Psi[\mathcal{T}^{(t)}] &= \left[ \mu(x_1^{(t)}) \prod_{h = 1}^H \rho^{(t)}(a_h^{(t)} | x_h^{(t)}) \right] \left[ \prod_{h = 1}^{H-1} P^\Psi(x_{h+1}^{(t)} | x_h^{(t)}, a_h^{(t)}) \right].
\end{align*}

%\begin{align*}
%	\mathbb{P}_\Phi[\mathcal{T}^{(t)}] &= \mu(x_1^{(t)}) \rho^{(t)}(a_1^{(t)} | x_1^{(t)}) \prod_{h=1}^{H-1} P^\Phi(x_{h+1}^{(t)} | x_h^{(t)}, a_h^{(t)}) \rho^{(t)}(a_{h+1}^{(t)} | x_{h+1}^{(t)}) \\
%	&= \left[ \mu(x_1^{(t)}) \prod_{h = 1}^H \rho^{(t)}(a_h^{(t)} | x_h^{(t)}) \right] \left[ \prod_{h = 1}^{H-1} P^\Phi(x_{h+1}^{(t)} | x_h^{(t)}, a_h^{(t)}) \right]
%\end{align*}
%
%and
%\begin{align*}
%	\mathbb{P}_\Psi[\mathcal{T}^{(t)}] &= \mu(x_1^{(t)}) \rho^{(t)}(a_1^{(t)} | x_1^{(t)}) \prod_{h=1}^{H-1} P^\Psi(x_{h+1}^{(t)} | x_h^{(t)}, a_h^{(t)}) \rho^{(t)}(a_{h+1}^{(t)} | x_{h+1}^{(t)}) \\
%	&= \left[ \mu(x_1^{(t)}) \prod_{h = 1}^H \rho^{(t)}(a_h^{(t)} | x_h^{(t)}) \right] \left[ \prod_{h = 1}^{H-1} P^\Psi(x_{h+1}^{(t)} | x_h^{(t)}, a_h^{(t)}) \right]
%\end{align*}
%

We deduce that:
\begin{align*}
	\mathbb{E}_\Psi[\mathcal{L}^{(t)}] &= \mathbb{E}_\Psi \left[ \sum_{h=1}^{H-1} \log \frac{P^\Psi(x_{h+1}^{(t)} | x_h^{(t)}, a_h^{(t)})}{P^\Phi(x_{h+1}^{(t)} | x_h^{(t)}, a_h^{(t)})} \right] = \sum_{(z, a, y) \in \cX \times \cA \times \cX} \mathbb{E}_\Psi\left[ \hat{N}_a^{(t)}(z, y) \right] \log \frac{P^\Psi(y | z, a)}{P^\Phi(y | z, a)},
%	&= \sum_{(x, a, y) \in \cX \times \cA \times \cX} \mathbb{E}_\Psi\left[ \sum_{\substack{h \in [H-1] \\ x_h^{(t)} = x, a_h^{(t)} = a, x_{h+1}^{(t)} = y}} \log \frac{P^\Psi(y | x, a)}{P^\Phi(y | x, a)} \right] \\
%	&= \sum_{(x, a, y) \in \cX \times \cA \times \cX} \mathbb{E}_\Psi\left[ \hat{N}_a^{(t)}(x, y) \right] \log \frac{P^\Psi(y | x, a)}{P^\Phi(y | x, a)},
\end{align*}
where $\hat{N}_a^{(t)}(z, y) := \sum_{h=1}^{H-1} \\indicator[x_h^{(t)} = z, a_h^{(t)} = a, x_{h+1}^{(t)} = y]$. Note that the dependency on the behavior policy $\rho^{(t)}$ becomes implicit in $\hat{N}_a^{(t)}(z, y)$.
In what follows, we use the following notations: $N^{\Psi, (t)}_a(z, y) := \mathbb{E}_\Psi [\hat{N}_a^{(t)}(z, y)]$, and for any subsets $X, Y$ of the set of contexts $\cX$, $N^{\Psi, (t)}_a(X, Y):=\sum_{z\in X,y\in Y} N^{\Psi, (t)}_a(z, y)$. 

Let us now simplify $\mathbb{E}_\Psi[\cL^{(t)}]$. Note that the terms for $z$ and $y$ involved in $\mathbb{E}_\Psi[\cL^{(t)}]$ are not equal to zero only if $x \in \{z, y\}$. There are three disjoint cases to consider. Recalling that $i = f(x)$,

\underline{Case 1. $\bm{z \not= x, y = x}$:}
\begin{align*}
	\sum_{a \in \cA} \sum_{z \not= x} N^{\Psi, (t)}_a(z, x) \log\frac{P^\Psi(x| z, a)}{P^\Phi(x | z, a)}
	&= \sum_{a \in \cA} \sum_{s \in \cS} N^{\Psi,(t)}_a(\tilde{f}^{-1}(s), x) \log\frac{c p(j | s, a)}{p(i | s, a)} - \sum_{a \in \cA} N_a^{\Psi,(t)}(x, x) \log\frac{c p(j | j, a)}{p(i | j, a)}.
\end{align*}

\underline{Case 2. $\bm{z = x, y \not= x}$:}
\begin{align*}
	\sum_{a \in \cA} \sum_{y \not= x} N^{\Psi,(t)}_a(x, y) \log\frac{P^\Psi(y | x, a)}{P^\Phi(y | x, a)} &= \sum_{a \in \cA} \sum_{s \in \cS} N^{\Psi,(t)}_a(x, \tilde{f}^{-1}(s)) \log\frac{p(s | j, a)}{p(s | i, a)} \\
	&\quad - \sum_{a \in \cA} \left[ N^{\Psi,(t)}_a(x, \tilde{f}^{-1}(i)) \log\left( 1 - q(x | i) \right) + N^{\Psi,(t)}_a(x, x) \log\frac{p(j | j, a)}{p(j | i, a)} \right].
\end{align*}

\underline{Case 3. $\bm{z = y = x}$:}
\begin{align*}
	\sum_{a \in \cA}  N^{\Psi, (t)}_a(x, x) \log\frac{P^\Psi(x | x, a)}{P^\Phi(x | x, a)} = \sum_{a \in \cA}  N^{\Psi, (t)}_a(x, x) \log\frac{c p(j | j, a)}{p(i | i, a)}.
\end{align*}

Combining the three cases yields:
\begin{align*}
	\mathbb{E}_\Psi[\mathcal{L}^{(t)}] &= \frac{H - 1}{n} \sum_{a \in \cA} \sum_{s \in \cS} \left[ \frac{n}{H - 1} N^{\Psi,(t)}_a(\tilde{f}^{-1}(s), x) \log\frac{c p(j | s, a)}{p(i | s, a)} + \frac{n}{H - 1} N^{\Psi,(t)}_a(x, \tilde{f}^{-1}(s)) \log\frac{p(s | j, a)}{p(s | i, a)} \right] \\
	&\quad + \frac{H - 1}{n} \underbrace{\sum_{a \in \cA} \left[ \frac{n}{H - 1} N^{\Psi, (t)}_a(x, x) \log\frac{p(i | j, a) p(j | i, a)}{p(i | i, a) p(j | j, a)} - \frac{n}{H - 1} N^{\Psi,(t)}_a(x, \tilde{f}^{-1}(i)) \log\left( 1 - q(x| i) \right) \right]}_{:= \Lambda_1^{(t)}}.
\end{align*}
From Proposition \ref{prop:m-pi} (proved in Section \ref{sec:m-asymptotics}), it can be easily seen that $|\Lambda_1^{(t)}| = \mathcal{O}(1 / n) = o(1)$.

Now we can relate $m_\rho^{\Psi,(t)}(s, a)$ to the $N^{\Psi, (t)}_a(s, x)$:
\begin{align*}
	& N_a^{\Psi, (t)}(s, x) = (H - 1) c q(x| i) p(j | s, a) m_\rho^{\Psi,(t)}(s, a),\\
	& N_a^{\Psi, (t)}(x, s) = (H - 1) c q(x | i) p(s | j, a) m_\rho^{\Psi,(t)}(j, a).
\end{align*}
Then, recalling the definition of $I_j^{(t)}(x; c, \Phi)$ (Eqn. \eqref{eqn:I}), we obtain:
\begin{align*}
	&\mathbb{E}_\Psi[\mathcal{L}^{(t)}] = \frac{H-1}{n} I_j^{(t)}(x; c, \Phi)
	+ \frac{H-1}{n} \Lambda^{(t)},
\end{align*}
where
\begin{equation*}
	\Lambda^{(t)} = \Lambda_1^{(t)} - \underbrace{\sum_{a \in \cA} \sum_{s \in \cS} n (1 - c q(x | i) p(j | s, a)) m_\rho^{\Psi,(t)}(s, a) \log\frac{1 - c q(x | i) p(j | s, a)}{1 - q(x | i) p(i | s, a)}}_{:=\Lambda_2^{(t)}}.
\end{equation*}
We now conclude the proof of Proposition \ref{prop:E-likelihood} by summing the above equalities over $t$. Recall that $I_j(x;c, \Phi) = \frac{1}{T} \sum_{t=1}^T I_j^{(t)}(x; c, \Phi)$ which implies:
\begin{align*}
	\EE_{\Psi}[\cL] = \sum_{t=1}^T \EE_\Psi[\cL^{(t)}] = \frac{T(H - 1)}{n} \left( I(x; \Phi) + \Lambda \right),
\end{align*}
where $\Lambda := \frac{1}{T} \sum_{t=1}^T \Lambda^{(t)}$. The proof is completed by applying Lemma \ref{lem:Lambda}.\ep

\begin{lemma}\label{lem:Lambda}
	$|\Lambda| = \mathcal{O}(1/n)$.
\end{lemma}
\begin{proof} We show that for all $t$, $|\Lambda^{(t)}| = \mathcal{O}(1/n)$.
	
	We first note from the fact that $\left| \log(1 - x) + x \right| \leq x^2$ for $|x| = o(1)$, we have $\left| \log \frac{1 - c p(j | s, a) q(x|i)}{1 - p(i | s, a) q(x|i)} \right| = \mathcal{O}\left( \frac{1}{n^2} \right)$.
	Because $\sum_{(s, a) \in \cS \times \cA} {m}_\rho^{\Psi,(t)}(s, a) = 1$, $\left| \Lambda_2^{(t)} \right|$ can be bounded as follows:
	\begin{equation*}
		\left| \Lambda_2^{(t)} \right| \leq n \sum_{(s, a) \in \cS \times \cA} {m}_\rho^{\Psi,(t)}(s, a) \left| \log\frac{1 - c p(j | s, a) q(x | i)}{1 - p(i | s, a) q(x | i)} \right|
		= \mathcal{O}\left( \frac{1}{n} \right).
	\end{equation*}
	Hence: $|\Lambda^{(t)}| \leq |\Lambda_1^{(t)}| + |\Lambda_2^{(t)}| \leq \mathcal{O}(1/n)$.
\end{proof}

\subsection{Proof of Proposition \ref{prop:Var-likelihood}}\label{sec:KL-Var}

We now compute $\Var_\Psi[\cL]$. For our analysis, we can just compute the variance for the model $\Psi$ constructed using the "optimal'' choices $j^\star$ and $c^\star$.
As $\cL^{(t)}$'s are independent, we have that $\Var_\Psi[\cL] = \sum_{t=1}^T \Var_\Psi[\cL^{(t)}]$, and thus we fix some $t$ and compute $\Var_\Psi[\cL^{(t)}]$. To simplify the notations, we ignore the dependency on $t$ throughout the proof. Denote $L_h = \log\frac{P^\Psi(x_{h+1} | x_h, a_h)}{P^\Phi(x_{h+1} | x_h, a_h)}$.
We have:
\begin{equation}\label{eq:var2}
	\Var_\Psi[\cL^{(t)}] = \sum_{h,h' = 1}^H \Cov_\Psi[L_h,L_{h'}]
	= \sum_{h,h' = 1}^H \left\{ \EE_\Psi[L_h L_{h'}] - \EE_\Psi[L_h]\EE_\Psi[L_{h'}] \right\}.
\end{equation}

Note that if $I(x; \Phi) = 0$, then we know from previous subsection that $\Var_\Psi[\cL^{(t)}] = 0$ as well, and so let us assume that $I(x; \Phi) > 0$.
Recall that $f(n) \asymp \tilde{f}(n)$ (resp. $f(n) \lesssim \tilde{f}(n)$) denotes $f(n)=\Theta(\tilde{f}(n))$ (resp. $f(n)={\cal O}(\tilde{f}(n))$). We divide up the computation of the r.h.s. of Eqn. \eqref{eq:var2} into three parts:

\underline{Part 1. $\bm{h' = h}$:}
\begin{align*}
	&\Cov_\Psi[L_h,L_{h'}] \leq \EE_\Psi[L_h^2] \\
	&= \sum_{(z, a) \in \cX \times \cA} \PP_\Psi[x_h = z] \rho(a | z) \sum_{y \in \cX} P^\Psi(y | z, a) \left( \log\frac{P^\Psi(y | z, a)}{P^\Phi(y | z, a)} \right)^2 \\
	&\asymp \frac{1}{n^2} \sum_{a \in \cA} \left\{ \rho(a | x) \sum_{y \in \cX \setminus \{x\}} \left( \log\frac{P^\Psi(y | x, a)}{P^\Phi(y |x, a)} \right)^2
	+ \sum_{z \in \cX \setminus \{x\}} \rho(a | z) \left( \log\frac{P^\Psi(x | z, a)}{P^\Phi(x | z, a)} \right)^2 \right\} \\
	&\lesssim \frac{1}{n},
\end{align*}
and thus $	\sum_{h' = h} \Cov_\Psi[L_h, L_{h'}] \lesssim \frac{H}{n}$.

\underline{Part 2. $\bm{h' = h + 1}$:}\\
We first note that for any $h, h' \in [H]$, $\EE_\Psi[L_h] \asymp \EE_\Psi[L_{h'}] \geq 0$ as
\begin{align*}
	\EE_\Psi[L_h] &= \sum_{(z, a) \in \cX \times \cA} \PP_\Psi[x_h = z] \rho(a | z) \sum_{y \in \cX} P^\Psi(y | z, a) \log\frac{P^\Psi(y | z, a)}{P^\Phi(y | z, a)} \\
	&= \sum_{(z, a) \in \cX \times \cA} \PP_\Psi[x_h = z] \rho(a | z) \KL\left( P^\Psi(\cdot | z, a) || P^\Phi(\cdot | z, a) \right) \geq 0
\end{align*}
and $\PP_\Psi[x_h = z] \asymp \PP_\Psi[x_{h'} = z] \asymp \frac{1}{n}$. Hence
\begin{equation}
	\Cov_\Psi[L_h,L_{h'}] \asymp \EE_\Psi[L_h L_{h'}] - \EE_\Psi[L_{h \wedge h'}]^2 \leq \EE_\Psi[L_h L_{h'}].
\end{equation}
Thus,
\begin{align*}
	&\Cov_\Psi[L_h,L_{h+1}] \lesssim \EE_\Psi[L_h L_{h+1}] \\
	&= \sum_{(z, a) \in \cX \times \cA} \PP_\Psi[x_h = z] \rho(a | z) \sum_{y \in \cX} P^\Psi(y | z, a) \\
	&\qquad \times\sum_{(a',z') \in \cA \times \cX} \rho(a' | y) P^\Psi(z' | y, a') \left( \log\frac{P^\Psi(y | z, a)}{P^\Phi(y | z, a)} \right) \left( \log\frac{P^\Psi(z' | y, a')}{P^\Phi(z' | y, a')} \right) \\
	&\asymp \frac{1}{n^3} \sum_{a,a' \in \cA} \left\{ \sum_{z, y, z' \in \cX} \rho(a | z) \rho(a' | y) \left( \log\frac{P^\Psi(y | z, a)}{P^\Phi(y | z, a)} \right) \left( \log\frac{P^\Psi(z' | y, a')}{P^\Phi(z' | y, a')} \right) \right\} \lesssim \frac{1}{n},
\end{align*}
and thus $\sum_{h' = h + 1} \Cov_\Psi[L_h, L_{h'}] \lesssim \frac{H}{n}$.

\underline{Part 3. $\bm{h' \geq h + 2}$:}\\
From $\Cov_\Psi[L_h,L_{h'}] = \EE_\Psi[L_h L_{h'}] - \EE_\Psi[L_h] \EE_\Psi[L_{h'}]$, we compute each term separately:
\begin{align*}
	& \EE_\Psi[L_h L_{h'}] \\
	&= \sum_{\substack{y_h, y_{h+1}, y_{h'}, y_{h'+1} \in \cX \\ b_h, b_{h'} \in \cA}} \PP_\Psi[x_h = y_h] \rho(b_h | y_h) P^\Psi(y_{h+1} | y_h, a_h) \\
	&\quad \times \left( \sum_{\substack{y_{h+2}, \cdots, y_{h'-1} \in \cX \\ b_{h+1}, \cdots, b_{h'-1} \in \cA}} \prod_{j=h+1}^{h'-1} \rho(b_j | y_j) P^\Psi(y_{j+1} | y_j, b_j) \right) \\
	&\quad\quad\quad \times \rho(b_{h'} | y_{h'}) P^\Psi(y_{h'+1} | y_{h'}, b_{h'}) \left( \log\frac{P^\Psi(y_{h+1} | y_h, b_h)}{P^\Phi(y_{h+1} | y_h, b_h)} \right) \left( \log\frac{P^\Psi(y_{h'+1} | y_{h'}, b_{h'})}{P^\Phi(y_{h'+1} | y_{h'}, b_{h'})} \right)
\end{align*}
and
\begin{align*}
	& \EE_\Psi[L_h] \EE_\Psi[L_{h'}] \\
	&= \left( \sum_{\substack{y_h, y_{h+1} \in \cX \\ b_h \in \cA}} \PP_\Psi[x_h = y_h] \rho(b_h | y_h) P^\Psi(y_{h+1} | y_h, a_h) \left( \log\frac{P^\Psi(y_{h+1} | y_h, b_h)}{P^\Phi(y_{h+1} | y_h, b_h)} \right) \right) \\
	&\quad \left( \sum_{\substack{y_{h'}, y_{h'+1} \in \cX \\ b_{h'} \in \cA}} \PP_\Psi[x_{h'} = y_{h'}] \rho(b_{h'} | y_{h'}) P^\Psi(y_{h'+1} | y_{h'}, a_{h'}) \left( \log\frac{P^\Psi(y_{h'+1} | y_{h'}, b_{h'})}{P^\Phi(y_{h'+1} | y_{h'}, b_{h'})} \right) \right) \\
	&= \sum_{\substack{y_h, y_{h+1}, y_{h'}, y_{h'+1} \in \cX \\ b_h, b_{h'} \in \cA}} \PP_\Psi[x_h = y_h] \rho(a_h | y_h) P^\Psi(y_{h+1} | y_h, a_h) \PP_\Psi[x_{h'} = y_{h'}] \rho(b_{h'} | y_{h'}) P^\Psi(y_{h'+1} | y_{h'}, b_{h'}) \\
	&\quad\quad\quad \left( \log\frac{P^\Psi(y_{h+1} | y_h, b_h)}{P^\Phi(y_{h+1} | y_h, b_h)} \right) \left( \log\frac{P^\Psi(y_{h'+1} | y_{h'}, b_{h'})}{P^\Phi(y_{h'+1} | y_{h'}, b_{h'})} \right).
\end{align*}
Then,
\begin{align*}
	& \EE_\Psi[L_h L_{h'}] - \EE_\Psi[L_h] \EE_\Psi[L_{h'}] \\
	&= \sum_{\substack{y_h, y_{h+1}, y_{h'}, y_{h'+1} \in \cX \\ b_h, b_{h'} \in \cA}} \PP_\Psi[x_h = y_h] \rho(b_h | y_h) P^\Psi(y_{h+1} | y_h, a_h) \rho(b_{h'} | y_{h'}) \\
	&\times P^\Psi(y_{h'+1} | y_{h'}, b_{h'}) Q_{h+1, h'}(y_{h+1}, y_{h'})  \left( \log\frac{P^\Psi(y_{h+1} | y_h, b_h)}{P^\Phi(y_{h+1} | y_h, b_h)} \right) \left( \log\frac{P^\Psi(y_{h'+1} | y_{h'}, b_{h'})}{P^\Phi(y_{h'+1} | y_{h'}, b_{h'})} \right),
\end{align*}
where
\begin{align*}
	Q_{h+1, h'}(y_{h+1}, y_{h'}) &:= \sum_{\substack{y_{h+2}, \cdots, y_{h'-1} \in \cX \\ b_{h+1}, \cdots, b_{h'-1} \in \cA}} \prod_{j=h+1}^{h'-1} \rho(b_j | y_j) P^\Psi(y_{j+1} | y_j, b_j) - \PP_\Psi[x_{h'} = y_{h'}] \\
	&=  \PP_\Psi[x_{h'} = y_{h'} | x_{h+1} = y_{h+1}] - \PP_\Psi[x_{h'} = y_{h'}].
\end{align*}
We can relate $|Q_{h+1, h'}|$ to the mixing time of $P^\Psi$:
\begin{align*}
	|Q_{h+1,h'}(y, z)| &= \left| \left( P^\Psi \right)^{h'-h-1}(y, z) - \mu \left(P^\Psi\right)^{h'-1}(z) \right| \\
	&\leq \left| \left( P^\Psi \right)^{h'-h-1}(y, z) - \Pi^\Psi(z) \right| + \left| \Pi^\Psi(z) - \mu \left(P^\Psi\right)^{h'-1}(z) \right| \\
	&\leq \sum_{z \in \cZ} \left| \left( P^\Psi \right)^{h'-h-1}(y, z) - \Pi^\Psi(z) \right| + \sum_{z \in \cZ} \left| \Pi^\Psi(z) - \mu\left(P^\Psi\right)^{h'-1}(z) \right| \\
	&\leq 2d_{TV}(\left( P^\Psi \right)^{h'-h-1}(y, \cdot), \Pi^\Psi) + 2 d_{TV}\left( \mu\left(P^\Psi\right)^{h'-1}, \Pi^\Psi \right) \\
	&\leq 2 \left( \delta\left(P^\Psi\right) \right)^{h'-h-2} d_{TV}(P^\Psi(y, \cdot), \Pi^\Psi) + 2 \left( \delta\left(P^\Psi\right) \right)^{h'-1} d_{TV}(\mu, \Pi^\Psi) \\
	&\lesssim \left( \delta\left(P^\Psi\right) \right)^{h'-1}.
\end{align*}
Using the triangle inequality,
\begin{align*}
\left| \sum_{h' \geq h + 2} \Cov_\Psi[L_h, L_{h'}] \right| &\lesssim \frac{1}{n^3} \sum_{h' \geq h + 2} \sum_{\substack{y_h, y_{h+1}, y_{h'}, y_{h'+1} \in \cX \\ b_h, b_{h'} \in \cA}} \rho(b_h | y_h) \rho(b_{h'} | y_{h'}) | Q(y_{h+1}, y_{h'}) |\\
	& \qquad \times \left| \left( \log\frac{P^\Psi(y_{h+1} | y_h, b_h)}{P^\Phi(y_{h+1} | y_h, b_h)} \right) \left( \log\frac{P^\Psi(y_{h'+1} | y_{h'}, b_{h'})}{P^\Phi(y_{h'+1} | y_{h'}, b_{h'})} \right) \right| \\
	&\lesssim \frac{1}{n} \sum_{h' \geq h + 2} \left( \delta\left(P^\Psi\right) \right)^{h'-1}\lesssim \frac{H}{n}.
\end{align*}

Finally, combining all above completes the proof of Proposition \ref{prop:Var-likelihood}.\ep

% !TEX root = ./main.tex

\subsection{Properties of the rate function $I(x;\Phi)$}\label{subsec:ratefunction}

In this subsection, we provide properties of the rate function $I(x;\Phi)$ that are useful for the analysis. We also provide intermediate results used throughout our proofs.

\subsubsection{Necessary and sufficient conditions for $I(x;\Phi)=0$}
%For simplicity\footnote{One can easily extend the discussion to more general cases; see Appendix \ref{app:adaptive}}, here, let us assume that we use the same fixed policy throughout the episodes i.e. $\rho^{(t)} = \rho$.
\begin{proposition}
	\label{prop:I=0}
	$I(x; \Phi) \geq 0$, and the equality holds iff there exist $j \in \cS$ and $c > 0$ such that both of the following holds:
	\begin{enumerate}
		\item $p(f(x) | s, a) = c p(j | s, a), \quad \forall (s, a) \in \cS \times \cA$,
		\item $p(s | f(x), a) = p(s | j, a), \quad \forall (s, a) \in \cS \times \cA$.
	\end{enumerate}
	%	{\color{red} Furthermore, if we consider the necessary-sufficient condition for $\liminf_{n \rightarrow \infty} I(x;\Phi) = 0$, then the RHS of both conditions must be asymptotically $o\left(\frac{1}{|\cS|}\right)$.}
\end{proposition}

Before we prove the above proposition, we give its interpretation. It states that from the observations on $x$, we cannot distinguish whether $x$'s latent state is $j$ or $f(x)$. Indeed, {\it 1.} represents the fact that the observations of transitions leading to context $x$ are statistically equivalent in $\Phi$ and $\Psi^{(x,j)}$, and {\it 2.} represents the fact that the observations of transitions from context $x$ are statistically equivalent in $\Phi$ and $\Psi^{(x,j)}$.
The additional $c$ takes into account the non-uniform emission probabilities of each context of each cluster.

\begin{proof}
	$I(x; \Phi) = 0$ if and only if there exists $j$ and $c$ such that $I_j(x; c, \Phi) = 0$.
	We now show that $I_j(x; c, \Phi)$ is actually equal to a mixture of $\KL$'s. From there, the results will follow immediately. Let $m_\rho^{\Psi}(s, a)=m_\rho^{\Psi, (t)}(s, a)$.
	
	Define $p^{out}_{k,a}(s) := p(s | k, a)$, and $r \in \cP([2K] \times \cA)$ as
	\begin{align}
		r(\bar{s}, a) &:=
		\begin{dcases}
			m_\rho^{\Psi}(s, a) p(f(x) | s, a) q(x | f(x)) &\quad (\bar{s} = 2s - 1), \\
			m_\rho^{\Psi}(s, a) \left( 1 - p(f(x) | s, a) q(x | f(x)) \right) &\quad (\bar{s} = 2s).
		\end{dcases}
	\end{align}
	Analogously, we define
	\begin{align}
		\tilde{r}(\bar{s}, a; c) &:=
		\begin{dcases}
			 m_\rho^{\Psi}(s, a) p(j | s, a) c q(x | f(x)) &\quad (\bar{s} = 2s - 1), \\
			m_\rho^{\Psi}(s, a) \left( 1 -  p(j | s, a) c q(x | f(x)) \right) &\quad (\bar{s} = 2s),
		\end{dcases}.
	\end{align}
	Then it is easy to see that actually,
	\begin{equation}
		I_j(x; c, \Phi) = n {\KL\Big( \tilde{r}(\cdot, \cdot; c) || r(\cdot, \cdot) \Big)}
		+ n c q(x | f(x)) \sum_{a \in \cA} m_\rho^{\Psi}(f(x), a) { \KL\left( p^{out}_{j,a}(\cdot) || p^{out}_{f(x),a}(\cdot) \right)}.
	\end{equation}
The proposition then follows from the fact that $\KL(p || q) \geq 0$ and $\KL(p || q) = 0$ iff $p = q$ a.e.
\end{proof}

\subsubsection{Alternative $\KL$-form of the rate function $I(x;\Phi)$}\label{sec:alternate-kl}

Next, we derive an alternative asymptotic $\KL$-form for the rate function $I(x;\Phi)$ that will be useful later in the analysis of the algorithm.
For simplicity, fix $t$. We denote $m_\rho := m_\rho^{\Phi,(t)}$. We first introduce the {\it alternative divergence}:
\begin{equation}
	\tilde{I}(x; \Phi) := \min_{j: j \not= f(x)} \inf_{c > 0} \tilde{I}_j(x; c, \Phi),
\end{equation}
where
\begin{equation}
	\tilde{I}_j(x; c, \Phi) := n \KL\Big( p^{in}_{\Phi, x}(\cdot, \cdot) || p^{in}_{\Psi, x}(\cdot, \cdot; c) \Big)
	+ c n q(x | f(x)) \sum_{a \in \cA} m_\rho(f(x), a) \KL\left( p^{out}_{f(x),a}(\cdot) || p^{out}_{j,a}(\cdot) \right).
\end{equation}
and where $p^{out}_{k,a}(\cdot) := p(\cdot | k, a)$, and recalling that $\cS = [K]$, for $\bar{s} \in [2K]$ and $a \in \cA$,
\begin{align}
\label{eqn:p-in}
	p^{in}_{\Phi, x}(\bar{s}, a) &:=
	\begin{dcases}
		m_\rho(s, a) p(f(x) | s, a) q(x | f(x)) &\quad (\bar{s} = 2s - 1), \\
		m_\rho(s, a) \left( 1 - p(f(x) | s, a) q(x | f(x)) \right) &\quad (\bar{s} = 2s),
	\end{dcases}\\
	p^{in}_{\Psi,x}(\bar{s}, a; c) &:=
	\begin{dcases}
		c m_\rho(s, a) p(j | s, a) q(x | f(x)) &\quad (\bar{s} = 2s - 1), \\
		m_\rho(s, a) \left( 1 - c p(j | s, a) q(x | f(x)) \right) &\quad (\bar{s} = 2s),
	\end{dcases}.
\end{align}
It's easy to check that $p^{out}_{k,a}(\cdot) \in \cP(\cS)$ and $p^{in}_{\Phi, x}(\cdot, \cdot), p^{in}_{\Psi, x}(\cdot, \cdot;c) \in \cP([2K] \times \cA)$.
Note that compared to the actual divergence $I$, the order of $\Phi$ and $\Psi$ is switched.

Roughly speaking, $p^{out}_{k,a}(s)$ describes the {\it outgoing} probabilities from the state-action pair $(k, a)$, and $p^{in}_{\Phi, x}(\bar{s}, a)$ (resp. $p^{in}_{\Psi, x}$) describes the {\it inward} probabilities into the context $x$ under $\Phi$ (resp. $\Psi$), up to order-wise negligible remainders. We then show the following  of $\tilde{I}$:
%\begin{proposition}
%	\label{prop:divergence-kl}
%	We have:
%	\begin{equation}
%		\mathbb{E}_\Psi[\mathcal{L}]
%		\leq \max(c, 1/c, 1) \frac{TH}{n} \tilde{I}_j(x; c, \Phi)
%		+ \frac{TH}{n} \Lambda,
%	\end{equation}
%	where $|\Lambda| = \mathcal{O}(1 / n) = o(1)$.
%\end{proposition}
\begin{proposition}
	\label{prop:I}
	For all $c > 0$, we have that $I_j(x; c, \Phi) = 0$ if and only if $\tilde{I}_j(x; c, \Phi) = 0$.
	Furthermore, we have that
	$\min(1, c, 1/c, 1/\eta) \tilde{I}_j(x; c, \Phi) \leq I_j(x; c, \Phi) \leq \max(1, c, 1/c, \eta) \tilde{I}_j(x; c, \Phi)$.
\end{proposition}
\begin{proof}
	The first part is trivial.

	For the second part, we start with the following observations:
	\begin{align*}
		{p^{in}_{\Psi, x}(\bar{s}, a) \over p^{in}_{\Phi,x}(\bar{s}, a; c)} &=
		\begin{dcases}
			c &\quad (\bar{s} = 2s - 1), \\
			\frac{1 - c p(j | s, a) q(x | f(x))}{1 - p(f(x) | s, a) q(x | f(x))} \simeq 1 &\quad (\bar{s} = 2s),
		\end{dcases} \\
		&\lesssim \max(c, 1),
	\end{align*}
	and
	\begin{align*}
		{p^{in}_{\Phi, x}(\bar{s}, a) \over p^{in}_{\Psi,x}(\bar{s}, a; c)} &=
		\begin{dcases}
			\frac{1}{c} &\quad (\bar{s} = 2s - 1), \\
			\frac{1 - p(f(x) | s, a) q(x | f(x))}{1 - c p(j | s, a) q(x | f(x))} \simeq 1 &\quad (\bar{s} = 2s),
		\end{dcases} \\
		&\lesssim \max(1/c, 1).
	\end{align*}
	Also, we have that $m_\rho(s, a) \sim m_\rho^\Psi(s, a)$.
	From the characterization of $I_j$ as a mixture of $\KL$'s (see proof of Proposition \ref{prop:I=0}), we conclude by applying Lemma \ref{lem:kl-invert}.
\end{proof}

\begin{proposition}
	\label{prop:I-additional}
	There exist nonnegative functions $\psi_1, \psi_2$, independent of $n$, such that for all $c > 0$, $x \in \cX$, $j \in \cS$, and BMDP $\Phi$ (satisfying Assumptions \ref{assumption:SA}-\ref{assumption:uniform}), $\psi_1(p, \eta, c) \leq \tilde{I}_j(x; c, \Phi) \leq \psi_2(p, \eta, c)$.
	Consequently, $I_j(x; \Phi) = \inf_{c > 0} I_j(x; c, \Phi)$ does not scale with $n$.
%	For all $c > 0$, $\poly_1(\eta, c, 1/\eta, 1/c) \leq \tilde{I}_j(x; c, \Phi) \leq \poly_2(\eta, c, 1/\eta, 1/c)$.
%	We may have $\poly_1 \equiv \poly_2 \equiv 0$, in which case $I(x; \Phi) = \tilde{I}(x; \Phi) = 0$ i.e. the given block MDP is not clusterable.
\end{proposition}
\begin{proof}
	For simplicity, denote $u_{s,a} = \frac{p(s|f(x),a)}{p(s | j,a)} \in [1/\eta, \eta]$ and $v_{s,a} = \frac{p(f(x)|s,a)}{p(j | s,a)} \in [1/\eta, \eta]$.
%	$K := \max_{1/\eta \leq u \leq \eta} (u - 1)^2$ and $K_c := \max_{1/\eta \leq u \leq \eta} ((u / c) - 1)^2$.
	Then,
	\begin{align*}
		& \tilde{I}_j(x; c, \Phi) \\
		&= n \KL\Big( p^{in}_{\Phi, x}(\cdot, \cdot) || p^{in}_{\Psi, x}(\cdot, \cdot; c) \Big)
		+ c n q(x | f(x)) \sum_{a \in \cA} m_\rho(f(x), a) \KL\left( p^{out}_{f(x),a}(\cdot) || p^{out}_{j,a}(\cdot) \right) \\
		&\overset{(a)}{\leq} \frac{n}{2} \sum_{\bar{s}, a} \frac{\left( p^{in}_{\Phi, x}(\bar{s}, a) - p^{in}_{\Psi, x}(\bar{s}, a; c) \right)^2}{p^{in}_{\Psi, x}(\bar{s}, a; c)}
		+ \frac{c}{2} n q(x | f(x)) \sum_{a \in \cA} m_\rho(f(x), a) \sum_s \frac{\left( p^{out}_{f(x),a}(s) - p^{out}_{j,a}(s) \right)^2}{p^{out}_{j,a}(s)} \\
		&\leq \frac{n}{2} \sum_{s, a} m_\rho(s, a) q(x | f(x)) \left( \frac{1}{c p(j | s, a)} + \frac{q(x | f(x))}{1 - c p(j | s, a) q(x | f(x))} \right) \left( p(f(x) | s, a) - c p(j | s, a) \right)^2 \\
		&\quad \ + c \eta^2 S \frac{\eta^4}{SA} \sum_{s, a} p(s | j,  a) \left( \frac{p(s | f(x),  a)}{p(s | j,  a)} - 1 \right)^2 \\
		&\overset{(b)}{\leq} c \eta^7 \frac{1}{SA} \sum_{s, a} \left( \frac{v_{s,a}}{c} - 1 \right)^2 + c \eta^7 \frac{1}{SA} \sum_{s, a} \left( u_{s,a} - 1 \right)^2 \\
		&= \eta^7 \frac{1}{SA} \sum_{s, a} \left( \frac{v_{s,a}^2}{c} + c - 2 v_{s,a} \right) + c \eta^7 \frac{1}{SA} \sum_{s, a} \left( u_{s,a} - 1 \right)^2
		\triangleq \psi_2(p, \eta, c)
%		\overset{(c)}{\leq} \poly_2(\eta, c, 1/\eta, 1/c)
%		&\leq \eta^6 \left( \frac{\eta^3}{c} + \eta c - 2 \right) + c \eta^5 (\eta - 1)^2 \max(1, 1/\eta^2)
	\end{align*}
	and
	\begin{align*}
		& \tilde{I}_j(x; c, \Phi) \\
		&= n \KL\Big( p^{in}_{\Phi, x}(\cdot, \cdot) || p^{in}_{\Psi, x}(\cdot, \cdot; c) \Big)
		+ c n q(x | f(x)) \sum_{a \in \cA} m_\rho(f(x), a) \KL\left( p^{out}_{f(x),a}(\cdot) || p^{out}_{j,a}(\cdot) \right) \\
		&\overset{(a)}{\geq} \frac{n}{2} \sum_{\bar{s}, a} \frac{\left( p^{in}_{\Phi, x}(\bar{s}, a) - p^{in}_{\Psi, x}(\bar{s}, a; c) \right)^2}{p^{in}_{\Phi, x}(\bar{s}, a) \vee p^{in}_{\Psi, x}(\bar{s}, a; c)}
		+ \frac{c}{2} n q(x | f(x)) \sum_{a \in \cA} m_\rho(f(x), a) \sum_s \frac{\left( p^{out}_{f(x),a}(s) - p^{out}_{j,a}(s) \right)^2}{p^{out}_{f(x),a}(s) \vee p^{out}_{j,a}(s)} \\
		&\geq \frac{n}{2} \sum_{s, a} m_\rho(s, a) q(x | f(x)) \frac{\left( p(f(x)|s,a) - c p(j|s,a) \right)^2}{p(f(x)|s,a) \vee c p(j|s,a)} \\
		&\quad \ + \frac{n}{2} \sum_{s, a} m_\rho(s, a) q(x | f(x))^2 \frac{\left( p(f(x) | s, a) - c p(j | s, a) \right)^2}{1 - q(x | f(x)) \left( p(f(x) | s, a) \wedge c p(j | s, a) \right)} \\
		&\quad \ + \frac{c}{2 \eta^6} \frac{1}{A} \sum_{s, a} \frac{\left( p(s | f(x), a) - p(s | j, a) \right)^2}{p(s | f(x), a) \vee p(s | j, a)} \\
		&\overset{(b)}{\geq} \frac{1}{2 \eta^6} \frac{1}{A} \sum_{s, a}  \frac{\left( p(f(x)|s,a) - c p(j|s,a) \right)^2}{p(f(x)|s,a) \vee c p(j|s,a)}
		+ \frac{c}{2 \eta^8} \frac{1}{SA} \sum_{s, a} \left( u_{s,a} - 1 \right)^2 \\
		&\geq \frac{1}{2 \eta^7} \frac{c}{1 \vee \frac{\eta}{c}} \frac{1}{SA} \sum_{s, a} \left( \frac{v_{s,a}}{c} - 1 \right)^2
		+ \frac{c}{2 \eta^8} \frac{1}{SA} \sum_{s, a} \left( u_{s,a} - 1 \right)^2
		\triangleq \psi_1(p, \eta, c)
%		\overset{(c)}{\geq} \poly_1(\eta, c, 1/\eta, 1/c)
%		&\geq \frac{n}{2} \sum_{s, a} m_\rho(s, a) q(x | f(x)) \left( \frac{1}{c p(j | s, a)} + \frac{q(x | f(x))}{1 - c p(j | s, a) q(x | f(x))} \right) \left( p(f(x) | s, a) - c p(j | s, a) \right)^2 \\
%		&\quad \ + c \eta^2 S \frac{\eta^4}{SA} \sum_{s, a} p(s | j,  a) \left( \frac{p(s | f(x),  a)}{p(s | j,  a)} - 1 \right)^2 \\
%		&\overset{(b)}{\leq} c \eta^7 \frac{1}{SA} \sum_{s, a} \left( \frac{v_{s,a}}{c} - 1 \right)^2 + c \eta^7 \frac{1}{SA} \sum_{s, a} \left( u_{s,a} - 1 \right)^2 \\
%		&= \eta^7 \frac{1}{SA} \sum_{s, a} \left( \frac{v_{s,a}^2}{c} + c - 2 v_{s,a} \right) + c \eta^7 \frac{1}{SA} \sum_{s, a} \left( u_{s,a} - 1 \right)^2 \\
%		&\leq \eta^6 \left( \frac{\eta^3}{c} + \eta c - 2 \right) + c \eta^5 (\eta - 1)^2 \max(1, 1/\eta^2),
	\end{align*}
	where $(a)$ follows from Lemma \ref{lem:kl-bound} and $(b)$ follows from the observation that $0 \leq \frac{q(x | f(x))}{1 - c p(j | s, a) q(x | f(x))} = o(1)$.
	The last point follows from Proposition \ref{prop:I} and the inequality we proved above.
%	 and $(c)$ follows from the fact that $u_{s,a}, v_{s,a}$ depend on $p$'s only.
\end{proof}
%
%\begin{corollary}
%	$\tilde{I}_j(x ; c, \Phi) = \Theta(1)$, implying that $I_j(x ; c, \Phi) = \Theta(1)$.
%\end{corollary}
%\begin{remark}
%	The assumption that $c$ does not depend on $n,S,A$ is not too restrictive.
%	To illustrate this, consider a simple case in which for some $x \in \cX$, $a \in \cA$, $j \not= f(x) \in \cS$, and $y \in f^{-1}(j)$, we have that $p(\cdot | j, a) = p(\cdot | f(x), a)$ and $q(y | j) p(j | s, a) = q(x | f(x)) p(f(x) | s, a)$.
%	In this particular case, intuitively, it is easy to see that any algorithm cannot cluster $x$ correctly (with high probability), as the inward and outward information of $x$ is the same as $y$ in a different cluster.
%	Indeed, when plugging in this model to our definition of $I$, the optimal choice of $c$ is $c^\star = \frac{q(y | j)}{q(x | f(x))}$.
%	From our regularity assumption, this implies that $c = \Theta(1)$.
%	It is safe to say that restricted to block MDPs with our regularity assumption, simiilar conclusion will hold, and thus the assumption that $c$ not depending on $n,S,A$ is not too far-off.
%\end{remark}

\begin{lemma}\label{lem:kl-invert}
	If $p, q$ are (discrete) probability distributions with the same support $\cZ$ satisfying $\max_{z \in \cZ} (p(z)/q(z), q(z)/p(z)) \leq \xi$ for some $\xi > 1$, then the following holds:
	\begin{equation}
		\max\left( \frac{\KL(p || q)}{\KL(q || p)}, \frac{\KL(q || p)}{\KL(p || q)} \right) \leq \xi.
	\end{equation}
\end{lemma}
\begin{proof}
	The statement then follows from the following lemma:
	\begin{restatable}[Lemma 19 of SM6.3 of \cite{SandersPY20} \& Lemma 6.3 of \cite{Csiszar06}]{lemma}{klbound}
		\label{lem:kl-bound}
		\begin{equation}
			\frac{1}{2} \sum_{z \in \cZ} \frac{(p(z) - q(z))^2}{p(z) \vee q(z)} \leq \KL(p || q) \leq \frac{1}{2} \sum_{z \in \cZ} \frac{(p(z) - q(z))^2}{q(z)}.
		\end{equation}
	\end{restatable}
\end{proof}

\subsubsection{Asymptotics of $m_\rho$}\label{sec:m-asymptotics}

We provide properties of the quantity: $m_\rho(s, a) := \frac{1}{H-1} \sum_{h = 1}^{H-1} \PP_\Phi[f(x_h^{(t)}) = s, a_h^{(t)} = a]$ (there is no dependence in $t$). These properties are extensively used in the appendices, and they also hold for $m_\rho^{\Psi}$ (up to some negligible remainders -- actually, it can be easily seen that $m_\rho(s, a) \sim m_\rho^\Psi(s, a)$).

\begin{proposition}\label{prop:mm}
	For $(s, a) \in \cS \times \cA$,
	\begin{equation}
		m_\rho(s, a)
		=  \frac{1}{H-1} \sum_{h=1}^{H-1}\sum_{x\in f^{-1}(s)} \rho(a|x) \sum_{z\in {\cal X}}\mu(z)\left( P_0 \right)^{h-1}(x|z),
	\end{equation}
	where remember that $P_0=\sum_{a \in \mathcal{A}} \rho(a | x) P(y | x, a)$, and by convention $\sum_z\mu(z)P_0^0(x|z)=\mu(x)$.
\end{proposition}
\begin{proof}
	By definition,
	\begin{align*}
		m_\rho(s, a) &= \frac{1}{H-1} \sum_{h=1}^{H-1} \PP_\Phi\left[ f(x_h^{(t)}) = s, a_h^{(t)} = a \right] \\
		&= \frac{1}{H-1} \sum_{h=1}^{H-1} \sum_{x \in f^{-1}(s)} \rho(a | x) \PP_\Phi\left[ x_h^{(t)} = x \right].
	\end{align*}

Now we have: $\PP_\Phi\left[ x_h^{(t)} = x \right]=\sum_z\mu(z)P_0^{h-1}(x | z)$ for $h\ge 1$. The result follows immediately.
\end{proof}

\begin{proposition}
	\label{prop:m-pi}
	The following holds for all $(s, a) \in \cS \times \cA$: with $\rho(a | s) := \frac{1}{n\alpha_s} \sum_{y \in f^{-1}(s)} \rho(a | y)$, $\rho \sim \cU(\cA)$ and $\mu \sim \cU(\cX)$,
	\begin{equation}
		\frac{1}{\eta^4 SA} \leq \frac{1}{\eta^2} \alpha_s \rho(a | s)\leq m_\rho(s, a) \leq \eta^2 \alpha_s \rho(a | s) \leq \frac{\eta^4}{SA}.
	\end{equation}

	i.e. $m_\rho(s, a) = \Theta\left( \frac{1}{SA} \right)$.
\end{proposition}

\begin{proof} The result follows directly from Proposition \ref{prop:mm} and from the results of Appendix \ref{app:equilibrium}.
\end{proof}

%\begin{proof}
%We only prove the upper bound as the lower bound follows from the same reasoning. From the regularity conditions, we have that for all $(z, x) \in \cX \times \cX$,
%\begin{align*}
%\left( P_0 \right)^{h-1}(z, x) &\leq \max_{z, x \in \cX} P_0(z, x) = \max_{z, x \in \cX} \sum_{a \in \cA} \rho(a | x) P(x|z,a) \\
%		&\leq \max_{z, x \in \cX} \sum_{a \in \cA} \rho(a | x) \frac{\eta^2}{n}
%		= \frac{\eta^2}{n},
%	\end{align*}
%and thus, denoting $\rho(a | s) := \frac{1}{n\alpha_s} \sum_{x \in f^{-1}(s)} \rho(a | x) = \frac{1}{T} \sum_{t=1}^T \sum_{x \in f^{-1}(s)} \rho^{(t)}(a | x)$,
%	\begin{equation*}
%		m_\rho(s, a) \leq \frac{1}{n T(H-1)} \sum_{t=1}^T \sum_{h=1}^{H-1} \sum_{x \in f^{-1}(s)} \sum_{z \in \cX} \rho^{(t)}(a | x) \frac{\eta^2}{n}
%		= \eta^2 \alpha_s \rho(a | s).
%	\end{equation*}
%\end{proof}

\subsubsection{Examples of rate function $I(x; \Phi)$}\label{sec:I-examples}
For simplicity, we consider the following environment: $\cX = \{x_1, \cdots, x_{10}\}$, $\cS = \{s_1, s_2\}$, $\cA = \{a_1, a_2\}$, and $f : \cX \rightarrow \cS$ defined such that $f^{-1}(s_1) = \{x_1, x_3, x_5, x_7, x_9\}$ and $f^{-1}(s_2) = \{x_2, x_4, x_6, x_8, x_{10}\}$.
We denote $P_k$ to be the transition probability matrix corresponding to action $a_k$.
Our exploration policy and initial distribution are set to be uniformly random, i.e., $\rho(\cdot | x) \sim \cU(\cA)$ for all $x \in \cX$ and $\mu(\cdot) \sim \cU(\cX)$.
As we only have two clusters, we have that $I(x; c, \Phi) = I_2(x; c, \Phi)$.

Furthermore, to clearly illustrate our change-of-measure argument, we fix the construction of alternate model $\Psi$ as follows: change $f$ to $g$ such that $g(x_1) = s_2$ and rest stay the same.
Also, we assume that $q$ is uniform over its support (or its respective cluster), which implies that the alternate $\tilde{q}$ is set as follows:
\begin{equation*}
	\tilde{q}(x_{2k-1} | s_1) = \frac{1}{4}, \ \tilde{q}(x_1 | s_2) = \frac{c}{5}, \  \tilde{q}(x_{2k} | s_2) = \frac{5 - c}{25}.
\end{equation*}
where $c \in (0, 5)$ is to be set later.

Lastly, we recall (from the proof of Proposition \ref{prop:I=0}) that
\begin{equation*}
	I(x_1; c, \Phi) = n {\KL\Big( \tilde{r}(\cdot, \cdot; c) || r(\cdot, \cdot) \Big)}
	+ \frac{c}{5} n \sum_{a \in \cA} m_\rho^{\Psi}(s_1, a) { \KL\left( p^{out}_{j,a}(\cdot) || p^{out}_{s_1,a}(\cdot) \right)},
\end{equation*}
where $p^{out}_{k,a}(s) := p(s | k, a)$, and $r, \tilde{r} \in \cP(\{s_1, \bar{s}_1, s_2, \bar{s}_2\} \times \cA)$ are defined as
\begin{align*}
	r(\bar{s}, a) &:=
	\begin{dcases}
		\frac{1}{5} m_\rho^{\Psi}(s, a) p(s_1 | s_k, a) &\quad (\bar{s} = s_k), \\
		m_\rho^{\Psi}(s, a) \left( 1 -\frac{1}{5} p(s_1 | s_k, a) \right) &\quad (\bar{s} = \bar{s}_k),
	\end{dcases}
\end{align*}
and
\begin{align*}
	\tilde{r}(\bar{s}, a; c) &:=
	\begin{dcases}
		\frac{c}{5} m_\rho^{\Psi}(s, a) p(s_2 | s_k, a) &\quad (\bar{s} = s_k), \\
		m_\rho^{\Psi}(s, a) \left( 1 -  \frac{c}{5} p(s_2 | s_k, a) \right) &\quad (\bar{s} = \bar{s}_k),
	\end{dcases}.
\end{align*}
From Proposition \ref{prop:mm}, we have that
\begin{equation*}
	m_\rho^\Psi(s, a)
	=  \frac{1}{20(H-1)} \sum_{h=1}^{H-1}\sum_{x\in g^{-1}(s)} \sum_{z\in {\cal X}} \left( \tilde{P}_0 \right)^{h-1}(x|z),
\end{equation*}
where $\tilde{P}_0 = \frac{1}{2} \sum_{a \in \cA} \tilde{P}(y | x, a)$ and $\tilde{P}(y | x, a) = p(g(y) | g(x), a) \tilde{q}(y | g(y))$.

For computational simplicity, we set the horizon length to be $H = 10$.
All the computations were done using Mathematica.

\begin{figure}[t]
	\centering
	\begin{subfigure}[b]{0.45\textwidth}
		\centering
		\includegraphics[width=\textwidth]{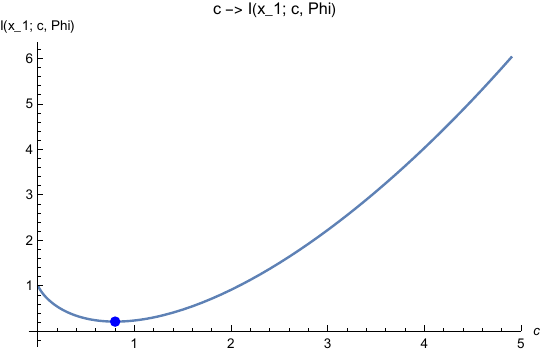}
		\caption{$I(x_1; \Phi) > 0$}
		\label{fig:I>0}
	\end{subfigure}
	\hfill
	\begin{subfigure}[b]{0.45\textwidth}
		\centering
		\includegraphics[width=\textwidth]{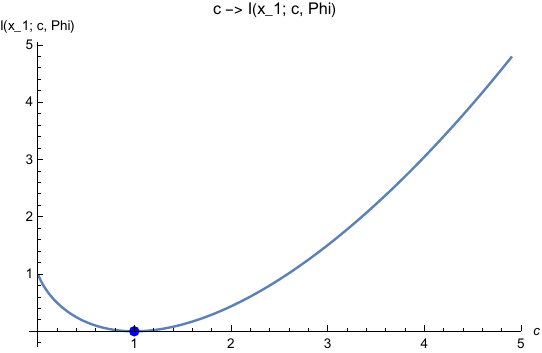}
		\caption{$I(x_1; \Phi) = 0$}
		\label{fig:I=0}
	\end{subfigure}
	\caption{Plot of $I(x_1; c, \Phi)$ as function of $c$. Note that the domain of $c$ is $(0, 5)$}
\end{figure}

\underline{$\bm{I(x_1; \Phi) > 0}$}
Consider the following instantiation:
\begin{equation*}
	P_1 = \begin{bmatrix}
		2/3 & 1/3 \\ 1/3 & 2/3
	\end{bmatrix}, \quad
	P_2 = \begin{bmatrix}
		1/2 & 1/2 \\ 1/2 & 1/2
	\end{bmatrix}.
\end{equation*}
Then we have that
\begin{align*}
	m_\rho^\Psi(s_1, a_1) &= m_\rho^\Psi(s_2, a_1) = \frac{73567181}{302330880} \approx 0.2433\\
	m_\rho^\Psi(s_1, a_2) &= m_\rho^\Psi(s_2, a_2) = \frac{77598259}{302330880} \approx 0.2567
\end{align*}

With some more computations, we have the explicit form of the divergence $I$ as follows:
\begin{align*}
	I(x_1; c, \Phi) = \frac{73567181}{4453496320} &\left[ 2c \log 2 + (15 - 2c)\log\frac{15 - 2c}{14} + 3(10 - c) \log\frac{10 - c}{9} \right. \\
	&\ \left. + (15 - c)\log\frac{15 - c}{13} + 6c\log c \right]
\end{align*}
and thus,
\begin{equation*}
	I(x_1; \Phi) = \inf_{c > 0} I(x_1; c, \Phi) \approx 0.2127 > 0,
\end{equation*}
where the minimum is attained at $c \approx 0.8023$.
We also provide a plot of $I(x_1; c, \Phi)$ as a function of $c$ in Figure \ref{fig:I>0}.

\underline{$\bm{I(x_1; \Phi) = 0}$}
Consider the following instantiation:
\begin{equation*}
	P_1 = \begin{bmatrix}
		1/2 & 1/2 \\ 1/2 & 1/2
	\end{bmatrix}, \quad
	P_2 = \begin{bmatrix}
		1/2 & 1/2 \\ 1/2 & 1/2
	\end{bmatrix}.
\end{equation*}
Then we have that
\begin{equation*}
	m_\rho^\Psi(s_1, a_1) = \frac{11}{45} \approx 0.2444, \quad m_\rho^\Psi(s_1, a_2) = m_\rho^\Psi(s_2, a_2) = \frac{23}{90} \approx 0.2556
\end{equation*}

Again, with some more computations, we have that
\begin{equation*}
	I(x_1; c, \Phi) = \frac{44}{45} \left[ (10 - c) \log\frac{10 - c}{9} + c \log c \right],
\end{equation*}
and thus,
\begin{equation*}
	I(x_1; \Phi) = \inf_{c > 0} I(x_1; c, \Phi) = 0,
\end{equation*}
where the minimum is attained at $c = 1$.
Again, we also provide a plot of $I(x_1; c, \Phi)$ as a function of $c$ in Figure \ref{fig:I=0}.

%	\begin{remark}
%		Surprisingly, in both cases, the function $c \mapsto I_j(x; c, \Phi)$ seems to be convex in its domain.
%		Although we don't prove this rigorously, we conjecture that the convexity (or maybe weak convexity \cite{Vial83}) holds for more general cases.
%	\end{remark}

\subsubsection{Relation to other BMDP separability notions}\label{app:separability-others}
$I(\Phi) > 0$ can be considered as a separability condition, as it implies that we can correctly ``separate''(cluster) all contexts.
In this section, we actually show that our notion of separability encompasses previously considered separability notions i.e. ours is the strongest.

\paragraph{$\gamma$-separability}
\cite{DuKJAD19a} considered a notion of $\gamma$-separability, in which the {\it backward probability vectors} of two different latent states should be sufficiently separated.
Precisely, for $\nu \in \cP(\cS \times \cA)$ with $\supp(\nu) = \cS \times \cA$, define
\begin{equation*}
	\bm{b}_\nu(s') = (b_\nu(s, a | s'))_{(s, a) \in \cS \times \cA},
	\quad b_\nu(s, a | s') := \frac{p(s' | s, a) \nu(s, a)}{\sum_{\tilde{s}, \tilde{a}} p(s' | \tilde{s}, \tilde{a}) \nu(\tilde{s}, \tilde{a})}.
\end{equation*}
\begin{definition}[Assumption 3.2 of \cite{DuKJAD19a}]
	For $\gamma > 0$, a BMDP is said to be {\bf $\gamma$-separable} if
	\begin{equation}
		\lVert \bm{b}_\nu(s') - \bm{b}_\nu(s'') \rVert_1 \geq \gamma, \quad \forall s' ,s'' \in \cS, s' \neq s''
	\end{equation}
	If no such $\gamma$ exists, then we say that the BMDP is {\bf $\gamma$-inseparable}.
\end{definition}

The following proposition shows that our notion of separability ($I(\Phi) > 0$) is {\it weaker} (in that it encompasses a broader range of BMDPs) than $\gamma$-separability:
\begin{proposition}
	If the BMDP $\Phi$ is $\gamma$-separable, then $I(\Phi) > 0$.
\end{proposition}
\begin{proof}
	We prove the contrapositive.
	Assume that $I(\Phi) = 0$.
	Then by Proposition \ref{prop:I=0}, there exists some $j \in \cS$ with $j \neq f(x)$ and $c > 0$ such that
	\begin{equation*}
		p(s | f(x), a) = p(s | j, a) \text{ and } p(f(x) | s, a) = c p(j | s, a), \quad \forall (s, a) \in \cS \times \cA.
	\end{equation*}
	WLOG fix some $\nu \in \cP(\cS \times \cA)$ with $\supp(\nu) = \cS \times \cA$.
	Then, for arbitrary $(s, a) \in \cS \times \cA$,
	\begin{align*}
		b_\nu(s, a | f(x)) &= \frac{p(f(x) | s, a) \nu(s, a)}{\sum_{\tilde{s}, \tilde{a}} p(f(x) | \tilde{s}, \tilde{a}) \nu(\tilde{s}, \tilde{a})} \\
		&= \frac{c p(j | s, a) \nu(s, a)}{\sum_{\tilde{s}, \tilde{a}} c p(j | \tilde{s}, \tilde{a}) \nu(\tilde{s}, \tilde{a})} \\
		&= \frac{p(j | s, a) \nu(s, a)}{\sum_{\tilde{s}, \tilde{a}} p(j | \tilde{s}, \tilde{a}) \nu(\tilde{s}, \tilde{a})}
		= b_\nu(s, a | j),
	\end{align*}
	which implies that $\bm{b}_\nu(f(x)) = \bm{b}_\nu(j)$.
\end{proof}
%%%% COMMENT IN
%Intuitively, the $\gamma$-separability only considers the backward probability vectors; thus it doesn't cover the case in which the backwards are equal, yet the forwards are different, in which case $I(\Phi) > 0$ i.e. clusterable in our setting.

\paragraph{Kinematic separability}
\cite{misra2020kinematic} considered a notion of kinematic inseparability, in which two contexts have the same forward and backward probabilities.

We recall that $P(y | x, a) = q(y | f(y)) p(f(y) | f(x), a)$ is the context transition probability kernel.
Similarly as above, we define the {\it contextual backward transition probability kernel}: given some $u \in \cP(\cX \times \cA)$ with $\supp(u) = \cX \times \cA$,
\begin{equation}
	P^{bwd}(x, a | y) := \frac{P(y | x, a) u(x, a)}{\sum_{x', a'} P(y | x', a') u(x', a')}.
\end{equation}

\begin{definition}[Definition 3 of \cite{misra2020kinematic}]
	Given a BMDP $\Phi$, two contexts $x_1, x_2 \in \cX$ are {\bf kinematically inseparable} if for every distribution $u \in \cP(\cX \times \cA)$ with $\supp(u) = \cX \times \cA$ the following holds: for all $(x, a) \in \cX \times \cA$,
	\begin{enumerate}
		\item[(C1)] $P(x | x_1, a) = P(x | x_2, a)$,
		
		\item[(C2)] $P^{bwd}(x, a | x_1) = P^{bwd}(x, a | x_2)$.
	\end{enumerate}
	If one of these conditions fails, then we say that $x_1, x_2$ are {\bf kinematically separable}.
\end{definition}
The above definition can be equivalently rewritten using latent transition and emission probabilities:
\begin{lemma}
\label{lem:kinematic}
	The above conditions for $x_1, x_2$ being kinematically inseparable are equivalent to the following: for all $(s, a) \in \cS \times \cA$,
	\begin{enumerate}
		\item[(C1')] $p(s | f(x_1), a) = p(s | f(x_2), a)$,
		
		\item[(C2')] $\frac{p(f(x_1) | s, a)}{\sum_{x', a'} p(f(x_1) | f(x'), a') u(x', a')} = \frac{p(f(x_2) | s, a)}{\sum_{x', a'} p(f(x_2) | f(x'), a') u(x', a')}$.
	\end{enumerate}
\end{lemma}
\begin{proof}
	Follows from straightforward computations.
\end{proof}

\cite{misra2020kinematic} then extended the notion of kinematic separability, which is defined between two contexts, to the whole BMDP:
\begin{definition}[Definition 4 of \cite{misra2020kinematic}]
\label{def:canonical}
	A BMDP $\Phi$ is in {\bf canonical form} if for any $x_1, x_2 \in \cX$ the following holds: $f(x_1) = f(x_2)$ if and only if $x_1$ and $x_2$ are kinematically inseparable.
	If this does not hold, then $\Phi$ is said to be not in canonical form.
\end{definition}

Finally, the following proposition shows that our notion of separability ($I(\Phi) > 0$) is {\it weaker} (in that it encompasses a broader range of BMDPs) than BMDP being in canonical form:
\begin{proposition}
\label{prop:kinematic2}
	If the BMDP $\Phi$ is in canonical form, then $I(\Phi) > 0$.
\end{proposition}
\begin{proof}
	Again, we prove the contrapositive.
	Assume that $I(\Phi) = 0$.
	From Proposition \ref{prop:I=0}, there exists some $j \in \cS$ with $j \neq f(x)$ and $c > 0$ such that
	\begin{equation*}
		p(s | f(x), a) = p(s | j, a) \text{ and } p(f(x) | s, a) = c p(j | s, a), \quad \forall (s, a) \in \cS \times \cA.
	\end{equation*}
	First part implies (C1').
	For the second part, WLOG fix some $u \in \cP(\cS \times \cA)$ with $\supp(u) = \cS \times \cA$.
	Then, for arbitrary $(s, a) \in \cS \times \cA$,
	\begin{equation*}
		\frac{p(f(x) | s, a)}{\sum_{x', a'} p(f(x) | f(x'), a') u(x', a')} = \frac{c p(j | s, a)}{\sum_{x', a'} c p(j | f(x'), a') u(x', a')}
		= \frac{p(j | s, a)}{\sum_{x', a'} p(j | f(x'), a') u(x', a')},
	\end{equation*}
	which precisely implies (C2').
	By Lemma \ref{lem:kinematic}, we have that $\Phi$ is not in canonical form.
\end{proof}

\newpage
% !TEX root = ./main.tex

\section{Bernstein-type Concentration for Markov Chains with Restarts and Applications}\label{app:concentration}

In this appendix, we present concentration results for Markov chains {\it with restarts}. These results will be crucial in the analysis of the performance of our algorithms.

\subsection{Concentration of Markov chains with restarts}

Consider $T$ i.i.d. episodes of BMDP of length $H$ generated under the behavior policy. Each episode corresponds to the trajectory of length $H$ of a Markov chain. We are interested in deriving concentration results for these trajectories for any $T$ and $H$. In particular, since $H$ can possibly be limited, we cannot use either the long-term properties of the Markov chain or the assumption that the chain starts in its steady-state regime.

Our concentration results differ from most prior Bernstein-type concentration bounds for Markov chains (refer to \cite{Paulin15} and references therein). Indeed existing bounds hold for a single trajectory and assume that the Markov chain starts from its stationary distribution. Considering restarts from an arbitrary distribution induces a bias term that vanishes with the number of observations (as $H$ grows large), and our results account for this bias.

%However this bias term affect the problem dependent constants in the sample complexity and results in sub-optimal dependencies. Our proofs relies on elementary calculations that completely remove the dependence on the starting distribution with the caveat of a sub-optimal dependence on the mixing time, which happens to be irrelevant in our analysis (we care about $n$ not $\eta$).

\begin{theorem}\label{thm:bernstein-restart}
Let $\{ (X_{h}^{(t)})_{h = 1}^H \}_{t \in [T]}$ be a collection of {\it i.i.d.} possibly time-inhomogeneous Markov chains over a finite state space $\cZ$, with transition probability matrices $\{P_h\}_{h \geq 1}$ and initial distribution $\mu \in \cP(\cZ)$. We assume that $\mu$ and $P_h$'s are $\eta$-regular (see Appendix \ref{app:equilibrium}), and that each $P_h$ admits a stationary distribution $\nu_h$.
Let $\left\{ \phi_h : \cX \rightarrow \RR \right\}$ be a collection of {\it bounded} measurable real-valued functions. Then we have that for all $u \geq 0$,
	\begin{align}
		\label{eqn:bernstein}
		\PP\left[ \sum_{t=1}^T \sum_{h=1}^H  \phi_h(X_{h}^{(t)}) - \EE_{\mu}[\phi_h(X_{h}^{(t)})] > u \right] & \le \exp\left(  -\frac{u^2}{ 2 T H V_{\mu, P, \phi} + \frac{2}{3} M_{P, \phi} u} \right),
	\end{align}
	where the variance and maximum deviation terms are defined as follows:
	\begin{align}
		V_{\mu,P, \phi} & := (1 + \sqrt{2}\eta(2\eta - 1))^2  \max_{z \in \cZ} \max_{1 \le \ell \le h \le H} \Var_{P_{\ell-1}(z, \cdot)}[\phi_h], \\
		M_{P, \phi} & := (2\eta - 1) \max_{h \in [H]} \Vert \phi_h \Vert_\infty,
	\end{align}
and where by convention $P_{0}(z, \cdot)=\mu(\cdot)$.
\end{theorem}

\subsection{Proof of Theorem \ref{thm:bernstein-restart}}

All the supporting lemmas are presented and proved in Section \ref{sec:bernstein-lemmas}. In this proof, we use the following notations. We define $\Pi_h := \bm{1}\nu_h$ (where $\bm{1}$ is a column vector with coordinates equal to 1 and $\nu_h$ is a row vector). Denote $\EE_\nu[\phi] := \EE_{X \sim \nu}[\phi(X)]$ and $\EE_\nu[\phi(X_h)] := \EE_{X_1 \sim \nu}[\phi(X_h)]$, and similarly for $\Var$.

Without loss of generality, we assume that $\nu_h\phi_h := \EE_{X \sim \nu_h}[\phi_h(X)] = 0$ for all $h \in [H]$. Indeed, if $\nu_h\phi_h \neq 0$, then we can write $\phi_h(X_{h}^{(t)})) - \EE_{\mu}[\phi_h(X_{h}^{(t)})] = (\phi(X_{h}^{(t)})) - \nu_h \phi_h) - \EE_{\mu}[(\phi_h(X_{h}^{(t)}) - \nu_h \phi_h]$, then continue the proof using $\tilde{\phi}_h: x \mapsto \phi_h(x) - \nu_h \phi_h$ instead.

We start by obtaining an upper bound on the moment generating function of $\sum_{t=1}^T \sum_{h=1}^H \phi_h(X_{h}^{(t)}) - \EE_{\mu}[\phi_h(X_{h}^{(t)})]$.
Using the fact that the trajectories are independent across episodes $t \in [T]$ and applying Lemma \ref{lem:mgf-nsmc}, we immediately obtain, for all $\lambda \geq 0$,
\begin{align*}
	&\EE_{\mu} \left[\exp\left( \lambda  \left( \sum_{t=1}^T \sum_{h=1}^H \phi_h(X_{h}^{(t)}) - \EE_{\mu}[\phi_h(X_{h}^{(t)})] \right)  \right)\right] \cr
	& \quad = \prod_{t=1}^T \EE_{\mu}\left[ \exp\left( \lambda \left( \sum_{h=1}^H \phi_h(X_{h}^{(t)}) - \EE_{\mu}[\phi_h(X_{h}^{(t)})] \right)  \right) \right] \le \exp\left(  \frac{T HV_{\mu, P, \phi}}{M^2_{P,\phi}} \varphi(\lambda M_{P, \phi})  \right),
\end{align*}
where we recall that $\varphi(x) = e^{x} - x - 1$, and where definitions of $V_{\mu, P, \phi}$ and $M_{P, \phi}$ are given in Lemmas \ref{lem:mgf-nsmc}. By Markov inequality, we have: for any $u \geq 0$,
\begin{align*}
	\PP\left( \sum_{t=1}^T \sum_{h=1}^H \phi_h(X_{h}^{(t)})) - \EE_{\mu}[\phi_h(X_{h}^{(t)})] > u \right) & \le \inf_{\lambda \geq 0 } \exp\left( \frac{T HV_{\mu, P, \phi}}{M^2_{P,\phi}} \varphi(\lambda M_{P, \phi}) - \lambda u \right) \\
	& \le \exp\left( - \frac{TH V_{\mu, P, \phi}}{M_{P, \phi}^2} \varphi^\star\left( \frac{M_{P, \phi} u}{THV_{\mu, P, \phi}}\right)\right) \\
	& \le \exp\left( -\frac{u^2}{2THV_{\mu, P, \phi} + \frac{2}{3} M_{P, \phi} u }\right),
\end{align*}
where we introduced $\varphi^\star(y) = (1+y)\log(1+y) - y$, the Fenchel dual\footnote{Recall its definition: $\varphi^\star(y) := \sup_{x} y x - \varphi(x)$.} of $\varphi$.
The last inequality then follows from $\varphi^\star(y) \ge \frac{y^2}{2 + \frac{2}{3} y}$ (see \cite{boucheron2013concentration} for the proof of this elementary inequality). The result then follows from the upper bounds of $V_{\mu, P, \phi}$ and $M_{P, \phi}$ derived in Lemma \ref{lem:nsmc-var-dev}. \ep

\subsection{Towards concentration inequalities in BMDPs}

Next, we specify the results of Theorem \ref{thm:bernstein-restart} to the case of homogenous Markov chains. This will be instrumental in our analysis since indeed the Markov chains induced in BMDPs are homogenous. The results resemble the concentration results established by \cite{SandersPY20} (see Proposition 10 of their supplementary material SM1) and the subsequent improvements established by \cite{Sanders21}, but there are several key differences. One is that we keep track of the asymptotics in $S$ and $A$, and another is that we consider restarts and have to deal with the absence of equilibrium assumption.

The next theorem is a direct application of Theorem \ref{thm:bernstein-restart} to homogenous Markov chains (e.g., BMDPs).

\begin{theorem}
\label{thm:bernstein-restart-bmdp}
Let $\{ (X_{h}^{(t)})_{h = 1}^H \}_{t \in [T]}$ be a collection of {\it i.i.d.} time-homogeneous Markov chains over a finite state space $\cZ$, with transition probability matrix $P$ and initial distribution $\mu \in \cP(\cZ)$. We assume that $\mu$ and $P$ are $\eta$-regular, and that $P$ admits a stationary distribution $\nu$. Let $\phi : \cX \rightarrow \RR$ be a bounded measurable real-valued function. Then we have that for all $u \geq 0$,
	\begin{align}
		\label{eqn:bernstein-bmdp}
		\PP\left[ \sum_{t=1}^T \sum_{h=1}^H  \phi(X_{h}^{(t)}) - \EE_{\mu}[\phi(X_{h}^{(t)})] > u \right] & \le \exp\left(  -\frac{u^2}{ 2 T H V_{\mu, P, \phi} + \frac{2}{3} M_{P, \phi} u } \right),
	\end{align}
	where the variance and maximum deviation terms are defined as follows:
	\begin{align}
		V_{\mu,P, \phi} & := (1 + \sqrt{2}\eta(2\eta - 1))^2 \max \left( \Var_{\mu}[\phi], \max_{z \in \cZ} \Var_{P(z, \cdot)}[\phi] \right), \\
		M_{P, \phi} & := (2\eta - 1) \Vert \phi \Vert_\infty.
	\end{align}
\end{theorem}

We can further simplify the bound depending on the choice of $\rho$:
\begin{itemize}
	\item For any $u$ satisfying $\lVert \phi \rVert_\infty u = o\left( TH V_{\mu, P, \phi} \right)$,
	\begin{align*}
		\PP\left( \left| \sum_{t=1}^T \sum_{h=1}^H  \phi(X_{h}^{(t)})) - 		\EE_{\mu}[\phi(X_{h}^{(t)})] \right| > u \right) & \le 2 \exp\left(  -\frac{u^2}{ 2 T H V_{\mu, P, \phi} } \right)
	\end{align*}

	\item For any $u$ satisfying $\lVert \phi \rVert_\infty u = \omega\left( TH V_{\mu, P, \phi} \right)$,
	\begin{align*}
		\PP\left( \left| \sum_{t=1}^T \sum_{h=1}^H  \phi(X_{h}^{(t)}) - 	\EE_{\mu}[\phi(X_{h}^{(t)})] \right| > u \right) & \le 2 \exp\left(  -\frac{u}{ \frac{2}{3} M_{P, \phi} } \right)
	\end{align*}
\end{itemize}

\subsection{Supporting lemmas for Theorem~\ref{thm:bernstein-restart}}\label{sec:bernstein-lemmas}
Let $(X_h)_{h = 1}^H$ be a (fixed) Markov chain over a finite state space $\cZ$ with transition probability matrices $\{P_h\}_{h \geq 1}$. We further assume that each $P_h$ admits a stationary distribution, denoted by $\nu_h$, and let $\Pi_h := \bm{1}\nu_h$.
We denote by $P_0 := \mu \in \cP(\cZ)$ its initial distribution ($X_1 \sim \mu$). We assume that $\mu$ and $P_h$'s are $\eta$-regular. Let
$\left\{ \phi_h : \cX \rightarrow \RR \right\}$ be a collection of {\it bounded} measurable real-valued functions.
Lastly, we denote $\EE_\nu[\phi] := \EE_{X \sim \nu}[\phi(X)]$ and $\EE_\nu[\phi(X_h)] = \EE_{X_1 \sim \nu}[\phi(X_h)]$, and similarly for $\Var$, and denote $\Vert \phi\Vert_\nu := \sqrt{\EE_\nu[\phi(X)^2]}$.
Furthermore, we regard any probability measures over $\cZ$ as a $|\cZ|$-dimensional row vector and all real-valued measurable functions over $\cZ$ as a $|\cZ|$-dimensional column vector\footnote{With this notation, note that for $\phi : \cZ \rightarrow \RR$, $\EE_\mu[\phi] = \mu\phi$.}.
%
% Additionally, following the notations of \cite{Levin17}, for any $\nu \in \cP(\cX)$, we will denote $\ell_2(\nu)$ the set of measurable functions $\phi$ such that $\EE_\nu[\phi(X)^2] < \infty$. We further equip $\ell_2(\nu)$ with the scalar product $\langle \phi, \psi\rangle_\nu = \EE_\mu[\phi(X) \psi(X)]$ for all $\psi, \phi \in \ell_2(\nu)$. With this, we also define the $\nu$-norm $\Vert \phi \Vert_\mu = \sqrt{\langle \phi, \phi \rangle_\nu}$ for all $\phi \in \ell_2(\nu)$, and $\nu$-operator norm $\Vert P \Vert_\mu = \sup_{\phi \in \ell_2(\nu): \Vert \phi\Vert_\nu = 1} \Vert P \phi \Vert$ for any linear operator on $P$ acting on $\ell_2(\nu)$.
%
%  for any $\nu \in \cP(\cZ)$,
% \begin{itemize}
% 	\item (Expectation) $\mu \phi = \EE_{X \sim \mu}[\phi(X)]$
% 	\item (Variance) $\Var_\mu[\phi] = \EE_{X \sim \mu}[ \left(\phi(X) - \EE_{X \sim \mu}[ \phi(X)] \right)^2 ]$
% 	\item ($\mu$-square-summable functions) $\phi \in \mathcal{L}_2(\mu)$ if and only if $\EE_{X \sim \mu} [ \phi(X)^2]  = \mu\phi^2 < \infty$
% 	\item ($\mu$-scalar-product) let $\phi, \psi \in \ell_2(\mu)$, we define $\langle \phi, \psi\rangle_\mu  = \EE_{X \sim \mu} [\psi(X) \phi(X)]$
% 	\item ($\mu$-norm) we have $\Vert \phi \Vert_{\mu} = \sqrt{\langle \phi, \phi \rangle_\mu}$
% 	\item ($\mu$-operator norm) let $P$ be some linear operator on $\ell_2(\mu)$, we define $\Vert P \Vert_\mu = \sup_{\phi: \Vert \phi \Vert_\mu = 1} \Vert P \phi\Vert_\mu $
% 	\item (infinity norm) we have $\Vert \cdot \Vert_\infty$
% \end{itemize}

\begin{lemma}\label{lem:mgf-nsmc}
	For all $\lambda> 0$, we have
	\begin{align*}
		\EE_\mu\left[\exp(\lambda \left( \sum_{h=1}^H \phi_h(X_h) - \EE_\mu[ \phi_h(X_h)] \right) )\right] \le \exp( \frac{H V_{\mu, P, \phi}}{ M_{P, \phi}^2} \varphi( \lambda M_{P, \phi} )),
	\end{align*}
	where $\varphi: x \mapsto e^x - x - 1$. The variance term and maximum deviation term are defined, respectively, as follows:
	\begin{align}
		V_{\mu, P, \phi} & := \max_{z \in \cZ, \ell \in [H]} \Var_{P_{\ell-1}(z, \cdot)}\left[ \sum_{h = \ell}^H \left(\prod_{i = \ell}^{h-1}  (P_i - \Pi_i)\right)\phi_h\right], \\
		M_{P, \phi} & :=  \max_{\ell \in [H]} \left\Vert \sum_{h = \ell}^H \left(\prod_{i = \ell}^{h-1}  (P_i - \Pi_i)\right)\phi_h  \right\Vert_\infty.
	\end{align}
	% \begin{align*}
		%   V & \lesssim \eta^2  \max_{h \in [H]} \Var_{P_{h-1}(x, \cdot)}[ \phi_h] \\
		%   M & \lesssim \eta \max_{h \in [H]} \Vert \phi_h \Vert_{\infty}
		% \end{align*}
\end{lemma}

\begin{proof} For notational convenience, let us introduce, for all $h \ge 1$, the $\vert \cZ\vert$-dimensional random {\it row} vector $\delta_h = (\indicator \lbrace X_h = z\rbrace)_{z \in \cZ}$.
	Then observe that we may write, via a telescoping sum, for all $h \ge 1$,
	\begin{align*}
		\phi_h(X_h) - & \EE_\mu[ \phi_h(X_h) ]  \\
		& = \delta_1 \left(\prod_{i = 1}^{h-1} P_i \right)\phi_h - \mu \left(\prod_{i = 1}^{h-1} P_i \right) \phi_h  + \sum_{\ell = 1}^{h-1} \delta_{\ell+ 1} \left(\prod_{i = \ell +1}^{h-1} P_i\right) \phi_h - \delta_{\ell} \left(\prod_{i = \ell}^{h-1} P_i \right) \phi_h \\
		& = (\delta_1 - \mu) \left(\prod_{i=1}^{h-1} P_{i} \right) \phi_h + \sum_{\ell =1}^{h-1} \left(\delta_{\ell + 1} - \delta_\ell P_{\ell} \right) \left( \prod_{i= \ell + 1}^{h-1} P_i \right) \phi_h.
	\end{align*}
	To further ease notations, we introduce $Z_1 = \delta_1 - \mu$ and $Z_\ell = \delta_\ell - \delta_{\ell - 1} P_{\ell-1}$ for $\ell \ge 2$, and $P_{\ell \to h} = \prod_{i = \ell}^{h-1} P_i$, with the convention that $P_{\ell \to \ell} = I$.
	With that, we may write $\phi_h(X_h) - \EE_\mu[ \phi_h(X_h)] = \sum_{\ell =1}^h Z_\ell P_{\ell \to h} \phi_h$.
	One important observation is that $Z_\ell {\bm 1} = 0$.

	On the other hand, from the fact that $(P_i - \Pi_i) \Pi_j = 0$ for all $i, j\ge 1$, we have that
	\begin{equation}
	\label{eq:P-Delta}
		P_{\ell \to h} =  \prod_{i = \ell}^{h-1} P_i
		= \prod_{i = \ell}^{h-1} (P_i - \Pi_i) + \Pi_\ell \prod_{j=\ell+1}^{h-1} P_j
		= \prod_{i = \ell}^{h-1} (P_i - \Pi_i) + \Pi_\ell P_{\ell+1 \rightarrow h}.
	\end{equation}

	Since $\Pi_j = {\bm 1}\nu_j$, we have that $\Pi_\ell P_{\ell+1 \rightarrow h} = \mathbf{1} \xi$ for some $\xi \in \RR^{1 \times \vert \cZ \vert}$, i.e.,
	\begin{align*}
		 Z_\ell \left( \Pi_\ell P_{\ell+1 \rightarrow h} \right) \phi_h = Z_\ell \mathbf{1} \xi \phi_h = 0, \quad \forall h > \ell \ge 1.
	\end{align*}

	Thus, introducing $\Delta_{\ell \to h} = \prod_{i = \ell}^{h-1} (P_i - \Pi_i)$, with the convention that $\Delta_{\ell \to \ell} = I$, we may finally write
	\begin{align}\label{eq:mgf-nsmc-1}
		\sum_{h=1}^H \phi_h(X_h) - \EE_\mu[\phi_h(X_h)] & = \sum_{h=1}^H \sum_{\ell = 1}^h Z_\ell \Delta_{\ell \to h} \phi_h = \sum_{\ell = 1}^H Z_\ell \left(  \sum_{h=\ell}^{H} \Delta_{\ell \to h} \phi_h \right).
	\end{align}
	Now, we are ready to upper bound the moment-generating function of the LHS.
	The convenience of Eqn. \eqref{eq:mgf-nsmc-1} is that $(Z_h)_{h \ge 1}$ is adapted to the filtration generated by the Markov chain $(X_h)_{h \ge 1}$, $\EE_\mu[Z_h \vert X_{h-1}] = 0$ for all $h \ge 2$, and $\EE_\mu[Z_1] = 0$. Using a standard Bernstein type upper bound technique~\citep{boucheron2013concentration}, we have for all $\ell \ge 2$, $\lambda > 0$,
	\begin{align*}
		\EE_\mu\left[ \exp(\lambda Z_\ell \left(  \sum_{h=\ell}^{H} \Delta_{\ell \to h} \phi_h \right)  ) \bigg \vert X_{\ell - 1} \right] \le \exp( \frac{\Var_{P_{\ell-1}(X_{\ell -1}, \cdot)}\left[ \sum_{h = \ell}^H \Delta_{\ell \to h} \phi_h \right] }{ C^2 } \varphi(\lambda C)),
	\end{align*}
	where $\varphi: x \mapsto e^x - x - 1$, and $C$ can be any positive constant that verifies $ \Vert \sum_{h=\ell}^{H} \Delta_{\ell \to h} f_h \Vert_\infty \le C$. Introducing $V_{\mu, P, \phi}$ and $M_{P, \phi}$ as defined in the statement of this Lemma, we see that for all $\ell \ge 2$,
	\begin{align*}
		\EE_\mu\left[ \exp(\lambda Z_\ell \left(  \sum_{h=\ell}^{H} \Delta_{\ell \to h} \phi_h \right) ) \bigg \vert X_{\ell - 1} \right] \le \exp( \frac{V_{\mu, P, \phi}}{ M_{\mu, P, \phi}^2} \varphi(\lambda M_{\mu, P, \phi})).
	\end{align*}
	Similarly, we also have
	\begin{align*}
		\EE_\mu\left[ \exp(\lambda Z_1 \left( \sum_{h=1}^{H} \Delta_{\ell \to h} \phi_h \right)) \right] \le \exp( \frac{V_{\mu, P, \phi}}{ M_{\mu, P, \phi}^2} \varphi(\lambda M_{\mu, P, \phi} ) ).
	\end{align*}
	Now using the fact that $(Z_h)_{h \ge 1}$ is adapted to the filtration generated by the Markov chain $(X_h)_{h \ge 1}$, we obtain via a peeling argument that
	\begin{align*}
		\EE_\mu\left[\exp(\lambda \left( \sum_{h=1}^H \phi_h(X_h) - \EE_\Phi[ \phi_h(X_h)] \right) )\right] \le \exp( \frac{H V_{\mu, P, \phi}}{ M_{\mu, P, \phi}^2} \varphi( \lambda M_{\mu, P, \phi} )), \quad \forall \lambda > 0.
	\end{align*}
\end{proof}

The following lemma provides simple bounds for $V_{\mu,P,\phi}$ and $M_{P,\phi}$.
\begin{lemma}\label{lem:nsmc-var-dev}
	For all $\ell \in [H]$ and all $z \in \cZ$,
	\begin{align}\label{eq:nsmc-var}
		\Var_{P_{\ell-1}(z, \cdot)} \left[ \sum_{h= \ell}^H \left(\prod_{i = \ell}^{h-1} (P_{i} - \Pi_{i}) \right) \phi_h \right] \le (1 + \sqrt{2}\eta(2\eta - 1))^2 \max_{\ell \le h \le H} \Var_{P_{\ell-1}(z, \cdot)}[\phi_h]
	\end{align}
	and
	\begin{align}\label{eq:nsmc-dev}
		\left\Vert  \sum_{h= \ell}^H \left(\prod_{i = \ell}^{h-1} (P_{i} - \Pi_{i}) \right) \phi_h \right\Vert_\infty \le  (2\eta - 1)\max_{\ell \le h \le H} \Vert \phi_h \Vert_\infty.
	\end{align}
\end{lemma}

\begin{proof}
	Recall that $\Delta_{\ell \to h} = \prod_{i = \ell}^{h-1} (P_i - \Pi_i)$ for $h > \ell \ge 1$, with the convention that $\Delta_{\ell \to \ell} = I$. Finally, remember that $P_{0}(z, \cdot):=\mu(\cdot)$.

	\underline{\textbf{\textit{Proof of Eqn. \eqref{eq:nsmc-var} (Variance term)}}}\\
	\begin{align*}
		\Var_{\mu} \left[ \sum_{h= \ell}^H \Delta_{\ell \to h} \phi_h \right] & =  \Var_{\mu} \left[ \sum_{h= \ell}^H \Delta_{\ell \to h}  (\phi_h  - \mu \phi_h \mathbf{1}) \right]\\
		& \le  \EE_{X \sim \mu}\left[ \left(\left(\sum_{h= \ell}^H \Delta_{\ell \to h}  (\phi_h  - \mu \phi_h \mathbf{1})\right) (X)  \right)^2 \right] \\
		& = \left\Vert \sum_{h= \ell}^H \Delta_{\ell \to h}  (\phi_h  - \mu \phi_h \mathbf{1}) \right\Vert_\mu^2 \\
		& \le \left(\sum_{h=\ell}^H \Vert \Delta_{\ell \to h} (  \phi_h - \mu \phi_h \mathbf{1}) \Vert_\mu  \right)^2
	\end{align*}
	Now, we note that for $h > \ell \ge 1$, we have:
	\begin{align*}
		\Vert \Delta_{\ell \to h} (\phi_h - \mu \phi_h \mathbf{1})\Vert_\mu^2 &\le \Vert \Delta_{\ell \to h} (\phi_h - \mu \phi_h \mathbf{1}) \Vert^2_\infty \\
		& = \max_{x \in \cZ} \left\vert \sum_{y} \Delta_{\ell \to h}(x, y) (\phi_h - \mu \phi_h \mathbf{1})(y) \right\vert^2 \\
		& \le \max_{x \in \cZ} \left\vert \sum_{y} \frac{\Delta_{\ell \to h}(x,y)^2}{\mu(y)} \right\vert  \Vert \phi_h - \mu \phi_h \mathbf{1} \Vert_\mu^2 \\
		& \le \left( \max_{x, y \in \cZ} \left \vert \frac{\Delta_{\ell \to h}(x,y)}{\mu(y)}\right \vert \right) \left( \max_{x \in \cZ} \sum_{y \in \cZ} \vert \Delta_{\ell \to h}(x,y) \vert \right) \Vert \phi_h - \mu \phi_h \mathbf{1} \Vert_\mu^2  \\
		& \le \left( \max_{x, y \in \cZ} \left \vert \frac{\Delta_{\ell \to h}(x,y)}{\mu(y)}\right \vert \right) \left( \max_{g \in \RR^{|\cZ|}: \lVert g \rVert_\infty \leq 1} \lVert \Delta_{\ell \rightarrow h} g \rVert_\infty \right)  \Vert \phi_h - \mu \phi_h \mathbf{1} \Vert_\mu^2 \\
		& \le 2 \left( \max_{x, y \in \cZ} \left \vert \frac{\Delta_{\ell \to h}(x,y)}{\mu(y)}\right \vert \right) \left( 1 - \frac{1}{\eta} \right)^{h-\ell}  \Vert \phi_h - \mu \phi_h \mathbf{1} \Vert_\mu^2,
	\end{align*}
where the last inequality follows from Lemma \ref{lem:nsmix}.

As $\mu$ and $(P_h)_{h \ge 1}$ satisfy the $\eta$-regularity property, so does $P_{\ell \to h}(x, \cdot)$ and $\Pi_\ell P_{\ell+1 \rightarrow h}$ with the same parameter $\eta$, for all $x \in \cZ$ (this can be easily proven using induction on the number of multiplicands $h - \ell$). Thus, from Eqn. \eqref{eq:P-Delta}, $\Delta_{\ell \to h}(x,\cdot) = P_{\ell \to h}(x, \cdot) - \Pi_\ell P_{\ell+1 \rightarrow h}$ is also $\eta$-regular, which implies that
	\begin{align*}
		\max_{x, y \in \cZ} \left \vert \frac{\Delta_{\ell \to h}(x,y)  }{\mu(y)}\right \vert \le \eta \vert \cZ\vert\left(\frac{\eta}{\vert\cZ\vert} - \frac{1}{\eta \vert \cZ \vert} \right) = \eta^2 - 1,
	\end{align*}
	which in turn gives
	\begin{equation*}
		\Vert \Delta_{\ell \to h} (\phi_h - \mu \phi_h \mathbf{1})\Vert_\mu^2 \le 2 (\eta^2 - 1)\left(1 - \frac{1}{\eta} \right)^{h-\ell} \Vert \phi_h - \mu \phi_h \mathbf{1} \Vert_\mu^2.
	\end{equation*}

Noting that $\Vert \phi_h - \mu \phi_h \mathbf{1} \Vert_\mu^2 = \Var_{\mu}[\phi_h]$, we finally obtain
	\begin{align*}
		\Var_{\mu} \left[ \sum_{h= \ell}^H \Delta_{\ell \to h}  \phi_h \right] &\le \left(\sum_{h=\ell}^H \Vert \Delta_{\ell \to h} (  \phi_h - \mu \phi_h \mathbf{1}) \Vert_\mu \right)^2 \\
		&\leq \left( \sqrt{\Var_\mu[\phi_\ell]} + \sum_{h=\ell+1}^H \sqrt{2 (\eta^2 - 1)\left(1 - \frac{1}{\eta} \right)^{h-\ell}} \sqrt{\Var_\mu[\phi_h]} \right)^2 \\
		&\leq \max_{\ell \le h \le H} \Var_{\mu}[\phi_h] \left( 1 + \sqrt{2(\eta^2 - 1)} \sum_{h = \ell + 1}^H \left(1 - \frac{1}{\eta}\right)^{(h - \ell)/2} \right)^2 \\
		& \le (1 + \sqrt{2}\eta(2\eta - 1))^2 \max_{\ell \le h \le H} \Var_{\mu}[\phi_h].
	\end{align*}

	\underline{\textbf{\textit{Proof of Eqn. \eqref{eq:nsmc-dev} (Maximum deviation term)}}}\\
	\begin{align*}
		\left\Vert  \sum_{h= \ell}^H \Delta_{\ell \to h} \phi_h  \right\Vert_\infty & \le \sum_{h = \ell}^H \left\Vert  \Delta_{\ell \to h} \phi_h  \right\Vert_\infty  \\
		& = \Vert \phi_\ell \Vert_\infty + \sum_{ h = \ell + 1}^{H} \left\Vert \Delta_{\ell \to h} \phi_h  \right\Vert_\infty \\
		& \le \Vert \phi_\ell \Vert_\infty + 2 \sum_{ h = \ell + 1}^{H} \left( 1 - \frac{1}{\eta}\right)^{h - \ell} \Vert \phi_h \Vert_\infty \\
		& \le (1 + 2(\eta - 1)) \max_{h \in [H]} \Vert \phi_h  \Vert_\infty \\
		&= (2\eta - 1) \max_{h \in [H]} \Vert \phi_h  \Vert_\infty,
	\end{align*}
	where we used the triangle inequality in the first inequality and Lemma \ref{lem:nsmix} in the third inequality.
\end{proof}

The following lemma is used above in the proof of Lemma \ref{lem:nsmc-var-dev}.
\begin{lemma}\label{lem:nsmix}
	Let $(P_h)_{h=1}^{H-1}$ be transition matrices (i.e. row-stochastic) over $\cZ$, with each $P_h$ having a stationary distribution $\nu_h$.
	Let $g \in \RR^\cZ$ be such that $\Vert g \Vert_\infty < \infty$ and let us denote $\Pi_h := {\bm 1} \nu_h$.
	Then
	\begin{align*}
		\left\Vert \prod_{h=1}^{H-1} (P_i - \Pi_i) g \right\Vert_\infty  \le 2 \left(\prod_{h=1}^{H-1} \delta(P_h) \right) \Vert g \Vert_\infty,
	\end{align*}
	where $\delta(P_h) = \max_{x, y \in \cZ} d_{TV}(P_h(x, \cdot), P_h(y, \cdot))$ is the Dobrushin's coefficient of $P_h$.
	In particular, if we assume that $P_h$'s are $\eta$-regular, it follows that
	\begin{align*}
		\left\Vert \prod_{h=1}^{H-1} (P_i - \Pi_i) g \right\Vert_\infty  \le 2 \left(1- \frac{1}{\eta} \right)^{H-1} \Vert g \Vert_\infty.
	\end{align*}
\end{lemma}
\begin{proof} First, we observe that
	\begin{align*}
		\left\Vert \prod_{h=1}^{H-1} (P_i - \Pi_i) f \right\Vert_\infty & \overset{(a)}{=} \left\Vert \left(\prod_{h=1}^{H-1}P_i - \Pi_1 \prod_{i =2}^{H-1} P_i\right) g \right\Vert_\infty \\
		&  = \max_{x \in \cZ} \left\vert \sum_{y \in \cZ}  \left(\prod_{h=1}^{H-1}P_i \right)(x, y) g(y) -  \sum_{y \in \cZ} \left(\Pi_1 \prod_{h=2}^{H-1}P_i \right)(x, y) g(y) \right\vert \\
		& \le 2 \max_{x \in \cZ} d_{TV}\left( \left(\prod_{h=1}^{H-1}P_i \right)(x, \cdot), \left(\Pi_1 \prod_{h=2}^{H-1}P_i \right)(x, \cdot)  \right) \Vert g \Vert_\infty
	\end{align*}
	where $(a)$ follows from the fact that $ (P_i - \Pi_i) \Pi_j = 0$ for all $i, j\ge 1$.
	Recall that $P_{\ell \to h} = \prod_{i = \ell}^{h-1} P_i$ $h \ge \ell \ge 1$, with the convention that $P_{\ell \to \ell} = I$.
	Then we may simply write
	\begin{align}\label{eq:nsmix-0}
		\left\Vert \prod_{h=1}^{H-1} (P_i - \Pi_i) g \right\Vert_\infty \le 2   \max_{x \in \cZ} d_{TV}\left(  P_{1 \to H} (x, \cdot), \Pi_1 P_{2 \to H} (x,\cdot) \right) \Vert g \Vert_\infty.
	\end{align}
Next, we formulate two claims to complete the proof:
	\begin{claim}\label{claim:nsmix-1}
		\begin{align}
			\max_{x \in \cZ} d_{TV}\left(  P_{1 \to H} (x, \cdot), \Pi_1 P_{2 \to H} (x,\cdot) \right) \le \delta(P_{1 \rightarrow H}).
%			\max_{x, y \in \cZ} d_{TV}\left(  P_{1 \to H}  (x, \cdot),   P_{1 \to H} (y,\cdot) \right),
		\end{align}
	\end{claim}
	\begin{claim}\label{claim:nsmix-2}
	\begin{align}
		\delta(P_{\ell \to h}) \le \delta(P_{\ell \to k}) \delta(P_{k \to h}), \quad \forall \ h \ge k \ge \ell.
	\end{align}
	\end{claim}
	Here, as stated in the Lemma, $\delta(P) := \max_{x,y \in \cZ}d_{TV}(P(x,\cdot), P(y, \cdot))$ is the Dobrushin's coefficient.

Assuming that the two above claims hold, from \eqref{eq:nsmix-0}, we have:
	\begin{align*}
		\left\Vert \prod_{h=1}^{H-1} (P_i - \Pi_i) g \right\Vert_\infty & \le 2   \max_{x \in \cZ} d_{TV}\left(  P_{1 \to H} (x, \cdot), \Pi_1 P_{2 \to H} (x,\cdot) \right) \Vert g \Vert_\infty \\
		& \le 2 \delta(P_{1 \to H}) \Vert f\Vert_\infty \\
		& \le 2 \left( \delta(P_{1 \to 2}) \times \delta(P_{2 \to 3}) \times \hdots \times \delta(P_{H-1 \to H}) \right) \lVert g \rVert
		= 2 \left(\prod_{h=1}^{H-1} \delta(P_h) \right) \Vert g \Vert_\infty,
	\end{align*}
	where we noted that $P_{h \to h+1} = P_h$.

Finally, we provide proof of the above two claims. Both are simple adaptations of the proofs of Lemma 4.10 and Lemma 4.11 of \cite{Levin17}, respectively. We provide these proofs for completeness.

	\underline{\textit{\textbf{Proof of Claim \ref{claim:nsmix-1}}}}

	By definition,
	\begin{align*}
		& \max_{x \in \cZ} d_{TV}\left(  P_{1 \to H} (x, \cdot), \Pi_1 P_{2 \to H} (x,\cdot) \right) = \max_{x \in \cZ} \max_{A \subseteq \cZ} \vert P_{1 \to H} (x, A) - \Pi_1 P_{2 \to H} (x, A) \vert \\
		& \quad\quad =  \max_{x \in \cZ} \max_{A \subseteq \cZ} \vert P_{1 \to H} (x, A) - \Pi_1 P_{1 \to H} (x, A)  \vert \quad (\Pi_1 P_1 = \Pi_1) \\
		& \quad\quad = \max_{x \in \cZ} \max_{A \subseteq \cZ} \left\vert P_{1 \to H} (x, A) -  \sum_{y \in \cZ} \nu_1(y) P_{1 \to H}(y, A) \right\vert \\
		& \quad\quad = \max_{x \in \cZ} \max_{A \subseteq \cZ} \left\vert \sum_{y \in \cZ} \nu_1(y) \left(P_{1 \to H} (x, A) -  P_{1 \to H}(y, A)\right) \right\vert \\
		& \quad\quad \le \max_{x \in \cZ} \sum_{y \in \cZ} \nu_1(y) \max_{A \subseteq \cZ} \left\vert \left(P_{1 \to H} (x, A) -  P_{1 \to H}(y, A)\right) \right\vert \\
		&\quad\quad = \max_{x \in \cZ} \sum_{y \in \cZ} \nu_1(y) d_{TV}(P_{1 \to H} (x,\cdot),  P_{1 \to H}(y, \cdot)) \\
		& \quad\quad \le \max_{x,y} d_{TV}(P_{1 \to H} (x,\cdot),  P_{1 \to H}(y, \cdot)).
	\end{align*}

	\underline{\textit{\textbf{Proof of Claim \ref{claim:nsmix-2}}}}

	Let $h \ge k \ge \ell \ge 1$ and fix $x, y \in \cZ$.
	From Proposition 4.7 of \cite{Levin17}, there exists a coupling $(X_k,Y_k)$ of $P_{\ell \to k}(x, \cdot)$ and $P_{\ell \to k}(y, \cdot)$ such that
	\begin{align*}
		d_{TV}(P_{\ell \to k}(x, \cdot), P_{\ell \to k}(y, \cdot)) = \PP(X_k \neq Y_k).
	\end{align*}
	We also have $ P_{\ell \to h}(x, w) = \sum_{z \in \cZ} P_{\ell \to k}(x, z) P_{k \to h}(z, w) = \EE[ P_{k \to h}(X_k, w)]$ and similarly $P_{\ell \to h}(y, w) = \EE[P_{k \to h}(Y_k, w) ]$, for all $w \in \cZ$. Thus, for any $A \subseteq \cZ$, we have
	\begin{align*}
		\left\vert P_{\ell \to h}(x, A) - P_{\ell \to h}(y,A)\right\vert & =  \left\vert \EE[ P_{k \to h}(X_k, A)] - \EE\left[P_{k \to h}(Y_k, A)\right] \right\vert \\
		& = \left\vert \EE[ \indicator\lbrace X_k \neq Y_k \rbrace (P_{k \to h}(X_k, A) -  P_{k \to h}(Y_k, A)) ] \right\vert \\
		& \le \EE[ \indicator\lbrace X_k \neq Y_k \rbrace \left\vert P_{k \to h}(X_k, A) -  P_{k \to h}(Y_k, A)\right\vert ] \\
		& \le \EE\left[ \indicator\lbrace X_k \neq Y_k \rbrace  d_{TV}(P_{k \to h}(X_k, \cdot), P_{k \to h}(Y_k, \cdot) )\right] \\
		& \le \PP (X_k \neq Y_k) \max_{x',y' \in \cZ} d_{TV}(P_{k \to h}(x', \cdot), P_{k \to h}(y', \cdot)) \\
		& = d_{TV}(P_{\ell \to k}(x, \cdot), P_{\ell \to k}(y, \cdot))  \max_{x',y' \in \cZ} d_{TV}(P_{k \to h}(x', \cdot), P_{k \to h}(y', \cdot))
	\end{align*}
	Finally, maximizing both sides over $A \subseteq \cZ$ and $x,y \in \cZ$ yields the desired result.
\end{proof}

\newpage
% !TEX root = ./main.tex
\newpage
\section{Proof of Theorem \ref{thm:initial-spectral} -- Initial Spectral Clustering}\label{app:init1}

In this appendix, we present the proof of Theorem \ref{thm:initial-spectral}. We further clarify the details of our initial spectral clustering algorithm and provide complementing proofs and comments.

\subsection{Algorithm, preliminaries, and notations}\label{app:preliminaries}

\paragraph{Notations.} We introduce notations used extensively throughout this appendix. We recall that $\hat{N}_a(x, y) = \sum_{t,h} \hat{N}_{a,t,h}(x, y)$ for all $x,y \in \cX,a \in \cA$, where we denote $\hat{N}_{a,t,h}(x, y) \triangleq \indicator \lbrace (x_h^{(t)},a_h^{(t)}, x_{h+1}^{(t)}) = (x, a, y) \rbrace$ for all $t \in [T],h \in [H]$, and use the short hand $\sum_{t,h} = \sum_{t=1}^{T} \sum_{h=1}^H$. Furthermore, when writing $\hat{N}_a$ and $\hat{N}_{a,t,h}$, we will think of these as matrices in $\RR^{n \times n}$. Finally, we will also define for all $a \in \cA$, $P_a \triangleq (P(y \vert x,a))_{x,y \in \cX}$.

\paragraph{Trimming.} In the trimming step, for each $a \in \cA$, we recall that $\Gamma_a$ is defined as a subset of contexts constructed from $\cX$, by removing $\gamma$ contexts with the highest number of visits. More precisely,
$$
\forall a \in \cA, \qquad \Gamma_a = \left\lbrace x\in \cX:  \left\vert \left\lbrace z\in \cX: \sum_{y} \hat{N}_a(x,y) < \sum_{y} \hat{N}_a(z,y) \right\rbrace  \right\vert  \ge \gamma \right\rbrace
$$
where we choose $\gamma = \lfloor n \exp\left( - \frac{TH}{nA} \log\left( \frac{TH}{nA}\right)\right) \rfloor$. The reason for this choice will be apparent in our analysis later on (see proof of Proposition \ref{prop:init-concentration}). Then, the trimmed matrix is simply defined as $N_{a,\Gamma_a} = ( \hat{N}_a(x,y) \indicator \lbrace x ,y \in \Gamma_a \rbrace)_{x,y \in \cX}$.

\paragraph{Two-step conditioning.} We further define for each $a \in \cA$,  $\tilde{N}_a  \triangleq \sum_{t,h} \tilde{N}_{a,t,h}$, where we set, for $t \in [T]$, $3 \le h \le H$, $\tilde{N}_{a,t,h} \triangleq \EE_\mu[\hat{N}_{a,t,h} \vert \hat{N}_{a,t,h-2}]$, $\tilde{N}_{a,t,2} \triangleq \EE_\mu[\hat{N}_{a,t,2}]$, and  $\tilde{N}_{a,t,1} \triangleq \EE_\mu[\hat{N}_{a,t,1}]$. In fact, we note that
\begin{align*}
\forall h \ge 3, \qquad \tilde{N}_{a,t,h} & = \diag( P_{0}(x_{t, h-1}, \cdot)) \diag(\rho(a | \cdot)) P_a,  \\
\tilde{N}_{a, t, 2} &  = \diag(\mu P_{0}) \diag(\rho(a | \cdot)) P_a \\
\tilde{N}_{a, t, 1} &  = \diag(\mu) \diag(\rho(a | \cdot)) P_a
\end{align*}
where recall that $P_0 \triangleq ( \sum_{a \in \cA} P(y \vert x, a) \rho(a \vert x))_{x, y \in \cX}$. Introducing $(\tilde{N}_a)_{a \in \cA}$ is crucial in the analysis of the concentration of the trimmed matrix.

\paragraph{$S$-rank approximations.} As described in our algorithm, for each $a \in \cA$, we build a matrix $\hat{M}_a$ that is an $S$-rank approximation of $\hat{N}_{a,\Gamma_a}$. More precisely, the procedure is as follows: (i) for each $a \in \cA$, via an SVD decomposition, we can write  $ \hat{N}_{a,\Gamma_a} = U_a^\top \diag(\sigma_{a,1}, \dots, \sigma_{a,n}) V_a^\top$ where $U_a, V_a$ are two $n \times n$ orthonormal matrices, and $\sigma_{a, 1} \ge \sigma_{a,2} \ge \dots \ge \sigma_{a, n}$ are the singular values of $\hat{N}_{a,\Gamma_a}$; (ii) we obtain an $S$-rank approximation of $\hat{N}_{a, \Gamma_a}$ by setting all but the first $S$ singular values to zero, that is $\hat{M}_a = U_a^\top \diag(\sigma_{a,1}, \dots, \sigma_{a, S}, 0, \dots, 0) V_a$.

\paragraph{Aggregation of information across actions.}
 In order to fully exploit all the observations we gather, we aggregate the information across different actions by stacking together the obtained $S$-rank approximation matrices $(\hat{M}_a)_{a \in \cA}$ to form a fat matrix  $\hat{M}$ of size $n \times 2 n A$. More precisely, we write
 \begin{align*}
   \hat{M} = \begin{bmatrix} (\hat{M}_1)^\top & \cdots & (\hat{M}_A)^\top &  \hat{M}_1 & \cdots & \hat{M}_A   \end{bmatrix}.
 \end{align*}
 Importantly, as shown later in this appendix, our random matrix $\hat{M}$ will concentrate around the matrix $\tilde{N}$, an $n \times 2nA$ matrix, defined as follows
 \begin{align*}
   \tilde{N} \triangleq \begin{bmatrix} (\tilde{N}_1)^\top & \cdots & (\tilde{N}_A)^\top &  \tilde{N}_1 & \cdots & \tilde{N}_A   \end{bmatrix}.
 \end{align*}
 As we shall see, the motivation behind our aggregation procedure stems from the fact that analyzing $\tilde{N}$ gives rise to a separability quantity that is tightly related to the rate function $I(\Phi)$ that appears in our lower bound.

 % As described in Algorithm \ref{alg:spectral-clustering}, we do this by constructing.

\paragraph{Weighted K-medians clustering.} Finally, we run a weighted K-medians clustering on the rows of $\hat{M}$. The procedure consists of the following two steps:
\begin{itemize}
  \item[\emph{(i)}] we re-weigh or normalize in an $\ell_1$ sense the rows of $\hat{M}$, by setting for all $x \in \cX$ such that $\hat{M}(x, \cdot) \neq 0$, $\tilde{M}(x , \cdot) = \hat{M}(x,\cdot) / \Vert \hat{M}(x, \cdot)\Vert_1$ and define $\cX_0 = \lbrace x \in \cX: \hat{M}(x, \cdot) = 0 \rbrace$. We further define for any $x\in {\cal X}$, $\tilde{R}(x,\cdot)= \tilde{N}(x,\cdot) / \Vert \tilde{N}(x, \cdot)\Vert_1$, which will be useful in the analysis.
  \item[\emph{(ii)}] for some $\epsilon > 0$, we solve the following $(1+ \epsilon)$-K-medians optimization problem on $\cX \backslash \cX_0$: find $\lbrace \hat{f}(x)\rbrace_{x \in \cX \backslash \cX_0}$ in $\cS^{\cX \backslash \cX_0}$ such that:
  \begin{align}\label{eq:opt-k-medians}
  & \sum_{s \in \cS} \min_{u_s \in \RR^n}  \sum_{x \in \cX \backslash \cX_0: \hat{f}(x) = s}  \Vert \hat{M}(x, \cdot)\Vert_1 \Vert \tilde{M}(x,\cdot) - u_s  \Vert_1 \le  \nonumber \\Spectral
  & \qquad \qquad \qquad  (1+\epsilon)  \min_{f \in \cS^\cX }  \sum_{s \in \cS} \min_{u_s \in \RR^n} \sum_{x \in \cX \backslash \cX_0: f(x) = s} \Vert \hat{M}(x, \cdot)\Vert_1 \Vert \tilde{M}(x,\cdot) - u_s  \Vert_1
  \end{align}
  and for all $x \in \cX_0$, we set $\hat{f}(x) = 0$.
\end{itemize}
We note that step \emph{(ii)} can be executed efficiently (see e.g. \cite{Chen18, Gao18}).

\subsection{Proof of Theorem \ref{thm:initial-spectral}}

Here we state a more precise version of Theorem \ref{thm:initial-spectral} and provide its proof.

\begin{theorem}\label{thm:initial-spectral:precise} Assume that $TH = \Omega(n)$ and $I(\Phi) > 0$. The clustering error rate of the initial spectral clustering algorithm satisfies:
  \begin{align*}
    \PP\left( \frac{\vert \cE\vert}{n} \le \poly(\eta) \left(1+ \frac{(2+\epsilon)}{\tilde{J}(\Phi, \rho)}\right) \sqrt{ \frac{SAn}{TH} }\right) \ge 1 - \frac{2}{n} - 2e^{-n} - 2e^{-\frac{TH}{2nA}}.
  \end{align*}
\end{theorem}

Observe that when $TH = \omega(n)$, then $1 - \frac{2}{n} - 2e^{-n} - 2e^{-\frac{TH}{2nA}} \underset{n\to \infty}{\longrightarrow 1}$, which justifies the claim of Theorem \ref{thm:initial-spectral} that $ \frac{\vert \cE \vert}{n} = \mathcal{O}\left(\sqrt{\frac{SAn}{TH}
}\right)$ with high probability.

\begin{proof}[Proof of Theorem \ref{thm:initial-spectral:precise}]
  We know from Proposition \ref{prop:k-medians} that the weighted K-medians clustering algorithm ensures that
  \begin{align*}
    \frac{\vert \cE\vert}{n} \le  \left(2 + \eta^2 + \frac{4(2 + \epsilon)\eta^2}{\tilde{J}(\Phi, \rho)} \right) \frac{ 2 \eta^5 n\sqrt{A}}{ TH} \Vert \hat{M} - \hat{N} \Vert_F.
  \end{align*}
  Next, by construction of $\hat{M}$, using Lemma \ref{lem:rank-S-approx}, we have
  \begin{align*}
    \frac{\vert \cE\vert}{n} \le  \left(2 + \eta^2 + \frac{4(2 + \epsilon)\eta^2}{\tilde{J}(\Phi, \rho)} \right) \frac{ 8 \eta^5 nA \sqrt{S}}{TH} \max_{a \in \cA}\Vert \hat{N}_{a, \Gamma_a} - \tilde{N}_a\Vert.
  \end{align*}
  Finally, applying Proposition \ref{prop:init-concentration}, we immediately obtain
  \begin{align*}
    \PP\left( \frac{\vert \cE\vert}{n} \le \poly(\eta) \left(1+ \frac{(2+\epsilon)}{\tilde{J}(\Phi, \rho)}\right) \sqrt{ \frac{SAn}{TH} }\right) \ge 1 - \frac{2}{n} - 2e^{-n} - 2e^{-\frac{TH}{2nA}}.
  \end{align*}
\end{proof}

\subsection{Separability}

\begin{definition}\label{def:J-sep}
For a given BMDP $\Phi$, we define the separability quantity as follows
\begin{align*}
  J(\Phi, \rho) \triangleq  \min_{\substack{\nu \in \cP(\cX, \eta^3), \\  x,y \in \cX: f(x)\neq f(y)}} \frac{1}{2A} \sum_{a \in \cA} J_1(\nu, x, y ;\Phi, \rho, a) + J_2(\nu, x, y;\Phi, \rho, a)
\end{align*}
where
\begin{align*}
  J_1(\nu, x, y; \Phi, \rho, a) & =  \sum_{z \in \cX}  nA \left\vert \nu(z)\rho(a \vert z)P_a(z,x) - \nu(z) \rho(a \vert z) P_a(z,y)\right\vert, \\
  J_2(\nu, x, y;\Phi, \rho,  a) & =  \sum_{z \in \cX}  nA \left\vert \nu(x)\rho(a \vert x) P_a(x,z) - \nu(y) \rho(a \vert y) P_a(y,z)\right\vert,
\end{align*}
and where $\cP(\cX, \eta^3) = \lbrace \nu:  \nu \textrm{ is a probability distribution over } \cX, \max_{z_1, z_2 \in \cX} \frac{\nu(z_1)}{\nu(z_2)} \le \eta^3\rbrace$, namely, the set of all $\eta^3$-regular probability distributions.
\end{definition}

It is not difficult to verify that under Assumptions \ref{assumption:SA}-\ref{assumption:uniform}, we have that $J(\Phi, \rho) \le \poly(\eta)$. %and that $J(\Phi) = \Omega(I(\Phi))$ where $\Omega( \cdot)$ is taken with respect to $n \to \infty $.

\begin{proposition}[Separability property] \label{prop:separability}
	Under Assumptions \ref{assumption:SA}-\ref{assumption:uniform}, the matrix $\tilde{N}$ satisfies the following: for all $x,y \in \cX$ such that $f(x) \neq f(y)$, we have:
  \begin{align}\label{eq:prop:sep}
    \Vert \tilde{N}(x,  \cdot) - \tilde{N} (y, \cdot) \Vert_1 \ge \frac{2TH}{n} J(\Phi, \rho).
  \end{align}
  Consequently, it also holds that:
  \begin{align}\label{eq:prop:sep2}
    \Vert \tilde{R}(x, \cdot) - \tilde{R}(y, \cdot) \Vert_1 \ge \tilde{J}(\Phi, \rho)
  \end{align}
 for some $\tilde{J}(\Phi, \rho) = \poly(1/\eta)J(\Phi, \rho)$. Furthermore, when $I(\Phi) > 0$, then $J(\Phi, \rho) >0$.
\end{proposition}

\begin{proof}[Proof of Proposition \ref{prop:separability}]
  Let $a \in \cA$, and $x,y \in \cX$ such that $f(x) \neq f(y)$. First, to ease notations, we introduce for all $z \in \cX$,
  $$
  \nu(z) =  \frac{1}{ T (H-1)}\left( \sum_{t = 1}^{T} \left(\mu(z) + (\mu P_{0})(z) +  \sum_{h=3}^{H-1} P_{0}(x_{h-1}^{(t)}, z) \right) \right),
  $$
  Clearly, $\nu = (\nu(z))_{z \in \cX}$ is a probability distribution over $\cX$, and moreover, one can easily see that under Assumptions \ref{assumption:linear}-\ref{assumption:uniform}, $\nu$ is $\eta^3$-regular.  More precisely, we also have, for all $z \in \cX$,
  \begin{equation*}
  \label{eq:nu}
  	 \frac{1}{\eta^3 n} \leq \nu(z) \leq \frac{\eta^3}{n}.
  \end{equation*}

  \textbf{\emph{(Proof of \eqref{eq:prop:sep}).}}  First, noting that $\tilde{N}_a(z, x) =  T H \nu(z) \rho(a | z) P_a(z, x)$, we have that
  \begin{align*}
  	\left\Vert \tilde{N}_a(\cdot, x) - \tilde{N}_a(\cdot, y) \right\Vert_1  & = \sum_{z \in \cX} \left\vert \tilde{N}_a(z, x) - \tilde{N}_a(z, y) \right\vert \\
  	& =  \frac{TH}{nA} \sum_{z \in \cX}  nA  \left\vert \nu(z) \rho(a \vert z )P_a(z,x) - \nu(z)P_a(z,y)\right\vert \\
    & \ge \frac{TH}{nA} J_1(\nu, x, y; \Phi, \rho,  a).
  \end{align*}
  Similarly, we have:
  \begin{align*}
  	\left\Vert \tilde{N}_a(x, \cdot) - \tilde{N}_a(y, \cdot) \right\Vert_1  &  = \sum_{z \in \cX} \left\vert \tilde{N}_a(x, z) - \tilde{N}_a(y, z) \right\vert \\
  	&  =  \frac{T H}{nA}  \sum_{z \in \cX} n A \left\vert \nu(x) \rho(a \vert x) P_a(x, z) - \nu(y)P_a(y, z)\right\vert  \\
    & \ge \frac{TH}{nA} J_2(\nu, x , y;\Phi, \rho, a).
  \end{align*}
  Therefore, we obtain
  \begin{align*}
    \Vert \tilde{N}(x, \cdot) - \tilde{N}(y, \cdot) \Vert_1  &  = \sum_{a \in \cA} \Vert \tilde{N}_a(x, \cdot) - \tilde{N}_a(y, \cdot) \Vert_1 + \Vert \tilde{N}_a( \cdot, x) - \tilde{N}_a(\cdot, y) \Vert_1   \\
    & \ge \frac{TH }{nA} \sum_{a \in \cA} J_1(\nu, x, y ;\Phi, \rho, a) + J_2(\nu, x, y; \Phi, \rho, a) \\
    & \ge \frac{2 TH }{n} J(\Phi, \rho).
  \end{align*}

  \textbf{\emph{(Proof of \eqref{eq:prop:sep2}).}} Next, under Assumption \ref{assumption:regularity}, in view of Proposition \ref{prop:mc0}, we have for all $t \ge 1,h \ge 1$,
    \begin{align*}
      \frac{1}{\eta^5 n^2A} \le \min_{x, y \in \cX} \tilde{N}_{a,t,h}(x, y) \le \max_{x, y \in \cX}\tilde{N}_{a,t,h}(x, y) \le \frac{\eta^5}{n^2A},
    \end{align*}
  which leads to
  \begin{align*}
    \frac{1}{\eta^5 nA} &  \le  \max_{x, y \in \cX} \left(\Vert\tilde{N}_{a,t,h}(x, \cdot)\Vert_1, \Vert \tilde{N}_{a,t,h}(\cdot, y)\Vert_1\right) \le \frac{\eta^5}{nA} \\
    \frac{1}{\eta^5 nA} & \le  \min_{x, y\in \cX} \left(\Vert\tilde{N}_{a,t,h}(x, \cdot)\Vert_1, \Vert \tilde{N}_{a,t,h}(\cdot, y)\Vert_1\right) \le \frac{\eta^5}{nA}
  \end{align*}
  and gives that
  \begin{align}\label{eq:regular_tilde_N}
    \frac{ 2TH}{\eta^5n} \le \min \left(\Vert \tilde{N}(x, \cdot) \Vert_1, \Vert \tilde{N}(y, \cdot)\Vert_1 \right) \le \frac{ 2\eta^5T H}{n}.
  \end{align}
  Therefore, recalling the definition of $\tilde{R}$, we have for all $x, y \in \cX$, such that $f(x) \neq f(y)$,
  \begin{align*}
    \Vert \tilde{R}(x, \cdot) - \tilde{R}(y, \cdot) \Vert_1 \ge \frac{\Vert \tilde{N}(x, \cdot) - \tilde{N}(y, \cdot) \Vert_1}{\min(\Vert \tilde{N}(x, \cdot)\Vert_1 , \Vert \tilde{N}(y, \cdot) \Vert_1)} \ge \tilde{J}(\Phi, \rho) = \frac{1}{\eta^5} J(\Phi, \rho).
  \end{align*}

   \textbf{\emph{(Proving $I(\Phi)> 0 \implies J(\Phi, \rho) > 0$).}} Now let us prove that if $J(\Phi, \rho) = 0$, then $I(\Phi) = 0$. Assume that $J(\Phi, \rho) = 0$, then there exist $x, y \in \cX$ with $f(x) \neq f(y)$, and $\nu \in \cP(\cX, \eta^3)$ such that for all $a \in \cA$,
  \begin{align}
     \sum_{z \in \cX}  nA \left\vert \nu(z)\rho(a \vert z)P_a(z,x) - \nu(z) \rho(a \vert z) P_a(z,y)\right\vert & = 0,  \label{eq:J-sep:1}\\
     \sum_{z \in \cX}  nA \left\vert \nu(x)\rho(a \vert x) P_a(x,z) - \nu(y) \rho(a \vert y) P_a(y,z)\right\vert &  = 0. \label{eq:J-sep:2}
  \end{align}
  Now observe that:
  \begin{itemize}
    \item[1.] from Eqn. \eqref{eq:J-sep:1}, we can immediately deduce that for all $a \in \cA$, $z \in \cX$, $P_a(z,x) = P_a(z,y)$. This entails that
    \begin{align}\label{eq:J-sep:c1}
      \forall s \in \cS, \forall a \in \cA, \qquad q(x, f(x)) p(f(x) \vert s,a) = q(y, f(y)) p(f(y) \vert s,a);
    \end{align}
    \item[2.] from Eqn. \eqref{eq:J-sep:2}, it must be the case that for all $a \in \cA$, $z \in \cX$, $\nu(x)\rho(a \vert x)P_a(x,z) = \nu(y)\rho(a \vert y)P_a(y,z)$. By summing over $z \in \cX$, we deduce that $\nu(x) \rho(a \vert x)= \rho(a \vert y) \nu( y) > 0$ because $\nu \in \cP(\cX, \eta^3)$. This entails that
    \begin{align}\label{eq:J-sep:c2}
          \forall s \in \cS, \forall a \in \cA, \qquad  p(s \vert f(x), a) = p(s \vert f(y), a).
    \end{align}
  \end{itemize}
  Now in view of Proposition \ref{prop:I=0}, we observe that \eqref{eq:J-sep:c1} and \eqref{eq:J-sep:c2} imply that $\min_{x \in \cX} I(x; \Phi) = 0$, which in turn implies that $I(\Phi) = 0$.

  %
  % \begin{enumerate}
  % 	\item Since $J_1(x, y; a, \Phi) = 0$, it must be that for all $z \in \cZ$, $P_a(z, x) = P_a(z, y)$.
  % 	Equivalently,
  % 	\begin{align}\label{eq:prop:sep:c1}
  % 		\forall s \in \cS, \forall a\in \cA, \quad q(x | f(x)) p(f(x) | s, a) = q(y | f(y)) p(f(y) | s, a).
  % 	\end{align}
  % 	\item Since $J_2(x, y; a, \Phi) = 0$, it must be that $\nu(x) P_a(x, z) = \nu(y)P_a(y, z)$ for all $z \in \cX$.
  % 	By summing over $z \in \cX$, we deduce that $\nu(x) = \nu(y) > 0$.
  % 	Thus, we must have that $P_a(y, z) = P_a(x, z)$ for all $z \in \cX$.
  % 	Equivalently,
  % 	\begin{align}\label{eq:prop:sep:c2}
  % 		\forall s \in \cS, \forall a\in \cA, \quad \rho(a | x) p(s \vert f(x), a) = \rho(a | y) p(s \vert f(y), a).
  % 	\end{align}
  % 	Summing over $s \in \cS$ gives $\rho(a | x) = \rho(a | y)$ for all $a \in \cA$, and thus
  % 	\begin{equation}
  % 		\forall s \in \cS, \forall a\in \cA, \quad p(s \vert f(x), a) = p(s \vert f(y), a).
  % 	\end{equation}
  % \end{enumerate}
  % Recalling the definition of $I_{min}$, the two conditions \eqref{eq:prop:sep:c1} and \eqref{eq:prop:sep:c2} imply that $I_{min} = 0$.

\end{proof}

\subsection{Weighted K-medians clustering}

\paragraph{The solution to \eqref{eq:opt-k-medians}.} The solution to the $(1+\epsilon)$-K-medians optimization \eqref{eq:opt-k-medians} on $\cX\backslash \cX_0$ can be obtained efficiently (e.g., see \cite{Chen18,Gao18} from which we took inspiration). Here, we recall that $\cX_0 = \lbrace x \in \cX: \Vert \hat{M}(x, \cdot)\Vert_1 = 0\rbrace$. Thus, let us denote $\lbrace \hat{f}(x)\rbrace_{x \in \cX \backslash \cX_0}$ in $\cS^{ \cX \backslash \cX_0}$ and $\hat{u}_1, \dots, \hat{u}_S \in \RR^n$ such a solution. We further set $\hat{f}(x) = 1$, for all $x \in \cX_0$, and define $\hat{U}(x, \cdot) = \hat{u}_{\hat{f}(x)}$ for all $x \in \cX$. With this, observe that by definition of $\hat{f}$ and $\hat{U} = (\hat{U}(x, \cdot))_{x \in \cX}$, we have
\begin{equation}\label{eq:k-medians}
\sum_{x \in \cX } \Vert \hat{M}(x, \cdot)\Vert_1 \Vert \tilde{M}(x,\cdot) - \hat{U}(x, \cdot)  \Vert_1 \le  (1+\epsilon)     \sum_{x \in \cX } \Vert \hat{M}(x, \cdot)\Vert_1 \Vert \tilde{M}(x,\cdot) - \tilde{R}(x, \cdot) \Vert_1.
\end{equation}
Note that the choice of $\epsilon$ is irrelevant in our analysis and can be viewed as a constant hyperparameter.

\paragraph{Linking $\vert \cE\vert$ to the geometry of the rows of $\hat{U}$ and $\tilde{R}$.} Next, we will need the technical Lemma \ref{lem:gao-geometry} which relates the number of misclassified contexts $\vert \cE\vert$ to that of the geometry of points, namely the rows of $\tilde{R}$ and $\hat{U}$ in our case, provided some separability condition is satisfied for $\tilde{R}$. The statement of the lemma is valid for any norm $\Vert \cdot \Vert$.

\begin{lemma}[Lemma 6 in \cite{Gao18}]\label{lem:gao-geometry}
	Assume that $\min_{x,y \in\cX: f(x) \neq f(y)} \lVert \tilde{R}(x, \cdot) - \tilde{R}(y, \cdot) \rVert \geq 2\xi$ for some $\xi > 0$. Then it holds that
	\begin{equation}
		|\cE| = \left| \min_{\sigma \in \Upsilon(\cS)} \bigcup_{s \in \cS} \hat{f}^{-1}(\sigma(s)) \setminus f^{-1}(s) \right|
		\leq |\cX_0| + (\eta^2 + 1) |\cX_1|.
	\end{equation}
  where we define $\cX_1 = \lbrace x \in \cX \backslash \cX_0: \Vert \hat{U}(x, \cdot) - \tilde{R}(x, \cdot) \Vert_1 \ge \xi \rbrace$.
\end{lemma}

The proof of Lemma \ref{lem:gao-geometry} is borrowed from \cite{Gao18}. The statement we provide here differs slightly from that provided by \cite{Gao18} as they considered the \emph{Degree Corrected Block Model} while in our case we consider BMDPs. However, the proof is essentially the same and is provided here for completeness (postponed end of this subsection).

Now, invoking Proposition \ref{prop:separability}, we have under the assumption that $I(\Phi)> 0$, for all $x, y \in \cX$ such that $f(x) \neq f(y)$, $\Vert \tilde{R}(x, \cdot) - \tilde{R}(y, \cdot)\Vert_1 \ge \tilde{J}(\Phi, \rho)$. Thus, applying Lemma \ref{lem:gao-geometry} specialised to the $\ell_1$ norm, we immediately obtain that
\begin{align}\label{eq:ce-sets-X0-X1}
  \vert \cE \vert \le \vert \cX_0 \vert + (\eta^2 + 1) \vert \cX_1\vert
\end{align}
where $\cX_1 = \lbrace x \in \cX\backslash \cX_0: \Vert \hat{U}(x, \cdot) - \tilde{R}(x, \cdot)\Vert_1 \ge \tilde{J}(\Phi, \rho)/2 \rbrace$.

Starting from our regularity Assumption \ref{assumption:regularity}, observe that
\begin{align*}
    \frac{ \vert \cX_0 \vert}{n}  & \leq  \frac{\eta^5}{2T H}\sum_{x \in \cX_0} \Vert \tilde{N}(x, \cdot)\Vert_1  \qquad \text{and} \qquad \frac{ \vert \cX_1 \vert}{n}  \leq  \frac{\eta^5}{2T H}\sum_{x \in \cX_1} \Vert \tilde{N}(x, \cdot)\Vert_1.
\end{align*}

\paragraph{Bounding $\sum_{x \in \cX_0} \Vert \tilde{N}(x, \cdot)\Vert_1$.} We have:
\begin{align*}
  \sum_{x \in \cX_0} \Vert \tilde{N}(x, \cdot)\Vert_1 & = \sum_{x \in \cX_0} \Vert \hat{M}(x, \cdot) - \tilde{N}(x, \cdot)\Vert_1 \\
  & \le \sum_{x \in \cX} \Vert \hat{M}(x, \cdot) - \tilde{N}(x, \cdot)\Vert_1 \\
  & \le n\sqrt{A} \Vert  \hat{M} - \tilde{N} \Vert_F,
\end{align*}
where the first equality holds by definition of $\cX_0$, and the last inequality holds by the equivalence of norms on matrices (essentially using Cauchy-Schwarz inequality).

\paragraph{Bounding $\sum_{x \in \cX_1} \Vert \tilde{N}(x, \cdot)\Vert_1$.} First, by triangle inequality,
\begin{align*}
  \sum_{x \in \cX_1} \Vert \tilde{N}(x, \cdot)\Vert_1 & \le \sum_{x \in \cX_1} \left(    \Vert \tilde{N}(x, \cdot) - \hat{M}(x, \cdot) \Vert_1 + \Vert \hat{M}(x ,\cdot)\Vert_1 \right) \\
  & \le n \sqrt{A} \Vert  \hat{M} - \tilde{N} \Vert_F + \sum_{x \in \cX_1}  \Vert \hat{M}(x ,\cdot)\Vert_1,
\end{align*}
where again for the second inequality, we used the equivalence between norms on matrices. Next, we have:
\begin{align*}
   \sum_{x \in \cX_1}  \Vert \hat{M}(x ,\cdot)\Vert_1 & \overset{(a)} {\le} \frac{2}{\tilde{J}(\Phi, \rho)} \sum_{x \in \cX_1}  \Vert \hat{M}(x ,\cdot)\Vert_1 \Vert \hat{U}(x, \cdot) - \tilde{R}(x, \cdot)\Vert_1 \\
   & \overset{(b)} {\le} \frac{2}{\tilde{J}(\Phi, \rho)} \sum_{x \in \cX_1}  \Vert \hat{M}(x ,\cdot)\Vert_1  \left(\Vert \tilde{M}(x, \cdot) - \hat{U}(x, \cdot) \Vert_1 + \Vert \tilde{M}(x, \cdot) - \tilde{R}(x, \cdot)\Vert_1 \right) \\
   & \overset{(c)} {\le} \frac{2(2 + \epsilon)}{\tilde{J}(\Phi, \rho)} \sum_{x \in \cX_1}  \Vert \hat{M}(x ,\cdot)\Vert_1   \Vert \tilde{M}(x, \cdot) - \tilde{R}(x, \cdot)\Vert_1 \\
   & \overset{(d)} {\le} \frac{4(2 + \epsilon)}{\tilde{J}(\Phi, \rho)} \sum_{x \in \cX_1} \frac{\Vert \hat{M}(x ,\cdot)\Vert_1 \Vert \tilde{M}(x, \cdot) - \tilde{N}(x, \cdot)\Vert_1}{\max( \Vert \hat{M}(x ,\cdot)\Vert_1, \Vert \tilde{N}(x, \cdot) \Vert_1  )} \\
   & \overset{(e)} {\le} \frac{4(2 + \epsilon)}{\tilde{J}(\Phi, \rho)}\sum_{x \in \cX_1} \Vert \tilde{M}(x, \cdot) - \tilde{N}(x, \cdot)\Vert_1 \\
   & \overset{(f)} {\le} \frac{4(2 + \epsilon)}{\tilde{J}(\Phi, \rho)} n\sqrt{A} \Vert  \hat{M} - \tilde{N} \Vert_F,
\end{align*}
where in \emph{(a)} we used the definition of $\cX_1$, in \emph{(b)} we used the triangular inequality, in \emph{(c)} we used the definition $\hat{U}$ and  \eqref{eq:k-medians}, in \emph{(d)} we used the elementary inequality that $\Vert \Vert x \Vert_1^{-1} x - \Vert y \Vert_1^{-1} y \Vert_1 \le \frac{2 \Vert x - y\Vert_1}{\max (\Vert x \Vert_1, \Vert y\Vert_1)}$ (this can be shown by triangular inequality and reverse triangular inequality), in \emph{(e)} we used $\frac{\Vert x \Vert_1}{\max( \Vert x \Vert_1, \Vert y\Vert_1)} \le 1$, and finally in \emph{(f)} we used equivalence of norms between matrices. In summary, we get:
\begin{align*}
  \sum_{x \in \cX_1} \Vert \tilde{N}(x, \cdot)\Vert_1 & \le \left( 1 + \frac{4(2 + \epsilon)}{\tilde{J}(\Phi, \rho)} \right) n\sqrt{A} \Vert \hat{M} - \hat{N} \Vert_F
\end{align*}

Thus, we have just proved that:
\begin{align}
  \frac{\vert \cX_0 \vert}{n} & \le  \frac{ \eta^5 n\sqrt{A} }{ 2 T H } \Vert \hat{M} - \hat{N} \Vert_F, \label{eq:X_0}\\
  \frac{\vert \cX_1 \vert}{n} & \le \left( 1 + \frac{4(2 + \epsilon)}{\tilde{J}(\Phi, \rho)} \right) \frac{ \eta^5 n\sqrt{A} }{ 2 TH } \Vert \hat{M} - \hat{N} \Vert_F \label{eq:X_1}.
\end{align}

Now, in view of the inequalities \eqref{eq:ce-sets-X0-X1}, \eqref{eq:X_0}, \eqref{eq:X_1}, we have established the following Proposition.

\begin{proposition}\label{prop:k-medians}
  Assume that $I(\Phi) > 0$, and that Assumptions \ref{assumption:SA}-\ref{assumption:uniform} hold, then the weighted $K$-medians clustering algorithm ensures that
  \begin{align*}
    \frac{\vert \cE\vert}{n} \le  \left(2 + \eta^2 + \frac{4(2 + \epsilon)\eta^2}{\tilde{J}(\Phi, \rho)} \right) \frac{ 2 \eta^5 n\sqrt{A}}{ TH} \Vert \hat{M} - \hat{N} \Vert_F.
  \end{align*}
\end{proposition}

\begin{proof}[Proof of Lemma \ref{lem:gao-geometry}]
	% This is precisely Lemma 6 of \cite{Gao18}, with $\theta_i = 1$, but we include the proof for the sake of completeness.
	For each $s \in \cS$, define
	\begin{equation*}
		\cC_s := \left\{ x \in f^{-1}(s) \cap (\cX \setminus \cX_0) \ : \ \Vert \hat{U}(x, \cdot) - \tilde{R}(x, \cdot) \Vert_1 < \xi \right\}.
	\end{equation*}
	By construction, we have that $\bigcup_{s \in \cS} \cC_s = \cX \setminus (\cX_0 \cup \cX_1)$, and that $\cC_s \cup \cC_{s'} = \emptyset$ whenever $s \neq s'$.
	Also, by assumption, it is easy to see that $\hat{f}(x) \neq \hat{f}(y)$ if $x, y$ are in different $\cC_s$'s. Following \cite{Chen18}, we partition $\cX$ into three groups:
	\begin{align*}
		R_1 &:= \left\{ s \in \cS \ : \ \cC_s = \emptyset \right\}, \\
		R_2 &:= \left\{ s \in \cS \ : \ \cC_s \neq \emptyset, \forall x, y \in \cC_s \ \hat{f}(x) = \hat{f}(y)\right\}, \\
		R_3 &:= \left\{ s \in \cS \ : \ \cC_s \neq \emptyset, \exists x \neq y \in \cC_s \ \hat{f}(x) \neq \hat{f}(y) \right\}.
	\end{align*}

	By the definition of $R_2$, we observe that the contexts in $\bigcup_{s \in R_2} \cC_s$ have the {\it same} partition induced by $f$ and $\hat{f}$ i.e. they can be considered to be correctly classified, up to a permutation.
	Thus,
	\begin{equation*}
		|\cE| \leq \left| \cX_0 \cup \cX_1 \right| + \left| \bigcup_{s \in R_3} \cC_u \right|
		\leq |\cX_0| + |\cX_1| + \left| \bigcup_{s \in R_3} \cC_u \right|.
	\end{equation*}
	Note that for each $s \in \cS$, $\cC_s$ contains at least two different cluster indices given by $\hat{f}$ i.e.
	\begin{equation*}
		|R_1| + |R_2| + |R_3|  = S \geq |R_2| + 2|R_3|,
	\end{equation*}
	which then implies that $|R_1| \geq |R_3|$. We now conclude the proof by noting that
	\begin{align*}
		\left| \bigcup_{s \in R_3} \cC_u \right| &\leq |R_3| \frac{\eta n}{S} \\
		&\leq \eta |R_1| \frac{n}{S} \\
		&\leq \eta^2 \left| \bigcup_{s \in R_1} f^{-1}(s) \right| \\
		&\leq \eta^2 |\cX_1|.
	\end{align*}
\end{proof}

%
%
%
%
% Lemma \ref{lem:gao-misclassification} is borrowed from \cite{Gao18} (see Lemma 6 in \cite{Gao18}), and allows to relate the number of misclassified nodes to that of the geometryc of points, namely the rows of $\tilde{R}$ in our case. The version we provide here is stated diffently to be applied in our setting. The lemma allos We also provide a proof for completeness.
%
% \begin{lemma}\label{lem:gao-misclassification}
% 	Let $\xi > 0$ be such that $\min_{f(x) \neq f(y)} \lVert \tilde{R}(x, \cdot) - \tilde{R}(y, \cdot) \rVert \geq 2\xi$.
% 	Then we have
% 	\begin{equation}
% 		|\cE| = \left| \min_{\sigma \in \Upsilon(\cS)} \bigcup_{s \in \cS} \hat{f}^{-1}(\sigma(s)) \setminus f^{-1}(s) \right|
% 		\leq |\cX_0| + (\eta^2 + 1) |\cX_1|.
% 	\end{equation}
%   where we define $\cX_0 = \lbrace x \in \cX; \hat{f}(x) = 1 \rbrace$, and $\cX_1 = \lbrace x \in \cX \backslash \cX_0: \Vert \tilde{M}(x, \cdot) - \tilde{R}(x, \cdot) \Vert_1 \ge \xi \rbrace$.
% \end{lemma}

\subsection{$S$-rank approximation}

Lemma \ref{lem:rank-S-approx} allows us to control the error in Frobeinus norm between $\hat{M}$ and $\tilde{N}$, which is necessary to bound the number of misclassified contexts to that in operator norm between the trimmed matrices $\hat{N}_{a, \Gamma_a}$ and $\tilde{N}_a$ for which we are able to provide a concentration bound.

\begin{lemma}\label{lem:rank-S-approx}
  Under Assumption \ref{assumption:SA}-\ref{assumption:uniform}, after Algorithm \ref{alg:spectral-clustering}, we have:
  \begin{align*}
    \Vert \hat{M} - \tilde{N}  \Vert_F \le  4 \sqrt{SA} \max_{a \in \cX}\Vert \hat{N}_{a, \Gamma_a} - \tilde{N}_a\Vert
  \end{align*}
\end{lemma}

\begin{proof}[Proof of Lemma \ref{lem:rank-S-approx}]
  We have
  \begin{align*}
  \Vert \hat{M} - \tilde{N} \Vert_F^2 & = 2 \sum_{a \in \cA} \Vert \hat{M}_a - \tilde{N}_a \Vert^2_F \\
  & \overset{(a)}{\le} 4 S \sum_{a \in \cA} \Vert \hat{M}_a - \tilde{N}_a \Vert^2 \\
  & \overset{(b)}{\le} 4 S \sum_{a \in \cA} \left(  \Vert \hat{M}_a - \hat{N}_{a, \Gamma_a} \Vert + \Vert  \hat{N}_{a, \Gamma_a}  -  \tilde{N}_a   \Vert \right)^2 \\
  & \overset{(c)}{\le} 4 S \sum_{a \in \cA} \left(  \sigma_{\vert \cS\vert + 1}( \hat{N}_{a,\Gamma_a})  + \Vert  \hat{N}_{a,\Gamma_a}  -  \tilde{N}_a   \Vert\right)^2 \\
  & \overset{(d)}{\le} 4 S \sum_{a \in \cA} \left(  \sigma_{\vert \cS\vert + 1}(\tilde{N}_a) + 2 \Vert \hat{N}_{a,\Gamma_a}  -  \tilde{N}_a   \Vert\right)^2 \\
  & \overset{(e)}{\le} 16 SA \max_{a \in \cA}\Vert  \hat{N}_{a,\Gamma_a}  -  \tilde{N}_a   \Vert^2
  \end{align*}
  where inequality $(a)$ follows from  $\Vert M \Vert_F \le \sqrt{\mathrm{rank}(M)} \Vert M\Vert$ and the facts that $\mathrm{rank}(\hat{M}_a) \le S$ by construction and  $\tilde{N}_a = L P_a$ for some random matrix $L$, and thus $\mathrm{rank}(\tilde{N}_a) \le \mathrm{rank}(P_a) \le S$ by the structure of the $P_a$. Inequality $(b)$ follows by triangular inequality. Inequality $(c)$ follows by construction of $\hat{M}_a$ for all $a \in \cA$. Inequality $(d)$ follows using Weyl's inequality, and finally inequality $(e)$ follows by noting again that $\tilde{N}_a$ is at most of rank $S $, thus $\sigma_{S+1}(\tilde{N}_a) = 0$.
\end{proof}

% !TEX root = ./main.tex

\subsection{Analysis of the trimmed random matrix}\label{app:trimming}

This subsection is devoted to the derivation of concentration results for the matrices $ \hat{N}_{a, \Gamma_a}$ obtained after trimming the observation matrices. These concentration results are central to the performance analysis of our algorithms. The proof techniques used here draw inspiration from those used in SBMs, BMCs, and matrix completion problems~\citep{Feige05,Keshavan10a,SandersPY20,Sanders21}. We adapt these techniques to our setting. The proof relies on the analysis of \emph{light} and \emph{heavy} couples, an involved net argument, where the analysis of the \emph{heavy} couples relies on the so-called \emph{discrepancy property}~\citep{Feige05,Keshavan10a}. Furthermore, the observations upon which the matrices $\hat{N}_{a, \Gamma_a}$ are built are not independent but rather possess a Markovian nature. To tackle this challenge, we follow a similar reasoning to that in \cite{SandersPY20}. Our setting is however different since we have to accommodate for restarts in episodic Block MDPs. This is done thanks to our new concentration bound (see Appendix \ref{app:concentration}).

We state below the main result of this subsection.

\begin{proposition}
\label{prop:init-concentration}
	For all $a \in \mathcal{A}$, the following holds:
	\begin{align}
		\mathbb{P}\left( \max_{a \in \cA}\left\lVert \hat{N}_{a, \Gamma_a} - \tilde{N}_a \right\rVert \le \poly(\eta) \sqrt{\frac{TH}{nA}} \right) \ge  1- \frac{2}{n} - 2e^{-n} - 2 e^{- \frac{TH}{2nA}}
	\end{align}
\end{proposition}

\begin{proof}[Proof of Proposition \ref{prop:init-concentration}]
	First, we express $\Vert \hat{N}_{a, \Gamma_a} - \tilde{N}_a \Vert$ using the variational form of the operator norm, then use the triangle inequality to obtain
	\begin{align*}
		\left\Vert \hat{N}_{a,\Gamma_a} - \tilde{N}_a \right\Vert & = \sup_{u,v \in \mathbb{S}^{n-1}} u^\top \left(  \hat{N}_{a,\Gamma_a} -  \tilde{N} \right) v  \le T_1 + T_2 + T_3,
	\end{align*}
	where $\mathbb{S}^{n-1}$ denotes the unit sphere in $\RR^n$, and $T_1, T_2$ and $T_3$ are defined as follows:
	\begin{align*}
		T_1 &\triangleq \left| \sup_{u,v \in \mathbb{S}^{n-1}}  \sum_{(x,y) \in \cL \cap \cK^c}   u_x v_y \hat{N}_{a}(x,y) \right|, \\
		T_2 &\triangleq \left| \sup_{u,v \in \mathbb{S}^{n-1}} \sum_{(x,y)  \in \mathcal{L}}    u_x v_y \hat{N}_{a}(x,y) - u^\top \tilde{N}_{a} v \right|, \\
		T_3 &\triangleq \left| \sup_{u,v \in \mathbb{S}^{n-1}} \sum_{(x,y) \in \cH \cap \cK}    u_x v_y \hat{N}_{a}(x,y)  \right|,
	\end{align*}
	with
	\begin{align*}
		\cL &\triangleq \lbrace (x, y): x,y \in \cX \text{ and } \vert u_x v_y \vert < m   \rbrace \qquad & (\textit{light couples})\\
		\cK & \triangleq \lbrace (x, y): x, y \in \cX \text{ and } x, y \in \Gamma_a \rbrace \qquad & (\textit{non-trimmed couples}) \\
		\cH & \triangleq \lbrace (x, y): x,y \in \cX \text{ and } \vert u_x v_y \vert \ge m \rbrace \qquad & (\textit{heavy couples})
	\end{align*}
	where we set $m = \frac{1}{n} \sqrt{\frac{TH}{nA}}$. This choice will appear suitable for our analysis. We recall that number of trimmed nodes is exactly $\gamma = \left \lfloor n \exp( - \frac{TH}{nA} \log\left(\frac{TH}{nA} \right)  ) \right \rfloor$. By Lemma \ref{lem:cb-T1}, Lemma \ref{lem:cb-T2} and Lemma \ref{lem:cb-T3}, that the terms $T_1$, $T_2$, and $T_3$ satisfy the following high probability bounds whenever $TH = \Omega(n)$:
	\begin{align}
		\PP\left( T_1 \le \poly(\eta) \sqrt{\frac{TH}{nA}}\right)  & \ge 1 - \frac{e^{-n}}{A}, \label{eq:matrix-conc-1} \\
		\PP\left( T_2 \le \poly(\eta) \sqrt{\frac{TH}{nA}}\right)  & \ge 1 - \frac{e^{-n}}{A}, \label{eq:matrix-conc-2}\\
		\PP\left( T_3 \le \poly(\eta) \sqrt{\frac{TH}{nA}} \right)  & \ge 1 - \frac{2}{An} - 2e^{- \frac{TH}{nA}}.\label{eq:matrix-conc-3}
	\end{align}
	The desired result follows from the above concentration results. Indeed, we first note that the event
	$$
	\left\lbrace T_1 \le \poly(\eta) \sqrt{\frac{TH}{nA}} \right\rbrace \bigcap \left\lbrace T_2 \le \poly(\eta) \sqrt{\frac{TH}{nA}} \right\rbrace \bigcap \left \lbrace T_3 \le \poly(\eta)\sqrt{\frac{TH}{nA} } \right \rbrace
	$$
	is a subset of the event
	$$
	\left \lbrace \left\Vert \hat{N}_{a,\Gamma_a} - \tilde{N}_a \right\Vert \le \poly(\eta)\sqrt{ \frac{TH}{nA} } \right\rbrace.
	$$ Thus, using the union bound, we conclude that:
	\begin{align*}
		\PP\left(\left\Vert \hat{N}_{a, \Gamma_a} - \tilde{N}_a \right\Vert \le \poly(\eta)\sqrt{ \frac{TH}{nA} }\right) \ge 1 - \frac{2}{An} - \frac{2e^{-n}}{A} - 2 e^{-TH/nA}
	\end{align*}
	which further implies, applying the union bound once more, that
	\begin{align*}
		\PP\left(\max_{a \in \cA}\left\Vert \hat{N}_{a, \Gamma_a} - \tilde{N}_a \right\Vert \le \poly(\eta)\sqrt{ \frac{TH}{nA} }\right) \ge 1 - \frac{2}{n} - 2e^{-n} - 2 e^{-TH/2nA},
	\end{align*}
	for $TH = \Omega(n)$ where $\Omega(\cdot)$ hides a dependence in $A \log(A)$ (where e used the fact that $TH \ge 2 nA \log(A)$ gives $e^{-\frac{TH}{nA} + \log(A)} \le e^{-\frac{TH}{2nA}}$).
\end{proof}

\paragraph{Bounding the contribution of the light couples.} In Lemmas \ref{lem:cb-T1} and \ref{lem:cb-T2}, we obtain bounds on the terms that depend on the light couples, that is when we sum over $(x, y) \in \cL$.

We start by bounding $T_1$:
\begin{lemma}[Bounding $T_1$] \label{lem:cb-T1}
  We have for all $TH = \Omega(n)$,
\begin{align*}
  \PP\left(  T_1 >   \poly(\eta) \left( 1 + \frac{\log(A)}{n} \right)\sqrt{\frac{TH}{nA}} \right) \le e^{-n - \log(A)}.
\end{align*}
\end{lemma}
\begin{proof}[Proof of Lemma \ref{lem:cb-T1}] First, let us observe that $T_1$ can be upper bounded as follows:
	\begin{align*}
		T_1  & \overset{(a)}{\le} \sup_{u,v \in \mathbb{S}^{n-1}} \sum_{(x,y) \in \cL \cap \cK^c} \vert u_x v_y\vert \hat{N}_a(x,y)  \\
		& \overset{(b)}{\le} m \sum_{(x,y) \in \cK^c} \hat{N}_a(x,y) \\
		& \le m  \hat{N}_a(\cX, \Gamma_a^c) \\
		& \le m  \max_{\cY: \vert \cY\vert = \gamma} \hat{N}_a(\cX, \cY)
	\end{align*}
	where in inequality \emph{(a)} we used the triangle inequality, in \emph{(b)} we used the fact that $(x, y) \in \cL$ (i.e., light couples). Next, we know from Lemma \ref{lem:cb-1} that
	for all $TH \ge 2 n\vert \cA \vert$, for all $\cY \subset \cX$ such that $\vert \cY \vert = \gamma$, for all $u > 0$,we have
	\begin{align*}
		\max_{\cY: \vert \cY\vert = \gamma} \PP( \hat{N}_a(\cX, \cY) > \poly(\eta) n (1 + \max (\sqrt{u}, u) ) ) \le e^{- n u}.
	\end{align*}
	Therefore, we have by union bound
	\begin{align*}
		\PP\left(  T_1 > \poly(\eta) m n (1 + \max(\sqrt{u}, u))  \right)  & \le \sum_{ \cY: \vert \cY \vert = \gamma} \PP\left( \hat{N}_a(\cX ,\cY)  > \poly(\eta) n (1 + \max(\sqrt{u}, u))  \right)  \\
		& \le e^{-n(u - \log(2))}
	\end{align*}
	where we used the upper bound $ \vert \lbrace \cY \subseteq \cX : \vert \cY \vert = \gamma \rbrace \vert \le 2^n$. To conclude, we simply recall that $m = \frac{1}{n} \sqrt{\frac{TH}{nA}}$, plug in its value in the final concentration, and choose $u$ sufficiently large.
\end{proof}

\begin{lemma}\label{lem:cb-1}
  Let $\cY \subseteq \cX$ such that $\vert \cY \vert = \left \lfloor n  \exp(-\frac{TH}{nA} \log\left( \frac{TH}{nA }  \right))\right\rfloor $. Assume that $TH \ge 2 n A$, then for all $u > 0$, we have
  \begin{align*}
    \PP\left( \left \vert \hat{N}_a(\cX, \cY) - \EE_\mu[\hat{N}_a(\cX, \cY)] \right\vert > \poly(\eta) n \max\left( \sqrt{u} , u  \right) \right) & \le 2 e^{- n u },
  \end{align*}
  Furthermore, we have
  \begin{align*}
    \PP\left(  \hat{N}_a(\cX, \cY)  > \poly(\eta) n (1 +   \max\left( \sqrt{u}, u \right)) \right) & \le 2 e^{- n u }.
  \end{align*}
\end{lemma}
\begin{proof}[Proof of Lemma \ref{lem:cb-1}]
	The result follows immediately from Theorem \ref{thm:bernstein-restart}.
	Indeed, consider the induced restarted Markov chain $(x_{h+1}^{(t)}, a_h^{(t)})_{t \ge 1, h \ge 1}$, referred to as $MC_1$ in Appendix \ref{app:equilibrium}.
	 By Proposition \ref{prop:mc2}, the transition kernel and initial distribution of $MC_1$ are both $\eta^3$-regular.
	 Introducing $\phi: (y,b) \mapsto \indicator\lbrace b = a , y \in \cY\rbrace$, we see that $\hat{N}_a(\cX, \cY) = \sum_{t=1}^T \sum_{h=1}^H \phi(x_{h+1}^{(t)}, a_{h}^{(t)})$. Therefore applying Theorem \ref{thm:bernstein-restart}, we obtain for all $u > 0$
	\begin{align*}
		\PP\left( \vert \hat{N}_a(\cX, \cY) - \EE_\mu[\hat{N}_a(\cX, \cY)] \vert > u \right) & \le 2 \exp\left( - \frac{u^2}{ 2 (1 + \sqrt{2}\eta^3(2\eta^3 - 1))^2 TH \frac{\gamma}{nA} \eta^3  + \frac{2}{3} (2\eta^3 - 1) u}  \right) \\
		& \le  2 \exp\left( -  \frac{1}{2} \min\left(\frac{nAu^2}{   c_1 TH\gamma}, \frac{u}{c_2} \right) \right),
	\end{align*}
	where $c_1 = (1 + \sqrt{2}\eta^3(2\eta^3 - 1))^2 \eta^3$ and $c_2 = \frac{1}{3} (2\eta^3 - 1)$.
	Reparameterizing by $u' = \frac{1}{2n} \min\left(\frac{nAu^2}{   c_1 TH\gamma}, \frac{u}{c_2} \right)$ yields
	\begin{align*}
		\PP\left( \vert \hat{N}_a(\cX, \cY) - \EE_\mu[\hat{N}_a(\cX, \cY)] \vert > \max\left( \sqrt{2c_1 n u' \frac{\gamma TH}{nA}}, 2c_2 n u' \right) \right) & \le 2 \exp\left( - n u' \right)
	\end{align*}
	Recalling that $\gamma = \left \lfloor n \exp(-\frac{TH}{n A} \log\left( \frac{TH}{n A}  \right))\right\rfloor $, we can easily verify that  $\gamma \le \frac{3 n^2 A}{TH}$ (which follows from the elementary inequalities that $x \log(x) + 1 \ge x$ and $e^{-x} \leq \frac{1}{x}$ for $x > 0$). Thus, we may simply write that
	\begin{align*}
		\PP\left( \vert \hat{N}_a(\cX, \cY) - \EE_\mu[\hat{N}_a(\cX, \cY)] \vert > n c \max\left( \sqrt{ u' }, u' \right) \right) & \le 2 \exp\left( - n u' \right),
	\end{align*}
	where $c = \max(\sqrt{6c_1}, 2c_2) = \poly(\eta)$.
	The additional bound follows from the fact that $\EE[\hat{N}_a(\cX, \cY)] \le 3 \eta^4 n$.
\end{proof}

The following lemma is a standard net argument, which will be useful for bounding $T_2$:
\begin{lemma}[$\epsilon$-net argument on the unit sphere]\label{lem:net-sphere}
	Let $W$ be an $n \times n$ random matrix, for all $\epsilon \in (0, 1/2)$, let $\cN_\epsilon$ an $\epsilon$-net of the unit sphere with respect to the Euclidian distance $\Vert \cdot \Vert_2$ and with minimal cardinality. Then, for any $\rho > 0$,
	\begin{align*}
		\PP( \Vert W \Vert > \rho ) \le \left( \frac{2}{\epsilon} + 1\right)^{2n} \max_{u, v \in \cN_\epsilon}\PP\left( u^\top W v >  (1- 2 \epsilon) \rho \right).
	\end{align*}
\end{lemma}
\begin{proof}[Proof of Lemma \ref{lem:net-sphere}]
	The existence of $\cN_\epsilon$ is shown in Corollary 4.2.13 of \cite{vershynin2018high}, which states that $|\cN_\epsilon| \leq \left( \frac{2}{\epsilon} + 1 \right)^n$.
	The statement then follows from Exercise 4.4.3 of \cite{vershynin2018high}, which relates the spectral norm of a random matrix to that of maximizing its rectangular form over $\cN_\epsilon$.
\end{proof}

We now bound $T_2$:
\begin{lemma}[Bounding $T_2$] \label{lem:cb-T2} We have, for $n \ge 2 \log(A)$,
  \begin{align*}
    \mathbb{P}\left(   T_2 > \poly(\eta) \sqrt{\frac{TH}{nA}}   \right) \le 2 e^{- n - \log(A)}.
  \end{align*}
\end{lemma}
\begin{proof}[Proof of Lemma \ref{lem:cb-T2}]
    %
  	% \paragraph{\emph{Bounding $T_2$ - Proof of \eqref{eq:matrix-conc-2}.}}
  	% We establish the following claim
  	% \begin{align*}
  	% 	\forall \rho > 0, \qquad \mathbb{P}\left(  T_{1}(u,v)  > \sqrt{\frac{TH}{n\vert \cA \vert}} \left(\rho +  \frac{3}{\eta^4}\right)   \right) \le 2 e^{- n\rho}
  	% \end{align*}
  	Our first step is to split the sum based on the parity of the time steps per episode. More precisely, we write $T_2 \le T_2^{even} + T_2^{odd}$ and define
  	\begin{align*}
  		T_2^{even} & =	\sup_{u,v \in \mathbb{S}^{n-1}}\sum_{t=1}^T \sum_{h=1}^{\left \lfloor \frac{H-1}{2}\right \rfloor} \left(\sum_{(x,y) \in \cL} u_x v_y \hat{N}_{a,t,2h}(x,y)\right) - u^\top \tilde{N}_{a,t,2h} v, \\
  		T_2^{odd}  & =	\sup_{u,v \in \mathbb{S}^{n-1}}\sum_{t=1}^T \sum_{h=0}^{\left \lfloor \frac{H}{2}\right \rfloor - 1} \left(\sum_{(x,y) \in \cL} u_x v_y \hat{N}_{a, t,2h+1}(x,y) \right) - u^\top \tilde{N}_{a,t,2h+1}v.
  	\end{align*}
    We note that $T_2^{even}$ and $T_2^{odd}$ are expressed as supremum over the unit sphere. To proceed we will deploy a net argument, but first let us define for all $u,v \in \mathbb{S}^{n-1}$,
    \begin{align*}
      T_2^{even}(u,v) & = \sum_{t=1}^T \sum_{h=1}^{\left \lfloor \frac{H-1}{2}\right \rfloor} \left(\sum_{(x,y) \in \cL} u_x v_y \hat{N}_{a, t,2h+1}(x,y)\right) - u^\top \tilde{N}_{a,t,2h+1}v, \\
      T_2^{odd}(u,v) & = \sum_{t=1}^T \sum_{h=0}^{\left \lfloor \frac{H}{2}\right \rfloor - 1}\left( \sum_{(x,y) \in \cL} u_x v_y \hat{N}_{a, t,2h+1}(x,y)\right) - u^\top \tilde{N}_{a,t,2h+1}v.
    \end{align*}
  	The analysis of $T_2^{even}(u,v)$ and $T_2^{odd}(u,v)$ will be the same, therefore, and without loss of generality, we only show how to obtain a concentration bound for $T_2^{odd}(u,v)$. We start by computing its moment-generating function. First, let $\lambda > 0$, $h \ge 3$, we have:
  	\begin{align*}
  		& \EE_\mu \left[ \exp( \lambda \left(\sum_{(x,y) \in \cL } u_x v_y \hat{N}_{a,t,h}(x,y)    \right)  - \lambda u^\top \tilde{N}_{a,t,h} v )  \bigg\vert \tilde{N}_{a,t,h-2} \right] \\
  		& \qquad = \left(\sum_{(x,y) \in \cX \times \cX } (\indicator\lbrace (x,y) \not\in \cL \rbrace + \indicator\lbrace (x,y) \in \cL \rbrace e^{\lambda u_x v_y} ) \tilde{N}_{a,t,h}(x,y)  \right) \exp(- \lambda u^\top \tilde{N}_{a,t,h} v ) \\
  		& \qquad = \left( 1 + \sum_{(x,y) \in \cL}  \tilde{N}_{a,t,h}(x,y) \left(e^{\lambda u_x v_y} - 1\right) \right) \exp(- \lambda u^\top \tilde{N}_{a,t,h} v ) \\
  		& \qquad \overset{(a)}{\le} \exp( \sum_{(x,y) \in \cL} \tilde{N}_{a,t,h} (x,y) \left(\lambda u_x v_y + \frac{ e^{\lambda \vert u_x v_y \vert}}{2}  (\lambda  u_x v_y )^2 \right)   - u^\top \tilde{N}_{a,t,h} v ) \\
  		& \qquad \overset{(b)}{\le} \exp( - \lambda \left(\sum_{(x,y) \in \cL^c}  u_x v_y \tilde{N}_{a,t,h} (x,y)\right)  + \frac{ \lambda^2 e^{\lambda m}}{2} \left(\sum_{ (x,y) \in \cL}  (u_x v_y)^2 \tilde{N}_{a,t,h}(x,y) \right)  ) \\
  		& \qquad \overset{(c)}{\le}  \exp( \lambda \frac{\eta^5}{n^2 A m}  + \lambda^2 e^{\lambda m} \frac{\eta^5 }{2n^2 A})
  	\end{align*}
    where in the inequality \emph{(a)}, we use the elementary inequalities $1+ x \le \exp(x)$, then $e^{\lambda  u_x v_y } - 1  \le \lambda  u_x v_y  + \frac{e^{\lambda\vert u_x v_y\vert}}{2} (\lambda u_x v_u )^2$, in the inequality \emph{(b)}, we used $\vert u_x v_u \vert \le m$ for all $(x, y) \in \cL$, and in inequality \emph{(c)}, we used the fact that $\max_{x, y \in \cX} \tilde{N}_{a,t,h}(x,y) \le \frac{\eta^5}{n^2 A}$ along with
  	\begin{align*}
  		\left\vert \sum_{(x,y) \in \cL^c}  u_x v_y \tilde{N}_{a,t,h}(x,y) \right\vert \le  \frac{\max_{(x,y) \in \cL} \vert \tilde{N}_{a,t,h}(x,y) \vert}{\min_{(i,j) \in \cL^c} \vert u_i v_j\vert  }  \sum_{(i,j) \in \cL^c} \vert u_i v_j \vert^2  \le \frac{\eta^5}{n^2Am}
  	\end{align*}
  	and
  	\begin{align*}
  		\sum_{ (x,y) \in \cL}  (u_x v_y)^2 \tilde{N}_{a,t,h}(x,y) \le \frac{\eta^5}{n^2 A}.
  	\end{align*}
    Therefore, using a peeling argument, we obtain, for all $\lambda > 0$,
    \begin{align*}
      & \EE_\mu \left[ \exp( \lambda  \sum_{t=1}^T \sum_{h=0}^{\left \lfloor \frac{H}{2}\right \rfloor - 1} \left( \sum_{(x,y) \in \cL } u_x v_y \hat{N}_{a,t,h}(x,y)    \right)  - \lambda u^\top \tilde{N}_{a,t,h} v )  \right] \\
      & \qquad \qquad \qquad \qquad \qquad \qquad \qquad \qquad \qquad \qquad \le \exp( \lambda \frac{\eta^5TH}{2 n^2 A m}  + \lambda^2 e^{\lambda m} \frac{\eta^5 TH}{4 n^2 A}).
    \end{align*}
    Using Markov inequality and reparameterizing by $\lambda = z/m$, we obtain
  	\begin{align*}
  		\mathbb{P}\left(  T_{2}^{odd}(u,v) > n m \eta^5 (\rho + 1)   \right) & \le \inf_{z > 0}\exp(   z \frac{\eta^5 TH }{n^2A m^2} + z^2 e^z \frac{\eta^5 TH}{2n^2 A  m^2}  - z n \eta^5(\rho+1)) \\
  		& \le  \inf_{z > 0}\exp\left( - n\eta^5  \left( z \rho - z^2 e^z /2      \right)\right),
  	\end{align*}
    where the last inequality follows by plugging in the value of $m = \frac{1}{n} \sqrt{\frac{TH}{nA}}$. Thus, at the end, say for $\rho \in [0, \kappa e^\kappa]$, after optimizing for $z \in (0, \kappa)$, we obtain
  	\begin{align*}
  		\mathbb{P}\left(  T_{2}^{odd}(u,v) > \sqrt{\frac{TH}{nA}} \left(\eta^{5/2} e^{\kappa/2} \sqrt{\rho} +  \eta^5\right)   \right) \le e^{- n\rho}.
  	\end{align*}
    Finally using an $\epsilon$-net argument with $\epsilon = 1/4$ (Lemma \ref{lem:net-sphere}), we get, for all $\kappa > 0$, $\rho \in (0, \kappa e^{\kappa} - 10)$,
    \begin{align*}
      \mathbb{P}\left(   T_2^{odd} > 2 \sqrt{\frac{TH}{n\vert \cA \vert}} \left( \eta^{5/2}e^{\kappa/2}\sqrt{\rho + 10} +  \eta^5 \right)   \right) \le e^{- n \rho}
    \end{align*}
    and similarly
    \begin{align*}
      \mathbb{P}\left(   T_2^{even} > 2 \sqrt{\frac{TH}{n\vert \cA \vert}} \left( \eta^{5/2}e^{\kappa/2}\sqrt{\rho + 10} +  \eta^5 \right)   \right) \le e^{- n \rho},
    \end{align*}
    which finally, by union bound, gives
    \begin{align*}
      \mathbb{P}\left(   T_2 > 4 \sqrt{\frac{TH}{n A}} \left( \eta^{5/2}e^{\kappa/2}\sqrt{\rho + 10} +  \eta^5 \right)   \right) \le 2 e^{- n \rho}.
    \end{align*}
		The final statement follows by choosing $\rho = 1$ and $\kappa$ large enough.
\end{proof}

\paragraph{Bounding the contribution of the heavy couples.}

The analysis of $T_3$ relies on the discrepancy property which is satisfied by the trimmed matrix. This property will be defined in the proof of Lemma \ref{lem:cb-T3} and remains crucial in order to obtain high probability bounds in the regime when $TH = \omega(n)$ and $TH = \mathcal{O}(n \log(n))$.

Lemma \ref{lem:avg-prop} will be required for the analysis of $T_3$.
\begin{lemma}\label{lem:avg-prop}
	We have
	\begin{align*}
		\PP\left( \max_{x \in \Gamma_a} \hat{N}_a(x, \cX) \le \poly(\eta) \frac{TH}{nA} \right) \ge 1 - 2 \exp\left( - \frac{TH}{n A}\right).
	\end{align*}
\end{lemma}
\begin{proof}[Proof of Lemma \ref{lem:avg-prop}]
	The proof follows from an immediate application of Theorem \ref{thm:bernstein-restart}. In fact, the proof is identical to that of Lemma \ref{lem:cb-1}. Therefore we refer the reader to that proof and omit it.
\end{proof}

Let us now present the analysis of $T_3$:
\begin{lemma}[Bounding $T_3$] \label{lem:cb-T3} For $TH = \Omega(n)$, we have
  \begin{align*}
    \PP\left( T_3 \le \poly(\eta) \sqrt{\frac{TH}{nA}}\right) \ge 1 - \frac{2}{nA} - 2e^{-\frac{TH}{nA}}.
  \end{align*}
\end{lemma}
\begin{proof}[Proof of Lemma \ref{lem:cb-T3}]
	Again, it will be convenient to write $T_3 \le T_3^{even} + T_3^{odd}$, where
	\begin{align*}
		T_3^{even} & =	\sup_{u,v \in \mathbb{S}^{n-1}}\sum_{t=1}^T \sum_{h=1}^{\left \lfloor \frac{H-1}{2}\right \rfloor} \sum_{(x,y) \in \cH \cap \cK} u_x v_y \hat{N}_{a,t,2h}(x,y), \\
		T_3^{odd}  & =	\sup_{u,v \in \mathbb{S}^{n-1}}\sum_{t=1}^T \sum_{h=0}^{\left \lfloor \frac{H}{2}\right \rfloor - 1} \sum_{(x,y) \in \cH \cap \cK} u_x v_y \hat{N}_{a,t,2h+1}(x,y) .
	\end{align*}
	Then without loss of generality, we may focus on $T_3^{odd}$, as the analysis of $T_3^{even}$ will follow similarly. In the sparse regime when we do not have enough observations, we cannot unfortunately use a standard argument that combines a uniform concentration bound with a net argument. Instead, we will use the so-called discrepancy property.

  \paragraph{Discrepancy property.} First, in order to declutter notations, we shall denote the random matrix  $Q = \big( \sum_{t=1}^T \sum_{h=0}^{\lfloor H/2 \rfloor - 1} \hat{N}_{a,t,2h+1}(x,y) \indicator\{ x, y \in \Gamma_a \} \big)_{x,y \in \cX} $, and for all $\cI, \cJ \subseteq \cX$ define the quantity $e(\cI, \cJ) \triangleq \sum_{(x,y) \in \cI \times \cJ}  Q(x,y)$. Now, to obtain the desired result, we will follow a similar approach as that used in \cite{keshavan2010matrix} which relies on showing that the matrix $Q$ satisfies the so-called \emph{discrepancy property}. We say that a random matrix $Q$ satisfies the \emph{discrepancy property}, if there exist $\xi_1, \xi_2 > 0$ such that for all $\cI, \cJ \subseteq \cX$, the following holds:
  \begin{itemize}
		\item[\emph{(i)}] $ \frac{e(\cI, \cJ) n^2 A}{\vert \cI \vert \vert \cJ\vert TH} \le \xi_1 $
    \item[\emph{(ii)}]$ e(\cI, \cJ) \log( \frac{e(\cI, \cJ) n^2 A}{\vert \cI \vert \vert \cJ\vert TH} ) \le \xi_2 \max(\vert \cI\vert, \vert \cJ \vert) \log\left( \frac{n}{\max(\vert \cI\vert, \vert \cJ \vert)} \right)$
	\end{itemize}

	We know from Remark 4.5 in \cite{keshavan2010matrix}, that if the matrix $Q$ satisfies the {discrepancy property}, then for some given $\epsilon$-net $\cN_\epsilon$ of the unit sphere $\mathbb{S}^{n-1}$, there exists an absolute constant $C>0$ such that
	$$
	\sup_{u,v\in \cN_\epsilon } \sum_{(i,j)\in \cH} u_x v_y Q(i,j) \le \sqrt{ \frac{TH}{n A} }
	$$
	Now to conclude, it remains to show that $Q$ satisfies the {discrepancy property} with high probability. Let $\cI, \cJ \subseteq \cX$, we may assume w.l.o.g that $\vert \cJ \vert \ge \vert \cI \vert$. We distinguish between two cases:

	\begin{itemize}
		\item [$\rightarrow$] \emph{Case 1: if $\vert \cJ\vert \ge n/5$.} First, because of trimming we have the average bounded degree property which follows from Lemma \ref{lem:avg-prop}
		\begin{align*}
			\PP\left( \max_{x \in \Gamma_a} \hat{N}_a(x, \cX) \le \poly(\eta) \frac{TH}{nA} \right) \ge 1 - 2 \exp\left( - \frac{TH}{n A}\right).
		\end{align*}
Thus, with probability at least $1 - 2 \exp\left( - \frac{TH}{n A}\right)$, we have
		$$
		e(\cI, \cJ) \le \vert \cI \vert \max_{x \in \Gamma_a} e(x, \cX) \le \frac{5 C \vert \cI \vert \vert \cJ\vert TH}{n^2 A}
		$$
		which leads to
		$$
		\frac{e(\cI, \cJ) n^2 \vert \cA \vert}{\vert \cI\vert \vert \cJ\vert} \le 5C
    $$
		whenever $TH = \Omega(n)$.
		\item [$\rightarrow$] \emph{Case 2: if $\vert \cJ\vert \le n/5$.}  We start by defining
    $$
     \mu(\cI,\cJ) \triangleq  \EE\left[  \sum_{x \in \cI, y \in \cJ}\sum_{t = 1}^T \sum_{h=0}^{\lfloor H/2\rfloor - 1} \hat{N}_{a,t, 2h+1}(x,y)\right]
    $$
    and note that in view of Assumptions \ref{assumption:SA}-\ref{assumption:uniform}, it can be easily verified that $ \frac{ TH \vert \cI \vert \vert \cJ \vert }{ 2 \eta^5  n^2 A} \le  \mu(\cI,\cJ) \le \frac{\eta^5 TH \vert \cI \vert \vert \cJ \vert }{2 n^2 A}$. Then letting $\kappa_\star(\vert \cI \vert, \vert \cJ\vert) = \max \lbrace  \kappa_0, t^\star( \cI , \cJ ) \rbrace $ where $\kappa_0$ is chosen large enough, and $t^\star(\cI, \cJ)$ is defined as the constant $t$ satisfying
    $$
    t\log(t) =  \frac{1}{ \mu(\cI, \cJ)}  \xi_2 \vert \cJ \vert \log(\frac{n}{\vert \cJ \vert}).
    $$
    Next, we define the event
    $$
    \cE = \bigcap_{ \cI, \cJ \subset \cX: \vert \cI \vert \le  \vert \cJ \vert \le n/5} \lbrace e(\cI, \cJ) \le \kappa_\star (\cI, \cJ) \mu(\cI, \cJ) \rbrace.
    $$
    We claim that whenever the event $\cE$ holds then $Q$ satisfies either condition \emph{(i)} or \emph{(ii)} of the discrepancy property. Indeed, let $\cI, \cJ \subseteq \cX$, such that $\vert \cI \vert \le \vert \cJ\vert \le n/5$, if $\kappa_\star(\vert \cI \vert, \vert \cJ\vert) = \kappa_0$, then under $\cE$, it holds that $e(\cI, \cJ) \le \kappa_0 \mu(\cI, \cJ) \le \frac{\kappa_0 \eta^5 TH \vert \cI \vert \vert \cJ \vert }{2 n^2 A}$, which clearly means that property \emph{(i)} is satisfied. On the other hand, if $\kappa_\star(\vert \cI \vert, \vert \cJ\vert) =  t^\star(\cI, \cJ)$, then under the event $\cE$, we have $ \frac{\eta^5 e(\cI, \cJ)}{2\mu(\cI, \cJ)} \log(\frac{\eta^5 e(\cI, \cJ)}{2\mu(\cI, \cJ)}) \le \eta^5 t^\star(\cI, \cJ)/2 \log(\eta^5t^\star(\cI, \cJ)/2) \le \eta^5 t^\star(\cI, \cJ) \log(t^\star(\cI, \cJ)) \le \frac{\eta^5}{\mu(\cI, \cJ)} \xi_2 \vert \cJ\vert \log(\frac{n}{\vert\cJ\vert})$ by monotonicity of $t\log(t)$ and when choosing $\kappa_0$ large enough so that $t_\star(\cI, \cJ) \ge \eta^5 / 2$ ($\kappa_0 \ge \eta^5/2$). This implies that $e(\cI, \cJ) \log(\frac{e(\cI, \cJ) n^2 A}{ TH \vert \cI \vert \vert \cJ\vert }) \le \eta^5 \xi_2 \vert \cJ\vert \log(\frac{n}{\vert\cJ\vert})$. Thus, the property \emph{(ii)} is satisfied in this case. It remains to show that the event $\cE$ holds with high probability. We have by union bound
    \begin{align*}
      \PP(\cE^c) & \le \sum_{\cI, \cJ \in \cX: \vert \cI\vert \le \vert \cJ  \vert \le n/5} \PP( e(\cI, \cJ) > k_\star(\cI, \cJ)  \mu(\cI, \cJ)) \\
      & \overset{(a)}{\le} \sum_{\cI, \cJ \in \cX: \vert \cI\vert \le \vert \cJ  \vert \le n/5} \exp( - \frac{1}{2}k_\star(\cI, \cJ) \log( k_\star(\cI, \cJ) ) \mu(\cI, \cJ)) \\
      & \overset{(b)}{\le} \sum_{\cI, \cJ \in \cX: \vert \cI\vert \le \vert \cJ  \vert \le n/5} \exp( - \frac{1}{2} \xi_2 \vert \cJ \vert \log(\frac{n}{\vert \cJ\vert})) \\
      & \le \sum_{1 \le i \le j \le n/5} 2 {n \choose i} {n \choose j} \exp(- \xi_2 j\log(n/j)) \\
      & \le \sum_{1 \le i \le j \le n/5} 2  \exp( (4  - \xi_2) j\log(n/j))
    \end{align*}
    where in inequality \emph{(a)}, we applied Lemma \ref{lem:claim-dp}, and in the inequality \emph{(b)}, we use the defintion of $\kappa_\star(\cI, \cJ)$ and finally chose $\xi_2 > 4$. In particular, choosing $\xi_2 = 7 + \log(A)/\log(n)$ ensures that
    \begin{align*}
      \PP(\cE^c) \le \frac{1}{n A}.
    \end{align*}
	\end{itemize}
We have just shown that
\begin{align*}
  \PP\left( Q \textrm{ satisifes the discrepency property} \right) \ge 1 - \frac{1}{nA} - e^{-\frac{TH}{nA}}.
\end{align*}
This implies that:
\begin{align*}
  \PP\left( T_3^{odd}\le \poly(\eta) \sqrt{\frac{TH}{nA}} \right) \ge 1 - \frac{1}{nA} - e^{-\frac{TH}{nA}}.
\end{align*}
We can show the same for $T_3^{even}$. Therefore, we can conclude that for $n \gtrsim A $,
\begin{align*}
  \PP\left( T_3 \le \poly(\eta) \sqrt{\frac{TH}{nA}}\right) \ge 1 - \frac{2}{nA} - 2e^{-\frac{TH}{nA}}.
\end{align*}
\end{proof}

\begin{lemma}\label{lem:claim-dp}
  Let $k \ge e^2 \eta^8$, then
	\begin{align*}
		\PP(e(\cI, \cJ)  \ge k \mu(\cI, \cJ)) \le \exp( - \frac{1}{2}k \log(k) \mu(\cI, \cJ) ).
	\end{align*}
\end{lemma}

	\begin{proof}[Proof of Lemma \ref{lem:claim-dp}]
		We can easily verify see, from Assumptions \ref{assumption:SA}-\ref{assumption:uniform}, that $ \frac{ TH \vert \cI \vert \vert \cJ \vert }{ 2 \eta^5  n^2 A} \le  \mu(\cI,\cJ) \le \frac{\eta^5 TH \vert \cI \vert \vert \cJ \vert }{2 n^2 A}$. Next, we compute the moment-generating function of $e(\cI, \cJ)$. Let $\lambda> 0$, we have, via a peeling argument,
		\begin{align*}
			\EE[\exp( \lambda e(\cI, \cJ) ) ] & = \EE\left[ \exp(  \lambda \sum_{i \in \cI, j \in \cJ}\sum_{t = 1}^T \sum_{h=0}^{\lfloor H/2\rfloor - 1} \hat{N}_{a,t, 2h+1}(x,y) )  \right] \\
			& \le \prod_{t=1}^T \prod_{h=1}^{\lfloor H/2 \rfloor - 1} \left( 1  + \frac{\eta^2  \vert \cI \vert \vert \cJ \vert}{ n^2 A} e^\lambda \right) \\
			& \le  \prod_{t=1}^T \prod_{h=1}^{\lfloor H/2 \rfloor - 1} \exp( \frac{\eta^2  \vert \cI \vert \vert \cJ \vert}{ n^2 A} e^\lambda  ) \\
			& \le \exp( \frac{\eta^2  TH \vert \cI \vert \vert \cJ \vert}{ 2 n^2 A} e^\lambda  ).
		\end{align*}
		Now by Markov inequality, we have:
		\begin{align*}
			\PP \left(  e(\cI, \cJ) \ge k \mu(\cI, \cJ) \right) & \le \inf_{\lambda > 0} \EE[ \exp(\lambda ( e(\cI, \cJ) -  k \mu(\cI, \cJ)))] \\
			& \le \inf_{\lambda > 0} \exp( \frac{\eta^2  TH \vert \cI \vert \vert \cJ \vert}{ 2 n^2 A} e^\lambda  - \lambda k \mu(\cI, \cJ) ) \\
			& = \exp\left(   k \mu(\cI, \cJ) \left(1 - \log( k \mu(\cI, \cJ) \frac{2 n^2 A}{ \eta^2 TH \vert \cI \vert \vert \cJ \vert} )\right)\right) \\
			& \le \exp\left(   k \mu(\cI, \cJ) \left(1 - \log( \frac{k}{\eta^4} )\right)\right) \\
			& \le \exp( -\frac{1}{2}k\log(k) \mu(\cI, \cJ) )
		\end{align*}
		where in the last inequality we used $ k \ge e^2 \eta^8$. This concludes the proof of the claim.
	\end{proof}

%\subsection{Proofs - Trimmed random matrix}

\newpage
% !TEX root = ./main.tex

\section{Proof of Theorem \ref{thm:likelihood-improvement} (i) -- Iterative Likelihood Improvement}
\label{app:likelihood-improvement}

The likelihood improvement steps are inspired by our lower bound.
Specifically, the set of misclassified contexts is divided into two sets: the first set corresponds to "well-behaved'' contexts where the empirical lower bound divergence is high, meaning that these contexts are likely to be classified accurately; the second set corresponds to the other contexts.

For the proof, we introduce the following notation: for $x \in \cX$ and $j \in \cS$,
\begin{align}
	\hat{I}_j(x;  \Phi) \triangleq \sum_{a \in \cA} \sum_{s \in \mathcal{S}} &\left(  \hat{N}_a(f^{-1}(s), x) \log \frac{p(f(x) | s, a)}{\tilde{c}_j p(j | s, a)} +  \hat{N}_a(x, f^{-1}(s)) \log \frac{p(s | f(x), a)}{p(s | j, a)} \right),
\end{align}
where $\tilde{c}_j := \frac{\sum_{s ,a} m_\rho(s, a) p(j | s, a)}{\sum_{s, a} m_\rho(s, a) p(f(x) | s, a)}$.
It can be easily shown that $1/\eta \leq \tilde{c}_j \leq \eta$.
Moreover, defining $p^{bwd}(s, a | j) := \frac{m_\rho(s, a) p(j | s, a)}{\sum_{\tilde{s} \in \cS} \sum_{\tilde{a} \in \cA} m_\rho(\tilde{s}, \tilde{a}) p(j | \tilde{s}, \tilde{a})}$, we have that $\frac{p^{bwd}(s, a | f(x))}{p^{bwd}(s, a | j)} = \frac{p(f(x)|s,a)}{\tilde{c}_j p(j|s,a)}$.

Following \cite{YunP16} and \cite{SandersPY20}, we start by defining the subset of  "well-behaved'' contexts:

\begin{definition}[Well-behaved contexts]
\label{def:well-defined}
The set of {\bf well-behaved} contexts $\mathcal{H} \subset \mathcal{X}$ is the largest subset of $\Gamma \triangleq \bigcup_{a \in \cA} \Gamma_a$ with the following properties: for $x \in \mathcal{H}$,
\begin{itemize}
		\item[(H1)] For all $j \neq f(x)$,
		\begin{equation}
				\label{eq:well-behaved1}
				\hat{I}_j(x; \Phi) \geq \frac{1}{4\eta^2} \frac{TH}{n} I(x; \Phi),
			\end{equation}

		\item[(H2)]
		\begin{equation}
				\sum_{a \in \cA} \left\{ \hat{N}_a(x, \cX \setminus \cH) + \hat{N}_a(\cX \setminus \cH, x) \right\} \leq 2 \left( \log \frac{TH}{n}\right)^2.
			\end{equation}
	\end{itemize}
\end{definition}

%\begin{definition}
%	The set of {\bf well-behaved} contexts $\mathcal{H} \subset \mathcal{X}$ is the largest subset of $\Gamma \triangleq \bigcup_{a \in \cA} \Gamma_a$ with the following properties: for $z_i \in \mathcal{H}$,
%	\begin{itemize}
%		\item[(H1)] Whenever $z_i \in f^{-1}(i)$, for all $j \not= i$,
%		\begin{equation}
%		\label{eq:well-behaved1}
%			\hat{I}_{z_i, j}^U(\alpha, p) \geq \frac{1}{2(1 + 2\eta^5)} \frac{TH}{n} I^U(\alpha, p).
%		\end{equation}
%
%		\item[(H2)] For all $a \in \cA$,
%		\begin{equation}
%			\hat{N}_a\left( z_i, \mathcal{X} \setminus \mathcal{H} \right) + \hat{N}_a\left( \mathcal{X} \setminus \mathcal{H}, z_i \right) \leq 2  \log \left(\frac{TH}{nA}\right)^2.
%		\end{equation}
%	\end{itemize}
%\end{definition}

Let $\mathcal{E}^{(\ell)}$ be the set of misclassified contexts after the $\ell$-th iteration of Algorithm 2, and let $\mathcal{E}_\mathcal{H}^{(\ell)} := \mathcal{E}^{(\ell)} \cap \mathcal{H}$. The basic idea is to show that $\mathcal{E}_\mathcal{H}^{(\ell)}$ vanishes for $\ell = \floor{\log(nA)}$, and then we obtain a worst-case upper bound for the error rate by simply setting all vertices in $\mathcal{H}^\complement$ to be misclassified.

In Section~\ref{sec:bndHc}, we show the following:
\begin{proposition}
\label{prop:Hc-bound}
	If $I(\Phi) > 0$, then for some universal constant $C > 0$, we have, w.h.p.
	\begin{equation}
		\left| \mathcal{E}_{\mathcal{H}^\complement}^{(\ell)} \right| \leq \left| \mathcal{H}^\complement \right| = \mathcal{O} \left( \sum_{x \in \cX} \exp\left( -C \frac{TH}{n} I(x; \Phi) \right) \right).
	\end{equation}
\end{proposition}

In Section~\ref{sec:bndEl}, we show the following:
\begin{proposition}
\label{prop:HE-bound}
	If $I(\Phi) > 0$, then w.h.p.,
	\begin{equation}
		 \left| \mathcal{E}_\mathcal{H}^{(\ell)} \right| = 0 \quad \mbox{when}\quad \ell = \log(nA) .
	\end{equation}
\end{proposition}
From the above results and the fact that $\left| \mathcal{E}^{(\ell)} \right| = \left| \mathcal{E}_{\mathcal{H}^\complement}^{(\ell)} \right| + \left| \mathcal{E}_{\mathcal{H}}^{(\ell)} \right|$, we conclude the proof of Theorem~\ref{thm:likelihood-improvement}.

\subsection{Proof of Proposition \ref{prop:Hc-bound} -- Bounding $| \mathcal{H}^\complement |$} \label{sec:bndHc}
First, note that
\begin{equation}
	\frac{|\Gamma^\complement|}{n} = \frac{\left| \bigcap_{a \in \cA} \Gamma_a^\complement \right|}{n} \leq \frac{\min_{a \in \cA} \left| \Gamma_a^\complement \right|}{n} = \exp \left( -\frac{TH}{nA} \right) \stackrel{n\to\infty}{\longrightarrow} 0,
\end{equation}
i.e., the number of contexts {\it not} in $\Gamma$ is negligible, and thus we only need to worry about $\mathcal{H}^\complement \cap \Gamma$. From now on, all the contexts are assumed to be in $\Gamma$.

The proof consists of two parts.

\underline{\textbf{\textit{The set of $x \in \mathcal{X}$ such that (H1) does not hold is bounded}}}\\
Denote $\mathcal{H}_1 \triangleq \left\{ x \in \cX : \text{ (H1) holds} \right\}$.
To bound the cardinality of such set, we start by bounding $|\cH_1^\complement|$, which follows from the following concentration result, whose proof is postponed to Section \ref{proof:concentration-I}:

\begin{proposition}[Concentration regarding the rate function $I$]
	\label{prop:concentration-I}
%	With $\frac{TH}{n}$ large enough (larger than a constant that does not depend on $n, S, A$), suppose that the trajectories were generated using some fixed (memoryless) policy $\rho$.
%	Then,
	For any $x \in \cX$ and $j \in \cS$,
	\begin{equation}
		\label{eq:concentration-I}
		\PP\left[ \sum_{t=1}^T \sum_{h=1}^{H-1} \phi_j\left( \tilde{x}_h^{(t)} \right) < \frac{1}{4\eta^2} \frac{TH}{n} I(x; \Phi) \right] \leq 2 \exp\left( - 2C \frac{TH}{n} I(x; \Phi) \right),
	\end{equation}
	where $C > 0$ is an universal constant, $\tilde{x}_h^{(t)} :
	= (x_h^{(t)}, a_h^{(t)}, x_{h+1}^{(t)})$, and
	\begin{align*}
		\phi_j(\tilde{x}_h^{(t)}) &\triangleq \sum_{a \in \cA} \sum_{s \in \mathcal{S}} \left( \indicator\left[x_h^{(t)} = x, a_h^{(t)} = a, x_{h+1}^{(t)} \in f^{-1}(s)\right] \log \frac{p(s | f(x), a)}{\tilde{c}_j p(s | j, a)} \right. \\
		&\qquad\qquad\qquad \left. +  \indicator\left[ x_h^{(t)} \in f^{-1}(s), a_h^{(t)} = a, x_{h+1}^{(t)} = x \right] \log \frac{p(f(x) | s, a)}{p(j | s, a)} \right).
	\end{align*}
\end{proposition}

Then by observing that $\hat{I}_j(x; \Phi) = \sum_{t=1}^T \sum_{h=1}^{H-1} \phi_j(\tilde{x}_h^{(t)})$, from Proposition \ref{prop:concentration-I},
\begin{align*}
	\EE\left[ \left| \mathcal{H}_1^\complement \right| \right] &= \sum_{x \in \cX} \PP_\Phi\left[ \exists j \neq f(x) \quad \text{s.t.} \quad \hat{I}_j(x; \Phi) < \frac{1}{4\eta^2} \frac{TH}{n} I(x; \Phi) \right] \\
	&\leq \sum_{x \in \cX} \sum_{j \neq f(x)} \PP_\Phi\left[ \hat{I}_j(x; \Phi) < \frac{1}{4\eta^2} \frac{TH}{n} I(x; \Phi) \Bigg| x, j \right] \\
	&\leq (S - 1) \sum_{x \in \cX} \exp\left( - 2C \frac{TH}{n} I(x; \Phi) \right).
\end{align*}
%where $C$ is the universal constant introduced in the lower bound.

We then conclude using Markov inequality:
\begin{align*}
	\PP\left[ \left| \mathcal{H}_1^\complement \right| \geq \sum_{x \in \cX} \exp\left( -C \frac{TH}{n} I(x; \Phi) \right) \right] &\leq \frac{\EE\left[ \left| \mathcal{H}_1^\complement \right| \right]}{\sum_{x \in \cX} \exp\left( -C \frac{TH}{n} I(x; \Phi) \right)} \\
	&\leq (S - 1) \sum_{x \in \cX} \exp\left( -C \frac{TH}{n} I(x; \Phi) \right) \rightarrow 0,
\end{align*}
where the $\rightarrow$ holds in the limit $n \rightarrow \infty$ when $I(\Phi) > 0$ and at least $TH = \omega(n)$.

\underline{\textbf{\textit{Final construction}}}\\
Now consider the following iterative constructions of sets $\{Z(t)\}_{t \geq 0}$:
\begin{enumerate}
	\item $Z(0) = \mathcal{H}_1^\complement$;

	\item $Z(t) = Z(t - 1) \cup \{z(t)\}$, where $z(t)$ does not satisfy (H2) w.r.t. $Z(t - 1)$, i.e.,
	\begin{equation*}
		\sum_{a \in \cA} \left\{ \hat{N}_a(z(t), Z(t-1)) + \hat{N}_a(Z(t-1), z(t)) \right\} > 2 \left(\log \frac{TH}{n}\right)^2;
	\end{equation*}

	\item If such $z(t)$ does not exist, stop and let $t^*$ be the total number of iterations. We lastly define $Z := Z(t^*)$.
\end{enumerate}

First, observe that if $x \in Z^\complement$, then $x$ satisfies (H1) and (H2), i.e., $x \in \cH$. By the maximality of $\cH$, $Z^\complement \subseteq \cH$, which implies that, w.h.p.,
\begin{equation*}
	\left| \cH^\complement \right| \leq |Z| \leq |\cH_1^\complement| + t^* \leq \sum_{x \in \cX} \exp\left( -C \frac{TH}{n} I(x; \Phi) \right) + t^*.
\end{equation*}

Thus it suffices to bound $t^*$.
Define $s \triangleq \floor*{2 \sum_{x \in \cX} \exp\left( -C \frac{TH}{n} I(x; \Phi) \right)}$.

Obviously we have that $|Z(0)| \leq \frac{s}{2}$. We consider two cases:
\begin{itemize}
	\item $s = 0$, i.e., $Z(0) = \emptyset$. Then, we have that for all $x \in \cX$,
	\begin{equation*}
		\sum_{a \in \cA} \left\{ \hat{N}_a(x, \emptyset) + \hat{N}_a(\emptyset, x) \right\} = 0 \leq 2 \left(\log \frac{TH}{n}\right)^2.
	\end{equation*}
This means that $t^* = 0$.

	\item $s \geq 1$. By construction, we have that $|Z(t)| \leq \frac{s}{2} + t$.
	We then bound $\hat{N}(Z(t), Z(t)) := \sum_{a \in \cA} \hat{N}_a(Z(t),Z(t))$ as follows:
	\begin{align*}
		& \hat{N}(Z(t), Z(t)) \\
		&= \hat{N}(z(t), Z(t - 1)) + \hat{N}(Z(t - 1), z(t)) + \hat{N}(Z(t - 1), Z(t - 1)) + \hat{N}(z(t), z(t)) \\
		&> \hat{N}(Z(t - 1), Z(t - 1)) + 2 \left(\log \frac{TH}{n}\right)^2.
	\end{align*}
	Unfolding the recursion gives $\hat{H}(Z(t -1), Z(t - 1)) > 2t \left(\log \frac{TH}{n}\right)^2$.

	We now claim that $t^* < \frac{s}{2}$ with high probability.
	To show this, we proceed by contradiction and suppose that $t^* \geq \frac{s}{2}$.
	Then when $t = \frac{s}{2}$, we have that $|Z(s/2)| \leq s$ and $\hat{N}(Z(s/2), Z(s/2)) \geq s \left(\log \frac{TH}{n}\right)^2$.
	However, the following lemma, whose proof is presented in Section \ref{proof-nonexistence}, shows that this event does {\it not} happen with high probability:
	\begin{lemma}
		\label{lem:S-nonexistence}
		Assume $s \geq 1$.
		Then, with high probability, no $W \subset {\cal X}$ with cardinality $s$  satisfying $\hat{N}(W, W) \geq s \log \left( \frac{TH}{n} \right)^2$.
		Precisely,
		\begin{equation}
			\PP\left[ \exists W\subset{\cal X} : |W|=s, \ \hat{N}(W, W) < s \left(\log \frac{TH}{n}\right)^2 \right] \leq 2 \exp\left( -\frac{1}{8} \frac{TH}{n} \log \frac{TH}{n} \right).
		\end{equation}
	\end{lemma}
\end{itemize}

In summary, we have that $t^* < \frac{s}{2}$. We deduce that: w.h.p.,
\begin{equation*}
	\left| \mathcal{H}^\complement \right| \leq \sum_{x \in \cX} \exp\left( -C \frac{TH}{n} I(x; \Phi) \right) + \frac{s}{2} = 2 \sum_{x \in \cX} \exp\left( -C \frac{TH}{n} I(x; \Phi) \right).
\end{equation*}

\subsection{Proof of Proposition \ref{prop:HE-bound} -- Bounding $| \mathcal{E}_\mathcal{H}^{(l)}|$} \label{sec:bndEl}
From the algorithm, we must have that
\begin{equation}
	E \triangleq
	\sum_{x \in \mathcal{E}_\mathcal{H}^{(\ell)}}
	\left[ \mathcal{L}^{(\ell)}(x, \hat{f}_{\ell+1}(x)) - \mathcal{L}^{(\ell)}(x, f(x)) \right] \geq 0,
\end{equation}
where we recall that
\begin{equation*}
	\mathcal{L}^{(\ell)}(x, j) = \sum_{a \in \cA} \sum_{s \in \cS} \left[ \hat{N}_a(x, \hat{f}_\ell^{-1}(s)) \log \hat{p}_\ell(s | j, a) + \hat{N}_a(\hat{f}_\ell^{-1}(s), x) \log \hat{p}_\ell^{bwd}(s, a | j) \right],
\end{equation*}
with
\begin{equation*}
	\hat{p}_\ell(j|s,a) = {\hat{N}_a(\hat{f}_\ell^{-1}(j), \hat{f}_\ell^{-1}(s)) \over \hat{N}_a(\hat{f}_\ell^{-1}(j), \cX)}, \quad \hat{p}_\ell^{bwd}(s, a | j) = {\hat{N}_a(\hat{f}_\ell^{-1}(s), \hat{f}_\ell^{-1}(j))\over \sum_{\tilde{a} \in \cA} \hat{N}_{\tilde{a}}(\cX, \hat{f}_\ell^{-1}(j))}.
\end{equation*}

We can decompose $E$ as $E = E_1 + E_2 + U$, where
\begin{align}
	E_1 = & \sum_{x \in \cE_\cH^{(\ell+1)}} \left\{ \sum_{a \in \cA} \sum_{s \in \cS} \left[ \hat{N}_a(x, f^{-1}(s)) \log \frac{p(s | \hat{f}_{\ell+1}(x), a)}{p(s | f(x), a)}\right.\right. \\
	& \left.\left. \qquad+ \hat{N}_a(f^{-1}(s), x) \log \frac{p^{bwd}(s, a | \hat{f}_{\ell+1}(x))}{p^{bwd}(s, a | f(x))} \right] \right\},
\end{align}
\begin{align}
	E_2 = \sum_{x \in \cE_\cH^{(\ell+1)}} &\left\{ \sum_{a \in \cA} \sum_{s \in \cS} \left[ \left( \hat{N}_a(x, \hat{f}_\ell^{-1}(s)) - \hat{N}_a(x, f^{-1}(s)) \right) \log \frac{p(s | \hat{f}_{\ell+1}(x), a)}{p(s | f(x), a)} \right. \right. \nonumber \\
	&\left. \left. + \left( \hat{N}_a(\hat{f}_\ell^{-1}(s), x) - \hat{N}_a(f^{-1}(s), x) \right) \log \frac{p^{bwd}(s, a | j)}{p^{bwd}(s, a | f(x))} \right] \right\},
\end{align}

and $U = E - E_1 - E_2$.

Let us denote $e^{(\ell)} = \left| \cE_\cH^{(\ell)} \right|$. In Section \ref{proof-E1E2U}, we derive bounds on $U, E_1, E_2$:

\begin{lemma}
\label{lem:E1E2U}
Assume that $I(\Phi) > 0$ and $TH = \omega(n)$. Then the following holds w.h.p.:\\
\begin{equation}
			\label{eq:E1-bound}
			-E_1 = \Omega \left( e^{(\ell+1)} \frac{TH}{n} \right).
		\end{equation}

		\begin{equation}
			\label{eq:U-bound}
			U = \mathcal{O}\left( e^{(\ell+1)} SA \left( \frac{e^{(\ell)}}{n} \frac{TH}{nA} + S \sqrt{\frac{TH}{nA}} \right) \right)
		\end{equation}

		\begin{equation}
			\label{eq:E2-bound}
			|E_2| \leq F_1 + F_2 + F_3,
		\end{equation}

		where
		\begin{equation*}
			F_1 = \mathcal{O}\left( \frac{TH}{n} \frac{e^{(\ell)}}{n} e^{(\ell+1)} \right),\quad
			F_2 = \mathcal{O}\left( \sqrt{e^{(\ell+1)} e^{(\ell)} \frac{THA}{n}} \right),
		\end{equation*}
		\begin{equation*}
			F_3 = \mathcal{O}\left( e^{(\ell+1)} \left( \log \frac{TH}{n} \right)^2 \right).
		\end{equation*}
\end{lemma}

With the above lemma, we can now quantify the minimum number of iterations $\ell$ for $e^{(\ell)}$ to vanish:
\begin{proposition}
	If $I(\Phi) > 0$, $TH = \omega\left( n \right)$ and $\frac{e^{(1)}}{n}= o\left( 1 \right)$ w.h.p., then after $\ell \geq \log (nA)$ iterations of the likelihood improvement, $e^{(\ell)} = 0$ w.h.p..
\end{proposition}

\begin{proof}
	From $E \geq 0$ and $I(\Phi) > 0$, the following holds a.s.:
	\begin{equation*}
		-E_1 \leq F_1 + F_2 + F_3 + |U|.
	\end{equation*}
	From Lemma \ref{lem:E1E2U}, we have that w.h.p.
	\begin{align*}
		e^{(\ell+1)} \frac{TH}{n} = \mathcal{O}&\left( \frac{TH}{n} \frac{e^{(\ell)}}{n} e^{(\ell+1)} + \sqrt{e^{(\ell+1)} e^{(\ell)} \frac{THA}{n}} + e^{(\ell+1)} \left( \log \frac{TH}{n} \right)^2 \right. \\
		&\quad\quad \ \left. + e^{(\ell+1)} SA \left( \frac{e^{(\ell)}}{n} \frac{TH}{nA} + S \sqrt{\frac{TH}{nA}} \right) \right).
	\end{align*}
	With the given assumptions, we have that w.h.p.
	\begin{align*}
		\frac{e^{(\ell+1)}}{e^{(\ell)}} \leq & \ \mathcal{O}\left(\frac{nA}{TH} \right).
	\end{align*}
We can readily see that when $\ell = \lfloor\log (nA)\rfloor$, $e^{(\ell)} = 0$ w.h.p., and we are done.
\end{proof}

\subsection{Postponed proofs -- Bounding $|\cH_1^\complement|$}
\subsubsection{Proof of Proposition \ref{prop:concentration-I}: concentration for the rate function $I$}\label{proof:concentration-I}
Let $x \in \cX$, $j \in \cS$. We first compute $\lVert \phi \rVert_\infty$:
\begin{align*}
	\left| \phi(\tilde{x}_h^{(t)}) \right| &\leq \sum_{a \in \cA} \sum_{s \in \mathcal{S}} \left(  \indicator\left[x_h^{(t)} = x, a_h^{(t)} = a, x_{h+1}^{(t)} \in f^{-1}(s)\right] \left| \log \frac{p(s | f(x), a)}{p(s | j, a)} \right| \right. \\
	&\quad \ \left. +  \indicator\left[x_h^{(t)} \in f^{-1}(s), a_h^{(t)} = a, x_{h+1}^{(t)} = x \right] \left| \log \frac{p(f(x) | s, a)}{\tilde{c}_j p(j | s, a)} \right| \right) \\
	&\leq 3\log\eta.
\end{align*}

We now compute $\phi^2$ in closed form:
\begin{align*}
	&\phi(X, A, Y)^2 \\
	&= \sum_{(a, s), (\tilde{a}, \tilde{s})} \left(  \indicator\left[X = x, A = a, Y \in f^{-1}(s)\right] \log \frac{p(s | f(x), a)}{p(s | j, a)} \right.\\
	&\qquad \qquad\qquad \qquad\left.+ \indicator\left[X \in f^{-1}(s), A = a, Y = x \right] \log \frac{p(f(x) | s, a)}{\tilde{c}_j p(j | s, a)} \right) \\
	&\qquad \qquad \left( \indicator\left[X = x, A = \tilde{a}, Y \in f^{-1}(\tilde{s})\right] \log \frac{p(\tilde{s} | f(x), \tilde{a})}{p(\tilde{s} | j, \tilde{a})} \right.\\
	&\qquad \qquad \qquad \qquad\left. +  \indicator\left[X \in f^{-1}(\tilde{s}), A = \tilde{a}, Y = x \right] \log \frac{p(f(x) | \tilde{s}, \tilde{a})}{\tilde{c}_j p(j | \tilde{s}, \tilde{a})} \right) \\
	&= \sum_{a, s} \left[  \indicator\left[X = x, A = a, Y \in f^{-1}(s)\right] \left( \log \frac{p(s | f(x), a)}{p(s | j, a)} \right)^2\right.\\
	& \left.\qquad\qquad+  \indicator\left[X \in f^{-1}(s), A = a, Y = x \right] \left( \log \frac{p(f(x) | s, a)}{\tilde{c}_j p(j | s, a)} \right)^2 \right] \\
	&\quad + \underbrace{2  \sum_{a \in \cA} \indicator\left[ X = x, A = a, Y = x \right] \left( \log\frac{p(f(x) | f(x), a)}{p(f(x) | j, a)} \right) \left( \log\frac{p(f(x) | f(x), a)}{\tilde{c}_j p(j | f(x), a)} \right)}_{\triangleq G} \\
	&\leq  \sum_{a, s} \left[  \indicator\left[X = x, A = a, Y \in f^{-1}(s)\right] \left( \log \frac{p(s | f(x), a)}{p(s | j, a)} \right)^2 \right. \\
	&\qquad\qquad\qquad \left. +  \indicator\left[X \in f^{-1}(s), A = a, Y = x \right] \left( \log \frac{p(f(x) | s, a)}{\tilde{c}_j p(j | s, a)} \right)^2 \right] + G.
\end{align*}

As $\Var[\phi] \leq \EE[\phi^2]$ (for any given probability measure), it suffices to derive an upper bound for $\EE_\nu[\phi^2]$ for $\nu \in \{\mu_{odd}, \mu_{even}, P_2(\cdot | x', a', y') \}$, where we recall that $\EE_\nu[\phi^2] = \EE_{X \sim \nu}[\phi(X)^2]$. Observe that for any choice of $\nu$, $\EE_\mu[|G|] = \mathcal{O}\left( \frac{1}{n^2} \right)$, possibly up to some factors involving $S, A$.

We first consider $\mu_{odd}$.
Recalling the definitions of $p^{in}, p^{out}, m_\rho$ (Section \ref{sec:alternate-kl}),
\begin{align*}
	& \EE_{\mu_{odd}}\left[ \phi^2 \right] \\
	&\leq  \sum_{a, s} \left[ q(x | f(x)) m_{\rho}(f(x), a) p(s | f(x), a) \left( \log \frac{p(s | f(x), a)}{p(s | j, a)} \right)^2
	\right. \\
	&\qquad\qquad\qquad \left. + m_{\rho}(s, a) p(f(x) | s, a) q(x | f(x)) \left( \log \frac{p(f(x) | s, a)}{\tilde{c}_j p(j | s, a)} \right)^2 \right] + \mathcal{O}\left( \frac{1}{n^2} \right) \\
	&\overset{(i)}{\leq} 2 \left[ q(x | f(x)) \sum_{a \in \cA} m_{\rho}(f(x), a) \left( \max_{s \in \cS} \frac{p^{out}_{f(x),a}(s) \vee p^{out}_{j,a}(s)}{p^{out}_{f(x),a}(s) \wedge p^{out}_{j,a}(s)} \right)^2 \KL( p^{out}_{f(x),a} || p^{out}_{j,a} ) \right. \\
	&\qquad\qquad\qquad \left. +  \left( \max_{\bar{s} \in [2K]} \frac{p^{in}_{\Phi,x}(\bar{s},a) \vee p^{in}_{\Psi,x}(\bar{s},a;\tilde{c}_j)}{p^{in}_{\Phi,x}(\bar{s},a) \wedge p^{in}_{\Psi,x}(\bar{s},a;\tilde{c}_j)} \right)^2 \KL( p^{in}_{\Phi,x} (\cdot,\cdot) || p^{in}_{\Psi,x} (\cdot,\cdot;\tilde{c}_j)) \right]
	+ \mathcal{O}\left( \frac{1}{n^2} \right) \\
	&\overset{(ii)}{\leq} 2 \eta^2 \left[ q(x | f(x)) \sum_{a \in \cA} m_{\rho}(f(x), a) \KL( p^{out}_{f(x),a} || p^{out}_{j,a} ) +  \KL( p^{in}_{\Phi,x} (\cdot,\cdot) || p^{in}_{\Psi,x} (\cdot,\cdot;\tilde{c}_j)) \right]
	+ \mathcal{O}\left( \frac{1}{n^2} \right) \\
	&\overset{(iii)}{\leq} \frac{2\eta^3}{n} \tilde{I}_j(x; \tilde{c}_j, \Phi)+ \mathcal{O}\left( \frac{1}{n^2} \right)
	\overset{(iv)}{\leq} \frac{2\eta^4}{n} I_j(x; \tilde{c}_j, \Phi)+ \mathcal{O}\left( \frac{1}{n^2} \right),
\end{align*}
where $(i)$ follows from Lemma~\ref{lem:SM6-3}, $(ii)$ and $(iii)$ follow from the facts that $p^{out}$ is $\eta$-regular and $1/\eta \leq \tilde{c}_j \leq \eta$, and $(iv)$ follows from Proposition \ref{prop:I}.

Similarly, we can bound $\EE_{\mu_{even}}\left[ \phi^2 \right] \leq \frac{2\eta^6}{n} I_j(x; \tilde{c}_j, \Phi)+ \mathcal{O}\left( \frac{1}{n^2} \right)$ and $\EE_{P_2^2(\cdot | x', a', y')}\left[ \phi^2 \right] \leq \frac{2\eta^4}{n} I_j(x; \tilde{c}_j, \Phi)+ \mathcal{O}\left( \frac{1}{n^2} \right)$.

\begin{lemma}[Lemma 19 of SM6.3 of \cite{SandersPY20}] \label{lem:SM6-3}
	When $\sum_{z \in \cZ} p(z) = \sum_{z \in \cZ} q(z) = 1$ and $\supp(p) = \supp(q) = \cZ$, then the following holds:
	\begin{equation*}
		\sum_{z \in \cZ} p(z) \left( \log \frac{p(z)}{q(z)} \right)^2 \leq 2 \left( \max_{z \in \cZ} \frac{p(z) \vee q(z)}{p(z) \wedge q(z)} \right)^2 \KL(p || q).
	\end{equation*}
\end{lemma}
%\klbound*
%\begin{lemma} \label{lem:klpqqp}
%When $\sum_{z \in \cZ} p(z) = \sum_{z \in \cZ} q(z) = 1$ and $\supp(p) = \supp(q) = \cZ$, then the following holds:
%$$ \sum_{z \in \cZ} \frac{1}{2}\frac{(p(z) - q(z))^2}{p(z) \vee q(z)} \leq \KL(p || q) \leq  \sum_{z \in \cZ} \frac{1}{2}\frac{(p(z) - q(z))^2}{p(z) \wedge q(z)}$$
%\end{lemma}
%\begin{proof}
%Let $h(x) = \sum_{z \in \cZ} (q(z) +(p(z)-q(z))x) \log \frac{q(z) +(p(z)-q(z))x}{q(z)}$. From the derivative of $f$,
%$$ \KL(p || q) = \int_{0}^1 \sum_{z \in \cZ} (p(z)-q(z)) \log \frac{q(z) +(p(z)-q(z))x}{q(z)} dx.$$
%Using $\log(1+x) \leq x$ and $|\log(1-|x|)| \geq |x|$, we can conclude Lemma~\ref{lem:klpqqp}.
%\end{proof}

In summary, we have:
\begin{align*}
	M_{P, \phi} &= 3(2\eta^3 - 1) \log\eta \triangleq C_1, \\
	V_{\mu, P, \phi} &\leq \frac{C_2}{n} \tilde{I}_j(x; \tilde{c}_j, \Phi),
\end{align*}
with $C_2 \triangleq 2 \left( 1 + \sqrt{2}\eta^3(2\eta^3 - 1) \right)^2 \eta^4$.

Recalling the definition of $\phi$,
\begin{align*}
	&\EE_{x_1^{(t)} \sim \mu}\left[ \sum_{t,h} \phi\left( \tilde{X}_h^{(t)} \right) \right] \\
	&= \sum_{a \in \cA} \sum_{s \in \cS} \left(  N_a(x, f^{-1}(s)) \log\frac{p(s | f(x), a)}{p(s | j, a)} +  N_a(f^{-1}(s), x) \log\frac{p(f(x) | s, a)}{\tilde{c}_j p(j | s, a)} \right) \\
	&= \frac{TH}{n} \sum_{a \in \cA} \sum_{s \in \cS} \left(  q(x | f(x)) m_\rho(f(x), a) p(s | f(x), a) \log\frac{p(s | f(x), a)}{p(s | j, a)} \right.\\
	&\qquad\qquad \left.+  m_\rho(s, a) p(f(x) | s, a) q(x | f(x)) \log\frac{p(f(x) | s, a)}{\tilde{c}_j p(j | s, a)} \right) \\
	&\stackrel{(i)}{\geq} \frac{TH}{\eta n} \tilde{I}_j(x; \tilde{c}_j, \Phi) (1 - o(1)) \\
	&\stackrel{(ii)}{\geq} \frac{TH}{2\eta^2 n} I_j(x; \tilde{c}_j, \Phi),
\end{align*}
where $(i)$ follows from $\tilde{c}_j \eta \geq 1$ and the definition of $\tilde{I}$, and $(ii)$ follows from Proposition \ref{prop:I}.

Recalling that $I(x; \Phi) \leq I_j(x; c, \Phi)$ for all $j \in \cS$ and $c > 0$, we conclude by applying our Bernstein-type concentration (Theorem \ref{thm:bernstein-restart-bmdp}):
\begin{align*}
	& \PP\left[ \sum_{t, h} \phi\left( \tilde{X}_h^{(t)} \right) < \frac{1}{4\eta^2} \frac{TH}{n} I(x; \Phi) \right] \\
	&\leq \PP\left[ \sum_{t, h} \phi\left( \tilde{X}_h^{(t)} \right) < \frac{1}{4\eta^2} \frac{TH}{n} I_j(x; \tilde{c}_j, \Phi) \right] \\
	&\leq \PP\left[ \sum_{t, h} \left( \phi\left( \tilde{X}_h^{(t)} \right) - \EE_{\mu}\left[ \phi\left( \tilde{X}_h^{(t)} \right) \right] \right) <  - \frac{1}{4\eta^2} \frac{TH}{n} I_j(x; \tilde{c}_j, \Phi) \right]\\
	&\leq \PP\left[ \sum_{t, h} \left( \phi\left( \tilde{X}_{2h}^{(t)} \right) - \EE_{\mu}\left[ \phi\left( \tilde{X}_{2h}^{(t)} \right) \right] \right) < -\frac{1}{8\eta^2} \frac{TH}{n} I_j(x; \tilde{c}_j, \Phi) \right] \\
	&\quad + \PP\left[ \sum_{t, h} \left( \phi\left( \tilde{X}_{2h-1}^{(t)} \right) - \EE_{\mu}\left[ \phi\left( \tilde{X}_{2h-1}^{(t)} \right) \right] \right) < -\frac{1}{8\eta^2} \frac{TH}{n} I_j(x; \tilde{c}_j, \Phi) \right] \\
	&\leq 2 \exp\left( - \frac{\left( \frac{1}{8\eta^2} \frac{TH}{n} I_j(x; \tilde{c}_j, \Phi) \right)^2 }{ 2TH \frac{C_2}{n} I_j(x; \tilde{c}_j, \Phi) + \frac{2}{3} C_1 \frac{1}{8\eta^2} \frac{TH}{n} I_j(x; \tilde{c}_j, \Phi) } \right) \\
	&= 2 \exp\left( - 2 C \frac{TH}{n} I_j(x; \tilde{c}_j, \Phi) \right)
	\leq 2 \exp\left( - 2 C \frac{TH}{n} I(x; \Phi) \right),
\end{align*}
where $C \triangleq \frac{1}{256\eta^4 C_2 + \frac{32}{3} \eta^2 C_1}$.

\subsubsection{Proof of Lemma \ref{lem:S-nonexistence}}\label{proof-nonexistence}
Let $W \subset \mathcal{X}$ be any subset of size $s \geq 1$.
As done previously, we split the summation into two parts:
\begin{equation}
	\hat{N}(W, W) = \hat{N}^{even}(W, W) + \hat{N}^{odd}(W, W),
\end{equation}

where $\tilde{H} \triangleq \floor{H/2}$, $I_{h,t} = \sum_{a \in \cA} \indicator\left[ X_h^{(t)} \in W, A_h^{(t)} = a, X_{h+1}^{(t)} \in W \right] = \indicator\left[ X_h^{(t)} \in W, X_{h+1}^{(t)} \in W \right]$,
\begin{equation}
	\hat{N}^{even}(W, W) \triangleq \sum_{t = 1}^T \sum_{h = 1}^{\tilde{H}} I_{2h, t}, \quad \hat{N}^{odd}(W, W) \triangleq \sum_{t = 1}^T \sum_{h = 1}^{\tilde{H}+1} I_{2h-1, t}.
\end{equation}

Again, we exploit the conditional independency structure:
\begin{align*}
	\EE\left[ e^{\theta \hat{N}^{even}(W,W )} \right] &= \EE\left[ \prod_{t=1}^T \prod_{h=1}^{\tilde{H}} e^{\theta I_{2h,t}} \right] \\
	&= \prod_{t=1}^T \EE\left[ \prod_{h=1}^{\tilde{H} - 1}  e^{\theta I_{2h,t}} \EE\left[ e^{\theta I_{2\tilde{H},t}} \left| \left\{ I_{2h, t} \right\}_{h=1}^{\tilde{H} - 1} \right. \right] \right] \\
	&= \prod_{t=1}^T \EE\left[ \prod_{h=1}^{\tilde{H} - 1}  e^{\theta I_{2h,t}} \EE\left[ e^{\theta I_{2\tilde{H},t}} \left| \indicator[X_{2\tilde{H} - 1} \in W] \right. \right] \right] \\
	&= \prod_{t=1}^T \EE\left[ \prod_{h=1}^{\tilde{H} - 1} e^{\theta I_{2h,t}} \left\{ \EE\left[ e^{\theta I_{2\tilde{H},t}} \left| X_{2\tilde{H} - 1} \in W \right. \right] \indicator[X_{2\tilde{H} - 1} \in W] \right.\right.\\
	&\qquad\qquad \left.\left.+ \EE\left[ e^{\theta I_{2\tilde{H},t}} \left| X_{2\tilde{H} - 1} \not\in W \right. \right] \indicator[X_{2\tilde{H} - 1}^{(t)} \not\in W] \right\} \right] \\
	&\leq \prod_{t=1}^T \EE\left[ \prod_{h=1}^{\tilde{H} - 1} e^{\theta I_{2h,t}} \left\{ \EE\left[ e^{\theta I_{2\tilde{H},t}} \left| X_{2\tilde{H} - 1} \in W \right. \right]\indicator[X_{2\tilde{H} - 1} \in W] + \indicator[X_{2\tilde{H} - 1}^{(t)} \not\in W] \right\} \right] \\
	&\overset{(*)}{\leq} \prod_{t=1}^T \EE\left[ \prod_{h=1}^{\tilde{H} - 1} e^{\theta I_{2h,t}} \left( 1 + e^\theta\frac{s^2 \eta^6}{n^2} \right) \right] \\
	&\leq \cdots \\
	&\leq \left( 1 + e^\theta\frac{s^2 \eta^6}{n^2} \right)^{T\tilde{H}} \\
	&\leq \exp\left( e^\theta\frac{T\tilde{H} s^2 \eta^6}{n^2} \right)
	\leq \exp\left( e^\theta\frac{TH s^2 \eta^6}{2n^2} \right),
\end{align*}
where $(*)$ follows from
\begin{align*}
	\EE\left[ e^{\theta I_{2\tilde{H},t}} \left| X_{2\tilde{H} - 1} \in W \right. \right] &= e^\theta \PP\left[ X_{2\tilde{H}} \in W, X_{2\tilde{H}+1} \in W \left| \right. X_{2\tilde{H} - 1} \in W \right] + 1 \\
	&= 1 + e^\theta \frac{\sum_{x, y, z \in W} \PP[ X_{2\tilde{H} - 1} = x, X_{2\tilde{H}} = y, X_{2\tilde{H}+1} = z ]}{\sum_{x \in W} \PP[ X_{2\tilde{H} - 1} = x]} \\
	&= 1 + e^\theta \frac{\sum_{x, y, z \in W} \PP[ X_{2\tilde{H} - 1} = x] \PP[X_{2\tilde{H}} = y | X_{2\tilde{H}-1} = x] \PP[ X_{2\tilde{H}+1} = z | X_{2\tilde{H}} = y]}{\sum_{x \in W} \PP[ X_{2\tilde{H} - 1} = x]} \\
	&\leq 1 + e^\theta \frac{s^2 \left( \frac{\eta^3}{n} \right)^2 \sum_{x \in W} \PP[ X_{2\tilde{H} - 1} = x]}{\sum_{x \in W} \PP[ X_{2\tilde{H} - 1} = x]}
	= 1 + e^\theta\frac{s^2 \eta^6}{n^2}.
\end{align*}

By Markov inequality, we have:
\begin{align*}
	\PP\left[ \hat{N}_a^{even}(W, W) \geq \frac{s}{2} \left(\log \frac{TH}{n}\right)^2 \right] &\leq \inf_{\theta \geq 0} \frac{\EE\left[ \exp\left( \theta \hat{N}_a^{even}(W, W) \right) \right]}{\exp\left( \theta \frac{s}{2} \left(\log \frac{TH}{n}\right)^2 \right)} \\
	&\leq \inf_{\theta \geq 0} \exp\left( e^\theta\frac{TH s^2 \eta^6}{2n^2} - \theta \frac{s}{2} \left(\log \frac{TH}{n}\right)^2 \right) \\
	&\overset{(i)}{\leq} \exp\left( -\frac{TH}{n} s \left( \frac{1}{2} \log \frac{TH}{n} - e^{\frac{\frac{TH}{n}}{\log \frac{TH}{n}}} \frac{\eta^6 s}{2n} \right) \right) \\
	&\overset{(ii)}{\leq} \exp\left( -\frac{s}{4} \frac{TH}{n} \log \frac{TH}{n} \right),
\end{align*}
where $(i)$ follows from choosing $\theta = \frac{\frac{TH}{n}}{\log \frac{TH}{n}}$, and $(ii)$ follows from a simple calculation:
\begin{align*}
	e^{\frac{\frac{TH}{n}}{\log \frac{TH}{n}}} \frac{\eta^6 s}{2n} &\simeq \frac{1}{n} \sum_{x \in \cX} \exp\left( \frac{\frac{TH}{n}}{\log \frac{TH}{n}} - C \frac{TH}{n} I(x; \Phi) \right) \\
	&=\frac{1}{n} \sum_{x \in \cX} \exp\left( -\frac{TH}{n} \left( C I(x; \Phi) - \frac{1}{\log \frac{TH}{n}} \right) \right)
	= o(1) \leq \frac{1}{4} \log\frac{TH}{n},
\end{align*}
where we recall the definition of $s$ and our assumptions that $\sum_{x \in \cX} \exp\left( - C \frac{TH}{n} I(x; \Phi) \right) > 0$, and $TH = \omega(n)$.

Similarly, we also have that
\begin{equation*}
	\PP\left[ \hat{N}_a^{odd}(W, W) \geq \frac{s}{2} \log \left( \frac{TH}{n} \right)^2 \right] \leq \exp\left( -\frac{s}{4} \frac{TH}{n} \log \frac{TH}{n} \right).
\end{equation*}

It follows from a simple union bound (combinatorial) argument that
\begin{align*}
	\EE&\left[ \left| \left\{ W : \hat{N}(W, W) \geq s \left( \log \frac{TH}{n} \right)^2, |W| = s \right\} \right| \right] \\
	&\leq \EE\left[ \left| \left\{ W : \hat{N}^{even}(W, W) \geq \frac{s}{2} \log\left( \frac{TH}{n} \right)^2, |W| = s \right\} \right| \right] \\
	&\qquad\qquad+ \EE\left[ \left| \left\{ W : \hat{N}^{odd}(W, W) \geq \frac{s}{2} \log\left( \frac{TH}{n} \right)^2, |W| = s \right\} \right| \right] \\
	&\leq 2 \left( \frac{en}{s} \right)^s \exp\left( -\frac{s}{4} \frac{TH}{n} \log \frac{TH}{n} \right) \\
	&= 2 \exp\left( -\frac{s}{4} \left( \frac{TH}{n} \log \frac{TH}{n} - \log\frac{en}{s} \right) \right) \\
	&\overset{(i)}{\leq} 2\exp\left( -\frac{s}{8} \frac{TH}{n} \log \frac{TH}{n} \right)
	\overset{(ii)}{\leq} 2 \exp\left( -\frac{1}{8} \frac{TH}{n} \log \frac{TH}{n} \right),
\end{align*}
where $(i)$ follows from $\log\frac{en}{s} \leq \frac{1}{2} \frac{TH}{n}\log\frac{TH}{n}$, and $(ii)$ follows from $s \geq 1$.

Finally, by Markov inequality, we conclude that
\begin{align*}
	\PP\left[ \left| \left\{ W : \hat{N}(W, W) \geq s \left( \log \frac{TH}{n} \right)^2, |W| = s \right\} \right| \geq 1 \right] \leq 2 \exp\left( -\frac{1}{8} \frac{TH}{n} \log \frac{TH}{n} \right) \rightarrow 0.
\end{align*}
\ep

\subsection{Postponed proofs - Bounding $|\cE_\cH^{(\ell)}|$}

\subsubsection{Proof of Lemma \ref{lem:E1E2U}: bounding $E_1, E_2, U$}\label{proof-E1E2U}
We bound each term separately.

\underline{\it Proof of \eqref{eq:E1-bound} - Lower bounding $-E_1$}\\
The lower bound for $-E_1$ can be obtained by recalling that the form is asymptotically the same to that of the rate function:
\begin{align*}
	-E_1 &= \sum_{x \in \cE_\cH^{(\ell+1)}} \left\{ \sum_{a \in \cA} \sum_{s \in \cS} \left[ \hat{N}_a(x, f^{-1}(s)) \log \frac{p(s | f(x), a)}{p(s | \hat{f}_{\ell+1}(x), a)} \right.\right.\\
	& \qquad\qquad \left.\left.+ \hat{N}_a(f^{-1}(s), x) \log \frac{p^{bwd}(f(x) | s, a)}{p^{bwd}(\hat{f}_{\ell+1}(x) | s, a)} \right] \right\} \\
	&=  \sum_{x \in \cE_\cH^{(\ell+1)}} \hat{I}_{\hat{f}_{\ell+1}(x)}(x;  \Phi) \\
	&\overset{(i)}{=} \Omega \left( \frac{TH}{n} \sum_{x \in \cE_\cH^{(\ell+1)}} I(x; \Phi) \right)
	\overset{(ii)}{=} \Omega \left( e^{(\ell+1)} \frac{TH}{n} \right),
\end{align*}
where $(i)$ follows from the condition (H1) for the well-defined contexts (Definition \ref{def:well-defined}), and $(ii)$ follows from our assumption that $I(\Phi) > 0$, which implies that $I(x; \Phi) > 0$ for all $x \in \cX$ (see Section \ref{sec:lower}).

\underline{\it Proof of \eqref{eq:U-bound} - Upper bounding $U$}\\
We again rewrite $U$ as $U = U^{in} + U^{out}$, with
\begin{align*}
	U^{in} &\triangleq \sum_{x \in \cE_\cH^{(\ell+1)}} \left\{ \sum_{a \in \cA} \sum_{s \in \cS} \left[ \hat{N}_a(x, \hat{f}_{\ell+1}^{-1}(s)) \left( \log \frac{\hat{p}_\ell(s | \hat{f}_{\ell+1}(x), a)}{\hat{p}_\ell(s | f(x), a)} - \log \frac{p(s | \hat{f}_{\ell+1}(x), a)}{p(s | f(x), a)} \right) \right] \right\} \\
	&= \sum_{x \in \cE_\cH^{(\ell+1)}} \left\{ \sum_{a \in \cA} \sum_{s \in \cS} \left[ \hat{N}_a(x, \hat{f}_{\ell+1}^{-1}(s)) \left( \log \frac{\hat{p}_\ell(s | \hat{f}_{l+1}(x), a)}{p(s | \hat{f}_{l+1}(x), a)} - \log \frac{\hat{p}_\ell(s | f(x), a)}{p(s | f(x), a)} \right) \right] \right\},
\end{align*}
and
\begin{align*}
	U^{out} &\triangleq \sum_{x \in \cE_\cH^{(\ell+1)}} \left\{ \sum_{a \in \cA} \sum_{s \in \cS} \left[ \hat{N}_a(x, \hat{f}_{\ell+1}^{-1}(s)) \left( \log \frac{\hat{p}^{bwd}_\ell(\hat{f}_{\ell+1}(x), a | s)}{\hat{p}^{bwd}_\ell(f(x), a | s)} - \log \frac{p^{bwd}(\hat{f}_{\ell+1}(x), a | s)}{p^{bwd}(f(x), a | s)} \right) \right] \right\} \\
	&= \sum_{x \in \cE_\cH^{(\ell+1)}} \left\{ \sum_{a \in \cA} \sum_{s \in \cS} \left[ \hat{N}_a(x, \hat{f}_{\ell+1}^{-1}(s)) \left( \log \frac{\hat{p}^{bwd}_\ell(\hat{f}_{l+1}(x), a | s)}{p^{bwd}(\hat{f}_{l+1}(x), a | s)} - \log \frac{\hat{p}^{bwd}_\ell(f(x), a | s)}{p^{bwd}(f(x), a | s)} \right) \right] \right\},
\end{align*}

We conclude by Lemma \ref{lem:p-estimation} and the triangle inequality: w.h.p.
\begin{align*}
	|U| &\leq |U^{in}| + |U^{out}| \\
	&\leq \mathcal{O} \left( e^{(l+1)} S \frac{TH}{n} \left( \frac{e^{(l)}}{n} + S\sqrt{\frac{nA}{TH}} \right) \right).
\end{align*}

\underline{\it Proof of \eqref{eq:E2-bound} - Upper bounding $E_2$}\\
From the regularity assumptions and the triangle inequality, we first have that
\begin{equation*}
	|E_2| \lesssim \sum_{x \in \cE_\cH^{(\ell+1)}} \left\{ \sum_{a \in \cA} \sum_{s \in \cS} \left[ \left| \hat{N}_a(x, \hat{f}_\ell^{-1}(s)) - \hat{N}_a(x, f^{-1}(s)) \right| + \left| \hat{N}_a(\hat{f}_\ell^{-1}(s), x) - \hat{N}_a(f^{-1}(s), x) \right| \right] \right\}.
\end{equation*}

We bound the first summation (the second summation follows the exact same argument): by the triangle inequality,
\begin{align*}
	& \sum_{x \in \cE_\cH^{(\ell+1)}} \sum_{a \in \cA} \sum_{s \in \cS} \left| \hat{N}_a(x, \hat{f}_\ell^{-1}(s)) - \hat{N}_a(x, f^{-1}(s)) \right| \\
	&\qquad\qquad \le\sum_{x \in \cE_\cH^{(\ell+1)}} \sum_{a \in \cA} \sum_{s \in \cS} \left| \hat{N}_a(x, \hat{f}_\ell^{-1}(s) \cap \cH) - \hat{N}_a(x, f^{-1}(s) \cap \cH) \right| \\
	&\qquad\qquad\quad \ + \sum_{x \in \cE_\cH^{(\ell+1)}} \sum_{a \in \cA} \sum_{s \in \cS} \left( \hat{N}_a(x, \hat{f}_\ell^{-1}(s) \cap \cH^\complement) + \hat{N}_a(x, f^{-1}(s) \cap \cH^\complement) \right) \\
	&\qquad\qquad= \sum_{x \in \cE_\cH^{(\ell+1)}} \sum_{a \in \cA} \sum_{s \in \cS} \left| \hat{N}_a(x, \hat{f}_\ell^{-1}(s) \cap \cH) - \hat{N}_a(x, f^{-1}(s) \cap \cH) \right| \\
	&\qquad\qquad\quad \ + 2 \sum_{x \in \cE_\cH^{(\ell+1)}} \sum_{a \in \cA} \hat{N}_a(x, \cX \setminus \cH).
\end{align*}

From the construction of $\cH$ (specifically (H2)), the second sum can be further bounded as follows:
\begin{equation*}
	2 \sum_{x \in \cE_\cH^{(\ell+1)}} \sum_{a \in \cA} \hat{N}_a(x, \cX \setminus \cH) \leq 4 e^{(\ell+1)} \left( \log \frac{TH}{n} \right)^2.
\end{equation*}

The first sum is bounded as follows:
\begin{align*}
	&\sum_{x \in \cE_\cH^{(\ell+1)}} \sum_{a \in \cA} \sum_{s \in \cS} \left| \hat{N}_a(x, \hat{f}_\ell^{-1}(s) \cap \cH) - \hat{N}_a(x, f^{-1}(s) \cap \cH) \right| \\
	&= \sum_{x \in \cE_\cH^{(\ell+1)}} \sum_{a \in \cA} \sum_{s \in \cS} \left| \sum_{y \in (\hat{f}_\ell^{-1}(s) \cap \cH) \setminus (f^{-1}(s) \cap \cH)} \hat{N}_a(x, y) - \sum_{y \in (f^{-1}(s) \cap \cH) \setminus (\hat{f}_\ell^{-1}(s) \cap \cH)} \hat{N}_a(x, y) \right| \\
	&\leq \sum_{x \in \cE_\cH^{(\ell+1)}} \sum_{a \in \cA} \sum_{s \in \cS} \sum_{y \in (\hat{f}_\ell^{-1}(s) \cap \cH) \triangle (f^{-1}(s) \cap \cH)} \hat{N}_a(x, y) \\
	&= 2 \sum_{x \in \cE_\cH^{(\ell+1)}} \sum_{a \in \cA} \sum_{y \in \cE_\cH^{(\ell)}} \hat{N}_a(x, y)
	= 2 \sum_{a \in \cA} \hat{N}_a\left( \cE_\cH^{(\ell+1)}, \cE_\cH^{(\ell)} \right).
\end{align*}

Next, we will use the following lemma related to the spectral norm of matrices:
\begin{lemma}[Lemma 20 of SM6.4 of \cite{SandersPY20}]
	For any matrix $B \in \RR^{n \times n}$ and any subsets $E, F \subset [n]$, we have $\sum_{r \in E} \sum_{c \in F} B(r, c) = 1_E^\intercal B 1_F$, where $1_E$ and $1_F$ are column vector such that $1_E(x) = \indicator[x \in E]$ for $x \in [n]$.
	Furthermore, we have that $1_E^\intercal B 1_F \leq \lVert B \rVert \sqrt{|E| |F|}$.
\end{lemma}
Thus,
\begin{align*}
	\sum_{a \in \cA} \hat{N}_a\left( \cE_\cH^{(\ell+1)}, \cE_\cH^{(\ell)} \right) &= \sum_{a \in \cA} N_a\left( \cE_\cH^{(\ell+1)}, \cE_\cH^{(\ell)} \right) + \sum_{a \in \cA} (\hat{N}_a - N_a)\left( \cE_\cH^{(\ell+1)}, \cE_\cH^{(\ell)} \right) \\
	&\leq \mathcal{O}\left( \frac{TH}{n^2} e^{(\ell+1)} e^{(\ell)} + \sum_{a \in \cA} \left\lVert \hat{N}_a - N_a \right\rVert_2 \sqrt{e^{(\ell+1)} e^{(\ell)}} \right).
\end{align*}
From Proposition \ref{prop:init-concentration}, we have that for each $a \in \cA$, w.h.p.
\begin{align*}
	\left\lVert \hat{N}_a - N_a \right\rVert_2 &\leq \left\lVert \hat{N}_a - \trim_{\Gamma_a}\left(\widehat{N}_{a}\right) \right\rVert_2 + \left\lVert \trim_{\Gamma_a}\left(\widehat{N}_{a}\right) - N_a \right\rVert_2 \\
	&\leq \cO\left( \left\lVert \hat{N}_a - \trim_{\Gamma_a}\left(\widehat{N}_{a}\right) \right\rVert_F + \sqrt{\frac{TH}{nA}} \right) \\
	&\leq \cO\left( \sqrt{ \left(\frac{TH}{n^2 A}\right)^2 n e^{-\frac{TH}{nA}}} + \sqrt{\frac{TH}{nA}} \right) 
	\leq \cO\left( \sqrt{\frac{TH}{nA}} \right).
\end{align*}

In summary, we have, w.h.p.,
\begin{align*}
	|E_2| = \mathcal{O} \left( \underbrace{\frac{TH}{n} \frac{e^{(\ell)}}{n} e^{(\ell+1)}}_{\triangleq F_1} + \underbrace{\sqrt{e^{(\ell+1)} e^{(\ell)} \frac{THA}{n} }}_{\triangleq F_2} + \underbrace{e^{(\ell+1)} \left( \log \frac{TH}{n} \right)^2}_{\triangleq F_3} \right).
\end{align*}
\ep

\subsubsection{Intermediate estimation errors for $p$ and $p^{bwd}$}
In this subsection, we bound the {\it estimation errors of $p$ and $p^{bwd}$} during the improvement steps.

One important remark is that $p$ and $p^{bwd}$ can be {\it precisely} written as ratios of expected numbers of observations of transitions (even without the stationarity assumption):
\begin{lemma}
	\label{lem:p-closed-form}
	For all $(s, a, s') \in \cS \times \cA \times \cS$,
	\begin{equation}
		p(s' | s, a) = \frac{N_a\left( f^{-1}(s), f^{-1}(s') \right)}{N_a\left( f^{-1}(s), \cX \right)},
		\quad p^{bwd}(s, a | s') = \frac{N_a(f^{-1}(s), f^{-1}(s'))}{\sum_{\tilde{a} \in \cA} N_{\tilde{a}}(\cX, f^{-1}(s'))}.
	\end{equation}
\end{lemma}
\begin{proof}
	Both follow from a simple chain of computations:
	\begin{align*}
		\frac{N_a\left( f^{-1}(s), f^{-1}(s') \right)}{N_a\left( f^{-1}(s), \cX \right)} &= \frac{TH \sum_{y \in f^{-1}(s')} m_\rho(s, a) p(s' | s, a) q(y | s')}{TH \sum_{z \in \cX} m_\rho(s, a) p(f(z) | s, a) q(z | f(z))} \\
		&= p(s' | s, a) \frac{1}{\sum_{\tilde{s} \in \cS} p(\tilde{s} | s, a)}
		= p(s' | s, a).
	\end{align*}
	and
	\begin{align*}
		\frac{N_a(f^{-1}(s), f^{-1}(s'))}{\sum_{\tilde{a} \in \cA} N_{\tilde{a}}(\cX, f^{-1}(s'))} &= \frac{TH \sum_{y \in f^{-1}(s')} m_\rho(s, a) p(s' | s, a) q(y | s')}{TH \sum_{\tilde{s} \in \cS} \sum_{\tilde{a} \in \cA} m_\rho(\tilde{s}, \tilde{a}) p(s' | \tilde{s}, \tilde{a})} \\
		&= \frac{m_\rho(s, a) p(s' | s, a)}{\sum_{\tilde{s} \in \cS} \sum_{\tilde{a} \in \cA} m_\rho(\tilde{s}, \tilde{a}) p(s' | \tilde{s}, \tilde{a})}
		= p^{bwd}(s, a | s').
	\end{align*}
\end{proof}

Now the intermediate error bound for $p$:
\begin{lemma}
	\label{lem:p-estimation}
	After $\ell$ rounds of improvement, the following holds: if $e^{(\ell)}$ at least satisfies Theorem \ref{thm:initial-spectral}, then for all $(s, a, s') \in \cS \times \cA \times \cS$, w.h.p.
	\begin{align}
		\left| \log \frac{\hat{p}_\ell(s' | s, a)}{p(s' | s, a)} \right| \leq \left| \frac{\hat{p}_\ell(s' | s, a) - p(s' | s, a)}{p(s' | s, a)} \right|
		= \mathcal{O}\left( S \left( \frac{e^{(\ell)}}{n} + S \sqrt{\frac{nA}{TH}} \right) \right)
	\end{align}
	and
	\begin{equation}
		\left| \log \frac{\hat{p}^{bwd}_\ell(s' | s, a)}{p^{bwd}(s' | s, a)} \right| \leq \left| \frac{\hat{p}^{bwd}_\ell(s' | s, a) - p^{bwd}(s' | s, a)}{p^{bwd}(s' | s, a)} \right|
		= \mathcal{O}\left( S \left( \frac{e^{(\ell)}}{n} + S \sqrt{\frac{nA}{TH}} \right) \right).
	\end{equation}
\end{lemma}
\begin{proof}
	We start with $p_\ell$.
	From the inequalities $\frac{x}{1 + x} \leq \log (1 + x) \leq x$ for $x > -1$ and Lemma \ref{lem:p-closed-form}, we have that
	\begin{align*}
		& \left| \log \frac{\hat{p}_\ell(s' | s, a)}{p(s' | s, a)} \right| \leq \left| \frac{\hat{p}_\ell(s' | s, a) - p(s' | s, a)}{p(s' | s, a)} \right| \\
		&\qquad\qquad\qquad= \left| \frac{N_a(f^{-1}(s), \cX)}{N_a(f^{-1}(s), f^{-1}(s'))} \frac{\hat{N}_a(\hat{f}_\ell^{-1}(s), \hat{f}_\ell^{-1}(s'))}{\hat{N}_a(\hat{f}_\ell^{-1}(s), \cX)} - 1 \right| \\
		&\qquad\qquad= \bigg| \underbrace{\frac{N_a(\hat{f}_\ell^{-1}(s), \hat{f}_\ell^{-1}(s'))}{N_a(f^{-1}(s), f^{-1}(s'))} \frac{N_a(f^{-1}(s), \cX)}{N_a(\hat{f}_\ell^{-1}(s), \cX)}}_{\hbox{left ratios}} \underbrace{\frac{\hat{N}_a(\hat{f}_\ell^{-1}(s), \hat{f}_\ell^{-1}(s'))}{N_a(\hat{f}_\ell^{-1}(s), \hat{f}_\ell^{-1}(s'))} \frac{N_a(\hat{f}_\ell^{-1}(s), \cX)}{\hat{N}_a(\hat{f}_\ell^{-1}(s), \cX)}}_{\hbox{right ratios}} - 1 \bigg|.
	\end{align*}

The "left ratios" capture the clustering error, both of which are concentrated around $1$ with high probability. Denoting $V \triangleq f^{-1}(s), \hat{V} \triangleq \hat{f}_\ell^{-1}(s), W \triangleq f^{-1}(s'), \hat{W} \triangleq \hat{f}_\ell^{-1}(s')$, we first compute an upper bound for $| N_a(\hat{V}, \hat{W}) - N_a(V, W) |$:
	\begin{align*}
		\left| N_a(\hat{V}, \hat{W}) - N_a(V, W) \right| &= \left| \left( N_a(V \setminus \hat{V}, W) + N_a(V \cap \hat{V}, \hat{W} \setminus W) \right) \right.\\
		& \qquad \qquad\left. - \left( N_a(\hat{V} \setminus V, \hat{W}) + N_a(V \cap \hat{V}, W \setminus \hat{W}) \right) \right| \\
		&\leq N_a(V \cap \hat{V}, W \triangle \hat{W}) + N_a(V \setminus \hat{V}, B) + N_a(\hat{V} \setminus V, \hat{W}) \\
		&\leq N_a\left( V, \mathcal{E}^{(\ell)} \right) + N_a(\mathcal{E}^{(\ell)}, W) + N_a(\mathcal{E}^{(\ell)}, \hat{W}) \\
		&\leq \mathcal{O}\left( \frac{TH}{n^2 A} \frac{n}{S} e^{(l)} \right)
		= \mathcal{O}\left( \frac{TH}{nSA} e^{(l)} \right),
	\end{align*}
	where $\triangle$ is the symmetric difference operator.

	Now we compute the asymptotics of the first left ratio:
	\begin{align*}
		\left| { \frac{N_a(\hat{f}_\ell^{-1}(s), \hat{f}_\ell^{-1}(s'))}{N_a(f^{-1}(s), f^{-1}(s'))}} - 1 \right| &= \left| \frac{1}{N_a(f^{-1}(s), f^{-1}(s'))} \left( N_a(\hat{V}, \hat{W}) - N_a(V, W) \right) \right| \\
		&\leq \mathcal{O} \left( \frac{1}{\frac{n^2}{S^2} \frac{TH}{n^2 A}} \frac{TH}{nSA} e^{(\ell)} \right) = \mathcal{O} \left( \frac{S e^{(\ell)} }{n} \right).
	\end{align*}
	Observe that the same bound also holds for the other left ratio.

	The "right ratios" are readily bounded using Proposition \ref{prop:concentration-all}, provided at the end of this appendix: w.h.p.
	\begin{align*}
		\left| {\frac{\hat{N}_a(\hat{f}_\ell^{-1}(s), \hat{f}_\ell^{-1}(s'))}{N_a(\hat{f}_\ell^{-1}(s), \hat{f}_\ell^{-1}(s'))}} - 1 \right| &\leq \mathcal{O} \left( \frac{1}{\frac{n^2}{S^2} \frac{TH}{n^2 A}}  \sqrt{\frac{nTH}{A}} \right) \\
		&= \mathcal{O} \left( S^2 \sqrt{\frac{nA}{TH}} \right).
	\end{align*}
	Similarly, the same bound also holds for the other right ratio, and combining them all gives our result.
	%	Combining them all gives our desired statement.

	We now turn to $p_\ell^{bwd}$.
	Using similar reasoning as previous, we first have that
	\begin{align*}
		\left| \log \frac{\hat{p}^{bwd}_\ell(s, a | s')}{p^{bwd}(s, a | s')} \right|  & \leq \left| \frac{\hat{p}^{bwd}_\ell(s, a | s') - p^{bwd}(s, a | s')}{p^{bwd}(s, a | s')} \right| \\
		&= \left| \frac{\sum_{\tilde{a} \in \cA} N_{\tilde{a}}(\cX, f^{-1}(s'))}{N_a(f^{-1}(s), f^{-1}(s'))} \frac{\hat{N}_a(\hat{f}_\ell^{-1}(s), \hat{f}_\ell^{-1}(s'))}{\sum_{\tilde{a} \in \cA} \hat{N}_{\tilde{a}}(\cX, \hat{f}_\ell^{-1}(s'))} - 1 \right| \\
		& = \left| { \frac{N_a(\hat{f}_\ell^{-1}(s), \hat{f}_\ell^{-1}(s'))}{N_a(f^{-1}(s), f^{-1}(s'))} \frac{\sum_{\tilde{a} \in \cA} N_{\tilde{a}}(\cX, f^{-1}(s'))}{\sum_{\tilde{a} \in \cA} N_{\tilde{a}}(\cX, \hat{f}_\ell^{-1}(s'))}}\right. \\
		&\left.\qquad\qquad \times {\frac{\hat{N}_a(\hat{f}_\ell^{-1}(s), \hat{f}_\ell^{-1}(s'))}{N_a(\hat{f}_\ell^{-1}(s), \hat{f}_\ell^{-1}(s'))} \frac{\sum_{\tilde{a} \in \cA} N_{\tilde{a}}(\cX, \hat{f}_\ell^{-1}(s'))}{\sum_{\tilde{a} \in \cA} \hat{N}_{\tilde{a}}(\cX, \hat{f}_\ell^{-1}(s'))}} - 1 \right|.
	\end{align*}
	All ratios in the above can be bounded as those involved in $p_\ell$. This completes the proof.
\end{proof}

\paragraph{Concentration of $\hat{N}_a$ around $N_a$}
We now provide the concentration result relating $N_a$ and $\hat{N}_a$, for {\it any} subsets $E, F \subset \cX$, used in the discussions above:
\begin{proposition}[Concentration of $N_a$ and $\hat{N}_a$ over all possible subsets]
	\label{prop:concentration-all}
	There exists an absolute constant $c > 0$ such that for any $a \in \mathcal{A}$,
	\begin{equation}
		\label{eq:master1-4}
		\PP\left[ \max_{E, F \subset \cX} \left| \hat{N}_a(E, F) - N_a(E, F) \right| \geq c \sqrt{\frac{n TH}{A}} \right] \leq 4\exp\left( - 2n \left( 1 - \log2 \right) \right).
	\end{equation}
\end{proposition}
\begin{proof}
	Let $E, F \subset \cX$, and let $\phi(X, A, Y) = \indicator[X \in E, A = a, Y \in F]$.
	We have that $\hat{N}_a(E, F) = \sum_{t,h} \phi\left( X_{h-1}^{(t)}, A_{h-1}^{(t)}, X_h^{(t)} \right)$ and $\lVert \phi \rVert_\infty = 1$.
	As for the proof of other concentration results for BMDPs (e.g. Proposition \ref{prop:concentration-I}), we consider $MC_{2, odd}$ and $MC_{2, even}$.
	We first have that $\EE_{\mu_{odd}}[\phi], \EE_{\mu_{even}}[\phi], \EE_{P_2^2(\cdot | x', a', y')}[\phi] \leq \frac{\eta^5}{A} \triangleq p$ for all $(x', a', y') \in \cX \times \cA \times \cX$.
	From this, the variances for all cases are bounded as follows:
	\begin{equation*}
		\Var[\phi] = \EE[\phi] \left( 1 - \EE[\phi] \right) \leq \EE[\phi] \leq p,
	\end{equation*}
	implying that $V_{\mu, P, \phi} \leq \left( 1 + \sqrt{2}\eta^3 (2\eta^3 - 1) \right)^2 p$.

	Thus for any $u = o\left( TH p \right)$,
	\begin{align*}
		& \PP\left[ \left| \sum_{t,h} \phi(X_{h}^{(t)})) - \EE_{\mu}[\phi(X_{h}^{(t)})] \right| > u \right] \\
		&\quad \le \PP\left[ \left| \sum_{t,h}  \phi(X_{2h}^{(t)})) - \EE_{\mu}[\phi(X_{2h}^{(t)})] \right| > \frac{u}{2} \right] + \PP\left[ \left| \sum_{t,h} \phi(X_{2h+1}^{(t)})) - \EE_{\mu}[\phi(X_{2h+1}^{(t)})] \right| > \frac{u}{2} \right] \\
		&\quad \leq 4 \exp\left( -\frac{u^2}{ 2TH \left( 1 + \sqrt{2}\eta^3 (2\eta^3 - 1) \right)^2 p } \right).
	\end{align*}

	Choose $u = c \sqrt{\frac{n TH}{A}} = o\left( TH p \right)$ with $c^2 \geq 4 \left( 1 + \sqrt{2}\eta^3 (2\eta^3 - 1) \right)^2 \eta^5$. Then,
	\begin{align*}
		\PP\left[ \left| \hat{N}_a(E, F) - N_a(E, F) \right| \geq c \sqrt{\frac{n TH}{A}} \right] &\leq 4 \exp\left( -\frac{c^2 \frac{n TH}{A}}{ 2TH \left( 1 + \sqrt{2}\eta^3 (2\eta^3 - 1) \right)^2 \frac{\eta^5}{A} } \right) \\
		&\leq 4 \exp\left( - 2 n \right).
	\end{align*}

	We conclude by taking the union bound over all possible pairs $(E, F) \subset \cX$:
	\begin{align*}
		\PP\left[ \max_{E, F \subset \cX} \left| \hat{N}_a(E, F) - N_a(E, F) \right| \geq c \sqrt{\frac{n TH}{A}} \right] \leq 2^{2n} 4\exp\left( - 2n \right)
		\leq 4 \exp\left( -2 (1 - \log2) n \right).
	\end{align*}
\end{proof}

\newpage
% !TEX root = ./main.tex

\section{Proof of Theorem \ref{thm:likelihood-improvement} (ii) -- Estimation of the Latent Transitions and Emission Probabilities}\label{appendix:estimation}

In this appendix, we establish concentration results on the {\it final} estimation error of the latent state transitions rates and the emission probabilities.

\subsection{Preliminaries on the estimators}

\paragraph{Estimation under $\hat{f}$.} Let us recall, that given an estimated decoding function $\hat{f}$, we estimate the latent state transition probabilities $\hat{p}$ and the emission probabilities $\hat{q}$ as follows:
\begin{align*}
	  \forall s,s' \in \cS, \forall a \in \cA, \qquad \hat{p}(s' \vert s , a) & =   \frac{\sum_{t,h} \indicator\lbrace \hat{f}(x_{h}^{(t)}) = s, a_{h}^{(t)} = a,  \hat{f}(x_{h+1}^{(t)}) = s' \rbrace}{\sum_{t,h} \indicator\lbrace \hat{f}(x_{h}^{(t)}) = s, a_{h}^{(t)} = a  \rbrace}, \\
		  \forall x \in \cX, \forall s \in \cS, \qquad \hat{q}(x \vert s) & = \frac{\sum_{t,h} \indicator\lbrace \hat{f}(x_{h}^{(t)}) = s, x_{h}^{(t)} = x \rbrace}{\sum_{t,h} \indicator\lbrace \hat{f}(x_{h}^{(t)}) = s  \rbrace}.
\end{align*}

Here, we use the short hand $\sum_{t,h} \triangleq \sum_{t=\lfloor T/2\rfloor +1}^T \sum_{h=1}^H$. 
%However, it is worth mentioning that the results we provide here hold even if the first batch of observations counting from $t= 1$ up to $t = \lfloor T/2 \rfloor$.
Without explicitly mentioning it, we set $\hat{p}(s'\vert s, a) = 0$ (resp. $\hat{q}(x \vert s) = 0$) whenever $\sum_{t,h} \indicator\lbrace \hat{f}(x_{h}^{(t)}) = s, a_{h}^{(t)} = a  \rbrace = 0$ (resp. $\sum_{t,h} \indicator\lbrace \hat{f}(x_{h}^{(t)}) = s  \rbrace = 0$), and as we shall see this will not happen with high probability.

\paragraph{Estimation under the true $f$.} We will denote the estimates of the latent transition probabilities and emission probabilities under the true clustering function $f$ by $\hat{p}_f$ and $\hat{q}_f$, respectively. They are defined as follows:
\begin{align*}
	 \forall s,s' \in \cS, \forall a \in \cA \qquad \hat{p}_f(s' \vert s , a) & = \frac{\sum_{t,h} \indicator\lbrace f(x_{h}^{(t)}) = s, a_{h}^{(t)} = a,  f(x_{h+1}^{(t)}) = s' \rbrace}{\sum_{t,h} \indicator\lbrace f(x_{h}^{(t)}) = s, a_{h}^{(t)} = a  \rbrace} \\
		 \forall x \in \cX, \forall s \in \cS, \qquad \hat{q}_f(x \vert s) & = \frac{\sum_{t,h} \indicator\lbrace f(x_{h}^{(t)}) = s, x_{h}^{(t)} = x \rbrace}{\sum_{t,h} \indicator\lbrace f(x_{h}^{(t)}) = s  \rbrace}.
\end{align*}
Again, without explicitly mentioning it, we set $\hat{p}_f(s'\vert s, a) = 0$ (resp. $\hat{q}_f(x \vert s) = 0$) whenever $\sum_{t,h} \indicator\lbrace f(x_{h}^{(t)}) = s, a_{h}^{(t)} = a  \rbrace = 0$ (resp. $\sum_{t,h} \indicator\lbrace f(x_{h}^{(t)}) = s  \rbrace = 0$), and as we shall see this will not happen with high probability.

\paragraph{Notations.} To declutter notations, we introduce for all $t \in [T],h \in [H]$, $\forall x \in \cX$, $\forall X \subseteq \cX$, $\forall a \in \cA$, $\delta_{t,h,a}(x) \triangleq \indicator \{x_{h}^{(t)} = x, a_{h}^{(t)} = a\}$, $\delta_{t,h}(x) \triangleq \indicator \{x_{h}^{(t)} = x\}$, $\delta_{t,h,a}(X) \triangleq \indicator \{x_{h}^{(t)} \in \cX, a_h^{(t)} = a\}$ and $\delta_{t,h}(X) \triangleq \indicator \{x_{h}^{(t)} \in \cX\}$. We will further write $\hat{p}(s,a), p(s,a), \hat{q}(s), q(s)$ instead of $\hat{p}(\cdot \vert s,a), p(\cdot \vert s,a), \hat{q}(\cdot \vert s), q(\cdot \vert s)$
%
%
% Additionally, for all $a \in \cA$, $x \in \cX$, $X, Y  \subseteq \cX$, we also write $\hat{N}_{x} \triangleq  \sum_{t,h} \delta_{t,h}(x)$, and $\hat{N}_{X} \triangleq \sum_{t,h} \delta_{t,h}(X), \hat{N}_{x, X} \triangleq  \sum_{t,h} \delta_{t,h}(x)\delta_{t,h}(X), \hat{N}_{a, X} \triangleq  \sum_{t,h} \delta_{t,h,a}(X), \hat{N}_{a, X \rightarrow Y} \triangleq  \sum_{t,h} \delta_{t,h,a}(X) \delta_{t,h+1}(Y)$. With these new notations, we may simply express our estimators as follows: for all $x \deompositionin \cX$, $a \in \cA$, $s,s' \in \cS$,
% \begin{align*}
% 	\hat{p}(s' \vert s,a) = \frac{ \hat{N}_{a, \hat{f}^{-1}(s) \rightarrow \hat{f}^{-1} (s') }  }{ \hat{N}_{a, \hat{f}^{-1}(s) } }\qquad & \text{and} \qquad \hat{q}(x \vert s)  = \frac{\hat{N}_{x, \hat{f}^{-1}(s)}}{\hat{N}_{\hat{f}^{-1}(s)}} \\
% 	\hat{p}_f(s' \vert s,a) = \frac{ \hat{N}_{a, f^{-1}(s) \rightarrow f^{-1} (s') }  }{ \hat{N}_{a, f^{-1}(s) } }\qquad & \text{and} \qquad \hat{q}_f(x \vert s)  = \frac{\hat{N}_{x, f^{-1}(s)}}{\hat{N}_{f^{-1}(s)}}
% \end{align*}

\subsection{Proof of Theorem \ref{thm:likelihood-improvement} (ii)}

% First, we establish Proposition \ref{prop:ee_q}, which says that estimation error between $\hat{q}$ and $q$ in $\ell_1$ norm, ignoring the dependencies on $S$ and $n$, is controlled by two terms, the first one is of the order $\sqrt{n/TH}$ and corresponds to the error that one would obtain in the absence of misclassfication errors, and second term is $\vert \cE \vert/n$ and corresponds to the proportion of misclassfied contexts.

Here we present the precise statement of Theorem \ref{thm:likelihood-improvement} - (ii), as Proposition \ref{prop:ee_p_q}.

\begin{proposition}\label{prop:ee_p_q}
  Under Assumptions \ref{assumption:SA}-\ref{assumption:uniform}, the estimators $\hat{p}$ and $\hat{q}$ satisfy:
	\begin{itemize}
		\item[(i)] for all $ TH = \Omega(\log(n))$, the event
	    % \begin{align}\label{eq:qerr1}
	    %   \PP\left( \max_{s\in \cS,a \in \cA} \Vert \hat{p}(s,a) - p(s,a)\Vert_1 \le \poly(\eta) SA \left(\sqrt{\frac{S + \log(nSA)}{TH}} + \frac{\vert \cE \vert}{n} \right)  \right) \ge 1 - \frac{4}{n} - e^{-\frac{TH}{n}}
	    % \end{align}
			\begin{align}
	      \max_{s\in \cS,a \in \cA} \Vert \hat{p}(s,a) - p(s,a)\Vert_1 \le \poly(\eta) SA \left(\sqrt{\frac{S + \log(nSA)}{TH}} + \frac{\vert \cE \vert}{n} \right)
	    \end{align}
			holds with probability at least $ 1 - \frac{4}{n} - e^{-\frac{TH}{n}}$.
		\item[(ii)]for all $ TH = \Omega(n)$, the event
	    \begin{align}
	      \max_{s \in \cS}\Vert \hat{q}(s) - q(s)\Vert_1 \le \poly(\eta) S \left( \sqrt{\frac{n}{TH}} +  \frac{\vert \cE \vert}{n} \right)
	    \end{align}
			holds with probability at least $1 - \frac{4}{n} - e^{-\frac{TH}{n}}$
			% \begin{align}\label{eq:qerr1}
			% 	\PP\left( \max_{s \in \cS}\Vert \hat{q}(s) - p(s)\Vert_1 \le \poly(\eta) S \left( \sqrt{\frac{n}{TH}} +  \frac{\vert \cE \vert}{n} \right)  \right) \ge 1 - \frac{4}{n} - e^{-\frac{TH}{n}}
			% \end{align}
	\end{itemize}
\end{proposition}

\begin{proof}[Proof of Proposition \ref{prop:ee_p_q}] The proof is an immediate consequence of the estimation error decomposition Lemmas \ref{lem:p_err_dec} and \ref{lem:q_err_dec}, and the concentration bounds provided in Lemma \ref{lem:c1}, Lemma \ref{lem:c2}, Lemma \ref{lem:p_true}, and Lemma \ref{lem:q_true}.
\end{proof}

\begin{proposition}\label{prop:ee_q_avg}
	Under Assumptions \ref{assumption:SA}-\ref{assumption:uniform}, the estimator and $\hat{q}$ satisfy: for all $V \in \RR^n$, for all $\rho > 0$, for $TH \ge \poly(\eta)S( \rho + \log(S))$,
	\begin{align*}
		\PP\left( \max_{s \in \cS} \Big\vert \big( \hat{q}(s) - q(s) \big) V \Big\vert \le \poly(\eta) \Vert V \Vert_\infty S \left( \sqrt{  \frac{\rho + \log(S)}{TH} } +  \frac{\vert \cE\vert}{n} + \frac{\rho}{TH}  \right) \right)  \ge 1 - 4e^{-\rho}
	\end{align*}
\end{proposition}

\begin{proof}[Proof of Proposition \ref{prop:ee_q_avg}] The proof is an immediate consequence of the estimation error decomposition Lemma \ref{lem:q_err_dec}, and the concentration bounds provided in Lemma \ref{lem:c1}, Lemma \ref{lem:c2}, and Lemma \ref{lem:q_true}.
\end{proof}

\subsection{Estimation error decomposition lemmas}

A key step in the proof of Theorem \ref{thm:likelihood-improvement} - (ii) is to establish Lemmas \ref{lem:p_err_dec} and \ref{lem:q_err_dec}. These lemmas allow us to obtain upper bounds on the estimation error of $\hat{p}$ and $\hat{q}$ that only depend on the estimation error of $\hat{p}_f$ and $\hat{q}_f$ and the total number of misclassified nodes $\sum_{t,h} \indicator\lbrace \hat{f}(x_h^{(t)}) \neq f(x_h^{(t)}) \rbrace$ where here we use the slight abuse of notations that $\sum_{t,h} = \sum_{t = \lfloor T/2\rfloor + 1}^T \sum_{h=1}^{H+1}$.

Here, we state Lemma \ref{lem:p_err_dec} which will serve in the analysis of the estimation error of $\hat{p}$.

\begin{lemma}[First Error Decomposition]\label{lem:p_err_dec}
  The estimator $\hat{p}$ satisfies the following error decomposition: for all $s \in \cS, a \in \cA$, provided that $\hat{N}_a(f^{-1}(s)) \neq 0$, we have
  \begin{align}
       \Vert \hat{p}(s,a) - p(s,a) \Vert_1 \le   \Vert \hat{p}(s,a) - p(s,a) \Vert_1 + \frac{6 \sum_{t,h} \indicator\lbrace \hat{f}(x_h^{(t)}) \neq f(x_h^{(t)}) \rbrace}{\hat{N}_{a}(f^{-1}(s))}.
  \end{align}
\end{lemma}

Next, we state Lemma \ref{lem:q_err_dec} which will serve to analyze the estimation error of $\hat{q}$.

\begin{lemma}[Second Error Decomposition]\label{lem:q_err_dec}
  The estimator $\hat{q}$ satisfies the following error decomposition: for all $s \in \cS$, provided that $\hat{N}({f^{-1}(s)}) \neq 0$, we have
  \begin{align}
    \Vert \hat{q}(s) - q(s) \Vert_1 \le \Vert \hat{q}_f(s) - q(s) \Vert_1 +  \frac{4 \sum_{t,h} \indicator\lbrace \hat{f}(x_h^{(t)}) \neq f(x_h^{(t)}) \rbrace}{\hat{N}(f^{-1}(s))}.
  \end{align}
  Moreover, for any $V \in \RR^n$, it holds
  \begin{align}\label{eq:err2}
     \Big\vert \big(\hat{q}(s) - q(s) \big) V  \Big\vert & \le \Big\vert \big(\hat{q}_f(s) - q(s) \big) V \Big\vert +  \frac{4 \sum_{t,h} \indicator\lbrace \hat{f}(x_h^{(t)}) \neq f(x_h^{(t)}) \rbrace \Vert V\Vert_\infty}{\hat{N}(f^{-1}(s))}.
  \end{align}
\end{lemma}

%
% \begin{lemma}[First Error Decomposition]\label{lem:p_err_dec}
%   The estimator $\hat{p}$ satisfies the following error decomposition: for all $s \in \cS, a \in \cA$, provided that $\hat{N}_{f^{-1}(s)} \neq 0$, we have
%   \begin{align}\label{eq:err1}
%        \Vert \hat{p}(s,a) - p(s,a) \Vert_1 \le   \Vert \hat{p}(s,a) - p(s,a) \Vert_1 + \frac{5 \vert \cE \vert \max_{x \in \cX} \hat{N}_x}{\hat{N}_{a,f^{-1}(s)}}
%   \end{align}
% \end{lemma}
%
% \begin{lemma}[Second Error Decomposition]\label{lem:q_err_dec}
%   The estimator $\hat{q}$ satisfies the following error decomposition: for all $s \in \cS$, provided that $\hat{N}_{f^{-1}(s)} \neq 0$, we have
%   \begin{align}\label{eq:err1}
%     \Vert \hat{q}(s) - q(s) \Vert_1 \le \Vert \hat{q}_f(s) - q(s) \Vert_1 +  \frac{4 \vert \cE \vert\max_{x \in \cX} \hat{N}_x}{\hat{N}_{f^{-1}(s)}}.
%   \end{align}
%   Moreover, for any $V \in \RR^n$, it holds
%   \begin{align}\label{eq:err2}
%      \Big\vert \big(\hat{q}(s) - q(s) \big) V  \Big\vert & \le \Big\vert \big(\hat{q}_f(s) - q(s) \big) V \Big\vert + +  \frac{4 \vert \cE \vert\max_{x \in \cX} \hat{N}_x \Vert V\Vert_\infty}{\hat{N}_{f^{-1}(s)}}.
%   \end{align}
% \end{lemma}

\subsection{Concentration bounds}

\begin{lemma}\label{lem:c1}
	Under Assumptions \ref{assumption:SA}-\ref{assumption:uniform}, we have for all $\rho > 0$,
	\begin{align}\label{eq:c1}
		\PP\left( \frac{1}{TH}\sum_{t,h} \indicator\{ \hat{f}(x_{h}^{(t)}) \neq f(x_{h}^{(t)})\} \le \poly(\eta) \left(\frac{\vert \cE \vert}{n}  +  \frac{\rho}{TH} \right)\right) \ge 1 - e^{-\rho}
	\end{align}
	Consequently, we have:
	\begin{align}\label{eq:c1:consequence}
		\PP\left( \frac{1}{TH}\sum_{t,h} \indicator\{ \hat{f}(x_{h}^{(t)}) \neq f(x_{h}^{(t)})\} \le  \poly(\eta) \frac{\vert \cE \vert}{n} \right) \ge 1 - e^{-\frac{TH}{n}}
	\end{align}
\end{lemma}

\begin{lemma}\label{lem:c2}
	Under Assumptions \ref{assumption:SA}-\ref{assumption:uniform}, forall $\rho > 0$, we have:
\begin{itemize}
	\item[(i)] for all $TH \ge \poly(\eta) S (\rho + \log(S))$
	\begin{align}
			\PP\left( \min_{s \in \cS}\hat{N}(f^{-1}(s))  \ge \poly\left( \frac{1}{\eta}\right)\frac{TH}{S}  \right) \le 1 - e^{-\rho}.
	\end{align}
	\item[(ii)] For all $TH \ge \poly(\eta) SA (\rho + \log(SA))$, we obtain
	\begin{align}
			\PP\left( \min_{s \in \cS, a \in \cA} \hat{N}_a(f^{-1}(s))  \ge  \poly\left(\frac{1}{\eta}\right) \frac{TH}{SA}    \right) \le e^{-\rho}.
	\end{align}
\end{itemize}
\end{lemma}

An immediate consequence of Lemma \ref{lem:c2} is that for $TH = \Omega(\log(n))$, we have
\begin{align*}
	\PP\left( \min_{s \in \cS} \hat{N}(f^{-1}(s))  \ge  \poly\left(\frac{1}{\eta}\right) \frac{TH}{S}   \right) \ge 1 - \frac{1}{n}
\end{align*}
and
\begin{align*}
	\PP\left( \min_{s \in \cS, a \in \cA} \hat{N}_a(f^{-1}(s))  \ge  \poly\left(\frac{1}{\eta}\right) \frac{TH}{SA}   \right) \ge 1 - \frac{1}{n}.
\end{align*}
where $\Omega(\cdot)$ hides a dependence on $\poly(\eta)SA\log(SA)$.

\begin{lemma}\label{lem:p_true}
 Under Assumptions \ref{assumption:SA}-\ref{assumption:uniform}, for $TH = \Omega(\log(n))$, we have
 \begin{align}
 	\PP\left( \max_{s \in \cS, a \in \cA} \Big \Vert \hat{p}_f(s, a) - p(s, a) \Big \Vert_1 \le  \poly(\eta)  SA \sqrt{\frac{ S + \log(nSA)}{TH}} \right) \le 1 - \frac{3}{n}.
 \end{align}
\end{lemma}

\begin{lemma}\label{lem:q_true}
Under Assumptions \ref{assumption:SA}-\ref{assumption:uniform}, we obtain:
\begin{itemize}
	\item[(i)] for all $\rho> 0$, for all $V \in \RR^n$, for all $TH \ge \poly(\eta)S(\rho + \log(S))$,
	\begin{align}
			\PP\left( \max_{s \in \cS} \Big \vert \big(\hat{q}_f( s) - q( s) \big)V \Big \vert \le  \poly(\eta) \Vert V\Vert_\infty S  \sqrt{\frac{\rho + \log(S)}{TH}} \right) \ge 1- 3e^{-\rho}.
	\end{align}
	\item[(ii)] For $TH = \Omega(n)$,
	\begin{align}
			\PP\left(   \max_{s \in \cS}\Big\Vert \hat{q}_f( s) - q( s)  \Big \Vert_1 \le \poly(\eta) S\sqrt{\frac{n}{TH}}  \right) \ge 1 - \frac{3}{n}.
	\end{align}
\end{itemize}
%  For $TH = \Omega(\log(n))$, for all $V \in \RR^{n}$ such that $\Vert V \Vert_\infty < \infty$, we have
%   \begin{align*}
%       \PP\left( \max_{s \in \cS} \Big \vert \big(\hat{q}_f( s) - q( s) \big)V \Big \vert >   \poly(\eta) \Vert V\Vert_\infty S  \sqrt{\frac{\log(n)}{TH}} \right) \le \frac{3}{n}
%   \end{align*}
% Furthermore, for $TH = \Omega(n)$, we obtain
%  \begin{align*}
%      \PP\left(   \max_{s \in \cS}\Big\Vert \hat{q}_f( s) - q( s)  \Big \Vert_1 > \poly(\eta) S\sqrt{\frac{n}{TH}}  \right) \le \frac{3}{n}
%  \end{align*}
\end{lemma}

%
% \begin{lemma}\label{lem:1}
%   For $TH = \Omega(n)$, we have
%   \begin{align*}
%     \PP\left( \max_{x \in \cX} \hat{N}(x) \le  \poly(\eta) \left(\frac{TH}{n} + \log(n) \right) \right) \ge 1 - \frac{1}{n}
%   \end{align*}
% \end{lemma}
% \begin{lemma}\label{lem:2}
%   For $TH = \Omega(\log(n))$, we have
%   \begin{align*}
%     \PP\left( \min_{s \in \cS} \hat{N}(f^{-1}(s))  \ge  \poly\left(\frac{1}{\eta}\right) \frac{TH}{S}   \right) \ge 1 - \frac{1}{n}
%   \end{align*}
% \end{lemma}

\subsection{Proofs -- Estimation error decompositions}

\begin{proof}[Proof of Lemma \ref{lem:p_err_dec}] Let $U \in \RR^S$ such that $\Vert U \Vert_\infty \le 1$. We wish to relate the estimation error of $\vert (\hat{p}(s,a) - p(s,a) ) U \vert$ to that $\vert (\hat{p}_f(s,a) - p(s,a) ) U\vert$ and the number of misclassified nodes $\vert \cE \vert$. First, we start by writing
	\begin{align*}
		\hat{N}_{a}(f^{-1}(s))   \big(\hat{p}(s,a) - p(s,a) \big ) U = \Delta_1 + \Delta_2 + \Delta_3 + \Delta_4
	\end{align*}
	where we define
	\begin{align*}
		\Delta_1 & = \left( \hat{N}_{a}(f^{-1}(s))  - \hat{N}_{a}(\hat{f}^{-1}(s)) \right)   \big(\hat{p}(s,a) - p(s,a) \big ) U \\
		\Delta_2 & =  \sum_{s' \in \cS} \Big( \hat{N}_a(f^{-1}(s), f^{-1}(s')) - p(s' \vert s,a) \hat{N}_{a}(f^{-1}(s)) \Big) U(s') \\
		\Delta_3 & =  \sum_{s' \in \cS} \Big( \Big( \hat{N}_{a}(\hat{f}^{-1}(s),  f^{-1}(s')) \\
		& \qquad\qquad- \hat{N}_{a}(f^{-1}(s), f^{-1}(s')) \Big)  - p(s' \vert s,a) \left(\hat{N}_{a}(\hat{f}^{-1}(s))    -  \hat{N}_{a}(f^{-1}(s) )   \right)\Big) U(s') \\
		\Delta_4 & =  \sum_{s' \in \cS} \Big( \hat{N}_{a}( \hat{f}^{-1}(s), f^{-1}(s')) - \hat{N}_{a}(\hat{f}^{-1}(s), \hat{f}^{-1}(s')) \Big) U(s'). 
			\end{align*}

	\textbf{\emph{Bounding $\Delta_1$}}. The term $\Delta_1$ can be bounded as follows:
	\begin{align*}
		\vert \Delta_1 \vert & \le  \Big\vert \sum_{t,h} \delta_{t,h,a}(\hat{f}^{-1}(s)) - \delta_{t,h,a}(f^{-1}(s))\Big\vert \Vert\hat{p}(s,a) - p(s,a)\Vert_1 \\
		& \le  2 \sum_{t,h} \Big\vert \delta_{t,h,a}(\hat{f}^{-1}(s)) - \delta_{t,h,a}(f^{-1}(s))\Big\vert  \\
		& \le 2  \sum_{t,h} \sum_{x \in \cX} \indicator\lbrace \hat{f}(x) \neq f(x)\rbrace \indicator\lbrace x_{h}^{(t)} = x \rbrace \max\left( \indicator \lbrace f(x) = s\rbrace, \indicator \lbrace \hat{f}(x) = s\rbrace \right) \\
		& \le 2  \sum_{t,h} \sum_{x \in \cX} \indicator\lbrace \hat{f}(x) \neq f(x)\rbrace \indicator\lbrace x_{h}^{(t)} = x \rbrace
	\end{align*}
	where in the second inequality, we used the fact that $\Vert \hat{p}(s,a) - p(s,a)\Vert_1 \le 2$.

	\textbf{\emph{Bounding $\Delta_2$}}. We observe that
	\begin{align*}
		\vert \Delta_2 \vert & = \hat{N}_{a}(f^{-1}(s)) \Big\vert \sum_{s' \in \cS} \Big(\hat{p}_f(s'\vert s,a ) - p(s'\vert s,a) \Big)U(s') \Big\vert \\
		& \le \hat{N}_{a}(f^{-1}(s))  \Vert \hat{p}_f(s,a) -p(s,a)\Vert_1.
	\end{align*}

	\textbf{\emph{Bounding $\Delta_3$}}. We bound $\Delta_3$ as follows
	\begin{align*}
		\Delta_3 &  = \left \vert \sum_{s' \in \cS} \sum_{t,h} \Big( \delta_{t,h+1}(f^{-1}(s')) - p(s' \vert s,a)   \Big) \Big( \delta_{t,h,a}(\hat{f}^{-1}(s)) - \delta_{t,h,a}(f^{-1}(s))\Big) U(s')\right \vert \\
		&  \le  \sum_{s' \in \cS} \sum_{t,h} \left \vert \Big( \delta_{t,h+1}(f^{-1}(s')) - p(s' \vert s,a) \Big)U(s')  \right \vert  \left \vert \delta_{t,h,a}(\hat{f}^{-1}(s)) - \delta_{t,h,a}(f^{-1}(s)) \right \vert \\
		&  \le  \sum_{t,h} 2 \left \vert \delta_{t,h,a}(\hat{f}^{-1}(s)) - \delta_{t,h,a}(f^{-1}(s)) \right \vert \\
		&  \le  \sum_{t,h} 2  \sum_{x \in \cX} \indicator{\lbrace f(x) \neq \hat{f}(x)\rbrace} \indicator{ \lbrace x_{t,h} = x \rbrace}.
	\end{align*}

	\textbf{\emph{Bounding $\Delta_4$}}. We bound $\Delta_4$ as follows
	\begin{align*}
	\Delta_4 & = \left \vert \sum_{t,h}  \sum_{s' \in \cS} \left( \delta_{t,h+1}(f^{-1}(s')) - \delta_{t,h+1}(\hat{f}^{-1}(s'))  \right) \delta_{t,h,a}(\hat{f}^{-1}(s)) \right \vert \\
	 & \le \sum_{t,h}  \sum_{s' \in \cS} \left \vert  \delta_{t,h+1}(f^{-1}(s')) - \delta_{t,h+1}(\hat{f}^{-1}(s'))   \right \vert \delta_{t,h,a}(\hat{f}^{-1}(s)) \\
	 & \le \sum_{t,h}  \sum_{s' \in \cS} \left \vert  \delta_{t,h+1}(f^{-1}(s')) - \delta_{t,h+1}(\hat{f}^{-1}(s'))   \right \vert \\
	 &  \le \sum_{t,h}  \sum_{s' \in \cS} \indicator \lbrace \hat{f}(x_{h+1}^{(t)}) \neq f(x_{h+1}^{(t)})  \rbrace \max ( \indicator\{ f(x_{h+1}^{(t)}) = s' \}, \indicator\{ \hat{f}(x_{h+1}^{(t)}) = s' \}  )\\
	  & \le  2 \sum_{t,h}  \indicator \lbrace \hat{f}(x_{h+1}^{(t)}) \neq f(x_{h+1}^{(t)})  \rbrace.
	\end{align*}
	Finally, we conclude by writing
\begin{align*}
	 \Vert \hat{p}(s,a) - p(s,a) \Vert_1 \le   \Vert \hat{p}(s,a) - p(s,a) \Vert_1 + \frac{6 \sum_{t,h}  \indicator \lbrace \hat{f}(x_{h}^{(t)}) \neq f(x_{h}^{(t)})  \rbrace }{\hat{N}_{a}(f^{-1}(s))}
\end{align*}
provided $\hat{N}_{a}(f^{-1}(s)) > 0$, where we slightly abuse notations and use $\sum_{t,h} = \sum_{t = \lfloor T/2\rfloor + 1}^T \sum_{h=1}^{H+1}$.
\end{proof}

\begin{proof}[Proof of Lemma \ref{lem:q_err_dec}]
  Let $V \in \RR^n$, we start relating the estimation error of $\vert (\hat{q}(s) - q(s)) V\vert$ to that of $\vert (\hat{q}_f(s) - q(s)) V\vert$ and the misclassification error. Let us also recall that
	\begin{align*}
		\forall x \in \cX, s \in \cS, \quad \hat{q}(x \vert s) = \frac{\hat{N}(\lbrace x \rbrace \cap \hat{f}^{-1}(s) )}{\hat{N}(\hat{f}^{-1}(s))} \quad \textrm{and} \quad \hat{q}_f(x \vert s) = \frac{\hat{N}(\lbrace x \rbrace \cap f^{-1}(s) )}{\hat{N}(f^{-1}(s))}
	\end{align*}
	whenever $\hat{N}(\hat{f}^{-1}(s)), \hat{N}(f^{-1}(s)) > 0$, otherwise the estimates are set to zero. Now, we have
  \begin{align*}
    & \hat{N}(\hat{f}^{-1}(s))\Big\vert \big(\hat{q}( s) - q(s) \big) V \Big\vert   = \left \vert \sum_{x \in \cX} \Big( \hat{N}(\lbrace x \rbrace \cap \hat{f}^{-1}(s)) - q(x \vert s) \hat{N}(\hat{f}^{-1}(s)) \Big) V(x)\right \vert  \\
    &  \qquad \le  \left \vert \sum_{x \in \cX}  \left( \hat{N}( \lbrace x \rbrace \cap f^{-1}(s)) - q(x \vert s) \hat{N}(f^{-1}(s)) \right) V(x) \right \vert \\
    & \qquad \qquad +  \left \vert \sum_{t,h}\sum_{x \in \cX} \left( \delta_{t,h}(x)  - q(x \vert s) \right) \left( \delta_{t,h}(\hat{f}^{-1}(s)) - \delta_{t,h}(f^{-1}(s))\right)  V(x)  \right \vert \\
    & \qquad \le \left \vert  \sum_{x \in \cX}  \left( \hat{N}(\lbrace x \rbrace  \cap f^{-1}(s)) - q(x \vert s) \hat{N}(f^{-1}(s)) \right) V(x) \right \vert  \\
    & \qquad \qquad +   \left \vert \sum_{t,h}    \Big ( \sum_{x \in \cX} \left\vert \delta_{t,h}(x)  - q(x \vert s)  \right\vert \Big)  \Big\vert \delta_{t,h}(\hat{f}^{-1}(s))  - \delta_{t,h}( f^{-1}(s))  \Big\vert \Vert V \Vert_\infty \right \vert  \\
    & \qquad \le  \left \vert \sum_{x \in \cX} \left( \hat{N}(\lbrace x \rbrace \cap f^{-1}(s)) - q(x \vert s) \hat{N}(f^{-1}(s)) \right) V(x) \right \vert  \\
    & \qquad \qquad +   \left \vert \sum_{t,h}    \Big ( \sum_{x \in \cX} \left\vert \delta_{t,h}(x)  - q(x \vert s)  \right\vert \Big)  \Big( \sum_{x \in \cX} \delta_{t,h}(x) \indicator\{\hat{f}(x) \neq f(x)\} \Big)  \Vert V \Vert_\infty \right \vert  \\
    &  \qquad \le \left \vert \sum_{x \in \cX} \left( \hat{N}( \lbrace x \rbrace \cap f^{-1}(s)) - q(x \vert s) \hat{N}(f^{-1}(s)) \right)  V(x)\right \vert \\
    & \qquad \qquad +  2 \sum_{x \in \cX} \Big(\sum_{t,h} \delta_{t,h}(x) \Big) \indicator\{ \hat{f}(x) \neq f(x)\} \Vert V \Vert_\infty
  \end{align*}
 where we used at the end we used the fact that $\Vert \delta_{t,h}(\cdot) - q(\cdot \vert x)\Vert_\infty \le 1$. Thus, provided that $\hat{N}(\hat{f}^{-1}(s)) >  0$ and $\hat{N}(f^{-1}(s)) > 0 $, we obtain
\begin{align*}
  \hat{N}(\hat{f}^{-1}(s))\Big\vert \big(\hat{q}(s) - q(s) \big)V  \Big\vert \le \hat{N}(f^{-1}(s))\Big\vert \big(\hat{q}_f(s) - q(s) \big)V  \Big\vert + 2 \vert \cE \vert \max_{x \in \cX} \hat{N}(x) \Vert V \Vert_\infty
\end{align*}
where we recall that $\vert \cE \vert = \sum_{x \in \cX} \indicator\{ \hat{f}(x) \neq f(x) \}$. Furthermore,
\begin{align*}
  \hat{N}(f^{-1}(s))  \Big\vert \big(\hat{q}(s) - q(s) \big) V  \Big\vert & \le  \hat{N}(f^{-1}(s))\Big\vert \big(\hat{q}_f(s) - q(s) \big)V  \Big\vert + 2 \vert \cE \vert \max_{x \in \cX} \hat{N}(x) \Vert V \Vert_\infty \\
  & \qquad +  (\hat{N}(\hat{f}^{-1}(s)) -   \hat{N}(\hat{f}^{-1}(s)) ) \Big\vert \big(\hat{q}(s) - q(s) \big) V  \Big\vert \\
  & \le \hat{N}(f^{-1}(s)) \Big\vert \big(\hat{q}_f(s) - q(s) \big) V \Big\vert + 4 \vert \cE \vert \max_{x \in \cX} \hat{N}(x) \Vert V \Vert_\infty
\end{align*}
where we used the fact $ \Vert \hat{q}(s) - q(s) \Vert_1 \le 2$, and $ \vert \hat{N}(\hat{f}^{-1}(s)) -   \hat{N}(f^{-1}(s))  \vert \le \vert \cE \vert \max_{x \in \cX} \hat{N}(x)$. Therefore, by taking the supremum over $V$ such that $\Vert V \Vert_\infty \le 1$, we finally obtain, provided $\hat{N}(f^{-1}(s))> 0$, that
\begin{align*}
  \Vert \hat{q}(s) - q(s) \Vert_1 \le \Vert \hat{q}_f(s) - q(s) \Vert_1 +  \frac{4 \vert \cE \vert\max_{x \in \cX} \hat{N}(x)}{\hat{N}(f^{-1}(s))}.
\end{align*}
\end{proof}

\subsection{Proofs -- Concentration bounds}

\begin{proof}[Proof of Lemma \ref{lem:c1}] In view of Proposition \ref{prop:mc0}, we can easily verify that for all $t > \lfloor T/2\rfloor$,$h \in[H]$,
	$$
	\EE_\mu[\indicator\{\hat{f}(x_{h}^{(t)}) \neq f(x_{h}^{(t)}) \} \vert \hat{f}] \le \max_{x \in \cX}\EE_\mu[\indicator\{x_{h}^{(t)} = x \} \vert \hat{f}] \vert \cE \vert \le  \max_{x \in \cX}\mu^\top  P_0^{h-1} (x) \vert \cE \vert  \le \frac{\eta^2\vert \cE \vert}{n}.
	$$
Thus, $\EE_\mu[ \sum_{t,h} \indicator\{ \hat{f}(x_{h}^{(t)}) \neq f(x_{h}^{(t)})\} \vert \hat{f}] \le \frac{\eta^2 TH\vert \cE\vert}{2n}$. Now, we may apply Theorem \ref{thm:bernstein-restart} conditionally on $\hat{f}$, which gives for all $u >0$ (possibly depending on $\hat{f}$)
	\begin{align*}
    & \EE\left[ \indicator\left\{\sum_{t,h} \indicator\{ \hat{f}(x_{h}^{(t)}) \neq f(x_{h}^{(t)})\} >  \frac{\eta^2TH\vert \cE\vert}{2n}+ u \right\}  \Big\vert \hat{f} \right]  \\
		& \qquad \le \EE\left[ \indicator\left\{\sum_{t,h} \indicator\{ \hat{f}(x_{h}^{(t)}) \neq f(x_{h}^{(t)})\} >  \EE_\mu\left[\sum_{t,h} \indicator\{ \hat{f}(x_{h}^{(t)}) \neq f(x_{h}^{(t)})\} \Big \vert \hat{f}\right]+ u \right\}  \Big\vert \hat{f} \right]  \\
		& \qquad \le \exp\left(-\frac{u^2}{8 \eta^{10} \frac{TH \vert \cE\vert}{n} + \frac{4}{3} \eta^2u} \right) \\
		& \qquad \le  \exp\left(- \min \left( \frac{ nu^2}{8 \eta^{10}TH \vert \cE \vert}, \frac{u}{2\sqrt{2} \eta^5}  \right)  \right).
  \end{align*}
	Reparametrizing by $z = \frac{nu}{TH\vert \cE \vert}$ gives
	\begin{align*}
    & \EE\left[ \indicator\left\{\sum_{t,h} \indicator\{ \hat{f}(x_{h}^{(t)}) \neq f(x_{h}^{(t)})\} >  \frac{\eta^2TH\vert \cE\vert}{2n}+ \frac{TH \vert \cE \vert}{n} z \right\}  \Big\vert \hat{f} \right]  \\
		& \qquad \le  \exp\left(- \frac{TH \vert \cE \vert}{n}\min \left( \frac{ z^2}{8 \eta^{10}}, \frac{z}{2\sqrt{2} \eta^5}  \right)  \right).
	\end{align*}

	Further reparametrizing $ \rho = \frac{TH \vert \cE \vert}{n}\min \left( \frac{z^2}{8 \eta^{10}}, \frac{z}{2\sqrt{2} \eta^5}  \right)$ gives
	\begin{align*}
		 \EE\left[ \indicator\left\{\sum_{t,h} \indicator\{ \hat{f}(x_{h}^{(t)}) \neq f(x_{h}^{(t)})\} >  \frac{\eta^2TH\vert \cE\vert}{2n}+ 2\sqrt{2} \eta^5 \max \left\lbrace  \sqrt{\frac{TH \vert \cE \vert \rho}{n}   }, \rho \right\rbrace \right\}  \Big\vert \hat{f} \right] \le  e^{-\rho}.
	\end{align*}
	Finally, noting that $\max \left\lbrace  \sqrt{\frac{TH \vert \cE \vert \rho}{n}   }, \rho \right\rbrace \le \frac{TH \vert \cE \vert}{n} + \rho$, we conclude that for all $\rho > 0$
	\begin{align*}
		\PP\left( \frac{1}{TH}\sum_{t,h} \indicator\{ \hat{f}(x_{h}^{(t)}) \neq f(x_{h}^{(t)})\} > \left(\frac{\eta^2}{2} + s\sqrt{2} \eta^5\right) \frac{\vert \cE \vert}{n}  + 2 \sqrt{2} \eta^5 \frac{\rho}{TH}\right) \le e^{-\rho}.
	\end{align*}
\end{proof}

\begin{proof}[Proof of Lemma \ref{lem:c2}]
  In view of Proposition \ref{prop:mc0},  we know that $P_0$ is $\eta^3$-regular. We can easily verify that for all $t > \lfloor T/2 \rfloor+1$,$ h \in [H]$,
	\begin{align*}
		\EE_\mu[\indicator\{f(x_{t,h}) = s \}] & \ge \min_{x \in \cX} \EE_{\mu}[ \indicator\{x_{h}^{(t)} = x \}] \alpha_s n  = \min_{x \in \cX}\mu^\top (P_0)^{h-1} (x) \alpha_s n \ge \frac{1}{\eta^4 S} \\
		\EE_\mu[\indicator\{f(x_{t,h}) = s \}] & \le \max_{x \in \cX} \EE_{\mu}[ \indicator\{x_{h}^{(t)} = x \}] \alpha_s n  = \max_{x \in \cX}\mu^\top (P_0)^{h-1} (x) \alpha_s n \le \frac{\eta^4}{ S}.
	\end{align*}
	Thus, we have $\frac{TH}{2\eta^4 S}\le \EE_\mu[ \hat{N}(f^{-1}(s))] \le \frac{TH\eta^4}{2S}$. Now, applying Theorem \ref{thm:bernstein-restart}, we can immediately obtain that for all $ u > 0$,
  \begin{align*}
    \PP\left(  \frac{TH}{2\eta^3S} - u > \hat{N}(f^{-1}(s)) \right)  & \le \PP(  \EE_\mu[\hat{N}(f^{-1}(s))] > \hat{N}(f^{-1}(s))  + u )  \\
		& \le \exp\left(-\frac{u^2}{8 \eta^{10} \frac{TH}{S} + \frac{4}{3} \eta^2u }   \right) \\
		& \le \exp\left(- \min\left(\frac{Su^2}{8\eta^{11} TH}, \frac{u}{ 2\sqrt{2}\eta^{11/2}}   \right)  \right).
  \end{align*}
	Using a union bound gives us
  \begin{align*}
      \PP\left(  \frac{TH}{2\eta^3S} - u > \min_{s \in \cS}\hat{N}(f^{-1}(s)) \right) \le \exp\left(- \min\left(\frac{Su^2}{8\eta^{11} TH}, \frac{u}{ 2\sqrt{2}\eta^{11/2}}   \right) + \log(S) \right).
  \end{align*}
	Reparametrizing by $\rho = \min\left(\frac{Su^2}{8\eta^{11} TH}, \frac{u}{ 2\sqrt{2}\eta^{11/2}}   \right) - \log(S)$, yields
	\begin{align*}
      \PP\left(  \frac{TH}{2\eta^3S} - 2\sqrt{2}\eta^{11/2} \max \left\lbrace \sqrt{\frac{TH}{S} (\rho + \log(S))}, \rho + \log(S) \right\rbrace > \min_{s \in \cS}\hat{N}(f^{-1}(s)) \right) \le e^{-\rho}.
  \end{align*}
	Thus, for all $TH \ge 8 \eta^{17} S (\log(\rho) + \log(S))$, we obtain
	\begin{align*}
			\PP\left(  \frac{TH}{4\eta^3S} > \min_{s \in \cS}\hat{N}(f^{-1}(s)) \right) \le e^{-\rho}.
	\end{align*}
	Choosing $\rho = \log(n)$, we obtain that for $TH = \Omega(\log(n))$,
	\begin{align*}
			\PP\left(  \frac{TH}{4\eta^3S} > \min_{s \in \cS}\hat{N}(f^{-1}(s)) \right) \le \frac{1}{n}
	\end{align*}
	where $\Omega(\cdot)$ hides a dependence in $\poly(\eta) S \log(S)$.

	\medskip

	Following, a similar proof with the only exception that we use instead the fact $P_1$ is $\eta^2$-regular, we obtain that for $TH = \Omega(\log(n))$,
	\begin{align*}
			\PP\left(  \frac{TH}{4\eta^3SA} > \min_{s \in \cS, a \in \cA}\hat{N}_a(f^{-1}(s)) \right) \le \frac{1}{n}
	\end{align*}
	where $\Omega(\cdot)$ hides a dependence in $\poly(\eta) SA \log(SA)$.

\end{proof}

\begin{lemma}[$\epsilon$-net argument for $\ell_1$ norm] \label{lem:net-ell1}
  Let $q$ be a $d$-dimensional random vector, and $\epsilon \in (0, 1)$. Furthermore, let $\cN$ be an $\epsilon$-net of the unit ball with respect to $\Vert \cdot  \Vert_\infty$, with minimal cardinality. Then, for all $\rho> 0$, we have
  \begin{align*}
    \PP\left( \Vert q \Vert_1 >  \frac{\rho}{1 - \epsilon} \right) \le \left( \frac{3}{\epsilon} \right)^d \max_{V \in \cN} \PP\left(  q^\top V > \rho \right).
  \end{align*}
\end{lemma}

\begin{proof}[Proof of Lemma \ref{lem:p_true}]
Let $U \in \RR^S$ such that $\Vert U \Vert_\infty \le 1$. We have
\begin{align*}
	\hat{N}_a(f^{-1}(s))  \Big \vert  \big(\hat{p}_f(s,a) - p(s,a) \big) U \Big \vert  = \Big \vert \sum_{t,h} \sum_{s' \in \cS} ( \delta_{t,h+1}(f^{-1}(s')) - p(s' \vert s,a)) U(s')\delta_{t,h,a}(f^{-1}(s))  \Big \vert.
\end{align*}
We note by Hoeffding's lemma that for all  $\lambda > 0$, we have
\begin{align*}
		\EE\Big[ \exp\Big(\lambda \Big(   \sum_{s'\in \cS}  U(s')( \delta_{t,h+1}(f^{-1}(s')) - p(s' \vert s,a)) \delta_{t,h,a}(f^{-1}(s)) \Big)  \Big)  \Big \vert s_{h}^{(t)}\Big] \le \exp(\frac{\lambda^2}{2}).
\end{align*}
Using the above inequality alongside a peeling argument, we obtain
\begin{align*}
		\EE\Big[ \exp\Big(\lambda \Big(   \sum_{t,h} \sum_{s'\in \cS}  U(s')( \delta_{t,h+1}(f^{-1}(s')) - p(s' \vert s,a)) \delta_{t,h,a}(f^{-1}(s))\Big)  \Big) \Big] \le \exp(\frac{TH \lambda^2 }{2}).
\end{align*}
Now using Markov's inequality and optimizing over $\lambda > 0$, we obtain that for all $\rho > 0$
\begin{align} \label{eq:p_est_1}
	\PP\left(  \hat{N}_a(f^{-1}(s))   \big(\hat{p}_f( s, a) - p( s,a) \big)U > \rho  \right) \le \exp( -\frac{\rho^2}{2TH} ).
\end{align}
This gives via a standard union bound
\begin{align*}
	\PP\left( \min_{s \in \cS, a \in \cA}\hat{N}_a(f^{-1}(s))  \max_{s \in \cS, a \in \cA} \Big \vert \big(\hat{p}_f(s, a) - p(s, a) \big)U \Big \vert > \rho  \right) \le 2 \exp( -\frac{\rho^2}{2TH}  + \log(SA)).
\end{align*}
Reparametrizing by $\rho = \sqrt{TH (\rho' + \log(SA))}$, we may write for all $\rho' > 0$,
\begin{align*}
	\PP\left( \min_{s \in \cS, a \in \cA}\hat{N}_a(f^{-1}(s))  \max_{s \in \cS, a \in \cA} \Big \vert \big(\hat{p}_f( s,a) - p( s,a) \big)U \Big \vert >  \sqrt{TH (\rho' + \log(SA))} \right) \le 2 e^{-\rho'}.
\end{align*}
Using Lemma \ref{lem:c2}, we obtain that for all $TH \ge \poly(\eta)SA(\rho' + \log(SA))$,
\begin{align*}
		\PP\left( \max_{s \in \cS, a \in \cA} \Big \vert \big(\hat{p}_f(s, a) - p(s, a) \big)U \Big \vert > \poly(\eta)  SA \sqrt{\frac{ \rho' + \log(SA)}{TH}} \right) \le 3 e^{-\rho'}.
\end{align*}
Now, we apply an $\epsilon$-net argument using Lemma \ref{lem:net-ell1} with $\epsilon = 1/2$ to obtain
\begin{align*}
	\PP\left( \max_{s \in \cS, a \in \cA} \Big \Vert\hat{p}_f(s, a) - p(s, a) \Big \Vert_1 > \poly(\eta)  SA \sqrt{\frac{ \rho' + \log(SA)}{TH}} \right) \le 3 e^{-\rho' + \log(6)S}.
\end{align*}
Reparametrizing by $\rho'' =  \rho' + \log(6) S $
\begin{align*}
	\PP\left( \max_{s \in \cS, a \in \cA} \Big \Vert \big(\hat{p}_f(s, a) - p(s, a) \big)\Big \Vert_1 > \poly(\eta)  SA \sqrt{\frac{ \rho' + \log(SA) + S}{TH}} \right) \le 3 e^{-\rho'}.
\end{align*}
Choosing $\rho'' = \log(n)$, we finally obtain for $TH = \Omega(\log(n))$
\begin{align*}
	\PP\left( \max_{s \in \cS, a \in \cA} \Big \vert \hat{p}_f(s, a) - p(s, a) \Big \Vert_1 > \poly(\eta)  SA \sqrt{\frac{ S + \log(nSA)}{TH}} \right) \le \frac{3}{n}
\end{align*}
where $\Omega(\cdot)$ hides a dependence of order $\poly(\eta) SA\log(SA)$ .
\end{proof}

\begin{proof}[Proof of Lemma \ref{lem:q_true}]
  Let $V \in \RR^n$ such that $\Vert V \Vert_\infty< \infty$.

	\textbf{\emph{Proof of (i).}} We have
  \begin{align*}
    \hat{N}(f^{-1}(s))  \Big \vert  \big(\hat{q}_f(s) - q(s) \big ) V \Big \vert  = \Big \vert \sum_{t,h} \sum_{x \in \cX} ( \delta_{t,h}(x) - q(x \vert s)) \indicator \{ s_{t,h} = s\} V(x) \Big \vert.
  \end{align*}
  We note by Hoeffding's lemma that for all  $\lambda > 0$, we have
  \begin{align*}
      \EE\Big[ \exp\Big(\lambda \Big(   \sum_{x \in \cX} ( \delta_{t,h}(x) - q(x \vert s)) \indicator \{ s_{t,h} = s\}  V(x)\Big)  \Big)  \Big \vert s_{t,h}\Big] \le \exp(\frac{\lambda^2 \Vert V \Vert_\infty^2}{2}).
  \end{align*}
  Using the above inequality alongside a peeling argument, we immediately obtain
  \begin{align*}
      \EE\Big[ \exp\Big(\lambda \Big(   \sum_{t,h} \sum_{x \in \cX} ( \delta_{t,h}(x) - q(x \vert s)) \indicator \{ s_{t,h} = s\}  V(x)\Big)  \Big) \Big] \le \exp(\frac{TH \lambda^2 \Vert V \Vert_\infty^2}{2}).
  \end{align*}
  Now using Markov's inequality and optimizing over $\lambda > 0$, we obtain that for all $\rho > 0$
  \begin{align} \label{eq:q_est_1}
    \PP\left(  \hat{N}(f^{-1}(s))   \big(\hat{q}_f( s) - q( s) \big)V > \rho  \right) \le \exp( -\frac{\rho^2}{2TH\Vert V \Vert_\infty^2} ).
  \end{align}
  This gives via a standard union bound
  \begin{align*}
    \PP\left( \min_{s \in \cS}\hat{N}(f^{-1}(s))  \max_{s \in \cS} \Big \vert \big(\hat{q}_f( s) - q( s) \big)V \Big \vert > \rho  \right) \le 2 \exp( -\frac{\rho^2}{2TH\Vert V \Vert_\infty^2}  + \log(S)).
  \end{align*}
  Reparametrizing by $\rho = \Vert V\Vert_\infty \sqrt{TH (\rho' + \log(S))}$, we may write for all $\rho' > 0$,
  \begin{align*}
    \PP\left( \min_{s \in \cS}\hat{N}(f^{-1}(s))  \max_{s \in \cS} \Big \vert \big(\hat{q}_f( s) - q( s) \big)V \Big \vert > \Vert V\Vert_\infty  \sqrt{TH (\rho' + \log(S))} \right) \le 2 e^{-\rho'}.
  \end{align*}
	Using Lemma \ref{lem:c2}, we obtain that for all $TH \ge \poly(\eta)S(\rho' + \log(S))$,
	\begin{align*}
      \PP\left( \max_{s \in \cS} \Big \vert \big(\hat{q}_f( s) - q( s) \big)V \Big \vert > \poly(\eta) \Vert V\Vert_\infty  S \sqrt{\frac{ \rho' + \log(S)}{TH}} \right) \le 3 e^{-\rho'}.
  \end{align*}
	%
  % Choosing $\rho' = \log(n)$ and using Lemma \ref{lem:c2}, we can obtain for $TH = \Omega(\log(n))$
  % \begin{align*}
  %     \PP\left( \max_{s \in \cS} \Big \vert \big(\hat{q}_f( s) - q( s) \big)V \Big \vert > \poly(\eta) \Vert V\Vert_\infty  \frac{S \sqrt{TH (\log(n) + \log(S))}}{TH} \right) \le \frac{3}{n}
  % \end{align*}
	% where $\Omega(\cdot)$ hides dependences on $\poly(\eta) S \log(S)$

	\textbf{\emph{Proof of (ii)}} We start from the inequality \eqref{eq:q_est_1} and apply an $\epsilon$-net argument using Lemma \ref{lem:net-ell1} with $\epsilon = 1/2$ to obtain
	\begin{align*}
	  \PP\left(  \hat{N}(f^{-1}(s))   \Big\Vert \hat{q}_f( s) - q( s) \Big \Vert_1 > 2\rho  \right) \le \exp( -\frac{\rho^2}{2TH}  + \log(6) n ).
	\end{align*}
	Then, using a union bound we obtain
	\begin{align*}
	  \PP\left(  \min_{s \in \cS}\hat{N}(f^{-1}(s))   \max_{s \in \cS}\Big\Vert \hat{q}_f( s) - q( s)  \Big \Vert_1 > 2\rho  \right) \le \exp( -\frac{\rho^2}{2TH\Vert V \Vert_\infty^2}  + \log(6) n + \log(S) ).
	\end{align*}
	Reparametrizing by $\rho =  \sqrt{2TH (\rho' + \log(6) n + \log(S))}$
	\begin{align*}
	  \PP\left(  \min_{s \in \cS}\hat{N}(f^{-1}(s))   \max_{s \in \cS}\Big\Vert \hat{q}_f( s) - q( s) \Big \Vert_1 > 2\sqrt{2TH (\rho' + \log(6) n + \log(S))}  \right) \le e^{-\rho}.
	\end{align*}
	Choosing $\rho' = \log(n)$ and applying Lemma \ref{lem:c2}, we can obtain for $TH=\Omega(\log(n))$
	\begin{align*}
	    \PP\left( \max_{s \in \cS}\Big\Vert \hat{q}_f( s) - q( s)  \Big \Vert_1 > \poly(\eta) \frac{S\sqrt{2TH ( n + \log(n)+ \log(S))} }{TH}  \right) \le \frac{3}{n}.
	\end{align*}
\end{proof}

\newpage
% !TEX root = ./main.tex

\newpage
\section{Proofs of Theorems \ref{thm:lower-boundSC1} and \ref{thm:lower-boundSC2} -- Lower Bounds for Reward-Free RL}\label{app:reward-free-lower}

In this appendix, we provide the proofs of the sample complexity lower bounds in the (offline) reward-free RL.
On a high level, both rely on a specific construction of alternate BMDPs.

\subsection{Proof of Theorem \ref{thm:lower-boundSC1} -- minimax setting}
Recall that
$$\Lambda(\Phi) = \max_{v\in [-1,1]^S}\frac{1}{S}\sum_{s=1}^S\max_{a_{1}, a_{2}} \sum_{s'=1}^S (p(s'|s,a_{1}) - p(s'|s,a_{2})) v_{s'}.$$ 
In the following, we denote by $v^\star$ the $S$-dimensional vector achieving the maximum leading to $\Lambda(\Phi)$, and by $a_{1,s}^\star$ and $a_{2,s}^\star$ the two actions achieving this maximum for latent state $s$.
Such choices exist as the domains $[-1, 1]^S$ and $\cA$ are both compact.

We start by identifying a necessary condition for $\sum_{x\in f^{-1}(s)}|q(x |s) - \hat{q}(x |s)| \leq  \frac{\varepsilon}{\Lambda(\Phi)}$, where $\hat{q}$ would be estimated from the data.

From Lemma D.6 in \cite{jin2020reward}, for a given transition vector $q(x |s)$, we can define $M_s= e^{\Omega(n/S)}$ transition vectors $\{q^{(i)}(x |s) \}_{1\leq i \leq M_s}$ such that
\begin{itemize}
	\item $| q(x |s) - q^{(i)}(x |s)| =  \frac{8 \varepsilon }{ |f^{-1}(s)| \Lambda(\Phi)}$ for all $x \in f^{-1}(s)$,
	\item $\sum_{x\in f^{-1}(s)} | q^{(i)}(x |s) - q^{(j)}(x |s)| \geq \frac{2 \varepsilon}{\Lambda(\Phi)}.$
\end{itemize}

Let $q^{(0)}(x |s) = q(x |s).$ From Lemma~\ref{lem:kl-invert}, for all $i,j \in \{ 0,1,\dots, M_s\}$ such that $i \neq j$,
$$\KL (q^{(i)}(\cdot |s), q^{(j)}(\cdot |s)) = {\cal O}(  \varepsilon^2 / \Lambda(\Phi)^2).$$
Let $\mathbb{P}^{(0)}$ be the model defined by $q(\cdot|1),\dots q(\cdot|S)$ and $\mathbb{P}^{(v)}$ with $v\in [M_1]\times\dots \times [M_S]$ be the model defined by $q^{(v_1)}(\cdot|1),\dots q^{v_S}(\cdot|S)$. Let $N^\pi (s) $ be the expected number of visits to the $s$-th cluster under policy $\pi$, then
$$\KL( \mathbb{P}^{(v)},  \mathbb{P}^{(v')}) = \sum_{s=1}^S N^\pi (s) \KL ( q^{(v_s)}(\cdot |s), q^{(v'_s)}(\cdot |s)) = {\cal O}( TH\frac{\varepsilon^2}{\Lambda(\Phi)^2}).$$

When $\frac{ TH\frac{\varepsilon^2}{\Lambda(\Phi)^2}}{\sum_{s=1}^S \log M_s} = {\cal O}(1)$, 
from Theorem 2.5 in \cite{tsybakov09}, 
$$ \sum_{x\in f^{-1}(s)}|q(x |s) - \hat{q}(x |s)| \ge   \frac{\varepsilon}{\Lambda(\Phi)}\quad \mbox{for all}\quad s \quad \mbox{with probability}\quad\frac{1}{2}. $$

We now show that we can design a reward function $r$ such that $\frac{1}{H}(V^{\star}(r) - V^{\hat{\pi}_r}(r)) > \varepsilon$ when $\sum_{x\in f^{-1}(s)}|q(x |s) - \hat{q}(x |s)| \geq \frac{\varepsilon}{\Lambda(\Phi)}$ for all $s$. 

For this proof, let $V_h (x)$ be the value of $x$ from step $h$ and $\hat{V}_h(x)$ be the estimated value using $\hat{q}$. 
%\begin{align*}
%V_h(x) = & \EE_{p,q} [ \sum_{k=h}^H r_k(x_k^\pi,a_k^\pi) | x_h = x] \quad \mbox{and}\quad \hat{V}_h(x) =  \EE_{p,\hat{q}} [ \sum_{k=h}^H r_k(x_k^\pi,a_k^\pi) | x_h = x].
%\end{align*}

We first design $r_H$ as follows. Pick an action $a_H$ that will be optimal in all context $x$ for both model $(p,q)$ and $(p,\hat{q})$. Hence, $V_H (x) = \hat{V}_H(x) = r_H(x,a_H) $. Now we choose the reward function such that $r_H(x,a_H) = \frac{1+v^\star_s}{2}$ if $q(x|s)\geq \hat{q}(x|s)$ and $r_H(x,a_H) = \frac{1-v^\star_s}{2}$ if $q(x|s)< \hat{q}(x|s)$. Then, 
\begin{align*}
	&\frac{1}{S}\sum_{s = 1}^S \sum_{s' = 1}^S (p(s' | s,a^\star_{1,s})- p(s' | s,a^\star_{2,s}) )\left(\sum_{x\in f^{-1} (s')} (q(x|s')-\hat{q}(x|s') )V_H(x) \right) \cr
	\geq& \frac{1}{S}\sum_{s = 1}^S\sum_{s' = 1}^S (p(s' | s,a^\star_{1,s})- p(s' | s,a^\star_{2,s}) )\frac{\varepsilon}{\Lambda(\Phi)}\frac{v^\star_s}{4} = \frac{\varepsilon}{4}.
\end{align*}

We then set $r_{H-1}$ so that 
\begin{itemize}
	\item $r_{H-1}(x,a)=0$ for all $a \notin \{a^\star_{1,s},a^\star_{2,s} \} $ for all $x \in f^{-1}(s)$;
	\item $a^\star_{1,s}$ is the optimal action of $x \in f^{-1}(s)$ while $a^\star_{2,s}$ is the optimal action under $\hat{q}$ at time $H-1$, and so that
	\begin{align*}
		&r_{H-1}(x,a^\star_{1,s}) + \sum_{s'=1}^S p(s' | s,a^\star_{1,s})\sum_{x'\in f^{-1} (s')}q(x'|s')V_H(x') \cr
		&- \left( r_{H-1}(x,a^\star_{2,s}) + \sum_{s'=1}^S p(s' | s,a^\star_{2,s})\sum_{x'\in f^{-1} (s')}q(x'|s')V_H(x') \right)\cr
		=&\sum_{s' = 1}^S (p(s' | s,a^\star_{1,s})- p(s' | s,a^\star_{2,s}) )\frac{\varepsilon}{\Lambda(\Phi)}\frac{v^\star_s}{8}
	\end{align*}
\end{itemize}
Then, although 
\begin{align*}
	&\frac{1}{S}\sum_{s=1}^S \left( r_{H-1}(x,a^\star_{1,s}) + \sum_{s'=1}^S p(s' | s,a^\star_{1,s})\sum_{x'\in f^{-1} (s')}q(x'|s')V_H(x') \right)\cr
	&- \frac{1}{S}\sum_{s=1}^S \left( r_{H-1}(x,a^\star_{2,s}) + \sum_{s'=1}^S p(s' | s,a^\star_{2,s})\sum_{x'\in f^{-1} (s')}q(x'|s')V_H(x') \right)
	=\frac{\varepsilon}{8},
\end{align*}
the error on $\hat{q}$ makes the algorithm play $a^\star_{2,s}$ instead of $a^\star_{1,s}$ for all $s$, which induces $\Omega(\varepsilon)$  loss with respect to the optimal policy.

Analogously, at every $h$, we can design $r_h$ so that $a^\star_{1,s}$ is the optimal action of $x \in f^{-1}(s)$ while $a^\star_{2,s}$ is the optimal action under $\hat{q}$ at time $h$ and loose $\varepsilon$ value. Therefore, there exist $r_1,\dots, r_H$ such that
$$V^{\star}_h (r) - V^{\hat{\pi}_r}_h (r) = \Omega ((H-h)\varepsilon),\quad \mbox{for all } h.$$
Therefore,  
$\max_r \frac{1}{H}(V^{\star}(r) - V^{\hat{\pi}_r}(r)) \geq \varepsilon$, when $ \sum_{x\in f^{-1}(s)}|q(x |s) - \hat{q}(x |s)| \ge   \frac{\varepsilon}{\Lambda(\Phi)}\quad \mbox{for all } s$. \ep

\subsection{Proof of Theorem \ref{thm:lower-boundSC2} -- reward-specific setting}

We first establish the first term of the lower bound derived in Theorem \ref{thm:lower-boundSC2}. Consider a block MDP model such that every latent state $s$ has the same size $\frac{n}{S}$ and a unique optimal action $a_s$, defined as follows. 

{\bf Rewards.}
We consider a simple reward model such that every context state in $s\in \{1,\dots \lfloor S/2 \rfloor \}$ has $r_{h}(x,a) = 1$ for all $a$ and $h$  and every context state in $s\in \{\lfloor S/2 \rfloor +1 ,\dots , S\}$ has $r_{h}(x,a) = 0$ for all $a$ and $h$. Then, policies should visit $s\in \{1,\dots \lfloor S/2 \rfloor \}$ as many as possible to maximize the value. Note that from the reward information, we know whether $f(x)\in \{1,\dots \lfloor S/2 \rfloor \} $ or not. However, we do not have any prior knowledge from the reward about the exact membership among $\{1,\dots \lfloor S/2 \rfloor \} $ or among $\{\lfloor S/2 \rfloor +1, \dots S \}$.

{\bf Transitions.}
The transition $\sum_{v \in \{ 1,\dots \lfloor S/2 \rfloor \} } p(v|s,a_s) = 3/4$ and $\sum_{v \in \{ 1,\dots \lfloor S/2 \rfloor \} } p(v|s,a) = 1/4$ for $a\neq a_s$ and $q(x|s) = \frac{S}{n}$ for all $s$ and $x\in f^{-1}(s)$. One can easily check that $p$ and $q$ satisfy $I(x;\Phi)>0$. We design the transition model so that playing $a\neq a_s$ causes $\frac{\lfloor S/2 \rfloor}{2S}$ expected loss at the next time slot. 

{\bf Policy and clustering.}
From any given policy, we can simply design a clustering algorithm such that $x$ is classified to $s$ when $x$ selects $a_s$ with probability at least $1/2$ at a given time step $h$. From the clustering lower bound derived in Theorem \ref{thm:lower-bound}, to have less than $\varepsilon n$ misclassified context states, it is necessary to collect $TH= \Omega( n \log (1/\varepsilon))$ samples. Therefore, when $TH = {\cal O}( n \log (1/\varepsilon))$, every policy has to play the best policy $a_s$ with probability at most $1/2$ from at least $\varepsilon n$ context states, which makes 
$$\frac{1}{H}\left(V^* (r) - V^{\hat{\pi}} (r) \right) =\Omega (\varepsilon).$$   

\vspace{0.5cm}

To justify the second term of our lower bound derived in Theorem \ref{thm:lower-boundSC2}, we consider another Block MDP model. This model is similar to the previous model but with a slightly different transition kernel. The transitions are defined as follows: $q(x|s) = \frac{S}{n}$ for all $s$ and $x\in f^{-1}(s)$ and $\sum_{v \in \{ 1,\dots \lfloor S/2 \rfloor \} } p(v|s,a_s) = 1/2+ \varepsilon$ while $\sum_{v \in \{ 1,\dots \lfloor S/2 \rfloor \} } p(v|s,a) = 1/2$ for $a\neq a_s$. We also constraint $p$ so that $|p(v|s,a)-p(v|s,a')|\leq \frac{4\varepsilon}{S}$ for all $s$, $v$, and $a,a'$. Therefore, the model is designed so that failing to identify $a_s$ ends up with $\epsilon$ loss (since when a policy plays $a\neq a_s$, the policy loses $\varepsilon$ at the next time slot) and
$$\KL(p(\cdot |s,a),p(\cdot |s,a')) = {\cal O}(\varepsilon^2) \quad \mbox{for all}\quad a\neq a' ,$$
which comes from Lemma~\ref{lem:kl-invert}.

We now find the necessary condition to correctly identify $a_s$. From Theorem 2.5 in \cite{tsybakov09}, every policy fails to find $a_s$ with probability at least $1/2$ when $TH = {\cal O} (SA/\varepsilon^2)$. We thus have $\mathbb{E} \left[  \frac{1}{H}\left(V^{\star}(r) - V^{\hat{\pi}_r}(r) \right) \right] \geq \varepsilon$ with $TH = {\cal O} (SA/\varepsilon^2)$.\ep

% !TEX root = ./main.tex
\newpage
\section{Proofs of Theorems \ref{thm:rf1} and \ref{thm:rf2} -- Upper Bounds for Reward-Free RL}\label{app:reward-free-upper}

In this appendix, we provide proof of the reward-free guarantees of our algorithms. We start by introducing concepts and notations extensively used in our proofs. Then we provide the proofs of Theorem \ref{thm:rf1} and \ref{thm:rf2} which are fairly similar but with subtle differences. On a high level, the first step of these proofs is to use a value difference lemma to decouple the estimation error $\hat{p}$ and $\hat{q}$ from the clustering error due to the estimation of $\hat{f}$. In the second step, we use the specific concentration results for each setting to control the estimation error of $\hat{p}$ and $\hat{q}$. In the final step, we invoke our upper bound on the proportion of misclassified nodes, Theorem \ref{thm:likelihood-improvement}, to conclude.

\subsection{Preliminaries and notations}

\paragraph{Transitions and value functions.} Under the BMDP $\Phi = (p,q,f)$, we will denote the transition probabilities from the rich observations by $P$ where $P(y \vert x, a) = q(y, f(y))p(f(y)\vert f(x),a)$ for all $x,y \in \cX$, $a \in \cA$. Additionally, for a given reward $r$ and for all $h \in [H]$, we define the value function,  of a policy $\pi$, under BMDP $\Phi$, at step $h$, by
$$
V_h^\pi(x) = \EE_\Phi\left[ \sum_{k=h}^H r_k(x_k^\pi,a_k^\pi)\Big\vert x_h = x\right]
$$
where the dependence on $r$ is omitted for simplicity. Observe that we may simply write $V^\pi(r) = \EE_\mu[V_1^\pi(x_1)]$. Furthermore, we note that such value functions satisfy the following recursions:
$$
\forall h \in [H], \forall x \in \cX, \qquad V^\pi_h(x) = r_h(x, \pi_h(x)) + P(x_h, \pi_h(x)) V^\pi_{h+1}
$$
where we use the convention $V_{H+1}^\pi = 0$ and assume implicitly that the policy $\pi$ is deterministic for simplicity. Here, the notation $P(x_h, \pi_h(x)) V^\pi_{h+1}$ means $\sum_{y \in \cX} P(y \vert x_h, \pi_h(x)) V^\pi_{h+1}(y)$.
\paragraph{Empirical BMDP.} Through our estimation procedures we obtain the estimates $\hat{p}$, $\hat{q}$, and $\hat{f}$. With these, we will denote the empirical BMDP by $\hat{\Phi} = (\hat{p}, \hat{q}, \hat{f})$. The context transition probabilities  under $\hat{\Phi}$ will be denoted by $\hat{P}$. Additionally, for a given reward $r$, the value function under policy $\pi$, will be denoted by $\hat{V}^\pi(r)$, and at step $h$, by $\hat{V}_h^\pi$ for all $h \in [H]$.

\paragraph{True BMDP under inaccurate clustering.} We will also have to use the BMDP $\widetilde{\Phi} = (p,q, \hat{f})$ in our analysis. The context transition probabilities under $\widetilde{\Phi}$ will be denoted by $\widetilde{P}$. Additionally, for a given reward $r$, the value function under policy $\pi$, will be denoted by $\widetilde{V}^\pi(r)$, and at step $h$, by $\widetilde{V}_h^\pi$ for all $h \in [H]$.

\subsection{Proof of Theorem \ref{thm:rf1}}

Here we restate Theorem \ref{thm:rf1}.

% \begin{theorem}[Minimax setting]\label{thm:rf11}
%   Consider a BMDP $\Phi$ satisfying Assumptions \ref{assumption:SA}-\ref{assumption:uniform} and assume that $I(\Phi)>0$. Then, we have: for $TH = \omega(n)$, there exist universal constants $c, C >0$ such that the event
%     \begin{align*}
%        \frac{1}{H} \sup_{r} V^\star(r) - V^{\hat{\pi}_r}(r) \le \poly(\eta)    \sqrt{\frac{S^2A^2n\log(SA)}{TH}}
%     \end{align*}
%     holds with probability at least $1 - ce^{ -C\frac{TH}{n}}  \underset{n \to \infty}{\longrightarrow} 1$.
% \end{theorem}

\begin{manualtheorem}{6}[Minimax reward setting]
	Consider a BMDP $\Phi$ satisfying Assumptions \ref{assumption:SA}-\ref{assumption:uniform}. Further assume that $TH=\omega(n)$ and $I(\Phi)>0$. Then we have, w.h.p.,
$$
\sup_r {1\over H} (V^\star(r) - V^{\hat{\pi}_r}(r))= {\cal O}\left(\sqrt{\frac{nS^2A^2\log(SAH)}{TH}}\right).
$$
\end{manualtheorem}
%
%\begin{theorem}[Minimax reward setting]
%	Consider a BMDP $\Phi$ satisfying Assumptions \ref{assumption:SA}-\ref{assumption:uniform}. Further assume that $TH=\omega(n)$ and $I(\Phi)>0$. Then we have: w.h.p.
%	$$
%	\sup_r {1\over H} (V^\star(r) - V^{\hat{\pi}_r}(r))= {\cal O}\left(\sqrt{\frac{nS^2A^2\log(SAH)}{TH}}\right).
%	$$
%\end{theorem}
\begin{proof}[Proof of Theorem \ref{thm:rf1}] The proof is an immediate application of Theorem \ref{thm:likelihood-improvement} - (i) and Proposition \ref{prop:rf1}. Indeed, let us define the events
  \begin{align*}
    E_1 & = \left\lbrace \frac{\vert \cE\vert}{n} \le  \frac{1}{n} \sum_{x \in \cX} \exp\left(- C'I(x;\Phi) \frac{TH }{n}   \right) \right\rbrace, \\
    E_2 & = \left \lbrace \frac{1}{H} \sup_{r} V^\star(r) - V^{\hat{\pi}_r}(r) \le \poly(\eta) \left(   \sqrt{\frac{S^2A^2n\log(SA)}{TH}} + \frac{SA\vert \cE \vert}{n} \right) \right\rbrace.
  \end{align*}
  Under the event $E_1 \cap E_2$, we have
  \begin{align*}
    \frac{1}{H} \sup_{r} V^\star(r) - V^{\hat{\pi}_r}(r) & \le \poly(\eta) \left(  \sqrt{\frac{S^2A^2n\log(SA)}{TH}} + SA e^{-C' \frac{TH}{n} \min_{x \in \cX} I(x;\Phi)} \right) \\
    & \le \poly(\eta) \left(  \sqrt{\frac{S^2A^2n\log(SA)}{TH}} +  \frac{SAn}{C' \min_{x \in \cX} I(x;\Phi)TH} \right) \\
    & = \mathcal{O}\left( \sqrt{\frac{S^2A^2n\log(SA)}{TH}} \right)
  \end{align*}
  where we used the fact that $\exp(-C'\min_{x \in \cX}I(x;\Phi)\frac{TH}{n}) \le \frac{n}{C' \min_{x \in \cX}I(x;\Phi)TH}$. Therefore, by union bound,
  \begin{align*}
    \PP\left( \frac{1}{H} \sup_{r} V^\star(r) - V^{\hat{\pi}_r}(r) = \mathcal{O}\left( \sqrt{\frac{S^2A^2n\log(SA)}{TH}}\right) \right) & \ge 1 - \PP(E_1^c \cup E_2^c) \\
    & \ge 1 - \PP(E_1^c) - \PP(E_2^c).
  \end{align*}
  Next, in view of Theorem \ref{thm:likelihood-improvement} - (i) and Proposition \ref{prop:rf1}, we know that $\PP(E_1^c) \underset{n \to \infty }{\longrightarrow} 0$ and $\PP(E_2^c) \underset{n \to \infty }{\longrightarrow} 0$ for $TH = \omega(n)$. This concludes the proof.
\end{proof}

\begin{proposition}[Minimax setting]\label{prop:rf1}
  Under Assumptions \ref{assumption:SA}-\ref{assumption:uniform}, for all reward functions $r$,  provided $TH = \Omega(n)$, we have:
  \begin{align*}
    \PP\left( \frac{1}{H} \sup_{r} V^\star(r) - V^{\hat{\pi}_r}(r) \le \poly(\eta) \left(   \sqrt{\frac{S^2A^2n\log(SA)}{TH}} + \frac{SA\vert \cE \vert}{n} \right) \right) \ge 1 - \frac{14}{n} - e^{-\frac{TH}{n}}.
  \end{align*}
\end{proposition}

\begin{proof}[Proof of Proposition \ref{prop:rf1}] We start by applying Lemma \ref{lem:val_diff_I}, which ensures the following decomposition
\begin{align*}
   \frac{1}{H} \sup_{r}  V^\star(r) - V^{\hat{\pi}_r}(r)  \le 2\eta^2 \left(\max_{s \in \cS, a \in \cA} \Vert \hat{p}(s,a)  - p(s,a) \Vert_1  + \max_{s \in \cS}   \Vert \hat{q}(s) - q(s)\Vert_1 \right) +  \frac{4\eta^4\vert \cE \vert}{n}.
\end{align*}
Then, we have from Proposition \ref{prop:ee_p_q} that for all $ TH = \Omega(n)$,
\begin{align*}
  \PP\left(\max_{s\in \cS,a \in \cA} \Vert \hat{p}(s,a) - p(s,a)\Vert_1 \le \poly(\eta) SA \left(\sqrt{\frac{S + \log(nSA)}{TH}} + \frac{\vert \cE \vert}{n} \right) \right) \ge 1 - \frac{4}{n} - e^{-\frac{TH}{n}}
\end{align*}
and
\begin{align*}
    \PP\left(\max_{s \in \cS}\Vert \hat{q}(s) - q(s)\Vert_1 \le \poly(\eta) S \left( \sqrt{\frac{n}{TH}} +  \frac{\vert \cE \vert}{n} \right) \right) \ge 1 - \frac{4}{n} - e^{-\frac{TH}{n}}.
\end{align*}
Finally, combining the above three inequalities yields the result.
\end{proof}

\subsection{Proof of Theorem \ref{thm:rf2}}

Here we restate Theorem \ref{thm:rf2} in a more precise way.
% \begin{theorem}[Reward-specific setting]\label{thm:rf12}
%   Consider a BMDP $\Phi$ satisfying Assumptions \ref{assumption:SA}-\ref{assumption:uniform} and assume that $I(\Phi)>0$. Then, we have:
%   \begin{itemize}
%     \item[(i)] for $TH = \omega(n)$ and $TH = {\cal O}(n\log(n))$, there exists universal constants $c, C, C' >0$ such that the event
%     \begin{align*}
%        \frac{1}{H} \left(V^\star(r) - V^{\hat{\pi}_r}(r) \right) \le \poly(\eta) \left( \sqrt{\frac{S^3AH\log(SAHn)}{T}} + SH^2 e^{ - C'\frac{TH}{n} I(\Phi)} \right)
%     \end{align*}
%     holds with probability at least $1 - ce^{ -C\frac{TH}{n}}  \underset{n \to \infty}{\longrightarrow} 1$.
%     \item[(ii)] For $TH - \frac{n\log(n)}{C' I(\Phi)} = \omega(1)$, there exists some universal constant $c > 0$ such that the event
%     \begin{align*}
%     \frac{1}{H} \left(V^\star(r) - V^{\hat{\pi}_r}(r) \right) \le \poly(\eta)  \sqrt{\frac{S^3AH\log(SAHn)}{T}}
%     \end{align*}
%     holds with probability at least  $1 -\frac{c}{n} \underset{n \to \infty}{\longrightarrow} 1$.
%   \end{itemize}
% \end{theorem}

% \begin{proof}[Proof of Theorem \ref{thm:rf12}] Theorem \ref{thm:rf12} is an immediate application of Theorem \ref{thm:likelihood-improvement} - (i) and Proposition \ref{prop:rf2}. Note that when $TH- \frac{n\log(n)}{C' I(\phi)} = \omega(1)$, then $\vert \cE \vert/ n = 0$ with high probability which entails the result in \emph{(ii)}.
% \end{proof}

\begin{manualtheorem}{7}[Reward-specific setting]
	Let $C$ be the constant introduced in Theorem \ref{thm:likelihood-improvement}(i). Under the assumptions of Theorem \ref{thm:rf1}, we have for any reward function $r$, w.h.p.
\begin{align*}
	\frac{1}{H}(V^\star(r) - &  V^{\hat{\pi}_r}(r)) = {\cal O}\left( \sqrt{\frac{S^3A^2H\log(SAHn)}{T}} \right. \\
	&\ \ \ + \left. \frac{SH^2}{n} \sum_{x \in \cX} \exp\left( - C \frac{TH}{n} I(x; \Phi) \right) \right).
\end{align*}
\end{manualtheorem}
%
%\begin{theorem}[Reward-specific setting]\label{thm:rf12}
%  Consider a BMDP $\Phi$ satisfying Assumptions \ref{assumption:SA}-\ref{assumption:uniform} and assume that $I(\Phi)>0$. Then, there exists a universal constant $C'$ such that:
%  \begin{itemize}
%    \item[(i)] for $TH = \omega(n)$ and $TH = {\cal O}(n\log(n))$, the following event
%    \begin{align*}
%       \frac{1}{H} \left(V^\star(r) - V^{\hat{\pi}_r}(r) \right) = \mathcal{O} \left( \sqrt{\frac{S^3A^2H\log(SAHn)}{T}} + \frac{SAH^2 }{n}\sum_{x \in \cX} e^{ - C'\frac{TH}{n} I(x;\Phi) }    \right)
%    \end{align*}
%    holds with high probability.
%    \item[(ii)] For $TH - \frac{n\log(n)}{C' I(x;\Phi)} = \omega(1)$ for $x \in \cX$, the event
%    \begin{align*}
%    \frac{1}{H} \left(V^\star(r) - V^{\hat{\pi}_r}(r) \right) = \mathcal{O}\left(\sqrt{\frac{S^3A^2H\log(SAHn)}{T}} \right)
%    \end{align*}
%    holds with high probability.
%  \end{itemize}
%\end{theorem}
\begin{proof}[Proof of Theorem \ref{thm:rf2}]  The proof is an immediate application of Theorem \ref{thm:likelihood-improvement} - (i) and Proposition \ref{prop:rf2}.
Indeed, let us define the events
  \begin{align*}
    E_1 & = \left\lbrace \frac{\vert \cE\vert}{n} \le  \frac{1}{n} \sum_{x \in \cX} \exp\left(- C' \frac{TH}{n} I(x; \Phi)  \right) \right\rbrace, \\
    E_2 & = \left \lbrace \frac{1}{H}  V^\star(r) - V^{\hat{\pi}_r}(r) \le \poly(\eta) \left(   \sqrt{\frac{S^3AH^4\log(SAHn)}{TH}} + \frac{SH^2\vert \cE \vert}{n} \right) \right\rbrace.
  \end{align*}

  \textbf{Proving \emph{(i)}.} We note that under the event $E_1 \cap E_2$, we have
  \begin{align*}
    \frac{1}{H} V^\star(r) - V^{\hat{\pi}_r}(r) & \le \poly(\eta) \left(  \sqrt{\frac{S^3A^2\log(SAHn)}{TH}} + \frac{SAH^2 \sum_{x \in \cX} \exp\left( - C'\frac{TH}{n} I(x;\Phi) \right) }{n}   \right).
  \end{align*}
  Next, in view of Theorem \ref{thm:likelihood-improvement} - (i) and Proposition \ref{prop:rf2}, we know that $\PP(E_1^c) \underset{n \to \infty }{\longrightarrow} 0$ and $\PP(E_2^c) \underset{n \to \infty }{\longrightarrow} 0$ for $TH = \omega(n)$. This completes the proof of \emph{(i)}.

  \textbf{Proving \emph{(ii)}.} We start by noting that when  $TH- \frac{n\log(n)}{C' I(x;\Phi)} = \omega(1)$ for all $x\in \cX$, then we have $\vert \cE \vert < 1$ which simply implies that $\vert \cE \vert = 0$ (i.e., we recover the clusters exactly). Thus, under $E_1 \cap E_2$, we have in this case,
  \begin{align*}
    \frac{1}{H} V^\star(r) - V^{\hat{\pi}_r}(r) & = \mathcal{O} \left(  \sqrt{\frac{S^3A^2\log(SAHn)}{TH}} \right).
  \end{align*}
  And again, in view of Theorem \ref{thm:likelihood-improvement} - (i) and Proposition \ref{prop:rf2}, we know that $\PP(E_1^c) \underset{n \to \infty }{\longrightarrow} 0$ and $\PP(E_2^c) \underset{n \to \infty }{\longrightarrow} 0$ when $TH- \frac{n\log(n)}{C' I(x; \Phi)} = \omega(1)$ for all $x\in \cX$. This concludes the proof of \emph{(ii)}.
\end{proof}

\begin{proposition}[Reward-specific setting]\label{prop:rf2}
  Under Assumptions \ref{assumption:SA}-\ref{assumption:uniform}, we have for all reward functions $r$,  provided $TH = \Omega(n)$,
  \begin{align*}
    \PP\left( \frac{1}{H}\left(V^\star(r) - V^{\hat{\pi}_r}(r) \right) \le \poly(\eta) \left(   \sqrt{\frac{S^3A^2H^4\log(SAHn)}{TH}} + \frac{SH^2\vert \cE \vert}{n} \right) \right) \ge 1 - \frac{14}{n} - e^{-\frac{TH}{n}}.
  \end{align*}
\end{proposition}

\begin{proof}[Proof of Proposition \ref{prop:rf2}] We start by applying Lemma \ref{lem:val_diff_II}, which ensures the following decomposition
\begin{align*}
    \frac{1}{H} \left(V^\star(r) - V^{\hat{\pi}_r}(r) \right)  & \le \max_{s \in \cS, a \in \cA, h \in[H]} \Big\vert   \big(\hat{q}(s) - q(s) \big)  \hat{V}_{h+1}^{\hat{\pi}_r} \Big\vert   +  2 H  \max_{s \in \cS, a \in \cA } \Vert \hat{p}(s,a)  - p(s,a) \Vert_1     \\
     & \qquad   + \max_{s \in \cS, a \in \cA, h \in[H]}   \Big\vert   \big(\hat{q}(s) - q(s) \big)  \widetilde{V}^{\pi^\star}_{h+1} \Big\vert  + \frac{6\eta^2H \vert \cE \vert}{n}.
\end{align*}
From Lemma \ref{lem:concentration_avg_val}, we have the following concentration bounds that hold as long as $TH = \Omega(n)$,
\begin{align*}
    \PP\left(  \max_{s \in \cS, h \in[H] }\Big\vert (\hat{q}(s) - q(s)) \hat{V}^{\hat{\pi}_r}_{h} \Big \vert > \poly(\eta) SH\left( \sqrt{\frac{SA\log(SAHn)}{TH}} + \frac{\vert \cE \vert}{n} \right) \right) \le \frac{4}{n}
  \end{align*}
and
\begin{align*}
      \PP\left(  \max_{s \in \cS, h \in[H] }\Big\vert (\hat{q}(s) - q(s)  ) \widetilde{V}^{\pi^\star}_{h} \Big \vert > \poly(\eta) SH\left(\sqrt{\frac{SA\log(SAHn)}{TH}} + \frac{\vert \cE \vert}{n} \right) \right) \le \frac{4}{n}.
\end{align*}
Additionally from Proposition \ref{prop:ee_p_q} we also have for $TH = \Omega(\log(n))$
\begin{align*}
  \PP\left( \Vert \hat{p}(s,a) - p(s,a) \Vert_1 > \poly(\eta) SA\left( \sqrt{ \frac{S + \log(nSA)}{TH}}  +  \frac{\vert \cE\vert}{n}\right) \right) >  \frac{6}{n} + e^{-\frac{TH}{n}}.
\end{align*}
The final result follows from combining the above four inequalities.
\end{proof}

\subsection{Value difference lemmas}

A crucial step towards obtaining our reward-free guarantees is to establish value difference lemmas. These lemmas must account for the clustering error and whether we are in the minimax setting or reward-specific setting. We state and prove such lemmas (Lemma \ref{lem:val_diff_I} and Lemma \ref{lem:val_diff_II}).

\begin{lemma}[First Value Difference Lemma]\label{lem:val_diff_I}
  For any reward function $r$, let $\hat{\pi}_r$ be the optimal policy for the empirical model $\hat{\Phi}$. Under Assumptions \ref{assumption:SA}-\ref{assumption:uniform} for the true BMDP $\Phi$, we have:
  \begin{align*}
    \hfill \frac{1}{H} \sup_{r}  V^\star(r) - V^{\hat{\pi}_r}(r)  \le 2\eta^2 \left(\max_{s \in \cS, a \in \cA} \Vert \hat{p}(s,a)  - p(s,a) \Vert_1  + \max_{s \in \cS}   \Vert \hat{q}(s) - q(s)\Vert_1 \right) +  \frac{4\eta^4\vert \cE \vert}{n}.
  \end{align*}
\end{lemma}
\begin{lemma}[Second Value Difference Lemma]\label{lem:val_diff_II}
  Let $r$ be some reward function and let $\hat{\pi}_r$ be the optimal policy for the empirical model $\hat{\Phi}$. Under Assumptions \ref{assumption:SA}-\ref{assumption:uniform} for the true BMDP $\Phi$, we have:
  \begin{align*}
    \frac{1}{H} \left(V^\star(r) - V^{\hat{\pi}_r}(r) \right)  & \le \max_{s \in \cS, a \in \cA, h \in[H]} \Big\vert   \big(\hat{q}(s) - q(s) \big)  \hat{V}_{h+1}^{\hat{\pi}_r} \Big\vert   +  2 H  \max_{s \in \cS, a \in \cA } \Vert \hat{p}(s,a)  - p(s,a) \Vert_1     \\
     & \qquad   + \max_{s \in \cS, a \in \cA, h \in[H]}   \Big\vert   \big(\hat{q}(s) - q(s) \big)  \widetilde{V}^{\pi^\star}_{h+1} \Big\vert  + \frac{6\eta^2H \vert \cE \vert}{n}.
  \end{align*}
\end{lemma}

Lemma \ref{lem:val_diff_I} (resp. Lemma \ref{lem:val_diff_II}) is used to establish the reward-free guarantee in the \emph{minimax} (resp. \emph{reward-specific} setting. We note that Lemma \ref{lem:val_diff_II} has a worse dependence in $H$, than that of Lemma \ref{lem:val_diff_I}. This is due to the fact that $\hat{\Phi}$ does not necessarily satisfy Assumptions \ref{assumption:SA}-\ref{assumption:uniform} which are essential to get such improvement in $H$. Therefore, it may appear strange to use Lemma \ref{lem:val_diff_II}. It is however useful in order to obtain an improvement of order $n$ in the \emph{reward-specific} setting in contrast with the \emph{minimax} one, as will be apparent in our concentration bounds. The proofs of these lemmas rely on Lemma \ref{lem:val_diff_2}, Lemma \ref{lem:prob_err_decomposition}, and Lemma \ref{lem:max_value}.

\begin{lemma}\label{lem:val_diff_2}
   Under Assumptions \ref{assumption:SA}-\ref{assumption:uniform} for the true BMDP $\Phi$, for all rewards $r$, and any policy $\pi$, it holds that
  \begin{align*}
    \vert V^\pi(r) - \hat{V}^\pi(r) \vert \le H \max_{x \in \cX,a \in \cA, h \in[H]} \Big\vert \big (\widetilde{P}(x,a) - \hat{P}(x, a) \big) \hat{V}_{h+1}^\pi \Big \vert  +    \frac{\eta^2 H \vert \cE\vert}{n}.
  \end{align*}
  Alternatively, we also have
  \begin{align*}
    \vert V^\pi(r) - \hat{V}^\pi(r) \vert \le H \max_{x \in \cX,a \in \cA, h \in[H]} \Big\vert \big (\widetilde{P}(x,a) - \hat{P}(x, a) \big) \widetilde{V}_{h+1}^\pi \Big \vert  +    \frac{\eta^2 H \vert \cE\vert}{n}.
  \end{align*}
\end{lemma}

\begin{lemma}\label{lem:prob_err_decomposition}
  Under Assumption 1-2, for all $x \in \cX$, $a \in \cA$, we have
  \begin{align*}
    \big\Vert\widetilde{P}(x,a)  - \hat{P}(x,a)\big\Vert_1 & \le \frac{2  \eta^2\vert \cE \vert }{n} + \max_{s \in \cS} \Vert \hat{p}(s,a)  - p(s,a) \Vert_1  + \max_{s \in \cS}   \Vert \hat{q}(s) - q(s)\Vert_1.
  \end{align*}
  More precisely, for any $V \in \RR^n$, we have
  \begin{align*}
    \Big\vert \big(\widetilde{P}(x,a)  - \hat{P}(x,a)\big) V \Big\vert & \le \left(\frac{2 \eta^2 \vert \cE \vert }{n} + \max_{s \in \cS } \Vert \hat{p}(s,a)  - p(s,a) \Vert_1 \!\right) \Vert V  \Vert_\infty  + \max_{s \in \cS}  \Big\vert   \big(\hat{q}(s) - q(s) \big)  V \Big \vert.
  \end{align*}
\end{lemma}

In Lemma \ref{lem:max_value}, we establish that the centered value function under the regularity assumption satisfies an upper bound that is horizon free. This has to do with the fact that all the transition matrices have mixing times that are uniformly bounded by $\eta^2$. In fact, this result is easily generalizable to MDPs with finite state-action spaces that are communicating and aperiodic.

\begin{lemma}\label{lem:max_value}
Under Assumptions  \ref{assumption:SA}-\ref{assumption:uniform} for the true BMDP $\Phi$. Then, for all $h \in [H]$, the value function of a policy $\pi$ at step $h$, $V_h^\pi$ satisfies, for all $x \in \cX$, $a \in \cA$
\begin{align*}
  \Big\Vert V_{h+1}^\pi - P(x, a) V_{h+1}^\pi \mathbf{1}\Big\Vert_\infty \le (2\eta^2 - 1).
\end{align*}
\end{lemma}

\subsection{Concentration bounds on averaged optimal value functions}

The following lemma is a key technical result to establish a guarantee for the \emph{reward-free} setting which does not suffer a linear dependence in $n$.

\begin{lemma}\label{lem:concentration_avg_val}
  Assuming that $\hat{f}$ is estimated using the first $\lfloor TH/2 \rfloor$ observations, and that $\hat{p}, \hat{q}$ are estimated using the observations $\lfloor TH/2 \rfloor$ with $\hat{f}$. Then, under Assumptions \ref{assumption:SA}-\ref{assumption:uniform}, provided $TH =  \Omega(n)$, we have
  \begin{align*}
      \PP\left( \max_{s \in \cS, h \in [H]}\Big\vert ( \hat{q}(s) - q(s) ) \hat{V}^{\hat{\pi}_r}_{h} \Big \vert >  \poly(\eta) S H  \left(\sqrt{\frac{SA \log(HSn)}{TH}} + \frac{\vert \cE \vert}{n} \right) \right) \le \frac{4}{n}
  \end{align*}
  and
  \begin{align*}
        \PP\left(  \max_{s \in \cS, h \in[H] }\Big\vert (\hat{q}(s) - q(s)  ) \widetilde{V}^{\pi^\star}_{h} \Big \vert >  \poly(\eta) S H  \left(\sqrt{\frac{SA \log(HSn)}{TH}} + \frac{\vert \cE \vert}{n} \right) \right) \le \frac{4}{n}.
  \end{align*}
\end{lemma}
The challenge in obtaining the first concentration bound comes from the fact that $\hat{q}(s)$ and $\hat{V}^{\hat{\pi}_r}$ are dependent on each other. Fortunately, $\hat{V}^{\hat{\pi}_r}$ has a special structure that can be characterized; see Lemma \ref{lem:BMDP-LMDP}. Lemma \ref{lem:BMDP-LMDP} says that actually, the optimal value function under the BMDP model possesses a linear structure. This observation stems from the remark that the transition matrices in the BMDP are low rank (see \cite{zhang2022BMDP}). This low-rank structure can then be leveraged to ensure that, actually the value function of any policy satisfies the representation stated in the Lemma \ref{lem:BMDP-LMDP} (See for example,  Proposition 2.3. in \cite{jin2020reward} or Lemma 1 in \cite{Modi21}).

\begin{lemma}\label{lem:BMDP-LMDP} Let $\Phi$ be a BMDP, then, for all $h \in [H]$, there exists $\theta_h \in \RR^{SA}$, such that $\Vert \theta_h \Vert_2 \le \sqrt{SA}(H-h)$ such that the value function can be expressed as follows:
  \begin{align*}
    \forall x \in \cX, \qquad V_h^\star(x) = \max_{a \in \cA} \lbrace r_h(x,a) + \psi_f(x,a)^\top \theta_h \rbrace
  \end{align*}
  where for all $x\in \cX, a\in \cA$, $\psi_f(x,a)$ is an $SA$-dimensional column vector in $\lbrace 0, 1\rbrace^{SA}$ such that $\psi_f(x,a)(s,b) = \indicator\{(f(x), a) = (s,b)\}$ ($f$ is the latent state decoding function of $\Phi$).
\end{lemma}

Now using the representation of the optimal value function stated in Lemma \ref{lem:BMDP-LMDP}, we consider, for a given decoding function $f$, the set of all possible optimal value functions at step $h$ as we vary the latent transitions $p$ and emission probabilities $q$. More precisely, we define a set that contains all such functions as follows:
\begin{align*}
  \cV_h^\star(f; r) :=  \Big\lbrace V:  \exists \theta_h: \Vert\theta_h\Vert_2 \le M,
\forall x\in \cX, V(x) = \min\lbrace\max_{a \in \cA} \{ r_h(x,a) + \psi_f(x,a)^\top \theta_h \}, H\rbrace  \Big\rbrace
\end{align*}
with $M = \sqrt{SA}(H-h)$. Lemma \ref{lem:net-val} shows that we can construct a $\epsilon$-net of $\cV_h^\star(f;r)$ with a cardinality that grows exponentially only in $SA$ and not $n$ (see Lemma D.6 in \cite{jin2020reward}). This property is crucial to obtain an error rate that is independent of $n$. However, it is also important to note that such a net will still depend on the given reward functions $r$ and the clustering function $f$. It is for this reason that we won't be able to use such a net argument in the \emph{minimax} setting, and also why we split our budget of episodes into two parts in the design of our algorithm.

\begin{lemma}\label{lem:net-val} Let $f$ be any latent state decoding function and $r$ be any reward function, then there exists $\epsilon$-net $\cN_{\epsilon}(f;r)$ of $\cV_h^\star(f;r)$ with respect to $\Vert \cdot \Vert_\infty$, such that
  \begin{align*}
    \vert \cN_\epsilon(f;r)\vert \le  \left( 1 + \frac{2\sqrt{SA}(H-h)}{\epsilon} \right)^{SA}.
  \end{align*}
\end{lemma}

\subsection{Proofs -- Value difference lemmas}

\begin{proof}[Proof of Lemma \ref{lem:val_diff_I} and Lemma \ref{lem:val_diff_II}]
  The starting point for the proof of Lemma \ref{lem:val_diff_I} and Lemma \ref{lem:val_diff_II} is the same. We recall that for any given reward $r$, the policy $\hat{\pi}_r$ is optimal under $\hat{\Phi}$, thus, and in particular, $\hat{V}^{\pi_r^\star}(r) - \hat{V}^{\hat{\pi}_r}(r)  < 0$. From this observation, we immediately obtain
  \begin{align}\label{eq:val_diff_eq1}
    V^\star(r) - V^{\hat{\pi}_r}(r) & =  V^\star(r) - \hat{V}^{\pi_r^\star}(r) + \hat{V}^{\pi_r^\star}(r) - \hat{V}^{\hat{\pi}_r}(r) +  \hat{V}^{\hat{\pi}_r}(r)  -  V^{\hat{\pi}_r}(r) \nonumber\\
    & \le   V^\star(r) - \hat{V}^{\pi_r^\star}(r) +  \hat{V}^{\hat{\pi}_r}(r)  -  V^{\hat{\pi}_r}(r).
  \end{align}
  Now, let us introduce an intermediary model $\tilde{\Phi} = (p, q, \hat{f})$, and denote $\widetilde{V}^\pi$ the value function of policy $\pi$ under the model $\tilde{\Phi}$.

  \textbf{\emph{(Proof of Lemma \ref{lem:val_diff_I}).}} We start from \eqref{eq:val_diff_eq1} to write
  \begin{align*}
    V^\star(r) - V^{\hat{\pi}_r}(r) & \le 2 \max_{\pi} \left \vert V^\pi(r) - \hat{V}^{\pi}(r) \right\vert.
  \end{align*}
Then, an immediate application of Lemma \ref{lem:val_diff_2} gives us
  \begin{align*}
       V^\star(r) - V^{\hat{\pi}_r}(r)  & \le 2 H \max_{\pi}  \max_{x \in \cX,a \in \cA, h \in[H]} \Big\vert \big (\widetilde{P}(x,a) - \hat{P}(x, a) \big) \widetilde{V}_{h+1}^\pi \Big \vert  +  \frac{2 \eta^2 H \vert \cE\vert}{n}.
  \end{align*}
  We then observe that
  \begin{align*}
     \Big\vert \big (\widetilde{P}(x,a) - \hat{P}(x, a)  \big) \widetilde{V}_{h+1}^\pi & \Big \vert  \overset{(a)}{=} \Big\vert \big (\widetilde{P}(x,a) - \hat{P}(x, a)  \big) \left(\widetilde{V}_{h+1}^\pi -  \widetilde{P}(x,a)\widetilde{V}_{h+1}^\pi \mathbf{1}  \right)\Big \vert  \\
     & \overset{(b)}{\le} \Vert \widetilde{P}(x,a) - \hat{P}(x, a) \Vert_1 \Vert \widetilde{V}_{h+1}^\pi -  \widetilde{P}(x,a)\widetilde{V}_{h+1}^\pi \mathbf{1}  \Vert_\infty \\
     & \overset{(c)}{\le}  2\eta^2  \Vert \widetilde{P}(x,a) - \hat{P}(x, a) \Vert_1 \\
     & \overset{(d)}{\le}  2\eta^2  \left(\frac{2  \eta^2\vert \cE \vert }{n} + \max_{s \in \cS} \Vert \hat{p}(s,a)  - p(s,a) \Vert_1  + \max_{s \in \cS}   \Vert \hat{q}(s) - q(s)\Vert_1 \right)
  \end{align*}
  where $(a)$ follows by simply adding and subtracting a constant, $(b)$ follows from Holder's inequality, $(c)$ follows from applying Lemma \ref{lem:max_value} (we recall that $\mathbf{1} \in \RR^n$ and has all its entries equal to $1$), and $(d)$ follows from applying Lemma \ref{lem:prob_err_decomposition}. Therefore, we finally obtain that
  \begin{align*}
     V^\star(r) - V^{\hat{\pi}_r}(r)  \le 2\eta^2 H \left(\max_{s \in \cS} \Vert \hat{p}(s,a)  - p(s,a) \Vert_1  + \max_{s \in \cS}   \Vert \hat{q}(s) - q(s)\Vert_1 \right) +  \frac{4H\eta^4\vert \cE \vert}{n},
  \end{align*}
which leads to the desired solution by taking the supremum over all possible values of rewards $r$.

\textbf{\emph{(Proof of Lemma \ref{lem:val_diff_II}).}} Again, we start from \eqref{eq:val_diff_eq1} to write
  \begin{align*}
    V^\star(r) - V^{\hat{\pi}_r}(r)  \le   \underbrace{\left\vert V^{\hat{\pi}_r}(r) - \hat{V}^{\hat{\pi}_r}(r) \right\vert}_{\text{Term 1}} +  \underbrace{\left \vert V^{\pi^\star}(r) - \hat{V}^{\pi^\star}(r) \right\vert}_{\text{Term 2}}.
  \end{align*}
  In contrast to the proof of Lemma \ref{lem:val_diff_I}, we will have to analyze each of the above terms separately to take advantage of the fact that $\hat{\pi}_r$ is optimal under $\hat{\Phi}$, thus exploiting the fact that $\hat{V}^{\hat{\pi}_r}$ has a special structure.

  \textbf{\emph{(Bounding term 1).}} First, an immediate application of Lemma \ref{lem:val_diff_2} gives
  \begin{align*}
       \left\vert V^{\hat{\pi}_r}(r) - \hat{V}^{\hat{\pi}_r}(r) \right \vert \le H   \max_{x \in \cX,a \in \cA, h \in[H]}  \Big\vert \big (\widetilde{P}(x,a) - \hat{P}(x, a) \big) \hat{V}_{h+1}^{\hat{\pi}_r} \Big \vert  +  \frac{\eta^2 H \vert \cE\vert}{n}
  \end{align*}
  where we took the supremum over all possible optimal value functions at all steps $h \in [H]$. Then, applying Lemma \ref{lem:prob_err_decomposition}, we also have for all $h \in [H]$, $V \in  \cV_{h+1}^\star(\hat{f}, r)$,
  \begin{align*}
     \Big\vert \big (\widetilde{P}(x,a) - \hat{P}(x, a) \big) \hat{V}_{h+1}^{\hat{\pi}_r} \Big \vert \le    \max_{s \in \cS}  \Big\vert   \big(\hat{q}(s) - q(s) \big)  \hat{V}_{h+1}^{\hat{\pi}_r} \Big\vert + H  \max_{s \in \cS } \Vert \hat{p}(s,a)  - p(s,a) \Vert_1 + \frac{2 \eta^2 H \vert \cE \vert }{n}
  \end{align*}
  where we also used the fact that $\Vert V \Vert_\infty \le H$. Therefore, combining the above bounds gives
  \begin{align*}
   \frac{1}{H}  \left\vert V^{\hat{\pi}_r}(r) - \hat{V}^{\hat{\pi}_r}(r) \right \vert & \le \max_{s \in \cS, a \in \cA, h \in[H]}  \Big\vert   \big(\hat{q}(s) - q(s) \big) \hat{V}_{h+1}^{\hat{\pi}_r} \Big\vert \\
   & \qquad  +  H  \max_{s \in \cS, a \in \cA } \Vert \hat{p}(s,a)  - p(s,a) \Vert_1   + \frac{3\eta^2H \vert \cE \vert}{n}.
  \end{align*}

\textbf{\emph{(Bounding term 2).}} Next, again, by applying Lemma \ref{lem:val_diff_2}, we obtain
  \begin{align*}
       \left\vert V^{\pi^\star}(r) - \hat{V}^{\pi^\star}(r) \right \vert \le H  \max_{x \in \cX,a \in \cA, h \in[H]} \Big\vert \big (\widetilde{P}(x,a) - \hat{P}(x, a) \big) \widetilde{V}_{h+1}^{\pi^\star} \Big \vert  +  \frac{\eta^2 H \vert \cE\vert}{n}.
  \end{align*}
  Then, applying Lemma \ref{lem:prob_err_decomposition}, we also have for all $h \in [H]$,
  \begin{align*}
     \Big\vert \big (\widetilde{P}(x,a) - \hat{P}(x, a) \big) \widetilde{V}^{\pi_\star}_{h+1} \Big \vert \le    \max_{s \in \cS}  \Big\vert   \big(\hat{q}(s) - q(s) \big)   \widetilde{V}^{\pi_\star}_{h+1}\Big\vert + H  \max_{s \in \cS } \Vert \hat{p}(s,a)  - p(s,a) \Vert_1 + \frac{2 \eta^2 H \vert \cE \vert }{n}
  \end{align*}
  where we also used the fact that $\Vert V \Vert_\infty \le H$. Therefore, we obtain
  \begin{align*}
       \frac{1}{H}\left\vert V^{\pi^\star}(r) - \hat{V}^{\pi^\star}(r) \right \vert &   \le \max_{s \in \cS, a \in \cA, h \in[H]}   \Big\vert   \big(\hat{q}(s) - q(s) \big)  \hat{V}^{\pi^\star}_{h+1} \Big\vert  \\
       & \qquad +  H  \max_{s \in \cS, a \in \cA } \Vert \hat{p}(s,a)  - p(s,a) \Vert_1   + \frac{3\eta^2H \vert \cE \vert}{n}.
  \end{align*}
  Now combining the bounds on terms 1 and 2, obtain
  \begin{align*}
    \frac{1}{H} \left(V^\star(r) - V^{\hat{\pi}_r}(r) \right)  & \le \max_{s \in \cS, a \in \cA, h \in[H]}  \Big\vert   \big(\hat{q}(s) - q(s) \big) \hat{V}_{h+1}^{\hat{\pi}_r} \Big\vert  + \max_{s \in \cS, a \in \cA, h \in[H]}   \Big\vert   \big(\hat{q}(s) - q(s) \big)  \widetilde{V}^{\pi^\star}_{h+1} \Big\vert  \\
     & \qquad +  2 H  \max_{s \in \cS, a \in \cA } \Vert \hat{p}(s,a)  - p(s,a) \Vert_1   + \frac{6\eta^2H \vert \cE \vert}{n}.
  \end{align*}
\end{proof}

\begin{proof}[Proof of Lemma \ref{lem:val_diff_2}]
 We start by writing
\begin{align*}
  \vert V^\pi(r) - \hat{V}^\pi(r) \vert \le \vert V^\pi(r) - \widetilde{V}^\pi(r)\vert + \vert \widetilde{V}^\pi(r) - \hat{V}^\pi(r) \vert.
\end{align*}
Now, Let us define for all $h \in [H]$, $x \in \cX$,
\begin{align*}
  E_h^\pi(x) & = \vert V_h^\pi(x) - \widetilde{V}_h^\pi(x)\vert \\
  \widetilde{E}_h^\pi(x) & = \vert \widetilde{V}_h^\pi(x) - \hat{V}_h^\pi(x)\vert
\end{align*}
and note that
$$
\vert V^\pi(r) - \hat{V}^\pi(r)\vert \le \EE_{x_1}[E_1^\pi(x_1)] + \EE_{x_1}[\widetilde{E}_1^\pi(x_1)].
$$

\textbf{\emph{Step 1 -- (Bounding $E_1^\pi(x)$)}} First, we have for all $h \in [H]$, $x \in \cX$,
\begin{align*}
  E_h^\pi(x) & = \Big\vert P(x, \pi_h(x)) V_{h+1}^\pi - \widetilde{P}(x, \pi_h(x)) \widetilde{V}_{h+1}^\pi \Big\vert  \\
  & \le \Big \vert \big( P(x, \pi_h(x)) - \widetilde{P}(x, \pi_h(x)) \big) V_{h+1}^\pi  \Big \vert + \widetilde{P}(x, \pi_h(x)) \Big\vert V_{h+1}^\pi - \widetilde{V}_{h+1}^{\pi}\Big\vert \\
  & \le \Big \vert \big( P(x, \pi_h(x)) - \widetilde{P}(x, \pi_h(x)) \big) V_{h+1}^\pi  \Big \vert + \widetilde{P}(x, \pi_h(x)) E_{h+1}^\pi \\
  & \le \indicator\{\hat{f}(x) \neq f(x) \} \max_{y \in \cX} \Big \vert V_{h+1}^\pi(y) - P(x, \pi_h(x))V_{h+1}^\pi \Big\vert + \widetilde{P}(x, \pi_h(x)) E_{h+1}^\pi \\
  & \le 2 \eta^2 \indicator\{\hat{f}(x) \neq f(x) \} + \widetilde{P}(x, \pi_h(x)) E_{h+1}^\pi
\end{align*}
where we used the facts that (i) $\widetilde{P}(x,a) = P(x, a)$ whenever $\hat{f}(x) = f(x)$, (ii) $P$ is $\eta^2$-regular, (iii) Lemma \ref{lem:max_value} to obtain  $\max_{y \in \cX} \vert V_{h+1}^\pi(y) - P(x, \pi_h(x))V_{h+1}^\pi \vert \le 2\eta^2 - 1$, and  setting $E_{H+1}^\pi = 0$. We conclude after iterating the above recursion that
\begin{align*}
  \EE_{x_1}[E_1^\pi(x_1)] & \le \EE_{\widetilde{P}}\Big[ \sum_{h=1}^{H} \indicator\{ f(x_{h}^\pi) \neq \hat{f}(x_h^\pi)\}\Big] \\
  & \le \max_{x \in \cX}\widetilde{N}(x) \vert \cE \vert \\
  & \le \frac{\eta^2 H \vert \cE\vert}{n}
\end{align*}
where $\widetilde{N}(x) = \EE_{\widetilde{P}}[ \sum_{h=1}^H \indicator\{ x_h^\pi = x\}]$, and by $\eta^2$-regularity of  $\widetilde{P}$, we can easily verify that $\widetilde{N}(x) \le \frac{\eta^2H}{n}$.

\medskip

\textbf{\emph{Step 2 -- (Bounding $\tilde{E}_1^\pi(x)$)}} First, we verify that $\widetilde{E}_h^\pi$ satisfies the following recursion
\begin{align*}
  \widetilde{E}_h^\pi(x) & =  \Big\vert \widetilde{P}(x, \pi_h(x)) \widetilde{V}_{h+1}^{\pi}- \hat{P}(x, \pi_h(x)) \hat{V}_{h+1}^\pi \Big\vert \\
  & \le \Big\vert \big (\widetilde{P}(x, \pi_h(x)) - \hat{P}(x, \pi_h(x)) \big) \widetilde{V}_{h+1}^\pi \Big\vert +  \hat{P}(x, \pi_h(x)) \Big\vert  \widetilde{V}_{h+1}^\pi  - \hat{V}_{h+1}^\pi   \Big\vert \\
  & \le \Big\vert \big (\widetilde{P}(x, \pi_h(x)) - \hat{P}(x, \pi_h(x)) \big) \widetilde{V}_{h+1}^\pi \Big\vert +  \hat{P}(x, \pi_h(x)) \widetilde{E}_{h+1}^\pi \\
  & \le \max_{x \in \cX,a \in \cA} \Big\vert \big (\widetilde{P}(x,a) - \hat{P}(x, a) \big) \widetilde{V}_{h+1}^\pi \Big\vert +  \hat{P}(x, \pi_h(x)) \widetilde{E}_{h+1}^\pi
\end{align*}
with $\widetilde{E}_{H+1}^\pi = 0$. By iterating the above recursion, we then obtain
\begin{align*}
  \EE_{x_1}[\tilde{E}_1^\pi(x_1)] & \le \sum_{h=1}^H \max_{x \in \cX,a \in \cA} \Big\vert \big (\widetilde{P}(x,a) - \hat{P}(x, a) \big) \widetilde{V}_{h+1}^\pi \Big\vert \\
  & \le H \max_{x \in \cX,a \in \cA, h \in[H]} \Big\vert \big (\widetilde{P}(x,a) - \hat{P}(x, a) \big) \widetilde{V}_{h+1}^\pi \Big \vert.
\end{align*}
Similarly, we can also obtain
\begin{align*}
  \EE_{x_1}[\tilde{E}_1^\pi(x_1)] \le  H \max_{x \in \cX,a \in \cA, h \in[H]} \Big\vert \big (\widetilde{P}(x,a) - \hat{P}(x, a) \big) \hat{V}_{h+1}^\pi \Big \vert.
\end{align*}
This completes the proof of the lemma.
\end{proof}

\begin{proof}[Proof of Lemma \ref{lem:prob_err_decomposition}]
  First, we will start by introducing two intermediate probability distributions: $\widetilde{P}_1(x,a) = (q(y, \hat{f}(y)) p(\hat{f}(y) \vert f(x), a)  )_{y \in \cX}$, and $\widetilde{P}_2(x,a) = (\hat{q}(y, \hat{f}(y)) p(\hat{f}(y) \vert f(x), a)  )_{y \in \cX}$. We may then write
  \begin{align*}
    \Big\vert \big(\widetilde{P}(x,a)  - \hat{P}(x,a)\big) V \Big\vert & \le \Big\vert \big(\widetilde{P}(x,a)  - \widetilde{P}_1(x,a)\big) V \Big\vert \\
    & \quad  + \Big\vert \big(\widetilde{P}_1(x,a)  - \widetilde{P}_2(x,a)\big) V \Big\vert \\
    & \quad  + \Big\vert \big(\widetilde{P}_2(x,a)  - \hat{P}\; (x,a)\big) V \Big\vert.
  \end{align*}
  Note that $\widetilde{P}(y \vert x,a) = \widetilde{P}_1(y \vert x, a)$ whenever $f(y) = \hat{f}(y)$, thus
  \begin{align*}
    \Big\vert \big(\widetilde{P}(x,a)  - \bar{P}(x,a)\big) V \Big\vert & \le \Vert V \Vert_\infty \sum_{y \in \cX} \vert\bar{P}(y \vert x,a) - \widetilde{P}(y \vert x, a)\vert  \\
    & \le \Vert V \Vert_\infty \sum_{y \in \cX} \indicator\{\hat{f}(y) \neq f(y)\}  \max_{y \in \cX} q(y \vert s') p(s' \vert s, a) \\
    & \le \frac{\eta^2  \vert \cE \vert \Vert V \Vert_\infty}{n}.
  \end{align*}
  Next, we have
  \begin{align*}
    \Big\vert \big(\widetilde{P}_1(x,a)  - \widetilde{P}_2(x,a)\big) V \Big\vert  & \le \Big\vert \sum_{y \in \cX} \big(q(y \vert \hat{f}(y)) - \hat{q}(y \vert \hat{f}(y)) \big) p(\hat{f}(y) \vert \hat{f}(x), a) V(y)    \Big\vert \\
    & \le   \sum_{s \in \cS } p(s \vert \hat{f}(x), a)  \left \vert \sum_{y \in \hat{f}^{-1}(s)} \big(q(y \vert s) - \hat{q}(y \vert s \big)  V(y)  \right \vert \\
    & \le   \sum_{s \in \cS } p(s \vert \hat{f}(x), a) \left \vert  \sum_{y \in \cX} \big(q(y \vert s) - \hat{q}(y \vert s \big)  V(y) \right \vert     \\
    & \qquad +  \sum_{s \in \cS } p(s \vert \hat{f}(x), a) \left\vert \sum_{y \in \cX} \indicator\{\hat{f}(y) \neq f(y)\} q(y \vert s) V(y) \right \vert \\
    & \le \max_{s \in \cS}  \left \vert  \sum_{y \in \cX} \big(q(y \vert s) - \hat{q}(y \vert s \big)  V(y) \right \vert + \frac{\eta^2 \vert \cE\vert  \Vert V \Vert_\infty}{n}
  \end{align*}
  where we used that, by construction, $\hat{q}( y \vert s) = 0$ for all $s \not \in \hat{f}^{-1}(y)$. Finally, we have
  \begin{align*}
    \Big\vert \big(\widetilde{P}_2(x,a)  - \hat{P}(x,a)\big) V \Big\vert & = \left\vert  \sum_{y \in \cX}  \hat{q}(y \vert \hat{f}(y))\big(p( \hat{f}(y) \vert \hat{f}(x),a)  - \hat{p}( \hat{f}(y) \vert \hat{f}(x),a)\big) V(y) \right\vert \\
    & \le \left \vert \sum_{ s' \in \cS}  \big(p( s' \vert \hat{f}(x),a)  - \hat{p}( s' \vert \hat{f}(x),a)\big) \sum_{y \in \hat{f}^{-1}(s)}  \hat{q}(y \vert s) V(y)  \right\vert \\
    & \le \sum_{s' \in \cS} \vert p( s' \vert \hat{f}(x),a)  - \hat{p}( s' \vert \hat{f}(x),a) \vert \max_{s' \in \cS}  \left \vert  \sum_{y \in \hat{f}^{-1}(s)}  \hat{q}(y \vert s) V(y)  \right\vert \\
    & \le \left\Vert  p(\hat{f}(x),a)  - \hat{p}(\hat{f}(x),a)\right\Vert_1 \Vert V \Vert_\infty \\
    & \le \Vert V \Vert_\infty \max_{s \in \cS} \left\Vert  \hat{p}(s,a)  - p(s,a)\right\Vert_1
  \end{align*}
  To conclude, we obtain
  \begin{align*}
    \Big\vert \big(\widetilde{P}(x,a)  - \hat{P}(x,a)\big) V \Big\vert & \le \frac{2 \eta^2 \Vert V  \Vert_\infty}{n} + \max_{s \in \cS} \Vert \hat{p}(s,a)  - p(s,a) \Vert_1 \Vert V  \Vert_\infty  \\
    & \quad + \max_{s \in \cS}  \Big\vert  \sum_{y \in \cX} \big(q(y \vert s) - \hat{q}(y \vert s \big)  V(y) \Big \vert
  \end{align*}
\end{proof}

\begin{proof}[Proof of Lemma \ref{lem:max_value}]
  For simplicity, let us denote $g_h(x) = r_h(x, \pi_h(x))$ and $\EE_h[ \ \cdot \ ] = \EE^\pi\left[ \ \cdot \ \vert x_h = y\right] $. Then we note that
  \begin{align*}
    V_{h+1}^\pi(x) - P(y, \pi_h(y)) V_{h+1}^\pi & =  \EE_{h}\left[  \sum_{\ell = h}^H g_\ell(x_\ell) - \EE_{h}[g_\ell(x_\ell)] \right]
  \end{align*}
Since the transitions $(P(\cdot, \pi_h(\cdot)))_{h\in \cH}$ are $\eta^2$-regular, we obtain that
\begin{align*}
    \max_{x \in \cX}\Big\vert V_{h+1}^\pi(x) - P(y, \pi_h(y)) V_{h+1}^\pi \Big\vert \le (2\eta^2 - 1)
\end{align*}
\end{proof}

\subsection{Proofs -- Concentration bounds on averaged optimal value functions}

\begin{proof}[Proof of Lemma \ref{lem:concentration_avg_val}]
First, let us note that $\hat{V}_{h+1}^{\hat{\pi}_r} \in \cV^\star_{h+1}(\hat{f};r)$.   Let $\cN$ be an $\epsilon$-net of $\cV^\star(\hat{f};r)$ with respect to $\Vert \cdot \Vert_\infty$.  We have:
  \begin{align*}
    & \Big\vert ( \hat{q}(s) - q(s) ) \hat{V}^{\hat{\pi}_r}_{h} \Big \vert \le \sup_{V \in \cV^\star_{h}(\hat{f};r)} \Big\vert ( \hat{q}(s) - q(s) ) V \Big \vert \\
    & \qquad \le \max_{V \in\cN} \Big\vert ( \hat{q}(s) - q(s) ) \widetilde{V} \Big \vert +  \sup_{V \in \cV^\star_{h}(\hat{f};r)} \max_{\widetilde{V} \in \cN} \Big\vert ( \hat{q}(s) - q(s) ) (V - \widetilde{V}) \Big \vert   \\
    & \qquad \le \max_{V \in\cN} \Big\vert ( \hat{q}(s) - q(s) ) \widetilde{V} \Big \vert +  \Vert \hat{q}(s) - q(s) \Vert_1 \epsilon  \\
    & \qquad \le \max_{V \in\cN} \Big\vert ( \hat{q}(s) - q(s) ) \widetilde{V} \Big \vert +  2 \epsilon.
  \end{align*}
  Let us recall that the construction of $\hat{q}$ and $\hat{p}$ only uses half the subsequent budget of $TH/2$ observations. Thus, we have
  \begin{align*}
    \PP(  \Big\vert ( \hat{q}(s) - q(s) ) \hat{V}^{\hat{\pi}_r}_{h} \Big \vert > \rho ) & \le \PP(  \max_{V \in\cN} \Big\vert ( \hat{q}(s) - q(s) ) \widetilde{V} \Big \vert +  2 \epsilon > \rho ) \\
    & = \EE\left[ \EE\left[  \indicator\left\lbrace \max_{V \in\cN} \Big\vert ( \hat{q}(s) - q(s) ) \widetilde{V} \Big \vert +  2 \epsilon > \rho \right \rbrace \Big\vert   \hat{f} \right] \right] \\
    & \le \EE\left[ \EE\left[  \sum_{\widetilde{V} \in \cN} \indicator\left\lbrace \Big\vert ( \hat{q}(s) - q(s) ) \widetilde{V} \Big \vert +  2 \epsilon > \rho \right \rbrace \Big\vert   \hat{f} \right] \right] \\
    & \le \EE\left[ \vert \cN \vert \max_{\widetilde{V} \in \cN} \EE\left[  \indicator\left\lbrace \Big\vert ( \hat{q}(s) - q(s) ) \widetilde{V} \Big \vert +  2 \epsilon > \rho \right \rbrace \Big\vert   \hat{f} \right] \right].
  \end{align*}
  Applying Proposition \ref{prop:ee_q_avg}, we obtain for $TH \ge \poly(\eta)S( \rho' + \log(S))$,
  \begin{align*}
    \EE\left[  \indicator\left\lbrace \Big\vert ( \hat{q}(s) - q(s) ) \widetilde{V} \Big \vert + 2 \epsilon > \poly(\eta) S H  \left(\sqrt{\frac{\rho' + \log(S)}{TH}} + \frac{\vert \cE \vert}{n} + \frac{\rho'}{TH} \right)+ 2 \epsilon \right \rbrace \Big\vert   \hat{f} \right] \le 4e^{-\rho'}.
  \end{align*}
  Additionally, Lemma \ref{lem:net-val} ensures that $\vert \cN \vert \le \left(1 + \frac{H\sqrt{SA}}{\epsilon} \right)^{SA}$. Therefore,
  \begin{align*}
    \EE\left[ \vert \cN \vert \max_{\widetilde{V} \in \cN} \EE\left[  \indicator\left\lbrace \Big\vert ( \hat{q}(s) - q(s) ) \widetilde{V} \Big \vert +  2 \epsilon > \rho \right \rbrace \Big\vert   \hat{f} \right] \right] \le 4  \left(1 + \frac{\sqrt{SA}(H-h)}{\epsilon} \right)^{SA} e^{- \rho'}
  \end{align*}
  where we set $\rho = \poly(\eta) S H  \left(\sqrt{\frac{\rho' + \log(S)}{TH}} + \frac{\vert \cE \vert}{n} + \frac{\rho'}{TH} \right)+ 2 \epsilon$. Further, reparametrizing by $\rho'' = \rho' + SA \log( 1 + \frac{2H\sqrt{SA}}{\epsilon})$, we obtain
  \begin{align*}
    \PP\left(  \Big\vert ( \hat{q}(s) - q(s) ) \hat{V}^{\hat{\pi}_r}_{h} \Big \vert > \rho \right) \le 4 e^{-\rho''}
  \end{align*}
  where
  $$
  \rho = \poly(\eta) S H  \left(\sqrt{\frac{\rho'' + SA \log( 1 + \frac{2H\sqrt{SA}}{\epsilon}) + \log(S)}{TH}} + \frac{\vert \cE \vert}{n} + \frac{\rho'}{TH} \right)+ 2 \epsilon
  $$
  Now choosing $\epsilon = \sqrt{\frac{SA}{n}}$, we obtain that
  \begin{align*}
      \rho & \le \poly(\eta) S H  \left(\sqrt{\frac{\rho'' + SA \log(Hn) + \log(S)}{TH}} + \frac{\vert \cE \vert}{n} + \frac{\rho'' + SA\log(Hn)}{TH} \right)+ 2 \sqrt{\frac{SA}{n}} \\
      & \le \poly(\eta) S H  \left(\sqrt{\frac{\rho'' + SA \log(HSn)}{TH}} + \frac{\vert \cE \vert}{n} + \frac{\rho'' + SA\log(HSn)}{TH} \right)+ 2 \sqrt{\frac{SA}{TH}}
  \end{align*}
 where the second inequality holds as long as $TH = \Omega(n)$. Thus, choosing $\rho'' = \frac{1}{HSn}$ gives that for $TH = \Omega(n)$,
 \begin{align*}
   \PP\left( \Big\vert ( \hat{q}(s) - q(s) ) \hat{V}^{\hat{\pi}_r}_{h} \Big \vert >  \poly(\eta) S H  \left(\sqrt{\frac{SA \log(HSn)}{TH}} + \frac{\vert \cE \vert}{n} \right)+ 2 \sqrt{\frac{SA}{TH}} \right) \le \frac{4}{HSn}
 \end{align*}
 where we also used the fact that  $  \sqrt{\frac{SA \log(HSn)}{TH}} \ge \frac{SA \log(HSn)}{TH} $ for $TH = \Omega(\log(n))$. Therefore, using a union bound, we finally obtain that
 \begin{align*}
     \PP\left( \max_{s \in \cS, h \in [H]}\Big\vert ( \hat{q}(s) - q(s) ) \hat{V}^{\hat{\pi}_r}_{h} \Big \vert >  \poly(\eta) S H  \left(\sqrt{\frac{SA \log(HSn)}{TH}} + \frac{\vert \cE \vert}{n} \right) \right) \le \frac{4}{n}
 \end{align*}
 provided $TH = \Omega(n)$. Similarly, but this time without going through a net argument since $\widetilde{V}^\star_h$ conditionally on $\hat{f}$ is no more random, we can obtain that
  \begin{align*}
    \PP\left( \max_{s \in \cS, h \in[H]}\Big\vert (\hat{q}(s) - q(s) ) V^{\star}_{h} \Big \vert > \poly(\eta) S H \left(\sqrt{\frac{SA\log(HSn)}{TH}} + \frac{SH\vert \cE \vert}{n} \right)\right) \le \frac{4}{n}
  \end{align*}
  for $TH = \Omega(n)$.
\end{proof}

\newpage
% !TEX root = ./main.tex

\section{Towards Optimal Adaptive Exploration in BMDPs}\label{app:adaptive}

We have so far considered scenarios where the $T$ episodes were generated under a fixed behavior policy $\rho$. In this section, we briefly discuss how our results could be extended to the case where the behavior policy is (history-)adaptive.
Note that for reward-free RL, our lower bounds derived for the sample complexity of identifying near-optimal policies (see Theorem \ref{thm:lower-boundSC1} and \ref{thm:lower-boundSC2}) hold even for adaptive exploration policies; surprisingly, Theorem \ref{thm:rf1} and \ref{thm:rf2} establish that even using uniform behavior policy, we can match these lower bounds.
However, we believe that using adaptive behavior policies would actually yield significant performance improvements.

As before, we denote by $\{\cT^{(\tau)}\}_{\tau=1}^{T}$ the data collected during the $T$ episodes. In case of an adaptive exploration process, the policy at episode $t$, $\rho^{(t)}$, used to generate $\cT^{(t)}$ may be chosen depending on the previously observed data history $\{\cT^{(\tau)}\}_{\tau=1}^{t-1}$.

Before we present ways to extend our analysis to the case of adaptive behavior policies, we provide here a few pointers towards existing studies on the adaptive design of the exploration process for various inference tasks, especially for reward-free RL.
In tabular MDP, adaptive exploration for identifying near-optimal policies has extensively been studied with both minimax or problem-specific guarantees and in the reward-specific and reward-free settings, see e.g. \cite{zanette2019tighter, jin2020reward,Marjani21a,Marjani21b,Menard21} and references therein. Recently, \cite{Tarbouriech19,Tarbouriech21a} developed a framework to learn the model or a policy with prescribed frequencies at which (state, action) pairs are visited. All the aforementioned studies are restricted to general tabular MDPs.
For BMDPs, as already discussed in Section \ref{sec:related-works}, most of the proposed algorithms use adaptive policies in settings where the reward is already known and cluster recovery {\it isn't} the main objective.
Indeed, the problem of estimating the latent state decoding function in BMDPs has some similarities with the cluster recovery problem in SBMs, DCBMs, and BMCs.
However, even for SBMs, there isn't much work on adaptive sampling strategies; refer to \cite{YunP14, YunP19} for preliminary results.

\subsection{Lower bound on the latent state decoding error rate}
We apply the same change-of-measure argument to derive this lower bound. Fix a context $x$.
As the initial part of our lower bound proof is applicable to adaptive policies as well (from Appendix \ref{app:lower-bound}), we first recall them here.

We start with the lower bound of the individual error rate:
$$
\varepsilon_x = \PP_\Phi[x\in {\cal E}]\geq \frac{1}{2S} \exp\left( -\EE_\Psi[\cL] - \sqrt{2\eta S \Var_\Psi[\cL]} \right).
$$
We can further establish the connection between $\EE_\Psi[\cL]$ and the rate function $I_j(x;c,m_\rho^{\Psi}, \Phi)$ defined as:
\begin{align*}
	I^{(t)}_j(x; c, m_\rho^\Psi, \Phi) &:= n\sum_{a \in \cA} \sum_{s \in \cS} \left\{ c  q(x | f(x)) p(j | s, a) m_\rho^{\Psi,(t)}(s, a) \log\frac{c p(j | s, a)}{p(f(x) | s, a)}\right. \nonumber \\
	&\quad + c  q(x | f(x)) m_\rho^{\Psi,(t)}(j, a) p(s | j, a) \log\frac{p(s|j,a)}{p(s|f(x),a)} \nonumber \\
	&\quad + \left.  (1 - c q(x | f(x)) p(j | s, a)) m_\rho^{\Psi,(t)}(s, a) \log\frac{1 - c q(x | f(x)) p(j | s, a)}{1 - q(x | f(x)) p(f(x) | s, a)} \right\},
\end{align*}
where $m_\rho^{\Psi,(t)}(s, a)$ denotes the expected proportion of rounds spent in (latent state, action) pair $(s,a)$ under policy $\rho$ and model $\Psi$, in the $t$-th episode:
\begin{equation*}
	m_\rho^{\Psi,(t)}(s, a) := \frac{1}{H - 1} \sum_{h=1}^{H-1} \PP_\Psi[\tilde{f}(x_h^{(t)} )= s, a_h^{(t)} = a].
\end{equation*}
We add the explicit dependence on $m_\rho^{\Psi}$ for the rate function to remember that we can design the adaptive policy $\rho$ based on the desirable $m_\rho^\Psi$.

We also have:
$$\EE_\Psi[\cL] \leq \frac{TH}{n} (I_j(x;c,m_\rho^{\Psi},\Phi)+{\cal O}({1\over n}))\quad\hbox{and}\quad
\Var_\Psi[\cL] \leq \cO\left( \frac{TH}{n} \right).$$

By defining
\begin{equation}
	I(x; m_\rho^{\Psi}, \Phi) := \min_{j: j \not= f(x)} \min_{c>0} I_j(x; c,m_\rho^{\Psi}, \Phi),
\end{equation}
we can show, as previously, that under the adaptive policy $\rho$,
$$
\varepsilon_x =\mathbb{P}_\Phi[x\in {\cal E}] \ge \frac{1}{2\eta S} \exp(-{TH\over n} I(x; m_\rho^{\Psi}, \Phi)(1+o(1))).
$$ 
We deduce that the expected number of misclassified contexts under the adaptive policy $\rho$ is lower bounded as follows:
$$
\mathbb{E}_\Phi[|{\cal E}|] \gtrsim n \exp( - {TH\over n} I(m_\rho^{\Phi}, \Phi)(1+o(1))),
$$ 
where $\gtrsim$ hides additional (polynomial) dependency on $\eta, S, A$, and
$$
I(m_\rho^{\Phi}, \Phi) = -{n\over TH}\log\left( {1\over 2\eta S n}\sum_x \exp(-{TH\over n}I(x;m_\rho^\Phi, \Phi))\right).
$$
Note that we could replace $m_\rho^\Psi$ by $m_\rho^\Phi$ in the above definition, because $\Phi$ and $\Psi$ only differ from a single context. The formal justification is left for future work. Now, the difference with scenarios without adaptive exploration policies is that here the learner can design policies to optimize $m_\rho^{\Phi}$. We denote by ${\cal M}$ the set of possible expected proportions of rounds spent in the various $(s, a)$ pairs. ${\cal M}$ can be intricate to characterize for small $T$ but when $T$ grows large, we can proceed as in \cite{Marjani21b} and write the {\it navigation} constraints $m\in {\cal M}$ should asymptotically satisfy.

We now define the {\it adaptive rate function}:
\begin{equation}
{I'}(\Phi) := \max_{m\in {\cal M}} I(m, \Phi).
\end{equation}

Finally, we can state the following preliminary result:
\begin{theorem}
	\label{thm:adaptive-lower-bound}
	Let $\Phi$ be a BMDP satisfying Assumptions \ref{assumption:SA},\ref{assumption:linear},\ref{assumption:regularity},\ref{assumption:uniform}. Consider a clustering algorithm that is $\beta$-locally better-than-random for $\Phi$ with $\beta \ge \frac{2S\eta^2}{n}$,
	when applied to the data gathered over $T$ episodes, each of length $H$, using {\bf adaptive} $\rho$.
	Then, we have that
	\begin{equation*}
		\mathbb{E}_\Phi[|{\cal E}|] \gtrsim n\exp\left(-\frac{TH}{n} {I'}(\Phi)(1+o(1))\right).
	\end{equation*}
\end{theorem}

\subsection{Adaptive latent state decoding algorithms}

An interesting interpretation of the adaptive lower bound is that $m^\star(\Phi)\in \arg\max_{m\in {\cal M}} I(m, \Phi)$ represents an optimal way of navigating the BMDP to minimize the proportion of misclassified context. We can leverage this observation to design an adaptive algorithm. We can proceed in the following steps:
\begin{enumerate}
\item Using $\delta T$ episodes, explore the BMDP uniformly at random and run Algorithm \ref{alg:spectral-clustering} and \ref{alg:likelihood-improvement} to obtain the initial estimate $\hat{\Phi}_1=(\hat{f}_1, \hat{p}_1, \hat{q}_1)$\;
\item Compute $\widehat{m} \in \argmax_{m \in \cM(\hat{\Phi}_1)} \tilde{I}(x; m, \hat{\Phi})$ for each $x \in \cX$ (see Appendix \ref{sec:alternate-kl})\;
\item Using the remaining $(1 - \delta) T$ episodes, adaptively explore the BMDP such that the expected numbers of visits to each (estimated latent state-action) pair match $\widehat{m}$;
\item	Based upon the new samples, run the cluster improvement algorithm to output improved $\hat{\Phi} = (\hat{f}, \hat{p}, \hat{q})$.
\end{enumerate}

We emphasize that the above algorithm is not complete, and it just provides good design principles.

%
%\subsection{Adaptive Algorithm}
%Indeed, we can just apply the {\textsc GOSPRL} algorithm as proposed in \cite{Tarbouriech21a} to the solution obtained.
%However...
%
%
%\subsubsection{With generative model assumption}
%In this case, we simply sample the 
%
%
%Suppose that we're given the estimated $\hat{f}$, $\hat{p}$, and $\hat{\alpha}$ from the previous steps (i.e. after Algorithms \ref{alg:spectral-clustering}, \ref{alg:likelihood-improvement}, and \ref{alg:estimation}).

\newpage
% !TEX root = ./main.tex
\section{Beyond the $\eta$-regularity Assumption}\label{app:beyond-eta}

In this section, we describe how to recover our results under a weaker assumption than Assumption \ref{assumption:regularity}(ii), namely, only assuming that the dynamics of the latent state are ergodic.

\subsection{Relaxing Assumption \ref{assumption:regularity}(ii)}

We may relax Assumption \ref{assumption:regularity}(ii) in the following way as done for instance in \cite{Azizzadenesheli16a}.

\begin{assumption}\label{assumption:relax-eta} The Markov chains corresponding to the latent transition matrices, $p_a = (p(s'\vert s, a))_{s,s'\in \cS}$ for all $a \in \cA$, are aperiodic and communicating.
\end{assumption}

\noindent A direct consequence of Assumption \ref{assumption:relax-eta} is that we may define a regularity parameter for the $k$-step transition probabilities, which again will allow us to quantify the mixing properties of the latent dynamics in our BMDPs, namely those described by the transition matrices $p_a = (p(s'\vert s, a))_{s,s'\in \cS}$ for all $a \in \cA$. Additionally, aperiodicity may be further relaxed by simply requiring that a stationary distribution exists instead.

\begin{lemma}\label{lemma:relax-eta}
  Let $P$ be the transition matrix corresponding to a communicating and aperiodic Markov chain over a finite state space $\cZ$. Then, setting $k = \vert \cZ \vert^2$, we have:
  \begin{itemize}
    \item[(i)] there exists $\eta \ge 1$, such that
    $$
    \eta \ge \max_{z_1,z_2,z_3 \in \cZ} \left\lbrace \frac{P^k(z_1 \vert z_3)}{P^k(z_2 \vert z_3)}, \frac{P^k(z_1 \vert z_2)}{P^k(z_1 \vert z_3)} \right\rbrace;
    $$
    \item[(ii)] the Dobrushin's coefficient of $P^k$ is non-trivially bounded i.e., $\delta(P^k) \le 1- 1/\eta;$
    \item[(iii)] the mixing time of $P$ satisfies $t_{mix}(\varepsilon) \le \vert\cZ\vert^2 (\eta \log(1/\varepsilon) + 1)$.
  \end{itemize}
\end{lemma}

\begin{proof}[Proof sketch of Lemma \ref{lemma:relax-eta}] To establish $(i)$ we simply need to verify that for $k = \vert \cZ \vert^2$, $P^{k}(z_1 \vert z_2)> 0$ for all $z_1, z_2 \in \cZ$. Because the Markov chain is communicating, there exists $k_0 \ge 0$, such that for all $k \ge k_0$, the $k$-step transition probabilities are strictly positive. Corollary 1 of \cite{denardo1977periods} implies that $k_0 \le \vert \cZ \vert^2$.
\noindent Equipped with $(i)$, the proof of $(ii)$ follows immediately from Definition 3 and the proof of $(iii)$ follows similarly as in the proof of Propostion 5.
\end{proof}

\begin{remark} In view of Lemma \ref{lemma:relax-eta}, under Assumption \ref{assumption:relax-eta}, we can easily verify most of our results in Appendix \ref{app:equilibrium} generalize. Most importantly, for all $a \in \cA$, the $S^2$-step transition $P^{S^2}_a$ will be $\eta$-regular, even though the state space $\cX$ is of size $n \gg S$. Consequently, the $S^2$-step transitions of the induced Markov chains under our model will also inherit the $\eta$-regularity.
\end{remark}

\subsection{Bernstein-type concentration bounds for Markov chains with restarts.}

The extension of the concentration bound presented in Theorem \ref{thm:bernstein-restart-bmdp} for the time-homogeneous Markov chain immediately follows under the assumption that the Markov chain is aperiodic and communicating.

\begin{theorem}\label{thm:bernstein-1}
Let $\{ (X_{h}^{(t)})_{h = 1}^H \}_{t \in [T]}$ be a collection of {\it i.i.d.} time-homogeneous Markov chains over a finite state space $\cZ$, with transition probability matrix $P$ and initial distribution $\mu \in \cP(\cZ)$. We assume that $\mu$ and $P^k$ are $\eta$-regular, and that $P$ admits a stationary distribution $\nu$. Let $(\phi_h)_{h \ge 1}$ be sequence of mapping from $\cX$ to $\RR$, bounded and  measurable. Then we have that for all $u \geq 0$,
  \begin{align*}
      \PP\left( \sum_{t=1}^T \sum_{h=1}^H \phi_h(X_h^{(t)}) - \EE_\mu[\phi_h(X_h^{(t)})]  > u \right) \le k \exp\left( -\frac{u^2}{2TH V_{\mu,P,\phi} + \frac{2}{3} M_{P,\phi} u} \right)
  \end{align*}
  where
  \begin{align*}
    V_{\mu,P,\phi} & \le \textup{poly}(\eta) k^2 \max_{h \ge 1, i \in \lbrace 0, \dots, k-1\rbrace } \left\lbrace \textup{Var}_{P^k(x, \cdot)}[\phi_h],    \textup{Var}_{\mu P^i}[\phi_h] \right\rbrace, \\
    M_{P, \phi} & \le \textup{poly}(\eta)  k  \max_{h\ge 1} \Vert \phi_h\Vert^2.
  \end{align*}
\end{theorem}

\begin{proof}[Proof sketch of Theorem \ref{thm:bernstein-1}]
In fact, the proof steps are similar to those of the proof of Theorem \ref{thm:bernstein-restart} with some extra steps. First, recall from the proof of Lemma 7 that we may write
\begin{align*}
  \sum_{h=1}^H \phi_h(X_h) - \EE_\mu[\phi_h(X_h)] & = \sum_{\ell = 1}^H  Z_\ell\left(\sum_{h=\ell}^H (P-\Pi)^{h-\ell} \phi_h \right) = \sum_{i=0}^{k-1} S_i
\end{align*}
with
$$
S_i = \sum_{\ell = 1}^{H/k}  Z_{\ell k + i}\left(\sum_{h=\ell k + i}^H (P-\Pi)^{h-\ell} \phi_h \right)
$$
where we assume for simplicity that $H/k$ is an integer. We note that for each $i \in \lbrace 0, \dots, k-1\rbrace$, $S_i$ is a martingale. We will analyze each $S_i$ separately and then combine our concentration bounds via a union bound. Now let $i$ be fixed, following a similar argument as in the proof of Theorem \ref{thm:bernstein-restart}, we can conclude by upper bounding the terms
\begin{align*}
  V_{\mu,P,\phi} & = \max_{z \in \cZ, \ell k + i \in [H]} \left \lbrace  \textrm{Var}_{P^k(z, \cdot)} \left[ \sum_{h=\ell}^H  (P - \Pi)^{h-\ell} \phi_h \right], \textrm{Var}_{\mu P^i(z, \cdot)} \left[ \sum_{h=\ell k+i}^H  (P - \Pi)^{h-\ell} \phi_h \right] \right\rbrace  \\
  M_{P, \phi} & = \left\Vert \sum_{h = \ell k + i}^H (P - \Pi)^{h-\ell}\phi_h \right\Vert_\infty.
\end{align*}

We recall that we used Lemma \ref{lem:nsmc-var-dev} and Lemma \ref{lem:nsmix} for that. Now, we shall provide the key steps that need to be adapted in order to obtain similar lemmas. First, we can establish, following a similar reasoning as in Lemma \ref{lem:nsmc-var-dev}  that for any $\eta$-regular distribution $\mu$, we have
\begin{align*}
  \textrm{V}_{\mu} \left[ \sum_{h=\ell}^H (P-\Pi)^{h-\ell} \phi_h\right] & \le \left( \sum_{h = \ell}^H \left\Vert (P-\Pi)^{h-\ell} (\phi_h - \mu \phi_h \textbf{1})   \right\Vert_\infty \right)^2 \\
  & \le \left( \sum_{h = 0}^\infty k \max_{i \in [k]}\left\Vert (P-\Pi)^{kh} (\phi_{hk+i} - \mu \phi_{hk+i} \textbf{1})   \right\Vert_\infty \right)^2
\end{align*}
where we recall here that $k = \vert \cZ\vert^2$, upper bound naively the sum by the infinite sum, then split the sum into blocks of size $k$ to obtain the final inequality. Next, we can easily verify that for all $h \ge 1$, $P^{kh}$, $\Pi$ and $\mu$ are all $\eta$-regular since $P^k$ is $\eta$-regular. Thus, since $(P-\Pi)^k = P^k - \Pi$, we can obtain in a similar fashion as in the proof of Lemma \ref{lem:nsmc-var-dev}, we have
\begin{align*}
  \max_{x,y\in \cZ}  \left\vert \frac{(P-\Pi)^{kh}(x, y)}{\mu(x)} \right\vert  \le \eta^2 - 1
\end{align*}
Thus, we have for all $h \ge 1$,
\begin{align*}
   \max_{i \in [k]}\left\Vert (P-\Pi)^{kh} (\phi_{hk+i} - \mu \phi_{hk+i} \textbf{1})   \right\Vert_\infty & \le   2(\eta^2 - 1)\delta(P^k) \max_{h\ge 1} \Vert \phi_h - \mu\phi_h\textbf{1}\Vert_\mu \\
   & \le 2(\eta^2 - 1) \left(1 - \frac{1}{\eta} \right)^h \max_{t\ge 1} \Vert \phi_t - \mu\phi_t\textbf{1}\Vert_\mu
\end{align*}
where in an intermediate step we use
\begin{align*}
  \Vert (P-\Pi)^{kh} g \Vert_\infty \le 2 \delta(P^{kh}) \Vert g \Vert_\infty \le 2  (\delta(P^k))^h \Vert g \Vert_\infty \le 2\left( 1 - \frac{1}{\eta}\right)^{h} \Vert g\Vert_\infty
\end{align*}
for all $g$, such that $\Vert g \Vert_\infty < \infty$, where again the proof follows similarly as in the proof of Lemma \ref{lem:nsmix}. Therefore, we obtain at the end that for $\ell \in [H]$.
\begin{align*}
    V_{\mu, P, \Phi} & \le \textrm{poly}(\eta) k^2 \max_{h \ge 1}\textrm{Var}_\mu[\phi_h]  \\
    M_{P, \Phi} & \le \textrm{poly}(\eta) k \max_{h \ge 1}\Vert \phi_h \Vert_\infty
\end{align*}
\end{proof}

Now, to further extend the above concentration to time-inhomogeneous Markov chains, we need a somewhat stronger assumption on the sequence of the transition matrices $(P_h)_{h\ge 1}$. More precisely, along with the assumption that all $P_h$'s are communicating and aperiodic, we also need to impose that
\begin{align*}
  \forall m \ge 1, \qquad \lim_{n \to \infty} \delta\left( \prod_{i=m}^n P_i\right) = 0
\end{align*}
where we recall that $\delta(P)$ denotes the Dobrushin coefficient of the stochastic matrix $P$ (see Definition \ref{def:dob-coef} in Appendix \ref{app:equilibrium}).

One observation is that our results for BMDPs do not require the use of a concentration for a time-inhomogenous Markov chain, except for our horizon-free guarantee in Theorem \ref{thm:rf1}, which only needs the bounds established in Lemma \ref{lem:nsmc-var-dev} and Lemma \ref{lem:nsmix}.
Extension to the time-inhomogeneous case is left to future work.

\subsection{On the lower bound for the latent state decoding error rate}

To extend our lower bound (Theorem \ref{thm:lower-bound}), we have to be careful when using the change-of-measure argument. More precisely, when defining the alternative models, we have to make sure that the absolute continuity of the dynamics under the true model, with respect to the dynamics of the perturbed models, holds. This will result in technical manipulations that will require a new definition of $I(x;\Phi)$ and $I(\Phi)$.

\subsection{On the performance guarantee of the initial spectral clustering (Algorithm \ref{alg:spectral-clustering})}

The algorithm of the initial clustering will remain the same and we expect that its performance guarantees will still hold under the relaxed assumption. As for the analysis, equipped with Theorem \ref{thm:bernstein-1}, we can immediately obtain all the required concentration results to prove Theorem \ref{thm:initial-spectral}.

\paragraph{Adapting the analysis of the concentration of the trimmed matrix.} This will be the main part that requires changes in the proofs. First, all the concentration results used in proving Theorem \ref{thm:initial-spectral} that relied on Theorem \ref{thm:bernstein-restart} are actually very easily modified by instead using Theorem \ref{thm:bernstein-1}. We will only pay an additional price for the dependence on $S$ (which will not affect our results in terms of scaling in $n$). Second, our proof relied on the fact that the underlying Markov chains will visit each context $x \in \cX$ with a probability that is roughly of order $1/n$ within two steps. That is why we used a two-step conditioning (See Appendix \ref{app:preliminaries}). This can be easily fixed by instead using $(S^2 + 1)$-step conditioning. To this aim, we may define for all $a, t, h$, $\widetilde{N}_{a,t,h} = \EE[ \hat{N}_{a,t,h} \vert \hat{N}_{a,t,h-S^2-1}]$, and analyze the matrix differences
\begin{align*}
  \forall r \in [S^2], \qquad \sum_{t = 1}^T \sum_{\ell=1}^{H/S^2} \hat{N}_{a,t, \ell S^2+r} - \widetilde{N}_{a, t, \ell S^2+r}
\end{align*}
whenever this is required in the proof instead of the odd and even terms, and then conclude by union bound. This will only come at a cost of $\poly(S) \ll n$ in the new concentration bounds (see Appendix \ref{app:trimming}). Finally, we will obtain the same results up to some minor technical details.

\paragraph{Separability, $S$-rank approximation, and $\ell_1$-weighted-K-means.} These parts in the proof of Theorem \ref{thm:initial-spectral} will practically remain unchanged, except for the separability section for which we need to take into account the $(S^2 + 1)$-step conditioning. Again, we will obtain the same results up to minor technical details.

\subsection{On the guarantee of the iterative likelihood improvement (Algorithm \ref{alg:likelihood-improvement})}

Algorithm \ref{alg:likelihood-improvement} corresponding to the iterative likelihood improvement is inspired by our lower bound. Therefore, the expression of the updates in the algorithm might change depending on how the definition of $I(\Phi)$ and $I(x; \Phi)$ change. As for the analysis, we expect that the proof of Theorem \ref{thm:likelihood-improvement} will not change much except for the requirement of new concentration bounds which again will be immediately obtained from Theorem \ref{thm:bernstein-1}.

\subsection{On the guarantees for reward-free RL}

For the reward-specific setting, the horizon-free result of Theorem \ref{thm:rf1} may not be preserved, unless we are able to recover the claims of Lemmas \ref{lem:nsmc-var-dev} and \ref{lem:nsmix} with time-inhomogeneous Markov chains.
On the other hand, the results of Theorem \ref{thm:rf1} and Theorem \ref{thm:rf2} will still hold with perhaps slightly worse dependencies in $S,A,H$, but not in $n$ and $T$.

\newpage
% !TEX root = ./main.tex

 \begin{figure*}[!th]
 \centering
 \includegraphics[width=\textwidth]{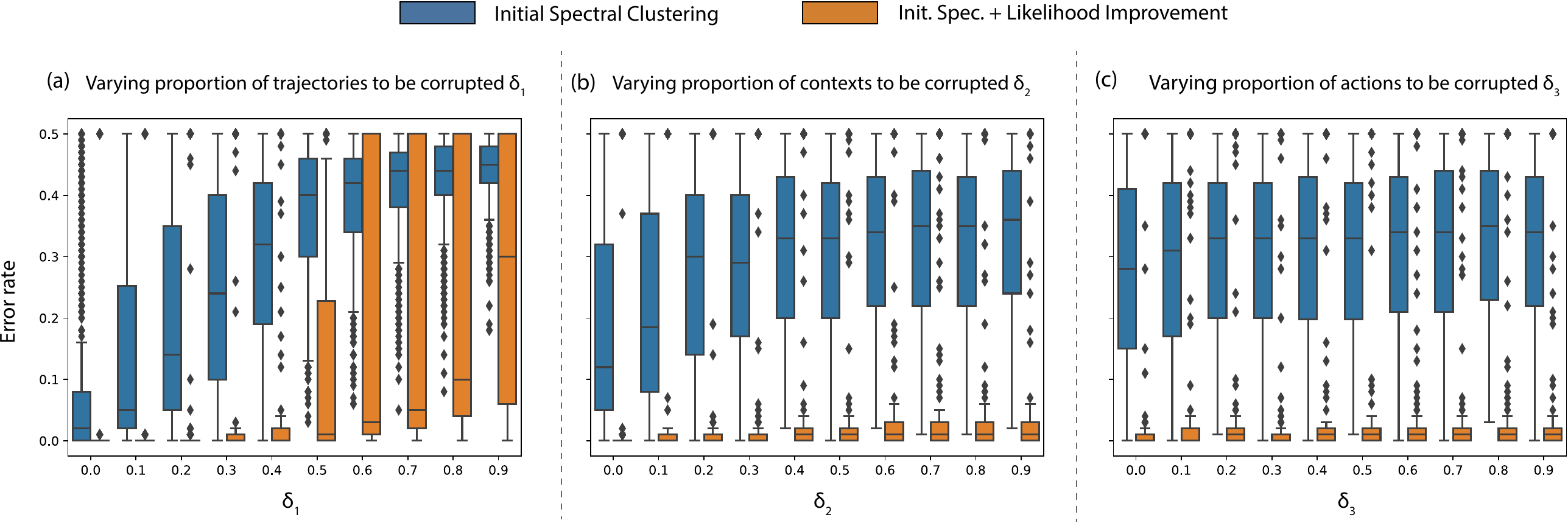}
 \caption{The clustering error rates with corruptions for various choices of (a) $\delta_1$'s, (b) $\delta_2$'s, and (c) $\delta_3$'s.}
 \label{fig:exp2}
 \end{figure*}

\section{Deferred Experimental Details}
\label{app:experiments}

\subsection{Missing details for Section~\ref{sec:experiments}}
The $k$-median clustering in Algorithm \ref{alg:spectral-clustering} is implemented using the pyclustering package~\citep{pyclustering}.
In each setting, each parameter was varied as follows: $T \in \{5, 10, \cdots, 45\}$ with $H = 100$ and $\eta = 4.0$; $H \in \{20, 30, \cdots, 100\}$ with $T = 30$ and $\eta = 4.0$; $\eta \in \{1.5, 2.0, \cdots, 6.0\}$ with $T = 30$ and $H = 100$.

\subsection{Additional experiment: clustering with random corruptions}

In this additional experiment (same setting as in Section~\ref{sec:experiments}), we show how robust the clustering algorithm is to corruption. For simplicity, we consider an oblivious adversary that chooses $\delta_1 T$ trajectories as well as $\delta_2 n$ contexts and $\delta_3 A$ actions (for each selected trajectory) to corrupt, uniformly at random.
Contexts are corrupted by replacing them with other contexts in different clusters, and actions are corrupted by replacing them with other actions.
Over all experiments, we fix $T = 30$, $H = 100$, and $\eta = 4.0$ vary each corruption ratio $\delta_1, \delta_2, \delta_3$ separately by fixing the other two ratios.
Precisely, each ratio was varied as follows: $\delta_i \in \{0, 0.1, \cdots, 0.9\}$ with $\delta_j = \delta_k = 0.4$, for $\{i, j, k\} = \{1, 2, 3\}$.

Figure \ref{fig:exp2} shows the results for clustering with corruption. As expected, the error rate increases (on average) as the corruption level becomes more severe.
Note that for our example, the performance of our clustering is particularly robust to $\delta_2$ (proportion of corrupted contexts) and $\delta_3$ (proportion of corrupted actions).
Surprisingly, it can be seen that up to a certain corruption level, our algorithm recovers near-exact clustering.

For future work, it would be interesting to theoretically and empirically investigate the effect of oblivious and adaptive adversarial corruption on clustering in BMDPs.
Empirically, one could start by constructing various other types of BMDP environments and similarly test out our algorithms.
Theoretically, one may take inspiration from recent progress on robust community detection \citep{sharoa2021robust,liu2022robust} or robust RL \citep{wu2021robust}.
Another future direction is to make the algorithm provably these corruptions. One possible approach is to utilize adaptive clustering schemes \citep{YunP14,YunP19} or active model estimation \citep{Tarbouriech19}.
	
\end{document}